%% file: paper_jmlr.tex
\documentclass[twoside,11pt]{article}

\usepackage{jmlr2e}
\usepackage{caption}
\usepackage{subcaption}
\usepackage{color}
\usepackage{amsmath, amsfonts, amssymb, amsbsy, amscd, bm, bbm,mathrsfs}

\usepackage{algpseudocode}
\usepackage{algorithm}

\usepackage{tikz}
\usepackage{array, booktabs}
\usepackage{xspace} 
\usepackage{relsize}
\usepackage{dsfont}
\usepackage{url}

\newtheorem{assumption}{Assumption}

\DeclareGraphicsExtensions{.eps,.pdf}

\newcommand{\e}{\mathrm{e}}

\newcommand{\R}{\mathbb{R}}

\newcommand{\goto}{\rightarrow}
\renewcommand{\d}{{\mathrm{d}}}
\newcommand{\N}{\mathcal{N}}
\renewcommand{\P}{\operatorname{\mathbb{P}}}
\newcommand{\E}{\operatorname{\mathbb{E}}}
\def\iid{\textrm{i.i.d.}\xspace}

\DeclareMathOperator{\se}{SE}

\newcommand{\bs}{{\bm s}}

\newcommand{\w}{{\bm w}}
\newcommand{\bb}{\bm{s}}

\newcommand{\bt}{\bm{t}}
\newcommand{\ind}{{\bf 1}}
\newcommand{\F}{\mathcal{F}}
\newcommand{\A}{\mathcal{A}}
\newcommand{\bM}{\bm{M}}
\newcommand{\bz}{\bm{z}}
\renewcommand{\emptyset}{\varnothing}
\newcommand{\setempty}{\emptyset}

\newtheorem{innercustomassm}{Assumption}
\newenvironment{customassm}[1]
  {\renewcommand\theinnercustomassm{#1}\innercustomassm}
  {\endinnercustomassm}

\DeclareGraphicsExtensions{.eps,.pdf}
\graphicspath{{figs/}}

\usepackage{lastpage}
\jmlrheading{24}{2023}{1-\pageref{LastPage}}{8/18; Revised 1/23}{4/23}{18-521}{Weijie J. Su and Yuancheng Zhu}
\ShortHeadings{HiGrad}{Su and Zhu}

\firstpageno{1}

\begin{document}

\title{HiGrad: Uncertainty Quantification for Online Learning and Stochastic Approximation}


\author{\name Weijie J.~Su \email suw@wharton.upenn.edu \\
       \addr  University of Pennsylvania, USA
       \AND
       \name Yuancheng Zhu \email yuancheng.zhu@gmail.com \\
       \addr Renaissance Technologies LLC, USA
}
\editor{Qiang Liu}
\maketitle

\begin{abstract}

Stochastic gradient descent (SGD) is an immensely popular approach for online learning in settings where data arrives in a stream or data sizes are very large. However, despite an ever-increasing volume of work on SGD, much less is known about the statistical inferential properties of SGD-based predictions. Taking a fully inferential viewpoint, this paper introduces a novel procedure termed HiGrad to conduct statistical inference for online learning, without incurring additional computational cost compared with SGD. The HiGrad procedure begins by performing SGD updates for a while and then splits the single thread into several threads, and this procedure hierarchically operates in this fashion along each thread. With predictions provided by multiple threads in place, a $t$-based
confidence interval is constructed by decorrelating predictions using covariance structures given by a Donsker-style extension of the Ruppert--Polyak averaging scheme, which is a technical contribution of independent interest. Under certain regularity conditions, the HiGrad confidence interval is shown to attain asymptotically exact coverage probability. Finally, the performance of HiGrad is evaluated through extensive simulation studies and a real data example. An R package \texttt{higrad} has been developed to implement the method.

\end{abstract}


\input{intro}

\input{method1}

\input{method2}

\input{split}

\input{extend}

\input{simulation}

\input{discuss}


\acks{We thank Guanghui Lan and Panagiotis Toulis for helpful comments about an early version of the manuscript. We are grateful to the three anonymous referees for
their constructive comments that helped improve the presentation
of this work. This work was supported in part by NSF through CAREER DMS-1847415, an Alfred Sloan Research Fellowship, and the Wharton
Dean's Research Fund.}

\appendix
\input{appendix_weak}
\input{appendix_strong}

\input{appendix_mis}

\vskip 0.2in
\bibliography{ref}

\end{document}

%% file: intro.tex
\section{Introduction}
\label{sec:introduction}

In recent years, scientific discoveries and engineering advancements have been increasingly driven by data analysis. Meanwhile, modern datasets exhibit new features that impose two challenges to conventional statistical approaches. First, as datasets grow exceedingly large, many basic statistical tasks such as maximum likelihood estimation (MLE) may become computationally infeasible. The other common feature is that data are frequently collected in an online fashion or computers do not have enough memory to load the entire dataset. As a consequence, we are often constrained from using batch learning methods such as gradient descent.

In this context, stochastic gradient descent (SGD), also known as incremental gradient descent, has been shown to resolve these two issues for online learning. SGD is used to find a minimizer of the optimization problem
\[
\min_{\theta} \ f(\theta) := \mathbb E f(\theta,Z),
\]
where the expectation is over the randomness embodied in the random data $Z$. Letting $N$ be the sample size, this method in its simplest form performs iterations according to
\begin{equation}\label{eq:sgd}
\theta_j = \theta_{j-1} - \gamma_j g(\theta_{j-1}, Z_j)
\end{equation}
for $j = 1, \ldots, N$, where $\gamma_j$'s are the step sizes, each $Z_j$ is a realization of $Z$, and $g$ is the gradient of $f(\theta, z)$ with respect to the first argument. These types of optimization problems appear ubiquitously in MLEs and, more broadly, in $M$-estimation \citep{huber1964robust}. As is clear, SGD makes only one pass over the data, thereby having a much lower computational cost than batch methods such as the Newton--Raphson method and gradient descent. These batch methods need to pass over the entire dataset in one iteration. Furthermore, SGD can discard data points on-the-fly after evaluating the gradient and, put slightly differently, SGD is online in nature, requiring essentially no memory cost. In addition to its computational efficiency and low memory cost, SGD achieves optimal convergence
rates under certain conditions \citep{nemirovskii1983problem,agarwal2010information,bach2013non}. Among others, these advantages have contributed to the immense popularity of SGD in large-scale machine learning problems \citep{zhang2004solving,duchi2011adaptive,lecun2015deep}.

These appealing features of SGD, however, are accompanied by the cost of having random solutions; as such, decision making based on SGD predictions might suffer from uncertainty. The randomness originates either from the stochasticity of data points in the online setting or from the random sampling scheme of SGD in the case of fixed datasets where multiple epochs are executed\footnote{If the fixed dataset is treated as a finite population, these two types of randomness are equivalent.}. This randomness is potentially non-negligible and could even jeopardize the interpretation of predictions at worst. To illustrate this, we apply SGD to the \textit{Adult} dataset hosted on the UCI Machine Learning Repository \citep{lichman2013machine} as an example. The dataset contains demographic
information of part of the 1994 US Census Database, and the goal is to predict whether a person's annual income exceeds \$50,000. To fit a logistic regression on the dataset, we run SGD for 25 epochs (approximately 750,000 steps of SGD updates), and use the estimated model to predict the probabilities for a randomly selected test set containing 1,000 sample units. The procedure above is repeated for a total of 500 times, and Figure~\ref{fig:adult_intro} plots the length of the 90\%-coverage empirical prediction interval (an interval covering 450 predicted probabilities) against the average predicted probability for each sample unit, showing the variability of SGD-predicted probabilities. Even with a relatively large number of passes through the training dataset, there are some test sample units with a large variability near 50\%. This is the regime where variability must be addressed since the decision based on predictions can be easily reversed.

\begin{figure}[!htp]
\centering
\input{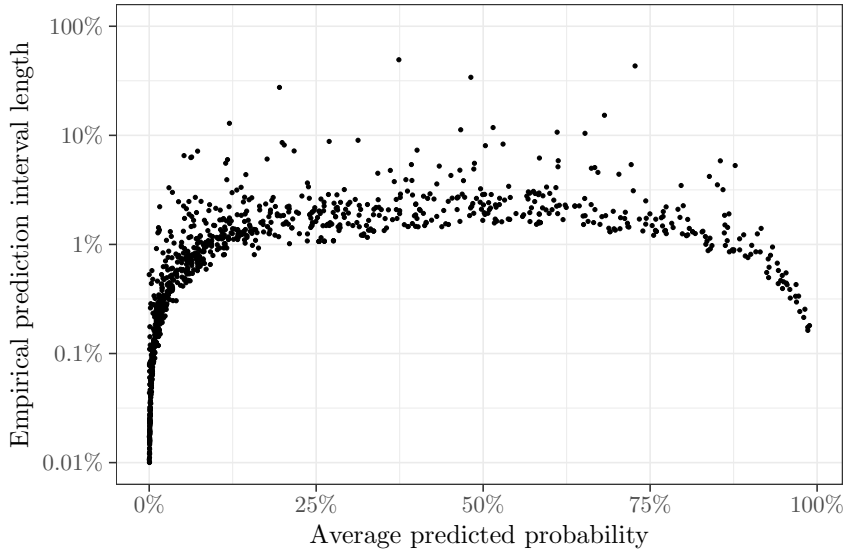}
\caption{Length of 90\% empirical prediction intervals versus average predicted probabilities on a test set of size 1,000 from the Adult dataset, calculated based on 500 independent SGD runs, each with 25 epochs.}
\label{fig:adult_intro}
\end{figure}

This paper aims to assess the uncertainty in SGD estimates via confidence intervals. Using the off-the-shelf bootstrap for this purpose is infeasible due to its prohibitively high computational cost and unsuitability for streaming data. In response, we propose a new method called HiGrad, short for Hierarchical Incremental GRAdient Descent, which estimates model parameters in an online fashion, just like SGD, and provides a confidence interval for the true population value. While SGD updates the iterate along \textit{one} direction, HiGrad updates its iterates using gradients along \textit{multiple} directions that are specified by a tree structure. An example of HiGrad is shown in Figure \ref{fig:higrad_intro}.


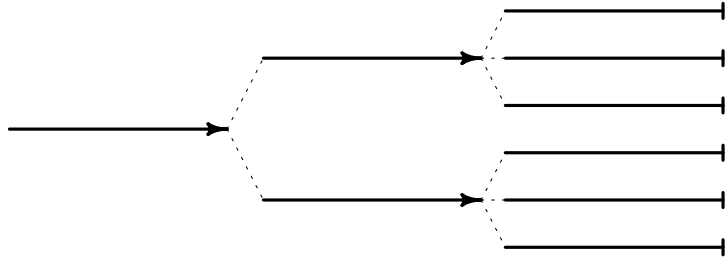
\begin{figure}[!htp]
\begin{center}
\input{figs/higrad_cartoon_simple.tex}
\caption{Graphical illustration of the HiGrad tree.
Here we have three levels. At the end of the first level, the segment is split into two;
at the end of the second level, each segment is further split into three. There are six threads in this HiGrad, each defined as a path from the root node to one of the six leaf nodes.}
\label{fig:higrad_intro}
\end{center}
\end{figure}

More specifically, HiGrad begins by performing SGD iterations for a certain number of steps and then splits at the end of the single segment into several segments. This resembles the splitting strategy used in the multilevel Monte Carlo method~\citep{giles2015multilevel}, which splits a sum to be estimated into several pieces for estimation separately. On top of that, however, HiGrad hierarchically operates in this fashion at every level until leaf nodes, generating multiple threads\footnote{A path from the root to a leaf node.}. Moreover, it naturally fits the online setting and requires no more computational effort compared with SGD. In particular, the HiGrad algorithm agrees with the vanilla SGD restricted to every thread of the tree. With the HiGrad iterates in place, a weighted average across each thread yields an estimate. These multiple estimates are used to construct a $t$-based confidence interval for the quantity of interest by recognizing the correlation structure, which is obtained by making use of the
Ruppert--Polyak normality result for averaged SGD iterates \citep{ruppert1988,polyak1990,polyak1992}. Under certain conditions, the HiGrad confidence interval is shown to have asymptotically correct coverage probabilities, and its center, referred to as the HiGrad estimator, achieves the same statistical efficiency as the vanilla SGD.

At a high level, HiGrad integrates the ideas of \textit{contrasting} and \textit{sharing}, two competing ingredients that require balancing. On the one hand, contrasting is gained by hierarchically splitting the threads to get more than one estimate, which allows us to measure the associated variability. On the other hand, every two threads \textit{share} some segments in order to \textit{elongate} the total length between the root and a leaf. The benefit of having a longer thread is that it ensures better convergence and accuracy of the solutions. On the contrary, splitting SGD at the beginning (see,
for example, \cite{jain2016parallelizing}) with the same computational budget $N$ gives much \textit{shorter} threads and, as a consequence, it might lead to a significant bias of the solutions, as demonstrated by simulation studies in Section \ref{sec:simulations}. To facilitate the use of HiGrad in practice, we set a default configuration of this method in our R package \texttt{higrad} (\url{https://cran.r-project.org/web/packages/higrad/}) through balancing between contrasting and sharing, showing its satisfactory performance in a variety of scenarios in Section \ref{sec:simulations}.

This paper contributes to the rich literature on online learning. As a modern online learning tool, SGD has a root extended to stochastic approximation, which was pioneered by \cite{robbins1951stochastic} and \cite{kiefer1952stochastic} in the 1950s (see \cite{lai2003stochastic} for an overview). More recently, the optimization and machine learning communities have been extensively studying SGD \citep{zhang2004solving,nemirovski2009robust,recht2011hogwild,rakhlin2012making}, mostly focused on the convergence of SGD iterates or generalization error bounds. Much less work has been done taking an inferential point of view on SGD. That said, very recently there has been a flurry of interesting activities on statistical inference for SGD \citep{toulis2017asymptotic,chen2016,li2017statistical,fang2017scalable,mandt2017stochastic,lan2012validation}. In short, \cite{toulis2017asymptotic} propose the implicit SGD, showing its robustness to step sizes and carrying over the Ruppert--Polyak normality to this method; in \cite{chen2016}, the authors first develop an asymptotically valid inference approach based on averaged SGD iterates. Their procedure takes the form of a new batch-means estimator that is derived by truncating the SGD iterates into blocks as a way to decorrelate nearby SGD iterates; further, \cite{li2017statistical} argue that discarding some intermediate iterates helps to reduce correlations of SGD iterates; in a different route, \cite{fang2017scalable} consider an inferential procedure through running
perturbed-SGD in parallel. The HiGrad procedure significantly differs from this work in that it essentially provides a new template for online learning, including SGD as the simplest example. For future research, it would be of great interest to explore potential benefits of this new template for purposes other than providing a confidence interval.


The remainder of the paper is structured as follows. We introduce the HiGrad algorithm in Section \ref{sec:method-1}, along with a sketch proof of the coverage properties of the HiGrad confidence interval and a form of statistical optimality. Section \ref{sec:how-split} considers the choice of parameters that determine the HiGrad procedure. In Section \ref{sec:extensions}, some practical extensions and improvements for implementing the algorithm are discussed. Results on a set of simulation studies and a real data example are presented in Section \ref{sec:simulations}. We conclude the paper in Section \ref{sec:discussion} with suggested future research. Technical details of the proofs are deferred to the appendix.


%% file: figs/higrad_cartoon_simple.tex
\begin{tikzpicture}[x=1pt,y=1pt]
\definecolor{fillColor}{RGB}{255,255,255}
\path[use as bounding box,fill=fillColor,fill opacity=0.00] (0,0) rectangle (289.08,108.41);
\begin{scope}
\path[clip] (  0.00,  0.00) rectangle (289.08,108.41);
\definecolor{drawColor}{RGB}{0,0,0}

\path[draw=drawColor,line width= 1.2pt,line join=round,line cap=round] ( 10.71, 54.20) -- ( 93.07, 54.20);
\definecolor{fillColor}{RGB}{0,0,0}

\path[draw=drawColor,line width= 1.2pt,line join=round,line cap=round,fill=fillColor] ( 85.74, 56.22) --
	( 87.31, 55.36) --
	( 89.10, 54.73) --
	( 91.05, 54.33) --
	( 93.07, 54.20) --
	( 93.07, 54.20) --
	( 91.05, 54.07) --
	( 89.10, 53.68) --
	( 87.31, 53.04) --
	( 85.74, 52.19) --
	( 85.74, 52.19) --
	( 85.91, 52.24) --
	( 86.07, 52.39) --
	( 86.22, 52.63) --
	( 86.34, 52.95) --
	( 86.43, 53.33) --
	( 86.48, 53.75) --
	( 86.50, 54.20) --
	( 86.48, 54.65) --
	( 86.43, 55.08) --
	( 86.34, 55.46) --
	( 86.22, 55.78) --
	( 86.07, 56.02) --
	( 85.91, 56.16) --
	( 85.74, 56.22) --
	cycle;

\path[draw=drawColor,line width= 0.4pt,dash pattern=on 1pt off 3pt ,line join=round,line cap=round] ( 93.07, 54.20) -- (106.79, 81.05);

\path[draw=drawColor,line width= 0.4pt,dash pattern=on 1pt off 3pt ,line join=round,line cap=round] ( 93.07, 54.20) -- (106.79, 27.36);

\path[draw=drawColor,line width= 1.2pt,line join=round,line cap=round] (106.79, 81.05) -- (189.15, 81.05);

\path[draw=drawColor,line width= 1.2pt,line join=round,line cap=round,fill=fillColor] (181.83, 83.06) --
	(183.40, 82.21) --
	(185.19, 81.57) --
	(187.13, 81.18) --
	(189.15, 81.05) --
	(189.15, 81.05) --
	(187.13, 80.92) --
	(185.19, 80.52) --
	(183.40, 79.89) --
	(181.83, 79.04) --
	(181.83, 79.04) --
	(182.00, 79.09) --
	(182.16, 79.23) --
	(182.30, 79.47) --
	(182.42, 79.79) --
	(182.51, 80.17) --
	(182.57, 80.60) --
	(182.59, 81.05) --
	(182.57, 81.50) --
	(182.51, 81.92) --
	(182.42, 82.30) --
	(182.30, 82.62) --
	(182.16, 82.86) --
	(182.00, 83.01) --
	(181.83, 83.06) --
	cycle;

\path[draw=drawColor,line width= 1.2pt,line join=round,line cap=round] (106.79, 27.36) -- (189.15, 27.36);

\path[draw=drawColor,line width= 1.2pt,line join=round,line cap=round,fill=fillColor] (181.83, 29.37) --
	(183.40, 28.52) --
	(185.19, 27.88) --
	(187.13, 27.49) --
	(189.15, 27.36) --
	(189.15, 27.36) --
	(187.13, 27.23) --
	(185.19, 26.83) --
	(183.40, 26.20) --
	(181.83, 25.34) --
	(181.83, 25.34) --
	(182.00, 25.40) --
	(182.16, 25.54) --
	(182.30, 25.78) --
	(182.42, 26.10) --
	(182.51, 26.48) --
	(182.57, 26.91) --
	(182.59, 27.36) --
	(182.57, 27.80) --
	(182.51, 28.23) --
	(182.42, 28.61) --
	(182.30, 28.93) --
	(182.16, 29.17) --
	(182.00, 29.32) --
	(181.83, 29.37) --
	cycle;

\path[draw=drawColor,line width= 0.4pt,dash pattern=on 1pt off 3pt ,line join=round,line cap=round] (189.15, 81.05) -- (198.30, 98.94);

\path[draw=drawColor,line width= 0.4pt,dash pattern=on 1pt off 3pt ,line join=round,line cap=round] (189.15, 81.05) -- (198.30, 63.15);

\path[draw=drawColor,line width= 0.4pt,dash pattern=on 1pt off 3pt ,line join=round,line cap=round] (189.15, 81.05) -- (198.30, 81.05);

\path[draw=drawColor,line width= 0.4pt,dash pattern=on 1pt off 3pt ,line join=round,line cap=round] (189.15, 27.36) -- (198.30, 45.25);

\path[draw=drawColor,line width= 0.4pt,dash pattern=on 1pt off 3pt ,line join=round,line cap=round] (189.15, 27.36) -- (198.30,  9.46);

\path[draw=drawColor,line width= 0.4pt,dash pattern=on 1pt off 3pt ,line join=round,line cap=round] (189.15, 27.36) -- (198.30, 27.36);

\path[draw=drawColor,line width= 1.2pt,line join=round,line cap=round] (198.30, 98.94) -- (280.66, 98.94);

\path[draw=drawColor,line width= 1.2pt,line join=round,line cap=round] (280.66, 96.10) --
	(280.66, 98.94) --
	(280.66,101.79);

\path[draw=drawColor,line width= 1.2pt,line join=round,line cap=round] (198.30, 63.15) -- (280.66, 63.15);

\path[draw=drawColor,line width= 1.2pt,line join=round,line cap=round] (280.66, 60.31) --
	(280.66, 63.15) --
	(280.66, 66.00);

\path[draw=drawColor,line width= 1.2pt,line join=round,line cap=round] (198.30, 81.05) -- (280.66, 81.05);

\path[draw=drawColor,line width= 1.2pt,line join=round,line cap=round] (280.66, 78.20) --
	(280.66, 81.05) --
	(280.66, 83.89);

\path[draw=drawColor,line width= 1.2pt,line join=round,line cap=round] (198.30, 45.25) -- (280.66, 45.25);

\path[draw=drawColor,line width= 1.2pt,line join=round,line cap=round] (280.66, 42.41) --
	(280.66, 45.25) --
	(280.66, 48.10);

\path[draw=drawColor,line width= 1.2pt,line join=round,line cap=round] (198.30,  9.46) -- (280.66,  9.46);

\path[draw=drawColor,line width= 1.2pt,line join=round,line cap=round] (280.66,  6.61) --
	(280.66,  9.46) --
	(280.66, 12.31);

\path[draw=drawColor,line width= 1.2pt,line join=round,line cap=round] (198.30, 27.36) -- (280.66, 27.36);

\path[draw=drawColor,line width= 1.2pt,line join=round,line cap=round] (280.66, 24.51) --
	(280.66, 27.36) --
	(280.66, 30.20);
\end{scope}
\end{tikzpicture}

%% file: method1.tex
\section{The HiGrad Procedure}
\label{sec:method-1}

\subsection{Problem statement}
\label{sec:problem-statement}


Let $f(\theta)$ be a convex function defined on Euclidean space $\R^d$ and denote by $\theta^\ast$ the unique minimizer of $f(\theta)$. Suppose the objective function $f$ is given by an expectation
\[
f(\theta) = \E f(\theta, Z),
\]
where $f(\theta, z)$ is a loss function and $Z$, throughout the paper, denotes a random variable drawn from an (unknown) infinite population or a finite population $\{z_1, \ldots, z_m\}$. In the latter case, the objective function is $f(\theta) = \frac1m \sum_{l=1}^m f(\theta, z_l)$. We sample $N$ \iid~data points $Z_1, Z_2, \ldots, Z_N$ from the population, each having the same distribution as $Z$ (in the case of finite population, the number $N/m$ is often called epochs). For an observation unit $Z = z$, we have access to a noisy gradient $g(\theta, z)$ that obeys
\[
\E g(\theta, Z) = \nabla f(\theta)
\] 
for all $\theta$. Namely, $g(\theta, Z)$ is unbiased for $\nabla f(\theta)$ or is, equivalently, the partial derivative of $f(\theta, Z)$ with respect to $\theta$.

A rich class of such problems is ubiquitous in statistics and machine learning: $f(\theta, z)$ is taken to be the negative log-likelihood function and the random variable is written as $z = (x, y)$, with $x \in \R^d$ being the feature vector and $y \in \R$ being the response or label. Although the joint distribution of $(X, Y)$ is typically unknown, the conditional distribution of $Y$ given $X$ is often assumed to be specified by the parameter $\theta^\ast$. Below is a list of several representative problems frequently encountered in practice (up to constants independent of $\theta$).
\begin{itemize}
\item \textbf{Linear regression}: $f(\theta, z) = \frac12 (y - x^\top \theta)^2$.
\item
\textbf{Logistic regression}: $f(\theta, z) = -y x^\top \theta + \log\left( 1 + \e^{x^\top \theta} \right)$.
\end{itemize}
The examples above fall into the broad class of generalized linear models (GLM). In its canonical form without dispersion, a GLM density takes the form $p_{\theta}(y|x) = h(y) \e^{y x^\top \theta - b(x^\top \theta)}$, where $h$ is the base measure and the function $b$ satisfies $b'(x^\top \theta) = \E(Y|X = x)$ and is convex. Ignoring the factor $\log h$, the negative log-likelihood is
\[
f(\theta, z) = -y x^\top\theta + b(x^\top\theta),
\]
which is convex. Hence, the objective function $f(\theta)$, derived through integrating out the randomness of $Z = (X, Y)$,  is also convex. This is the case for the two examples above\footnote{In the case of linear regression, $f(\theta, z)$ includes an additional term $y^2/2$. This term does not affect the minimizer $\theta^\ast$.}. In addition, two popular types of problems are also included. Below, $\|\cdot\|$ denotes the $\ell_2$ norm and $\lambda > 0$.
\begin{itemize}
\item
\textbf{Penalized generalized linear regression}: $f(\theta, z) = - y x^\top\theta + b(x^\top\theta) + \lambda \|\theta\|^2$.

\item
\textbf{Huber regression}: $f(\theta, z) = \rho_{\lambda}(y - x^\top \theta)$, where $\rho_{\lambda}(a) = a^2/2$ for $|a| \le \lambda$ and $\rho_{\lambda}(a) = \lambda |a| - \lambda^2/2$ otherwise.
\end{itemize}
The first one includes ridge regression as a well-known example, using an $\ell_2$ penalty to impose regularization. The function $b$ should satisfy some growth assumption and Poisson regression, as a result, is excluded. The second is an instance of robust estimation, which is used to make regression less sensitive to outliers. However, it is worth pointing out that the formal treatment given later in Section \ref{sec:asympt-valid-conf} considers a much broader class of problems.

As the model gets more and more complex, of practical importance is often the predictive performance of the model rather than the interpretation of a single unknown parameter. In the context above, a plethora of statistical and machine learning problems can be cast as estimating a univariate function $\mu_x(\theta)$ evaluated at $\theta^\ast$. Put concretely, imagine that we observe the feature $X = x$ of a freshly sampled data point and would like to predict the conditional mean of $Y$ given $X = x$:
\[
\mu_x(\theta^\ast) \equiv \E (Y| X = x).
\]
Above, note that the conditional mean is a function of the model parameter $\theta^\ast$ and independent variable $x$. In the aforementioned examples, $\mu_x(\theta) = x^\top \theta$ (linear regression), $\mu_x(\theta) = \e^{x^\top \theta}/(1 + \e^{x^\top \theta})$ (logistic regression) and, more broadly, $\mu_x(\theta) = \E(Y|X = x) \equiv b'(x^\top \theta)$ (generalized linear
models). Generally, $\mu_x(\theta)$ can be any smooth univariate function of $\theta$.

The main goal of this paper is to attach some confidence statements, such as a confidence interval, to an estimate of $\mu_x(\theta^\ast)$ solely based on noisy gradient information evaluated at the sample $Z_1, \ldots, Z_N$. Given that $N$ is exceedingly large or data is available in a stream, one challenge is to evaluate the noisy gradient $g$ only \textit{once} for each $Z_j$. While the SGD algorithm \eqref{eq:sgd} fulfills the computational constraint, it fails to provide a confidence interval in a natural way. On the other hand, bootstrap is a flexible technique to yield confidence statements but does not, however, scale up to large datasets since it typically requires 50 times more computational cost. Next, we present the HiGrad algorithm as an approach to bringing together all considered needs.

\subsection{The HiGrad tree}
\label{sec:higrad-algorithm}

The HiGrad algorithm is best visualized by its tree structure. A HiGrad tree is parameterized by $(B_1, B_2, \ldots, B_K)$ and $(n_0, n_1, \ldots, n_K)$, where $K$ is the depth of the tree and $n_i$ denotes segment length. All $B_i \ge 2$ and $n_i$ are positive integers. Figure \ref{fig:higrad_intro} illustrates an example of $(B_1, B_2) = (2, 3)$, while Figure \ref{fig:higrad} illustrates $(B_1, B_2) = (2, 2)$. At level 0, the root segment is comprised of $n_0$ data points and has $B_1$ child segments. Each of these segments at level 1 is comprised of $n_1$ data points and has $B_2$ child segments. Recurse this process according to the parameters $(B_1, B_2, \ldots, B_K)$ and $(n_0, n_1, \ldots, n_K)$ until the HiGrad tree has $K+1$ levels. Write $L_k := n_0 + n_1 + \cdots + n_k$ for $k = 0, 1, \ldots, K$ (as a convention, set $L_{-1} = 0$). A path connecting the root node (the left end of the leaf segment) and a leaf node (a last level node) is called a thread. Note that there are $T := B_1B_2 \cdots B_K$
threads, each of which traverses $K+1$ segments, having data points totaling $L_K$. As a constraint, the number of data units in the full tree should equal the total number of observations, that is,
\begin{equation}\label{eq:sum_N}
n_0 + B_1n_1 + B_1 B_2 n_2 + B_1 B_2 B_3 n_3 + \cdots + B_1 B_2 \cdots B_K n_K = N.
\end{equation}

The HiGrad algorithm is to run SGD on the tree structure above. Given a sequence of step sizes $\{\gamma_j\}_{j=1}^{\infty}$, HiGrad begins by iterating
\begin{equation}\nonumber
\theta_j^{\setempty} = \theta_{j-1}^{\setempty} - \gamma_j g(\theta_{j-1}^{\setempty}, Z_j^{\setempty})
\end{equation}
for $j = 1, \ldots, n_0$, starting from $\theta_0^{\setempty} = \theta_0$. Above, the superscript $\setempty$ denotes the root segment and $Z^{\setempty} := \{Z_j^{\setempty}\}_{j=1}^{n_0}$ is a sub-sample of the observations $\{Z_1, \ldots, Z_N\}$. Next, HiGrad proceeds to all level 1 segments, at one of which, say segment $\bs = (b_1)$ for $1 \le b_1 \le B_1$, it iterates according to
\[
\theta_j^{\bs} = \theta_{j-1}^{\bs} - \gamma_{\mathsmaller{L_0}+j} g(\theta^{\bs}_{j-1}, Z_j^{\bs})
\]
for $j = 1, \ldots, n_1$, where $\theta_0^{\bs} = \theta_{n_0}^{\setempty}$ (the last iterate from the previous segment). More generally, consider a segment $\bs = (b_1,\cdots, b_k)$ at level $k$, where $1 \le b_i \le B_i$ for $i = 1, \ldots, k$. At this segment, the procedure is updated according to
\[
\theta_j^{\bb} = \theta_{j-1}^{\bb} - \gamma_{\mathsmaller{L_{k-1}}+j} \,g(\theta^{\bb}_{j-1}, Z_j^{\bb})
\]
for $j = 1, \ldots, n_k$, with the initial point $\theta^{\bs}_0$ being the last iterate from the segment $\bs^{-} := (b_1, \ldots, b_{k-1})$. Through the whole procedure, the $N$ data points should be partitioned as
\begin{equation}\label{eq:n_partitioning}
\{Z_1, \ldots, Z_N\} = \cup \{ Z_j^{\bs}: 1 \le j \le n_{\#\bs}\},
\end{equation}
where the union is taken over all the segments $\bs$ and $\#\bs = k$ if $\bs = (b_1, \ldots, b_k)$ (with the convention that $\#\setempty = 0$). When the $N$ data points are sampled from a finite population, different segments might share some samples that are drawn multiple times.

To make the HiGrad algorithm online, a formal description of the construction is given in Algorithm \ref{algo:1}, with more details to be specified in the next section. In particular, the data stream can feed segments at the same level in a cyclic manner, thus enabling the HiGrad algorithm to be implemented in an \textit{online} fashion.

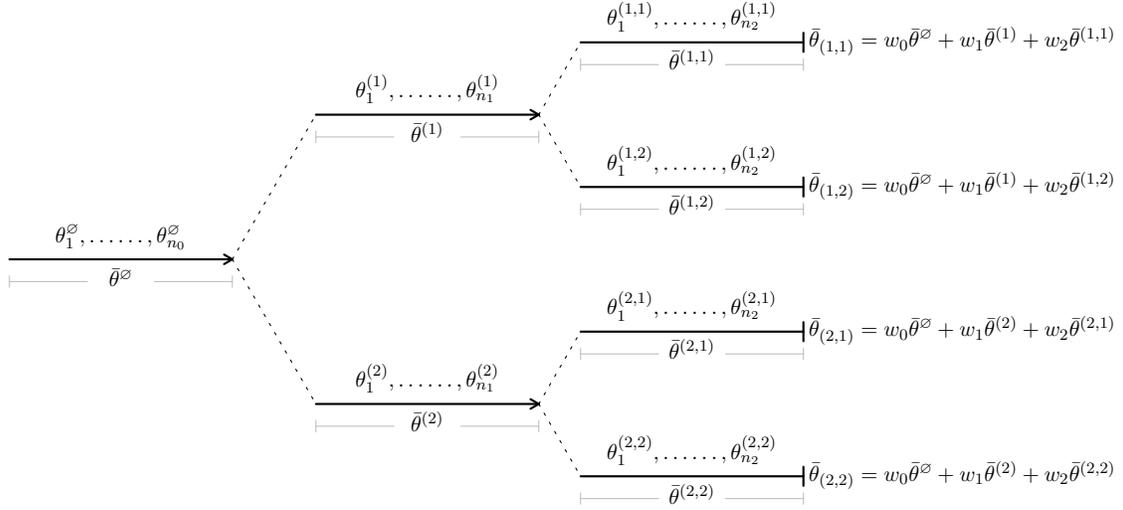
\begin{figure}[t]
\begin{center}
\input{figs/higrad_cartoon.tex}
\caption{Graphical illustration of the HiGrad algorithm. Here we have three levels and at the end of each level, each segment is split into two segments. Averages are obtained for each level and at each leaf a weighted average is calculated. The weights $w_j$ are detailed in Section \ref{sec:decorr-thre}, and more discussion about the tree structure is given in Section~\ref{sec:how-split}.}
\label{fig:higrad}
\end{center}
\end{figure}

\begin{algorithm}[h]
\caption{The HiGrad Algorithm} 
\label{algo:1}
\begin{algorithmic}[1]
\State \textbf{input:} HiGrad tree structure $(B_1, \ldots, B_K)$ and $(n_0, n_1, \ldots, n_K)$, partition of the dataset $\{Z_1, \ldots, Z_N\} = \cup_{\bs} \{ Z_j^{\bs}: 1 \le j \le n_{\#\bs}\}$, step sizes $(\gamma_1, \ldots, \gamma_{L_K})$, initial point $\theta_0$, univariate function $\mu_x(\theta)$, and level $\alpha$
\State \textbf{output:} $\overline\theta^{\bs}$ for all segments $\bs$ and confidence interval for $\mu_x(\theta^\ast)$
\vspace{0.5em}
\State Set $\overline\theta^{\bs} = 0$ for all segments $\bs$
\Function{SegmentHiGrad}{$\theta, \bb$}
\State $\theta_0^{\bs} = \theta$
\State $k = \#\bb$
\For{$j = 1$ \textbf{to} $n_k$}
\State $\theta_j^{\bs} \leftarrow \theta_{j-1}^{\bs} - \gamma_{j+L_{k-1}} \,g(\theta^{\bs}_{j-1}, Z_j^{\bs})$
\State $\overline\theta^{\bs} \leftarrow \overline\theta^{\bs} + \theta^{\bs}_j/n_k$
\EndFor
\If{$k < K$}
\For{$b_{k+1} = 1$ \textbf{to} $B_{k+1}$}
\State $\bb^+ \leftarrow (\bb, b_{k+1})$
\State \textbf{execute} \textsc{SegmentHiGrad}$\left( \theta_{n_k}^{\bb}, \bb^+\right)$
\EndFor
\EndIf
\EndFunction
\vspace{0.5em}
\State \textbf{execute} \textsc{SegmentHiGrad}$\left( \theta_0, \emptyset\right)$
\State Construct a confidence interval in the form of \eqref{eq:conf} using \eqref{eq:se}
\end{algorithmic}
\end{algorithm}

\subsection{A $t$-confidence interval}
\label{sec:decorr-thre}

Restricted to one thread, HiGrad amounts to performing the vanilla SGD \eqref{eq:sgd} for $L_K \equiv n_0 + n_1 + \cdots + n_K$ steps. Thus, HiGrad yields $T$ sets of vanilla SGD results, and the ultimate goal is to utilize these results to obtain an estimator of $\mu_x^\ast := \mu_x(\theta^\ast)$ with a confidence interval. To this end, we start by introducing some notation to facilitate our discussion. Given any segment $\bs = (b_1, \ldots, b_k)$ of the HiGrad tree\footnote{Throughout the paper, we use non-bold letters to denote scalars, vectors, and matrices, except for the case where it is necessary to emphasize that the notion is not a scalar, for example, $\bs, \bt$, and $\bm\mu_x$.}, denote by $\overline \theta^{\bb}$ the average of the $n_k$ iterates in $\bs$, that is\footnote{The letter $\bt$ is placed in subscript for a thread as a way to distinguish a thread from a segment.},
\[
\overline \theta^{\bb} = \frac1{n_k} \sum_{j=1}^{n_k} \theta^{\bb}_j.
\]
Averaged SGD is known to achieve optimal convergence rates and to make estimates robust for strongly convex objectives \citep{moulines2011non,rakhlin2012making} (for more work related to averaged SGD, see \cite{bach2013non,cardot2013efficient,duchi2016local,jain2017markov,liang2017statistical,fan2018statistical}). Let $\w = (w_0, w_1, \ldots, w_K)$ be a vector of weights such that $w_0 + w_1 + \cdots + w_K = 1$ and $w_i \ge 0$. Then, for any thread $\bt = (b_1, \ldots, b_K)$, write $\overline \theta_{\bt}$ for the weighted average over the $K+1$ segments through $\bt$, that is,
\begin{equation}\label{eq:theta_ave}
\overline \theta_{\bt} = \sum_{k=0}^K w_k \overline \theta^{(b_1, \ldots, b_k)}.
\end{equation}
For notational convenience, we suppress the dependence of $\w$ on $\overline\theta_{\bt}$. Denote by $\bm\mu_x \in \R^{T}$ the $T$-dimensional vector consisting of all $\mu_x^{\bt} := \mu_x(\overline\theta_{\bt})$ defined for every thread $\bt$, and write $\mu_x^\ast = \mu_x(\theta^\ast)$ for short.

Now, we turn to infer $\mu_x^\ast$ based on the $T$-dimensional vector $\bm\mu_x$. This requires recognizing the correlation structure of the $T$ threads. For two different threads $\bt = (b_1, \ldots, b_K)$ and $\bt' = (b_1', \ldots, b_K')$ with $1 \le b_k, b_k' \le B_k$ for $k = 1, \ldots, K$, the number of data points shared by $\bt$ and $\bt'$ vary from $n_0$ to $n_0 + n_1 + \cdots + n_{K-1}$. Intuitively, the more they share, the larger the correlation is. This point is made explicit by Lemma \ref{lm:pj} in Section \ref{sec:proof-ideas}, which, loosely speaking, states that as the length of the data stream $N \goto \infty$, under certain conditions the vector $\bm\mu_x$ is asymptotically normally distributed with mean $\mu_x^\ast \bm{1}: = (\mu_x^\ast, \mu_x^\ast, \ldots, \mu_x^\ast)^\top$ and covariance proportional to $\Sigma \in \R^{T \times T}$. The covariance $\Sigma$ is defined as
\begin{equation}\label{eq:sigma}
\Sigma_{\bt, \bt'} = \sum_{k=0}^p \frac{w_k^2 N}{n_k}
\end{equation}
for any two threads $\bt, \bt'$ that agree exactly on the first $p$ segments. In particular, the diagonal entries all equal
\[
\Sigma_{\bt, \bt} = \sum_{k=0}^K \frac{w_k^2 N}{n_k}.
\]
Making use of this distributional property, Proposition \ref{prop:t} in the next section devises an (asymptotic) pivotal quantity for $\mu_x^\ast$. This pivot suggests estimating $\mu_x^\ast$ using the sample mean of $\bm\mu_x$:
\begin{equation}\label{eq:average}
\overline \mu_x := \frac1T \sum_{\bt \in \mathcal{T}} \mu_x^{\bt},
\end{equation}
where $\mathcal T$ denotes the set of all threads. We propose
\begin{equation}\label{eq:conf}
\left[\overline\mu_x - t_{T - 1, 1-\frac{\alpha}{2}}\se_x, \quad \overline\mu_x + t_{T - 1, 1-\frac{\alpha}{2}}\se_x \right]
\end{equation}
as a $t$-based confidence interval for $\mu_x^\ast$ at nominal level $1 - \alpha$. Above, $t_{T-1, 1-\frac{\alpha}{2}}$ is the $1 - \frac{\alpha}{2}$ quantile of the $t$-distribution with $T-1$ degrees of freedom, and the standard error $\se_x$ takes the form
\begin{equation}\label{eq:se}
\se_x = \sqrt{\frac{\bm 1^\top \Sigma \bm 1 \, (\bm\mu_x^\top - \overline\mu_x \, \bm 1^\top) \Sigma^{-1} (\bm\mu_x - \overline\mu_x \, \bm 1)}{T^2(T - 1)}}.
\end{equation}
Formal statements are given in Section \ref{sec:asympt-valid-conf}, along with explicit conditions required by the results, and Section \ref{sec:proof-ideas} sheds light on how the correlation structure of $\bm\mu_x$ is derived and used in obtaining the confidence interval.

In passing, we remark that the HiGrad confidence interval construction, which will be discussed starting from \eqref{eq:delta_exp} in Section \ref{sec:proof-ideas}, relies on the delta method to linearly approximate $\mu_x$ near $\theta^\ast$. To improve on the linear approximation in the case of a large curvature of $\mu_x$ at $\theta^\ast$, one could consider a certain bijective function $\eta(\cdot)$ in a neighborhood of $\mu_x^\ast$ and construct a confidence interval for $\eta^\ast_x := \eta(\mu_x^\ast)$. Note that many interesting examples of $\mu_x$ in generalized linear models depend on $\theta$ only through $x^\top \theta$, and thus a good choice of $\eta$ could be the link function, satisfying $\eta(\mu_x(\theta)) = x^\top\theta$. With this transformation in place, we may construct a confidence interval for $\eta^\ast_x$ as earlier. By recognizing the correspondence between $\eta$ and $\mu_x$, a confidence interval for $\mu_x^\ast$ can be derived by simply inverting the endpoints of that for $\eta_x^\ast$.


%% file: figs/higrad_cartoon.tex
\begin{tikzpicture}[x=1pt,y=1pt]
\definecolor{fillColor}{RGB}{255,255,255}
\path[use as bounding box,fill=fillColor,fill opacity=0.00] (0,0) rectangle (433.62,213.83);
\begin{scope}
\path[clip] (  0.00,  0.00) rectangle (433.62,213.83);
\definecolor{drawColor}{RGB}{0,0,0}

\path[draw=drawColor,line width= 0.8pt,line join=round,line cap=round] ( 16.06,106.92) -- (100.36,106.92);

\path[draw=drawColor,line width= 0.8pt,line join=round,line cap=round] ( 97.24,105.11) --
	(100.36,106.92) --
	( 97.24,108.72);

\path[draw=drawColor,line width= 0.4pt,dash pattern=on 1pt off 3pt ,line join=round,line cap=round] (100.36,106.92) -- (131.98,161.67);

\path[draw=drawColor,line width= 0.4pt,dash pattern=on 1pt off 3pt ,line join=round,line cap=round] (100.36,106.92) -- (131.98, 52.16);

\path[draw=drawColor,line width= 0.8pt,line join=round,line cap=round] (131.98,161.67) -- (216.28,161.67);

\path[draw=drawColor,line width= 0.8pt,line join=round,line cap=round] (213.15,159.87) --
	(216.28,161.67) --
	(213.15,163.48);

\path[draw=drawColor,line width= 0.8pt,line join=round,line cap=round] (131.98, 52.16) -- (216.28, 52.16);

\path[draw=drawColor,line width= 0.8pt,line join=round,line cap=round] (213.15, 50.35) --
	(216.28, 52.16) --
	(213.15, 53.97);

\path[draw=drawColor,line width= 0.4pt,dash pattern=on 1pt off 3pt ,line join=round,line cap=round] (216.28,161.67) -- (232.09,189.05);

\path[draw=drawColor,line width= 0.4pt,dash pattern=on 1pt off 3pt ,line join=round,line cap=round] (216.28,161.67) -- (232.09,134.30);

\path[draw=drawColor,line width= 0.4pt,dash pattern=on 1pt off 3pt ,line join=round,line cap=round] (216.28, 52.16) -- (232.09, 79.54);

\path[draw=drawColor,line width= 0.4pt,dash pattern=on 1pt off 3pt ,line join=round,line cap=round] (216.28, 52.16) -- (232.09, 24.78);

\path[draw=drawColor,line width= 0.8pt,line join=round,line cap=round] (232.09,189.05) -- (316.39,189.05);

\path[draw=drawColor,line width= 0.8pt,line join=round,line cap=round] (316.39,185.44) --
	(316.39,189.05) --
	(316.39,192.67);

\path[draw=drawColor,line width= 0.8pt,line join=round,line cap=round] (232.09,134.30) -- (316.39,134.30);

\path[draw=drawColor,line width= 0.8pt,line join=round,line cap=round] (316.39,130.68) --
	(316.39,134.30) --
	(316.39,137.91);

\path[draw=drawColor,line width= 0.8pt,line join=round,line cap=round] (232.09, 79.54) -- (316.39, 79.54);

\path[draw=drawColor,line width= 0.8pt,line join=round,line cap=round] (316.39, 75.92) --
	(316.39, 79.54) --
	(316.39, 83.15);

\path[draw=drawColor,line width= 0.8pt,line join=round,line cap=round] (232.09, 24.78) -- (316.39, 24.78);

\path[draw=drawColor,line width= 0.8pt,line join=round,line cap=round] (316.39, 21.17) --
	(316.39, 24.78) --
	(316.39, 28.39);

\node[text=drawColor,anchor=base,inner sep=0pt, outer sep=0pt, scale=  0.75] at ( 58.21,113.47) {$\theta_1^{\emptyset},\dots\dots,\theta_{n_0}^{\emptyset}$};

\node[text=drawColor,anchor=base,inner sep=0pt, outer sep=0pt, scale=  0.75] at (174.13,168.23) {$\theta^{(1)}_1,\dots\dots,\theta^{(1)}_{n_1}$};

\node[text=drawColor,anchor=base,inner sep=0pt, outer sep=0pt, scale=  0.75] at (174.13, 58.72) {$\theta^{(2)}_1,\dots\dots,\theta^{(2)}_{n_1}$};

\node[text=drawColor,anchor=base,inner sep=0pt, outer sep=0pt, scale=  0.75] at (274.24,195.61) {$\theta^{(1,1)}_1,\dots\dots,\theta^{(1,1)}_{n_2}$};

\node[text=drawColor,anchor=base,inner sep=0pt, outer sep=0pt, scale=  0.75] at (274.24,140.85) {$\theta^{(1,2)}_1,\dots\dots,\theta^{(1,2)}_{n_2}$};

\node[text=drawColor,anchor=base,inner sep=0pt, outer sep=0pt, scale=  0.75] at (274.24, 86.09) {$\theta^{(2,1)}_1,\dots\dots,\theta^{(2,1)}_{n_2}$};

\node[text=drawColor,anchor=base,inner sep=0pt, outer sep=0pt, scale=  0.75] at (274.24, 31.34) {$\theta^{(2,2)}_1,\dots\dots,\theta^{(2,2)}_{n_2}$};

\node[text=drawColor,anchor=base,inner sep=0pt, outer sep=0pt, scale=  0.75] at ( 58.21, 96.61) {$\bar\theta^{\emptyset}$};

\node[text=drawColor,anchor=base,inner sep=0pt, outer sep=0pt, scale=  0.75] at (174.13,151.37) {$\bar\theta^{(1)}$};

\node[text=drawColor,anchor=base,inner sep=0pt, outer sep=0pt, scale=  0.75] at (174.13, 41.85) {$\bar\theta^{(2)}$};

\node[text=drawColor,anchor=base,inner sep=0pt, outer sep=0pt, scale=  0.75] at (274.24,178.75) {$\bar\theta^{(1,1)}$};

\node[text=drawColor,anchor=base,inner sep=0pt, outer sep=0pt, scale=  0.75] at (274.24,123.99) {$\bar\theta^{(1,2)}$};

\node[text=drawColor,anchor=base,inner sep=0pt, outer sep=0pt, scale=  0.75] at (274.24, 69.23) {$\bar\theta^{(2,1)}$};

\node[text=drawColor,anchor=base,inner sep=0pt, outer sep=0pt, scale=  0.75] at (274.24, 14.48) {$\bar\theta^{(2,2)}$};
\definecolor{drawColor}{RGB}{190,190,190}

\path[draw=drawColor,line width= 0.4pt,line join=round,line cap=round] ( 16.06, 98.49) -- ( 45.57, 98.49);

\path[draw=drawColor,line width= 0.4pt,line join=round,line cap=round] ( 16.06,100.65) --
	( 16.06, 98.49) --
	( 16.06, 96.32);

\path[draw=drawColor,line width= 0.4pt,line join=round,line cap=round] ( 70.86, 98.49) -- (100.36, 98.49);

\path[draw=drawColor,line width= 0.4pt,line join=round,line cap=round] (100.36, 96.32) --
	(100.36, 98.49) --
	(100.36,100.65);

\path[draw=drawColor,line width= 0.4pt,line join=round,line cap=round] (131.98,153.24) -- (161.49,153.24);

\path[draw=drawColor,line width= 0.4pt,line join=round,line cap=round] (131.98,155.41) --
	(131.98,153.24) --
	(131.98,151.08);

\path[draw=drawColor,line width= 0.4pt,line join=round,line cap=round] (186.78,153.24) -- (216.28,153.24);

\path[draw=drawColor,line width= 0.4pt,line join=round,line cap=round] (216.28,151.08) --
	(216.28,153.24) --
	(216.28,155.41);

\path[draw=drawColor,line width= 0.4pt,line join=round,line cap=round] (131.98, 43.73) -- (161.49, 43.73);

\path[draw=drawColor,line width= 0.4pt,line join=round,line cap=round] (131.98, 45.90) --
	(131.98, 43.73) --
	(131.98, 41.56);

\path[draw=drawColor,line width= 0.4pt,line join=round,line cap=round] (186.78, 43.73) -- (216.28, 43.73);

\path[draw=drawColor,line width= 0.4pt,line join=round,line cap=round] (216.28, 41.56) --
	(216.28, 43.73) --
	(216.28, 45.90);

\path[draw=drawColor,line width= 0.4pt,line join=round,line cap=round] (232.09,180.62) -- (261.60,180.62);

\path[draw=drawColor,line width= 0.4pt,line join=round,line cap=round] (232.09,182.79) --
	(232.09,180.62) --
	(232.09,178.45);

\path[draw=drawColor,line width= 0.4pt,line join=round,line cap=round] (286.89,180.62) -- (316.39,180.62);

\path[draw=drawColor,line width= 0.4pt,line join=round,line cap=round] (316.39,178.45) --
	(316.39,180.62) --
	(316.39,182.79);

\path[draw=drawColor,line width= 0.4pt,line join=round,line cap=round] (232.09,125.86) -- (261.60,125.86);

\path[draw=drawColor,line width= 0.4pt,line join=round,line cap=round] (232.09,128.03) --
	(232.09,125.86) --
	(232.09,123.70);

\path[draw=drawColor,line width= 0.4pt,line join=round,line cap=round] (286.89,125.86) -- (316.39,125.86);

\path[draw=drawColor,line width= 0.4pt,line join=round,line cap=round] (316.39,123.70) --
	(316.39,125.86) --
	(316.39,128.03);

\path[draw=drawColor,line width= 0.4pt,line join=round,line cap=round] (232.09, 71.11) -- (261.60, 71.11);

\path[draw=drawColor,line width= 0.4pt,line join=round,line cap=round] (232.09, 73.28) --
	(232.09, 71.11) --
	(232.09, 68.94);

\path[draw=drawColor,line width= 0.4pt,line join=round,line cap=round] (286.89, 71.11) -- (316.39, 71.11);

\path[draw=drawColor,line width= 0.4pt,line join=round,line cap=round] (316.39, 68.94) --
	(316.39, 71.11) --
	(316.39, 73.28);

\path[draw=drawColor,line width= 0.4pt,line join=round,line cap=round] (232.09, 16.35) -- (261.60, 16.35);

\path[draw=drawColor,line width= 0.4pt,line join=round,line cap=round] (232.09, 18.52) --
	(232.09, 16.35) --
	(232.09, 14.18);

\path[draw=drawColor,line width= 0.4pt,line join=round,line cap=round] (286.89, 16.35) -- (316.39, 16.35);

\path[draw=drawColor,line width= 0.4pt,line join=round,line cap=round] (316.39, 14.18) --
	(316.39, 16.35) --
	(316.39, 18.52);
\definecolor{drawColor}{RGB}{0,0,0}

\node[text=drawColor,anchor=base west,inner sep=0pt, outer sep=0pt, scale=  0.75] at (318.5,187.33) {$\bar\theta_{(1,1)}=w_0\bar\theta^{\emptyset}+w_1\bar\theta^{(1)}+w_2\bar\theta^{(1,1)}$};

\node[text=drawColor,anchor=base west,inner sep=0pt, outer sep=0pt, scale=  0.75] at (318.5,132.57) {$\bar\theta_{(1,2)}=w_0\bar\theta^{\emptyset}+w_1\bar\theta^{(1)}+w_2\bar\theta^{(1,2)}$};

\node[text=drawColor,anchor=base west,inner sep=0pt, outer sep=0pt, scale=  0.75] at (318.5, 77.82) {$\bar\theta_{(2,1)}=w_0\bar\theta^{\emptyset}+w_1\bar\theta^{(2)}+w_2\bar\theta^{(2,1)}$};

\node[text=drawColor,anchor=base west,inner sep=0pt, outer sep=0pt, scale=  0.75] at (318.5, 23.06) {$\bar\theta_{(2,2)}=w_0\bar\theta^{\emptyset}+w_1\bar\theta^{(2)}+w_2\bar\theta^{(2,2)}$};
\end{scope}
\end{tikzpicture}

%% file: method2.tex
\subsection{Correct coverage probabilities}
\label{sec:asympt-valid-conf}
The subject of this section is to provide theoretical support for the HiGrad confidence interval. We begin by stating the assumptions needed for the main theoretical results, Proposition \ref{prop:t}, Theorem \ref{prop:conf}, and Theorem \ref{prop:pred}. As earlier, $\|\cdot\|$ denotes the $\ell_2$ norm for a vector and the spectral norm for a matrix.

\begin{assumption}[Regularity of the objective]\label{ass:cvx}
The objective function $f(\theta)$ is differentiable and convex, and its gradient $\nabla f$ is Lipschitz continuous, that is, for some $L > 0$,
\[
\|\nabla f(\theta_1) - \nabla f(\theta_2)\| \le L \|\theta_1 -\theta_2\|
\]
holds for all $\theta_1$ and $\theta_2$. In addition, the Hessian $\nabla^2 f(\theta)$ exists in a neighborhood of $\theta^\ast$ with $\nabla^2 f(\theta^\ast)$ being positive-definite, and it is locally Lipschitz continuous in the sense that there exists $L', \delta_1 > 0$ such that
\[
\left\| \nabla^2 f(\theta) - \nabla^2 f(\theta^\ast) \right\| \le L' \|\theta - \theta^\ast\|_2
\]
if $\|\theta - \theta^\ast\| \le \delta_1$.
\end{assumption}

In the next assumption, denote by $V :=\E \left[ g(\theta^\ast, Z) g(\theta^\ast, Z)^\top \right]$ and $\epsilon = g(\theta, Z) - \nabla f(\theta)$. Thus, $V = \E_{\theta^\ast} \epsilon \epsilon^\top$ by using the fact $\nabla f(\theta^\ast) = 0$. Note that $\epsilon$ has mean zero and its distribution in general depends on $\theta$.
\begin{assumption}[Regularity of noisy gradient]\label{ass:reg}
There exists a constant $C > 0$ such that
\[
\left\| \E_{\theta} \epsilon \epsilon^\top - V \right\| \le C(\|\theta-\theta^\ast\| + \|\theta-\theta^\ast\|^2)
\]
for all $\theta$. Moreover, assume there exists a constant $\delta_2 > 0$ such that
\[
\sup_{\|\theta - \theta^\ast\| \le \delta_2} \E_{\theta} \|\epsilon\|^{2+\delta_2} < \infty.
\]

\end{assumption}

Assumptions exactly the same as or basically equivalent to the above two have been made in a series of papers working on averaged SGD and beyond, see \cite{ruppert1988,polyak1990,polyak1992,moulines2011non,fort2012central,dieuleveut2016nonparametric,chen2016,toulis2017asymptotic,li2017statistical,fang2017scalable} and references therein. Specifically, Assumption \ref{ass:cvx} considers a form of local strong convexity of the objective $f$ at the minimizer $\theta^\ast$. More precisely, the positive-definiteness of the Hessian $\nabla^2 f$ at $\theta^\ast$ together with the local Lipschitz continuity of the Hessian implies that $f(\theta) - \delta \|\theta\|^2/2$ is convex on $\{\theta: \|\theta - \theta^\ast\| \le \delta\}$ for some small $\delta > 0$. Hence, the idea here is that we first
run SGD in one thread and after a number of steps, the iterate would be sufficiently close to $\theta^\ast$ so that the strong convexity kicks in. This viewpoint is consistent with the current opinion about SGD that it automatically adapts to local strong convexity (see \cite{bach2013non,bach2014adaptivity,gadat2017optimal}). In the first display of Assumption \ref{ass:reg}, the term $\|\theta-\theta^\ast\|$ is used to ensure the continuity of the covariance of $\epsilon$ at $\theta^\ast$, while the second term $\|\theta-\theta^\ast\|^2$ controls the growth of the covariance. Recognizing that the first two assumptions remain to hold if both $\delta_1, \delta_2$ are replaced by $\min\{\delta_1, \delta_2\}$, we shall simply use $\delta > 0$ for both cases.

These two assumptions are generally satisfied for the four aforementioned examples in Section \ref{sec:problem-statement}. Below, we only consider the example of linear regression. Note that
\[
f(\theta) = \E \frac12 (Y - X^\top\theta)^2 = \frac12 \theta^\top \left[\E X X^\top \right] \theta - \left[\E Y X\right]^\top \theta + \frac12 \E Y^2,
\]
which is a simple quadratic function. Hence, Assumption \ref{ass:cvx} readily follows as long as $\E X X^\top$ exists, that is, $\|X\|$ has a second moment, and is positive-definite. The positive-definiteness holds if the vector $X \in \R^d$ is in generic positions \citep{tibshirani2012degrees}, for example, having probability density well-defined in a small ball. Next, Assumption \ref{ass:reg} is satisfied if $\E \|X\|^{4+c} < \infty$ and $\E |Y|^{2+c} \|X\|^{2+c} < \infty$ for a sufficiently small $c > 0$. More details and the other examples are considered in the appendix.


Now, we are in a position to state our main theoretical result, namely, Proposition \ref{prop:t}. Throughout the paper, the function $\mu_x$ is differentiable in a neighborhood of $\theta^\ast$ and $\frac{\d \mu_x(\theta)}{\d \theta}\Big|_{\theta = \theta^\ast} \ne 0$.
The weights $\w$ are taken to be any fixed vector of non-negative entries that sum to 1. Recall that $\overline\mu_x$ and $\se_x$ are defined in \eqref{eq:average} and \eqref{eq:se}, respectively.

\begin{proposition}\label{prop:t}
Let $K$ and $B_1, \ldots, B_K$ be fixed. For each $k$, assume $n_k/N$ converges to a nonzero constant as $N \goto \infty$. Under Assumptions \ref{ass:cvx} and \ref{ass:reg}, taking step sizes $\gamma_j = \frac{c_1}{(j + c_2)^{\alpha}}$ for fixed $\alpha \in (0.5, 1), c_1 > 0$ and $c_2$ ensures the following convergence in distribution as $N \goto \infty$:
\[
\frac{\overline\mu_x - \mu_x^\ast}{\se_x} \Longrightarrow t_{T - 1}.
\]
\end{proposition}

\begin{remark}
The choice of step sizes $\gamma_j \asymp j^{-\alpha}$ obeys 
\[
\sum_{j=1}^{\infty} \gamma_j = \infty \text{ and } \sum_{j=1}^{\infty} \gamma_j^2 < \infty.
\] 
In particular, the step sizes vanish to zero at a rate slower than $O(j^{-1})$, and this is shown to be necessary for the averaged SGD to outperform the Robbins--Monro algorithm (see, for example, \cite{citeulike:2621242})\footnote{The assumption on step sizes can be relaxed significantly by Assumption 3 in the appendix, without affecting the validity of any results in this section. We opt for the present one in the main text for its simplicity.}. 
\end{remark}

The two theorems below are immediate consequences of Proposition \ref{prop:t}. We prefer to state them as theorems rather than corollaries as to highlight their key roles in this paper. Throughout the paper, significance level $\alpha \in (0, 1)$ is fixed.
\begin{theorem}[Confidence intervals]\label{prop:conf}
Under the assumptions of Proposition \ref{prop:t}, the HiGrad confidence interval is asymptotically correct. That is,
\[
\lim_{N \goto \infty} \P\left(\mu_x^\ast \in \left[\overline \mu_x - t_{T - 1, 1-\frac{\alpha}{2}}\se_x, \quad \overline\mu_x + t_{T - 1, 1-\frac{\alpha}{2}}\se_x \right] \right) = 1 - \alpha.
\]
\end{theorem}

In words, the deterministic value $\mu_x^\ast$ is contained in the (random) interval \[
\left[\overline \mu_x - t_{T - 1, 1-\frac{\alpha}{2}}\se_x, \overline\mu_x + t_{T - 1, 1-\frac{\alpha}{2}}\se_x \right]
\]
with probability tending to $1-\alpha$.

Although continuing to hold, Theorem \ref{prop:conf} might lose interpretability in the case of model misspecification (for example, $f$ is not the negative log-likelihood). As a consequence, the HiGrad confidence interval \eqref{eq:conf} might merely cover a value irrelevant to its own interpretation.


Theorem \ref{prop:pred} below provides a prediction interval with correct asymptotic coverage $1 - \alpha$. Derived from widening the HiGrad confidence interval by a factor of $\sqrt{2}$, this prediction interval covers the estimator in \eqref{eq:average} computed from a fresh dataset following the same distribution with probability tending to $1 - \alpha$. Even in the case of model misspecification, it has substantive interpretation. For instance, its length shall shed light on the variability of the estimator $\overline\mu_x$.

\begin{theorem}[Prediction intervals]\label{prop:pred}
Let $\{Z_j'\}_{j=1}^{N}$ be an independent copy of the sample $\{Z_j\}_{j=1}^{N}$. Under the assumptions of Proposition \ref{prop:t}, apply the same HiGrad procedure to $\{Z_j'\}_{j=1}^{N}$ and get the estimator $\overline\mu_x'$. Then, we have
\[
\lim_{N \goto \infty} \P\left(\overline\mu_x' \in \left[\overline \mu_x - \sqrt{2}t_{T - 1, 1-\frac{\alpha}{2}}\se_x, \quad \overline\mu_x + \sqrt{2} t_{T - 1, 1-\frac{\alpha}{2}}\se_x \right] \right) = 1 - \alpha.
\]
\end{theorem}

\begin{remark}
The proof of this result follows from the simple fact that 
\[
\frac{\overline\mu_x - \overline\mu_x'}{\sqrt{2}\se_x} = \frac{(\overline\mu_x - \mu_x^\ast)- (\overline\mu_x'- \mu_x^\ast)}{\sqrt{2} \se_x}
\]
converges weakly to $t_{T - 1}$. Using optimal weights, HiGrad can also give a prediction interval for the vanilla SGD. Details are stated in Theorem \ref{prop:pred_vanilla} in Section \ref{sec:optimal-weights}. As a caveat, while a wide prediction interval implies large variability of the estimator, a short one does not necessarily ensure trustworthiness of the estimator due to a potentially large bias. 
\end{remark}

\subsection{Optimality}
\label{sec:optimal-weights}

The weights $\w$ have been treated so far as a generic nonnegative vector that sums to one. Moving forward, this section aims to identify a certain $\w$ that leads to the smallest asymptotic variance of the estimator $\overline\mu_x$. We begin with the fact\footnote{The notation $\bm 1$ denotes a column vector with all entries being 1. Its dimension is often clear from the context.}
\begin{equation}\nonumber
\sqrt{N} (\overline\mu_x - \mu_x^\ast) \Rightarrow \N\left( 0, \frac{\sigma^2 \bm 1^\top \Sigma \bm 1}{T^2} \right).
\end{equation}
Above, $\sigma^2$ is a constant independent of $\w$ (see more details in Section \ref{sec:proof-ideas}) and $\Sigma$ given in \eqref{eq:sigma} depends on $\w$ (recall that $\Sigma_{\bt, \bt'} = \sum_{k=0}^p \frac{w_k^2 N}{n_k}$ for any two threads $\bt, \bt'$ that agree exactly on the first $p$ segments). The display above reveals that $\overline\mu_x$ attains the minimum asymptotic variance if $\bm{1}^\top \Sigma \bm 1$ is minimized. The result below highlights the optimal weights in this sense.

\begin{proposition}\label{prop:weights}
Under the assumptions of Proposition \ref{prop:t}, $\bm{1}^\top \Sigma \bm 1$ attains the minimum if and only if
\begin{equation}\label{eq:wk}
w_k = \frac{n_k\prod_{i=0}^k B_i}{N}
\end{equation}
for all $k = 0, \ldots, K$.
\end{proposition}

Note that the optimal weights are independent of the choices of step sizes as long as the step sizes are specified by the assumptions of Proposition~\ref{prop:t}.

\begin{proof}[Proof of Proposition \ref{prop:weights}]
Let $p(\bt, \bt')$ denote the number of shared segments between two threads $\bt$ and $\bt'$. Note that
\[
\begin{aligned}
\bm{1}^\top \Sigma \bm 1 &= \sum_{\bt, \bt'}\sum_{k=0}^{p(\bt, \bt')} \frac{w_k^2 N}{n_k} = \sum_{k=0}^K \sum_{\bt, \bt'} \frac{w_k^2 N}{n_k} \mathbbm{1}(k \le p(\bt, \bt'))\\
&= \sum_{k=0}^K \frac{w_k^2 N}{n_k} \sum_{\bt, \bt'} \mathbbm{1}(k \le p(\bt, \bt')).
\end{aligned}
\]
To proceed, note that\footnote{Note that we use the convention that $B_0 = 1$.}
\[
\sum_{\bt, \bt'} \mathbbm{1}(k \le p(\bt, \bt')) = B_0\cdots B_k (B_{k+1}\cdots B_K)^2 = \frac{T^2}{\prod_{i=0}^k B_i}.
\]
Hence, we get
\[
\begin{aligned}
\bm{1}^\top \Sigma \bm 1 &= N T^2 \sum_{k=0}^K \frac{w_k^2}{n_k \prod_{i=0}^k B_i}\\
&= T^2 \left[\sum_{k=0}^K n_k\prod_{i=0}^k B_i \right] \left[\sum_{k = 0}^K\frac{w_k^2}{n_k\prod_{i=0}^k B_i } \right]\\
& \ge T^2 \left[\sum_{k=0}^K \sqrt{w_k^2} \right]^2\\
& = T^2.
\end{aligned}
\]
Above, we have made use of \eqref{eq:sum_N} and the Cauchy--Schwarz inequality, which is reduced to an equality if and only if \eqref{eq:wk} holds for all $k = 0, \ldots, K$.

\end{proof}

In particular, the proof suggests that asymptotic variance of the HiGrad estimator with the optimal weights is $\sigma^2 T^2 /(N T^2) = \sigma^2/N$, no matter the configuration of the HiGrad tree $(B_1, \ldots, B_K)$ and $(n_0, \ldots, n_K)$. As a special case, taking $K = 0$ and $n_0 = N$ shows that the HiGrad variance is the same as the vanilla averaged SGD. This fact demonstrates that the splitting strategy in HiGrad does not lose any statistical efficiency in providing uncertainty quantification. As a consequence, the discussion implies that the prediction interval in Theorem \ref{prop:pred} applies to the vanilla SGD using step sizes $\gamma_j = c_1/(j+c_2)^{\alpha}$ for $\alpha \in (0.5, 1)$ as well.

\begin{theorem}[Prediction intervals for vanilla SGD]\label{prop:pred_vanilla}
Under the assumptions of Theorem \ref{prop:pred}, apply the vanilla SGD to $\{Z_j'\}_{j=1}^{N}$ and get the estimator $\mu_x^{\mathrm{\tiny SGD}}$. Then, we have
\[
\lim_{N \goto \infty} \P\left(\mu_x^{\mathrm{\tiny SGD}} \in \left[\overline \mu_x - \sqrt{2}t_{T - 1, 1-\frac{\alpha}{2}}\se_x, \quad \overline\mu_x + \sqrt{2} t_{T - 1, 1-\frac{\alpha}{2}}\se_x \right] \right) = 1 - \alpha.
\]
\end{theorem}

This optimality of the HiGrad variance merits a stronger sense in the case where $f(\theta, Z)$ is the negative log-likelihood of $\theta$, that is, the model is correctly specified. In that case, $\sigma^2$ is shown in the discussion right below Lemma \ref{lm:pj} to coincide with the inverse of the Fisher information of $\mu_x(\theta)$ at $\theta^\ast$. Put more simply, the HiGrad procedure with the optimal weights achieves the Cram\'er--Rao lower bound among all (asymptotically) unbiased estimators.

\subsection{Proof sketch}
\label{sec:proof-ideas}

This section provides an overview of the proof of Proposition \ref{prop:t}, with an emphasis on high-level ideas rather than technical details, which can be found in the appendix. As will be shown, Proposition \ref{prop:t} is implied by Lemma \ref{lm:pj}. To state this lemma, we introduce some notations as follows. Consider the SGD rule \eqref{eq:sgd} for $j = 1, \ldots, n$. Write $n = n_0 + n_1 + \cdots + n_K$, and denote by $s_k = n_0 + \cdots + n_k$, with the convention that $s_{-1} = 0$. Define
\[
\overline\theta(k) = \frac1{n_k}\sum_{j={s_{k-1}+1}}^{s_k} \theta_j
\]
for $k = 0, \ldots, K$. Again, the choice of step sizes $\gamma_j = \frac{c_1}{(j + c_2)^{\alpha}}$ used here can be relaxed by Assumption 3 in the appendix.

\begin{lemma}\label{lm:pj}
For each $k$, assume $n_k/n$ converges to some nonzero constant as $n \goto \infty$. Under Assumptions \ref{ass:cvx} and \ref{ass:reg}, $\sqrt{n_0}(\overline\theta(0) - \theta^\ast), \sqrt{n_1}(\overline\theta(1) - \theta^\ast), \ldots, \sqrt{n_K}(\overline\theta(K) - \theta^\ast)$ converge weakly to $K+1$ \iid centered normal random variables.
\end{lemma}

This lemma is in fact a Donsker-style generalization of the normality of the celebrated Ruppert--Polyak averaging scheme \citep{ruppert1988, polyak1990, polyak1992}, and this generealization requires certain technical novelties. Explicitly, results in \cite{ruppert1988, polyak1990, polyak1992} state that, writing $\overline\theta$ for the sample mean of the vanilla SGD iterates $\theta_1, \ldots, \theta_n$, the random variable $\sqrt{n}(\overline\theta - \theta^\ast)$ converges to $\N(0, W)$ in distribution under the same assumptions as in Lemma \ref{lm:pj}. The covariance matrix $W$ takes the following sandwich form \citep{white1980heteroskedasticity}:
\[
W = H^{-1} V H^{-1},
\]
where $V$ has appeared in Assumption \ref{ass:reg} and $H$ is the Hessian $\nabla^2 f(\theta^\ast)$. Both $V$ and $H$ coincide with the Fisher information 
\[
I(\theta) = \E \, \nabla_{\theta} f(\theta, Z) \nabla_{\theta} f(\theta, Z)^\top = \E \nabla^2 f(\theta, Z)
\]
at $\theta = \theta^\ast$ if $f(\theta, Z)$ is taken to be the negative log-likelihood. Therefore, the averaged SGD iterates $\overline\theta$ matches the Cram\'er--Rao lower bound. Going back to Lemma \ref{lm:pj}, every $\sqrt{n_k}(\overline\theta(k) - \theta^\ast)$ converges to $\N(0, W)$ as the Ruppert--Polyak normality kicks in, and the (asymptotic) independence between these $K+1$ random variables is established in the proof of Lemma \ref{lm:pj} in the appendix by observing the rapid decaying of correlations among distant SGD iterates. As will be seen right below, the proof of Proposition \ref{prop:t} using Lemma \ref{lm:pj} holds regardless of the covariance $W$. In other words, the lemma does \textit{not} make full use of the Ruppert--Polyak normality result.

To obtain a confidence interval based on $\overline\mu_x$, one needs to specify the correlation structure of $\mu_x^{\bt}$ for all threads $\bt$. Lemma \ref{lm:pj} serves this purpose. To begin with, observe that
\begin{equation}\label{eq:delta_exp}
\mu_x(\theta) = \mu_x^\ast + (\theta - \theta^\ast)^\top \frac{\d \mu_x}{\d\theta}\Big|_{\theta=\theta^\ast} + o(\|\theta - \theta^\ast\|).
\end{equation}
Drop the small term $o(\|\theta - \theta^\ast\|)$ and denote by $\nu$ the column vector $\frac{\d \mu_x}{\d\theta} (\theta^\ast)$. Applying the Taylor expansion together with \eqref{eq:theta_ave} yields\footnote{Write $\bt = (b_1, \ldots, b_K)$ and let $\bt_k = (b_1, \ldots, b_k)$ for $k = 0, 1, \ldots, K$.}
\begin{equation}\label{eq:mu_ind}
\begin{aligned}
\mu_x(\overline\theta_{\bt}) &\approx \mu_x^\ast + \nu^\top \left( \sum_{k=0}^K w_k \overline \theta^{\bt_k} - \theta^\ast \right)\\
&= \mu_x^\ast + \nu ^\top\sum_{k=0}^K w_k (\overline \theta^{\bt_k} - \theta^\ast ).
\end{aligned}
\end{equation}
Now, suppose two threads $\bt, \bt'$ agree in the first $p$ segments, hence sharing the first $p$ summands in the second line of \eqref{eq:mu_ind}. Making use of the fact that the $K + 1$ summands are asymptotically independent as claimed by Lemma \ref{lm:pj}, the asymptotic covariance of $\sqrt{N} (\mu_x^{\bt} - \mu_x^\ast)$ and $\sqrt{N} (\mu_x^{\bt'} - \mu_x^\ast)$ equals
\[
\left( \sum_{k=0}^p \frac{w_k^2 N}{n_k} \right) \nu^\top W \nu.
\]
Consequently, the covariance of $\bm\mu_x \in \R^T$ is approximately given by $\Sigma$ up to a scaling factor of $\nu^\top W \nu/N$, which is independent of the weights $\w$ and the HiGrad tree structure. Taking $\sigma^2 := \nu^\top W \nu$, the discussion above is summarized in the following lemma:


\begin{lemma}\label{lm:delta}
Under the assumptions of Proposition \ref{prop:t}, $\sqrt{N}(\bm\mu_x - \mu_x^\ast \bm 1)$ converges weakly to a normal distribution with mean vector zero and covariance $\sigma^2 \Sigma$ as $N \goto \infty$. 
\end{lemma}

With Lemma \ref{lm:delta} in place, we are ready to give an informal proof of Proposition \ref{prop:t}.
\begin{proof}[Sketch proof of Proposition \ref{prop:t}]
First, we point out that the HiGrad estimator $\overline\mu_x$ in \eqref{eq:average} coincides with the \textit{least-squares estimator} of $\mu_x^\ast$. To see this, note that Lemma \ref{lm:delta} amounts to saying $\bm \mu_x \approx \mu_x^\ast \bm 1 + \bm z$ with $\bm z \sim \N(0, \sigma^2 \Sigma/N)$, which is equivalent to
\begin{equation}\label{eq:normal_approx}
\Sigma^{-\frac12} \bm\mu_x \approx (\Sigma^{-\frac12} \bm 1) \mu_x^\ast + \bm{\tilde z}.
\end{equation}
The noise term $\bm{\tilde z} = \Sigma^{-\frac12} \bm z \sim \N(0, \frac{\sigma^2}{N} \bm I)$ has been whitened. Thus, \eqref{eq:normal_approx} is a linear regression with $T$ observations and one unknown parameters $\mu_x^\ast$. The least-squares estimator of $\mu_x^\ast$ is
\[
\begin{aligned}
\widehat{\mu}_x &= (\bm 1^\top \Sigma^{-\frac12} \Sigma^{-\frac12} \bm 1)^{-1} \bm 1^\top \Sigma^{-\frac12}  \Sigma^{-\frac12} \bm\mu_x\\
& = (\bm 1^\top \Sigma^{-1} \bm 1)^{-1} \bm 1^\top \Sigma^{-1}\bm\mu_x.
\end{aligned}
\]
To proceed, recognize that $\bm{1}$ is an eigenvector of $\Sigma$ (denote by $\lambda$ the corresponding eigenvalue) due to the symmetric construction of $\Sigma$. Hence, we get
\[
\begin{aligned}
\widehat{\mu}_x & = (\bm 1^\top \Sigma^{-1} \bm 1)^{-1} \bm 1^\top \Sigma^{-1}\bm\mu_x\\
& = \left(\frac{\bm 1^\top \bm 1}{\lambda} \right)^{-1} \frac{\bm 1^\top}{\lambda} \, \bm\mu_x\\
& = \frac1{T} \sum_{\bt \in \mathcal{T}} \mu_x^{\bt},
\end{aligned}
\]
which is simply the sample mean $\overline\mu_x$. Moreover, the standard error of $\widehat{\mu}_x \equiv \overline\mu_x$ is
\[
\se_x = \frac{\widehat\sigma}{\sqrt N} \cdot \frac{\sqrt{\bm 1^\top \Sigma \bm 1}}{T},
\]
where
\[
\widehat\sigma^2 = \frac{N\|\Sigma^{-\frac12} (\bm\mu_x - \overline\mu_x \bm 1)\|^2}{T - 1} = \frac{N(\bm\mu_x^\top - \overline\mu_x \bm 1^\top) \Sigma^{-1} (\bm\mu_x - \overline\mu_x \bm 1)}{T - 1}.
\]
Note that the present form of $\se_x$ is equivalent to that in \eqref{eq:se}. Hence, the pivot
\[
\frac{\widehat\mu_x - \mu_x^\ast}{\se_x} = \frac{\overline\mu_x - \mu_x^\ast}{\se_x} = \frac{\sqrt{N}(\overline\mu_x - \mu_x^\ast)}{\widehat\sigma \sqrt{\bm 1^\top \Sigma \bm 1}\, / \, T}
\]
converges in distribution to a Student's $t$ random variable with $T - 1$ degrees of freedom. This concludes the proof.
\end{proof}


%% file: split.tex
\section{Configuring HiGrad}
\label{sec:how-split}
\newcommand{\Lci}{L_{\textnormal{CI}}}

The HiGrad algorithm takes as input $(B_1, \ldots, B_K)$ and $(n_0, n_1, \ldots, n_K)$, and this section aims to shed some light on how to choose the structural parameters. The main takeaway from this section is that HiGrad with $T = 4$ performs competitively in many practical problems.

With the goal of balancing contrasting and sharing, we consider the confidence interval length as a measure to evaluate HiGrad structures. While the results in Section \ref{sec:method-1} show that all HiGrad confidence intervals have the same coverage probability asymptotically, the average length of the confidence interval allows us to distinguish between different HiGrad structures. Apparently, a shorter confidence interval is better appreciated.

Denote by $\Lci = 2t_{T - 1, 1-\frac{\alpha}{2}}\se_x$ the length of the HiGrad confidence interval. Using the optimal weights, $\sqrt{N} (\overline\mu_x - \mu^\ast_x)$ is known to converge to $\N(0, \sigma^2)$ in distribution. Hence, $\sqrt{N}\se_x/\sigma$ follows $\chi_{T-1}/\sqrt{T-1}$ asymptotically, a rescaled chi random variable. As a consequence, the expectation of $\Lci$ equals\footnote{It is possible to construct examples where $\E \Lci$ is infinite by letting $\mu_x(\theta)$ grow very fast away from $\theta^\ast$. We omit these types of examples.}
\[
\frac{(2 + o(1)) t_{T - 1, 1-\frac{\alpha}{2}} \sigma \E \chi_{T-1}}{\sqrt{N(T - 1)}} = \frac{(2\sqrt{2}+o(1)) \sigma}{\sqrt{N}} \cdot \frac{t_{T-1, 1-\frac{\alpha}{2}} \Gamma\left(\frac{T}{2}\right)}{\sqrt{T-1}\, \Gamma\left( \frac{T-1}{2}\right)}.
\]
The expression above reveals that the average length depends on the tree structure only through $T$, the number of threads. Moreover, one can show that, for any fixed $0 < \alpha < 1$,
\begin{equation}\label{eq:decrease}
\frac{t_{T-1, 1-\frac{\alpha}{2}} \Gamma\left(\frac{T}{2}\right)}{\sqrt{T-1}\, \Gamma\left( \frac{T-1}{2}\right)}
\end{equation}
is a decreasing function of $T \ge 2$.

A direct consequence of this decreasing monotonicity is: the larger the number $T$ of HiGrad threads, the shorter the confidence interval on average (asymptotically). The literal meaning of this sentence suggests splitting more---or, equivalently, seeking more contrast---would imply a better confidence interval. From a practical perspective, however, splitting too much is not necessarily effective, and it could even lead the HiGrad results to be untrustworthy at worst because some segments would \textit{not} be long enough to ensure the normality in Lemma \ref{lm:pj}. In particular, the thread length $n_0 + n_1 + \cdots + n_K$ would not be long enough to guarantee convergence if the width $T$ is too large. This point is consistent with Figure \ref{fig:leng}, where the function in \eqref{eq:decrease} decreases noticeably when $T$ is small. However, the marginal gain by increasing $T$ becomes tiny once $T$ exceeds 4. In fact, the value at $T=4$ is $1.318$ times of the value at $T = \infty$
for $\alpha = 0.1$. Moreover, with a large $T$, either some segments or every thread would be relatively short. The former case undermines the correlation structures given by \eqref{eq:sigma} and thus might yield an undesired coverage probability of the HiGrad confidence interval, while the latter might even fail to achieve satisfactory convergence to the minimizer.

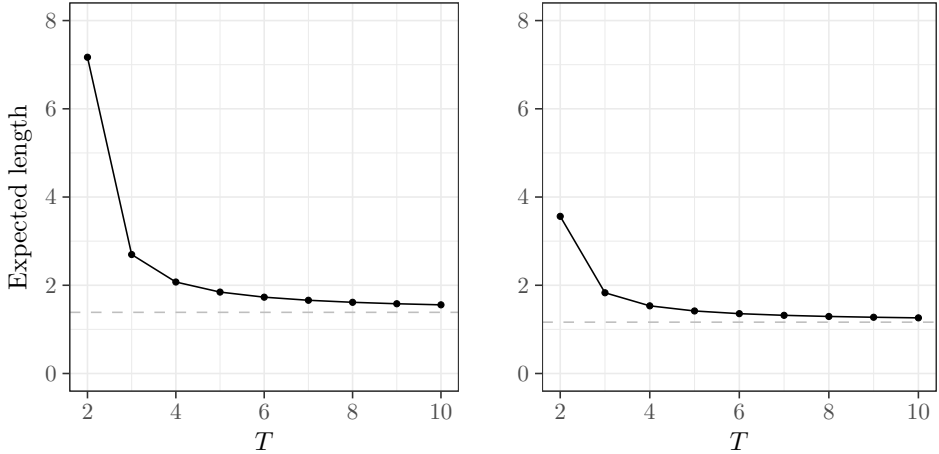
\begin{figure}[ht]
\begin{center}
\input{figs/conf_length.tex}
\caption{Rescaled expected length of confidence intervals versus $T$, the number of HiGrad threads. The left plot and right plot correspond to $\alpha = 0.05$ and $\alpha = 0.1$, respectively. The gray dashed lines indicate the confidence interval lengths at $T = \infty$.}
\label{fig:leng}
\end{center}
\end{figure}

To generate $T = 4$ threads in HiGrad, one could either set $K = 1, B_1 = 4$, or choose $K = 2, B_1 = B_2 = 2$. Since \textit{longer} HiGrad threads in general lead to \textit{better convergence}, we can distinguish between the two configurations by the length of threads. Explicitly, the latter case yields a longer thread in the case of an equal length of all threads and thus is preferred in this regard. More generally, let $T$ be a large number, and it is clear that the thread length is $2N/(T+1)$ in the setup where $n_0 = n_1, B_1 = T$, and $K = 1$. In contrast, if $n_0 = n_1 = \cdots = n_K$ and $B_1 = B_2 = \cdots = B_K = 2$, where $K = \log_2 T$, from \eqref{eq:sum_N} a little analysis shows that the thread length is
\[
\frac{(K+1) N}{ 2^{K+1} - 1} \approx \frac{(\log_2 T+1)N}{2T} = O\left( \frac{\log T }{T} \right) N,
\]
which is an order of magnitude larger than $O(1/T) N$ as in the direct splitting case. This comparison indeed demonstrates the benefit brought about by \textit{sharing} segments. As one of the two core ideas in HiGrad, sharing segments at early levels elongates threads.

In light of the above, an R package called \texttt{higrad} implementing this procedure sets the default tree structure to $K = 2, B_1 = B_2 = 2, T = 4$. This package is available at \url{https://cran.r-project.org/web/packages/higrad/}. From a practical perspective, however, the optimal HiGrad tree structure might vary with different problems, which necessitates the presentation of HiGrad in its most general form. To approach the optimal performance, one needs to choose a value of $T$ that is not too large or small. For example, one possibility is to leverage experiments to examine whether the HiGrad iterates bounce around some point. If not, it suggests that $T$ might be too large such that the SGD algorithm is not convergent~\citep{sordello2019robust}.


In addition to the tree structure, we still need to specify $n_0, n_1, \ldots, n_K$ under the constraint \eqref{eq:sum_N}. On the one hand, although the threads would be long if the segment length $n_k$ decreases fast as level $k$ increases, the covariance matrix given in \eqref{eq:sigma} might suffer from a bad condition number due to strong correlations among threads. As a result, the standard error given by \eqref{eq:se} might not be accurate. On the other hand, a rapidly increasing segment length would lead to a short thread. To balance the two considerations, the default values are set to $n_0 = n_1 = n_2 = N/7$ in the \texttt{higrad} package. The performance of this
configuration of the HiGrad algorithm is corroborated by extensive simulation studies shown in Section \ref{sec:simulations}, with satisfactory results across a range of examples. That being said, it is worth mentioning that this set of default parameters is employed in recognition of several considerations and constraints, and a different HiGrad configuration might be preferred in other settings.


%% file: figs/conf_length.tex
\begin{tikzpicture}[x=1pt,y=1pt]
\definecolor{fillColor}{RGB}{255,255,255}
\path[use as bounding box,fill=fillColor,fill opacity=0.00] (0,0) rectangle (361.35,180.67);
\begin{scope}
\path[clip] (  0.00,  0.00) rectangle (180.68,180.67);
\definecolor{drawColor}{RGB}{255,255,255}
\definecolor{fillColor}{RGB}{255,255,255}

\path[draw=drawColor,line width= 0.6pt,line join=round,line cap=round,fill=fillColor] ( -0.00,  0.00) rectangle (180.67,180.68);
\end{scope}
\begin{scope}
\path[clip] ( 28.09, 28.21) rectangle (175.17,175.17);
\definecolor{fillColor}{RGB}{255,255,255}

\path[fill=fillColor] ( 28.09, 28.21) rectangle (175.17,175.17);
\definecolor{drawColor}{gray}{0.92}

\path[draw=drawColor,line width= 0.3pt,line join=round] ( 28.09, 51.59) --
	(175.17, 51.59);

\path[draw=drawColor,line width= 0.3pt,line join=round] ( 28.09, 84.99) --
	(175.17, 84.99);

\path[draw=drawColor,line width= 0.3pt,line join=round] ( 28.09,118.39) --
	(175.17,118.39);

\path[draw=drawColor,line width= 0.3pt,line join=round] ( 28.09,151.79) --
	(175.17,151.79);

\path[draw=drawColor,line width= 0.3pt,line join=round] ( 51.49, 28.21) --
	( 51.49,175.17);

\path[draw=drawColor,line width= 0.3pt,line join=round] ( 84.92, 28.21) --
	( 84.92,175.17);

\path[draw=drawColor,line width= 0.3pt,line join=round] (118.35, 28.21) --
	(118.35,175.17);

\path[draw=drawColor,line width= 0.3pt,line join=round] (151.78, 28.21) --
	(151.78,175.17);

\path[draw=drawColor,line width= 0.6pt,line join=round] ( 28.09, 34.89) --
	(175.17, 34.89);

\path[draw=drawColor,line width= 0.6pt,line join=round] ( 28.09, 68.29) --
	(175.17, 68.29);

\path[draw=drawColor,line width= 0.6pt,line join=round] ( 28.09,101.69) --
	(175.17,101.69);

\path[draw=drawColor,line width= 0.6pt,line join=round] ( 28.09,135.09) --
	(175.17,135.09);

\path[draw=drawColor,line width= 0.6pt,line join=round] ( 28.09,168.49) --
	(175.17,168.49);

\path[draw=drawColor,line width= 0.6pt,line join=round] ( 34.78, 28.21) --
	( 34.78,175.17);

\path[draw=drawColor,line width= 0.6pt,line join=round] ( 68.21, 28.21) --
	( 68.21,175.17);

\path[draw=drawColor,line width= 0.6pt,line join=round] (101.63, 28.21) --
	(101.63,175.17);

\path[draw=drawColor,line width= 0.6pt,line join=round] (135.06, 28.21) --
	(135.06,175.17);

\path[draw=drawColor,line width= 0.6pt,line join=round] (168.49, 28.21) --
	(168.49,175.17);
\definecolor{drawColor}{RGB}{0,0,0}

\path[draw=drawColor,line width= 0.6pt,line join=round] ( 34.78,154.61) --
	( 51.49, 79.92) --
	( 68.21, 69.52) --
	( 84.92, 65.71) --
	(101.63, 63.78) --
	(118.35, 62.61) --
	(135.06, 61.84) --
	(151.78, 61.29) --
	(168.49, 60.88);
\definecolor{fillColor}{RGB}{0,0,0}

\path[draw=drawColor,line width= 0.4pt,line join=round,line cap=round,fill=fillColor] ( 34.78,154.61) circle (  1.16);

\path[draw=drawColor,line width= 0.4pt,line join=round,line cap=round,fill=fillColor] ( 51.49, 79.92) circle (  1.16);

\path[draw=drawColor,line width= 0.4pt,line join=round,line cap=round,fill=fillColor] ( 68.21, 69.52) circle (  1.16);

\path[draw=drawColor,line width= 0.4pt,line join=round,line cap=round,fill=fillColor] ( 84.92, 65.71) circle (  1.16);

\path[draw=drawColor,line width= 0.4pt,line join=round,line cap=round,fill=fillColor] (101.63, 63.78) circle (  1.16);

\path[draw=drawColor,line width= 0.4pt,line join=round,line cap=round,fill=fillColor] (118.35, 62.61) circle (  1.16);

\path[draw=drawColor,line width= 0.4pt,line join=round,line cap=round,fill=fillColor] (135.06, 61.84) circle (  1.16);

\path[draw=drawColor,line width= 0.4pt,line join=round,line cap=round,fill=fillColor] (151.78, 61.29) circle (  1.16);

\path[draw=drawColor,line width= 0.4pt,line join=round,line cap=round,fill=fillColor] (168.49, 60.88) circle (  1.16);
\definecolor{drawColor}{RGB}{190,190,190}

\path[draw=drawColor,line width= 0.6pt,dash pattern=on 4pt off 4pt ,line join=round] ( 28.09, 58.04) -- (175.17, 58.04);
\definecolor{drawColor}{gray}{0.20}

\path[draw=drawColor,line width= 0.6pt,line join=round,line cap=round] ( 28.09, 28.21) rectangle (175.17,175.17);
\end{scope}
\begin{scope}
\path[clip] (  0.00,  0.00) rectangle (361.35,180.67);
\definecolor{drawColor}{gray}{0.30}

\node[text=drawColor,anchor=base east,inner sep=0pt, outer sep=0pt, scale=  0.80] at ( 23.14, 32.14) {0};

\node[text=drawColor,anchor=base east,inner sep=0pt, outer sep=0pt, scale=  0.80] at ( 23.14, 65.54) {2};

\node[text=drawColor,anchor=base east,inner sep=0pt, outer sep=0pt, scale=  0.80] at ( 23.14, 98.94) {4};

\node[text=drawColor,anchor=base east,inner sep=0pt, outer sep=0pt, scale=  0.80] at ( 23.14,132.34) {6};

\node[text=drawColor,anchor=base east,inner sep=0pt, outer sep=0pt, scale=  0.80] at ( 23.14,165.74) {8};
\end{scope}
\begin{scope}
\path[clip] (  0.00,  0.00) rectangle (361.35,180.67);
\definecolor{drawColor}{gray}{0.20}

\path[draw=drawColor,line width= 0.6pt,line join=round] ( 25.34, 34.89) --
	( 28.09, 34.89);

\path[draw=drawColor,line width= 0.6pt,line join=round] ( 25.34, 68.29) --
	( 28.09, 68.29);

\path[draw=drawColor,line width= 0.6pt,line join=round] ( 25.34,101.69) --
	( 28.09,101.69);

\path[draw=drawColor,line width= 0.6pt,line join=round] ( 25.34,135.09) --
	( 28.09,135.09);

\path[draw=drawColor,line width= 0.6pt,line join=round] ( 25.34,168.49) --
	( 28.09,168.49);
\end{scope}
\begin{scope}
\path[clip] (  0.00,  0.00) rectangle (361.35,180.67);
\definecolor{drawColor}{gray}{0.20}

\path[draw=drawColor,line width= 0.6pt,line join=round] ( 34.78, 25.46) --
	( 34.78, 28.21);

\path[draw=drawColor,line width= 0.6pt,line join=round] ( 68.21, 25.46) --
	( 68.21, 28.21);

\path[draw=drawColor,line width= 0.6pt,line join=round] (101.63, 25.46) --
	(101.63, 28.21);

\path[draw=drawColor,line width= 0.6pt,line join=round] (135.06, 25.46) --
	(135.06, 28.21);

\path[draw=drawColor,line width= 0.6pt,line join=round] (168.49, 25.46) --
	(168.49, 28.21);
\end{scope}
\begin{scope}
\path[clip] (  0.00,  0.00) rectangle (361.35,180.67);
\definecolor{drawColor}{gray}{0.30}

\node[text=drawColor,anchor=base,inner sep=0pt, outer sep=0pt, scale=  0.80] at ( 34.78, 17.75) {2};

\node[text=drawColor,anchor=base,inner sep=0pt, outer sep=0pt, scale=  0.80] at ( 68.21, 17.75) {4};

\node[text=drawColor,anchor=base,inner sep=0pt, outer sep=0pt, scale=  0.80] at (101.63, 17.75) {6};

\node[text=drawColor,anchor=base,inner sep=0pt, outer sep=0pt, scale=  0.80] at (135.06, 17.75) {8};

\node[text=drawColor,anchor=base,inner sep=0pt, outer sep=0pt, scale=  0.80] at (168.49, 17.75) {10};
\end{scope}
\begin{scope}
\path[clip] (  0.00,  0.00) rectangle (361.35,180.67);
\definecolor{drawColor}{RGB}{0,0,0}

\node[text=drawColor,anchor=base,inner sep=0pt, outer sep=0pt, scale=  0.90] at (101.63,  6.06) {$T$};
\end{scope}
\begin{scope}
\path[clip] (  0.00,  0.00) rectangle (361.35,180.67);
\definecolor{drawColor}{RGB}{0,0,0}

\node[text=drawColor,rotate= 90.00,anchor=base,inner sep=0pt, outer sep=0pt, scale=  0.90] at ( 11.70,101.69) {Expected length};
\end{scope}
\begin{scope}
\path[clip] (180.68,  0.00) rectangle (361.35,180.67);
\definecolor{drawColor}{RGB}{255,255,255}
\definecolor{fillColor}{RGB}{255,255,255}

\path[draw=drawColor,line width= 0.6pt,line join=round,line cap=round,fill=fillColor] (180.68,  0.00) rectangle (361.35,180.68);
\end{scope}
\begin{scope}
\path[clip] (206.82, 28.21) rectangle (355.85,175.17);
\definecolor{fillColor}{RGB}{255,255,255}

\path[fill=fillColor] (206.82, 28.21) rectangle (355.85,175.17);
\definecolor{drawColor}{gray}{0.92}

\path[draw=drawColor,line width= 0.3pt,line join=round] (206.82, 51.59) --
	(355.85, 51.59);

\path[draw=drawColor,line width= 0.3pt,line join=round] (206.82, 84.99) --
	(355.85, 84.99);

\path[draw=drawColor,line width= 0.3pt,line join=round] (206.82,118.39) --
	(355.85,118.39);

\path[draw=drawColor,line width= 0.3pt,line join=round] (206.82,151.79) --
	(355.85,151.79);

\path[draw=drawColor,line width= 0.3pt,line join=round] (230.53, 28.21) --
	(230.53,175.17);

\path[draw=drawColor,line width= 0.3pt,line join=round] (264.40, 28.21) --
	(264.40,175.17);

\path[draw=drawColor,line width= 0.3pt,line join=round] (298.27, 28.21) --
	(298.27,175.17);

\path[draw=drawColor,line width= 0.3pt,line join=round] (332.14, 28.21) --
	(332.14,175.17);

\path[draw=drawColor,line width= 0.6pt,line join=round] (206.82, 34.89) --
	(355.85, 34.89);

\path[draw=drawColor,line width= 0.6pt,line join=round] (206.82, 68.29) --
	(355.85, 68.29);

\path[draw=drawColor,line width= 0.6pt,line join=round] (206.82,101.69) --
	(355.85,101.69);

\path[draw=drawColor,line width= 0.6pt,line join=round] (206.82,135.09) --
	(355.85,135.09);

\path[draw=drawColor,line width= 0.6pt,line join=round] (206.82,168.49) --
	(355.85,168.49);

\path[draw=drawColor,line width= 0.6pt,line join=round] (213.60, 28.21) --
	(213.60,175.17);

\path[draw=drawColor,line width= 0.6pt,line join=round] (247.47, 28.21) --
	(247.47,175.17);

\path[draw=drawColor,line width= 0.6pt,line join=round] (281.34, 28.21) --
	(281.34,175.17);

\path[draw=drawColor,line width= 0.6pt,line join=round] (315.21, 28.21) --
	(315.21,175.17);

\path[draw=drawColor,line width= 0.6pt,line join=round] (349.08, 28.21) --
	(349.08,175.17);
\definecolor{drawColor}{RGB}{0,0,0}

\path[draw=drawColor,line width= 0.6pt,line join=round] (213.60, 94.38) --
	(230.53, 65.45) --
	(247.47, 60.50) --
	(264.40, 58.56) --
	(281.34, 57.54) --
	(298.27, 56.91) --
	(315.21, 56.48) --
	(332.14, 56.18) --
	(349.08, 55.95);
\definecolor{fillColor}{RGB}{0,0,0}

\path[draw=drawColor,line width= 0.4pt,line join=round,line cap=round,fill=fillColor] (213.60, 94.38) circle (  1.16);

\path[draw=drawColor,line width= 0.4pt,line join=round,line cap=round,fill=fillColor] (230.53, 65.45) circle (  1.16);

\path[draw=drawColor,line width= 0.4pt,line join=round,line cap=round,fill=fillColor] (247.47, 60.50) circle (  1.16);

\path[draw=drawColor,line width= 0.4pt,line join=round,line cap=round,fill=fillColor] (264.40, 58.56) circle (  1.16);

\path[draw=drawColor,line width= 0.4pt,line join=round,line cap=round,fill=fillColor] (281.34, 57.54) circle (  1.16);

\path[draw=drawColor,line width= 0.4pt,line join=round,line cap=round,fill=fillColor] (298.27, 56.91) circle (  1.16);

\path[draw=drawColor,line width= 0.4pt,line join=round,line cap=round,fill=fillColor] (315.21, 56.48) circle (  1.16);

\path[draw=drawColor,line width= 0.4pt,line join=round,line cap=round,fill=fillColor] (332.14, 56.18) circle (  1.16);

\path[draw=drawColor,line width= 0.4pt,line join=round,line cap=round,fill=fillColor] (349.08, 55.95) circle (  1.16);
\definecolor{drawColor}{RGB}{190,190,190}

\path[draw=drawColor,line width= 0.6pt,dash pattern=on 4pt off 4pt ,line join=round] (206.82, 54.32) -- (355.85, 54.32);
\definecolor{drawColor}{gray}{0.20}

\path[draw=drawColor,line width= 0.6pt,line join=round,line cap=round] (206.82, 28.21) rectangle (355.85,175.17);
\end{scope}
\begin{scope}
\path[clip] (  0.00,  0.00) rectangle (361.35,180.67);
\definecolor{drawColor}{gray}{0.30}

\node[text=drawColor,anchor=base east,inner sep=0pt, outer sep=0pt, scale=  0.80] at (201.87, 32.14) {0};

\node[text=drawColor,anchor=base east,inner sep=0pt, outer sep=0pt, scale=  0.80] at (201.87, 65.54) {2};

\node[text=drawColor,anchor=base east,inner sep=0pt, outer sep=0pt, scale=  0.80] at (201.87, 98.94) {4};

\node[text=drawColor,anchor=base east,inner sep=0pt, outer sep=0pt, scale=  0.80] at (201.87,132.34) {6};

\node[text=drawColor,anchor=base east,inner sep=0pt, outer sep=0pt, scale=  0.80] at (201.87,165.74) {8};
\end{scope}
\begin{scope}
\path[clip] (  0.00,  0.00) rectangle (361.35,180.67);
\definecolor{drawColor}{gray}{0.20}

\path[draw=drawColor,line width= 0.6pt,line join=round] (204.07, 34.89) --
	(206.82, 34.89);

\path[draw=drawColor,line width= 0.6pt,line join=round] (204.07, 68.29) --
	(206.82, 68.29);

\path[draw=drawColor,line width= 0.6pt,line join=round] (204.07,101.69) --
	(206.82,101.69);

\path[draw=drawColor,line width= 0.6pt,line join=round] (204.07,135.09) --
	(206.82,135.09);

\path[draw=drawColor,line width= 0.6pt,line join=round] (204.07,168.49) --
	(206.82,168.49);
\end{scope}
\begin{scope}
\path[clip] (  0.00,  0.00) rectangle (361.35,180.67);
\definecolor{drawColor}{gray}{0.20}

\path[draw=drawColor,line width= 0.6pt,line join=round] (213.60, 25.46) --
	(213.60, 28.21);

\path[draw=drawColor,line width= 0.6pt,line join=round] (247.47, 25.46) --
	(247.47, 28.21);

\path[draw=drawColor,line width= 0.6pt,line join=round] (281.34, 25.46) --
	(281.34, 28.21);

\path[draw=drawColor,line width= 0.6pt,line join=round] (315.21, 25.46) --
	(315.21, 28.21);

\path[draw=drawColor,line width= 0.6pt,line join=round] (349.08, 25.46) --
	(349.08, 28.21);
\end{scope}
\begin{scope}
\path[clip] (  0.00,  0.00) rectangle (361.35,180.67);
\definecolor{drawColor}{gray}{0.30}

\node[text=drawColor,anchor=base,inner sep=0pt, outer sep=0pt, scale=  0.80] at (213.60, 17.75) {2};

\node[text=drawColor,anchor=base,inner sep=0pt, outer sep=0pt, scale=  0.80] at (247.47, 17.75) {4};

\node[text=drawColor,anchor=base,inner sep=0pt, outer sep=0pt, scale=  0.80] at (281.34, 17.75) {6};

\node[text=drawColor,anchor=base,inner sep=0pt, outer sep=0pt, scale=  0.80] at (315.21, 17.75) {8};

\node[text=drawColor,anchor=base,inner sep=0pt, outer sep=0pt, scale=  0.80] at (349.08, 17.75) {10};
\end{scope}
\begin{scope}
\path[clip] (  0.00,  0.00) rectangle (361.35,180.67);
\definecolor{drawColor}{RGB}{0,0,0}

\node[text=drawColor,anchor=base,inner sep=0pt, outer sep=0pt, scale=  0.90] at (281.34,  6.06) {$T$};
\end{scope}
\end{tikzpicture}

%% file: extend.tex
\section{Extensions}
\label{sec:extensions}
In this section, we showcase a number of extensions of HiGrad to incorporate some practical considerations and to improve efficiency. These extensions follow from results that have been developed in Section \ref{sec:method-1} without much additional effort, but might bring appreciable improvements in certain settings.

\paragraph{Flexible tree structures.} In its present formulation, the HiGrad tree is grown symmetrically across different branches. In fact, asymmetry is permitted, allowing for more flexibility to incorporate certain practical needs. Explicitly, after the first segment gets split into $B_1$ branches, we can build each subtree differently, with possibly different segment lengths and even various depths. Proposition \ref{prop:t} remains to hold if all the segments in the fully grown tree are asymptotically proportional to each other. 

An asymmetric HiGrad tree is favorable if it is used in a distributed environment once split. This point recognizes that, in distributed computing, datasets in their local machines are typically of different sizes and thus a symmetric HiGrad tree would inevitably incur heavy communication cost to guarantee consistency across all threads. Moreover, the number of total data points (or epochs in the finite population setting) is often unknown or not fixed a priori. An asymmetric structure admits more degrees of freedom to deal with such cases.

\paragraph{Batch size.} Mini-batch gradient descent is a trade-off between SGD and gradient descent. To update the iterate at each time, mini-batch gradient descent takes the average of the gradient over a certain number of data points so as to reduce the variance of the gradient. As a major advantage, it has been shown that mini-batch SGD outperforms the vanilla SGD in the low signal-to-noise ratio regime \citep{yin2018}. For the HiGrad algorithm, theoretical guarantees including Theorems \ref{prop:conf} and \ref{prop:pred} persist if iterations are updated in a mini-batch fashion. An interesting question for future research is, however, to determine how the mini-batch size affects the optimal HiGrad tree structures.

\paragraph{Multivariate generalizations.} The HiGrad algorithm seamlessly applies to the case where the function to estimate $\bm\mu_x$ is multivariate. In particular, our main theoretical result, Proposition \ref{prop:t}, and the subsequent Theorems \ref{prop:conf} and \ref{prop:pred} admit multivariate versions respectively, as follows. Denote by $p$ the dimension of $\bm\mu_x$ and let $\bM_x$ be a $T \times p$ matrix consisting of $T$ rows of $\bm\mu_x(\overline\theta_{\bt})$ for all threads $\bt$. As earlier in the univariate case, consider the simple multivariate linear regression
\[
\Sigma^{-\frac12} \bM_x = \Sigma^{-\frac12} \bm{1} (\bm\mu_x^\ast)^\top + \tilde{\bz},
\]
where $\bm\mu_x^\ast = \bm\mu_x(\theta^\ast)$. Note that the $T$ rows of $\tilde{\bm z}$ are approximately \iid~normal vectors. Omitting some technical details, we find that the least-squares solution is
\[
\overline{\bm\mu}_x = \frac1{T} \sum_{\bt \in \mathcal T} \bm\mu_x(\overline\theta_{\bt}),
\]
the sample covariance of $\tilde{\bz}$ is
\[
\widehat{\bm S}_x = \frac1{T-1} (\bM_x - \bm{1} \overline{\bm\mu}_x ^\top)^\top \Sigma^{-1} (\bM_x - \bm{1} \overline{\bm\mu}_x ^\top ),
\]
and Hotelling's \textit{T}-squared statistic reads
\begin{equation}\label{eq:T_hotelling}
\frac{(\bm{1}^\top \Sigma^{-\frac12} \bm{1})^2 \left( \overline{\bm\mu}_x -  \bm\mu_x^\ast \right)^\top \widehat{\bm S}^{-1}_x \left( \overline{\bm\mu}_x -  \bm\mu_x^\ast \right)}{T}.
\end{equation}

The result below generalizes Proposition \ref{prop:t} to the multivariate setting where the Jacobian $\frac{\partial \bm{\mu}_x(\theta)}{\partial \theta}$ exists and has full rank in a neighborhood of $\theta^\ast$. 

\begin{proposition}\label{prop:t2}
Under the assumptions of Proposition \ref{prop:t}, the HiGrad procedure ensures that the statistic \eqref{eq:T_hotelling} asymptotically follows Hotelling's $T$-squared distribution $T^2_{p, T-1}$ as $N \goto \infty$.
\end{proposition}

Above, note that $T^2_{p, T-1}$ is the same as the rescaled $F$ random variable $\frac{p(T-1)}{T - p} F_{p, T - p}$. Multivariate analogs of Theorems \ref{prop:conf} and \ref{prop:pred} for attaching confidence statements to the HiGrad estimator $\overline{\bm\mu}_x$ immediately follow from Proposition \ref{prop:t2}. For instance, an asymptotically $1 - \alpha$-coverage confidence region for $\bm\mu_x^\ast$ is
\[
\left\{ \bm\mu: \frac{(\bm{1}^\top \Sigma^{-\frac12} \bm{1})^2 \left( \overline{\bm\mu}_x - \bm\mu \right)^\top \widehat{\bm S}^{-1}_x \left( \overline{\bm\mu}_x - \bm\mu \right)}{T} \le T^2_{p, T-1, 1-\alpha}\right\}.
\]
Likewise, a prediction region for $\overline{\bm\mu}_x'$ obtained from a fresh dataset takes the same form except that $2T^2_{p, T-1, 1-\alpha}$ is used in place of $T^2_{p, T-1, 1-\alpha}$ above.

\paragraph{Burn-in and restarting.} Discarding a small portion of the iterates at the beginning, a trick referred to as burn-in, is widely adopted in practice (see, for example, \cite{chen2016,li2017statistical,chee2017convergence}). The rationale for using burn-in is that initial iterates can be far from the minimizer and thus it might improve the accuracy by returning the average of the last iterates. More concretely, burn-in is shown to improve the rate of convergence for non-smooth objectives \citep{rakhlin2012making}. As a generalization of SGD, HiGrad can easily incorporate this trick by, for example, discarding initial iterates in the first segment. In addition, the HiGrad algorithm seamlessly employs a similar trick called restarting in first-order methods \citep{o2015adaptive,su2016differential}, which resets the step size back to $\gamma_1$ after a certain
number of iterations. For example, restarting, in the HiGrad setting, can be applied at the beginning of each segment.


%% file: simulation.tex
\usetikzlibrary{shapes.geometric,calc}

\newcommand\score[2]{
\pgfmathsetmacro\pgfxa{#1+1}
\tikzstyle{scorestars}=[star, star points=5, star point ratio=2.25, draw,inner sep=0.1em,anchor=outer point 3]
\begin{tikzpicture}[baseline]
  \foreach \i in {1,...,#2} {
    \pgfmathparse{(\i<=#1?"black":"white")}
    \edef\starcolor{\pgfmathresult}
    \draw (\i*.8em,0) node[name=star\i,scorestars,fill=\starcolor]  {};
   }
   \pgfmathparse{(#1>int(#1)?int(#1+1):0}
   \let\partstar=\pgfmathresult
   \ifnum\partstar>0
     \pgfmathsetmacro\starpart{#1-(int(#1))}
     \path [clip] ($(star\partstar.outer point 3)!(star\partstar.outer point 2)!(star\partstar.outer point 4)$) rectangle 
    ($(star\partstar.outer point 2 |- star\partstar.outer point 1)!\starpart!(star\partstar.outer point 1 -| star\partstar.outer point 5)$);
     \fill (\partstar*1em,0) node[scorestars,fill=black]  {};
   \fi

,\end{tikzpicture}
}

\section{Numerical Examples}
\label{sec:simulations}

\subsection{Simulations}
\label{sec:simulations-1}

The empirical performance of HiGrad is evaluated in simulations from three perspectives: accuracy, coverage, and informativeness. Explicitly, accuracy is measured by the distance between the estimator averaged from all HiGrad threads and the true parameter, coverage is measured by the probability that the HiGrad confidence interval contains the true value, and informativeness is measured by the average length of the confidence interval.

Using the optimal weights, HiGrad is applied to linear regression and logistic regression. The former generates $Y$ from $\N(\mu_X(\theta^\ast), 1)$, whereas the latter generates $Y = 1$ with probability
\[
\frac{\e^{\mu_X(\theta^\ast)}}{\e^{\mu_X(\theta^\ast)} + 1}
\]
and $Y = 0$ otherwise, both conditional on the feature vector $X$. The quantity to be estimated in both cases takes the form $\mu_x(\theta) = x^\top \theta$ and $X$ follows a multivariate normal distribution $\N(0, I_d)$, where the dimension $d = 50$. The function $f(\theta, z)$ is taken to be the negative log-likelihood (see Section \ref{sec:problem-statement}). Upon a query from the HiGrad algorithm, a pair of $(X, Y)$ is generated according to the models described above. The step size $\gamma_j$ is set to $0.1 j^{-0.55}$ and $0.4 j^{-0.55}$ for linear regression and logistic regression, respectively, and $\theta_0$ is initialized randomly with a $N(0, 0.01I)$ distribution. Three types of the true coefficients $\theta^\ast$ are examined: a null case where $\theta_1^\ast = \cdots = \theta_d^\ast = 0$, a dense case where $\theta_1^\ast = \cdots = \theta_d^\ast = \frac{1}{\sqrt{d}}$, and a sparse case where $\theta_1^\ast = \cdots = \theta_{d/10}^\ast = \sqrt{10/d}$, $\theta_{d/10+1}^\ast = \cdots = \theta_d^\ast = 0$. Table \ref{table:configs} presents the HiGrad configurations considered in the simulation studies. Note that all of the four HiGrad configurations have $T = 4$ threads.

\begin{table}[H]
\centering
\begin{tabular}{>{}m{0.7in} | >{\arraybackslash}m{2.8in} }
\hline\hline
\includegraphics[height = 2em]{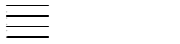} & $K = 1, B_1 = 4, n_0 = 0, n_1 = N/4$ \\ 
\includegraphics[height = 2em]{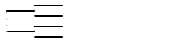} & $K = 2, B_1 = B_2 = 2, n_0 = 0, n_1 = N/6$ \\
\includegraphics[height = 2em]{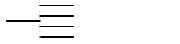} & $K = 1, B_1 = 4, n_0 = n_1 = N/5$ \\ 
\includegraphics[height = 2em]{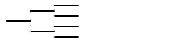} & $K = 2, B_1 = B_2 = 2, n_0 = n_1 = N/7$\\
\hline\hline
\end{tabular}
\caption{Configurations of HiGrad in the simulations.}
\label{table:configs}
\end{table}

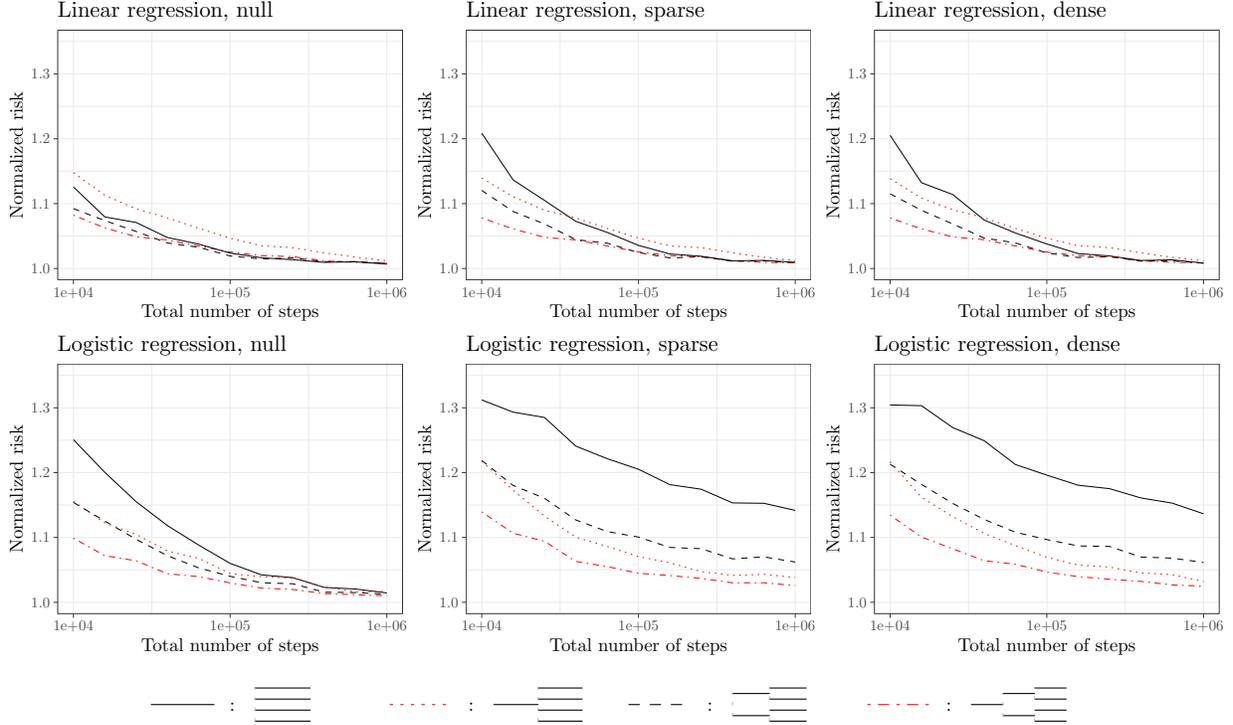
\begin{figure}[t]
\begin{center}
\resizebox{\textwidth}{!}{
\input{figs/accuracy_ggplot.tex}
}

\medskip
\input{figs/accuracy_legend.tex}
\caption{Estimation accuracy of HiGrad against the total number of iteration steps. The risk is averaged over 100 replicates and is further normalized by that of vanilla SGD. The four HiGrad configurations are described in Table~\ref{table:configs}.
}
\label{fig:estimation_accuracy}
\end{center}
\end{figure}

\paragraph{Accuracy.} Denote by $\overline\theta$ the average of all HiGrad thread estimates \eqref{eq:theta_ave} and record $\|\overline\theta - \theta^\ast\|$ as the estimation risk. The reported risks are averaged over 100 replicates, each with a total number $N$ of iterations varying from $10^4$ to $10^6$. Shown in Figure \ref{fig:estimation_accuracy} are plots of the HiGrad risks normalized by those of SGD in the same setting as a function of the number of steps. These plots demonstrate that, in general, a configuration with longer thread tends to yield smaller risk. In particular, the fourth configuration (dash-dotted red line) is with the longest thread length $3N/7$, indeed having the lowest risk in all six plots. On the contrary, the shortest thread length $N/4$ is from the first configuration (solid black
line), which yields the highest risk in most cases. As an aside, we point out that the case of a null $\theta^\ast$ appears to have the most accurate HiGrad results. This is not surprising as the algorithm is initialized near the origin.

\begin{figure}[t]
\begin{center}
\begin{tabular}{cc}
\small{Linear regression} & \small{Logistic regression} \\ 
\\
\input{figs/fig_simulation_lm.tex} & \input{figs/fig_simulation_logistic.tex}\\
\end{tabular}
\input{figs/legend_simulation.tex}
\end{center}
\caption{Coverage probability and length of the HiGrad confidence intervals. In both panels, the middle column presents the configuration graphically; the left column shows the coverage probabilities 
(with the nominal coverage 90\% indicated by a vertical gray line); the right column illustrates the average lengths of the confidence intervals. The color of the bar indicates the type of true parameters $\theta^\ast$, as shown in the legend at the bottom. 
}\label{fig:simulation}
\end{figure}
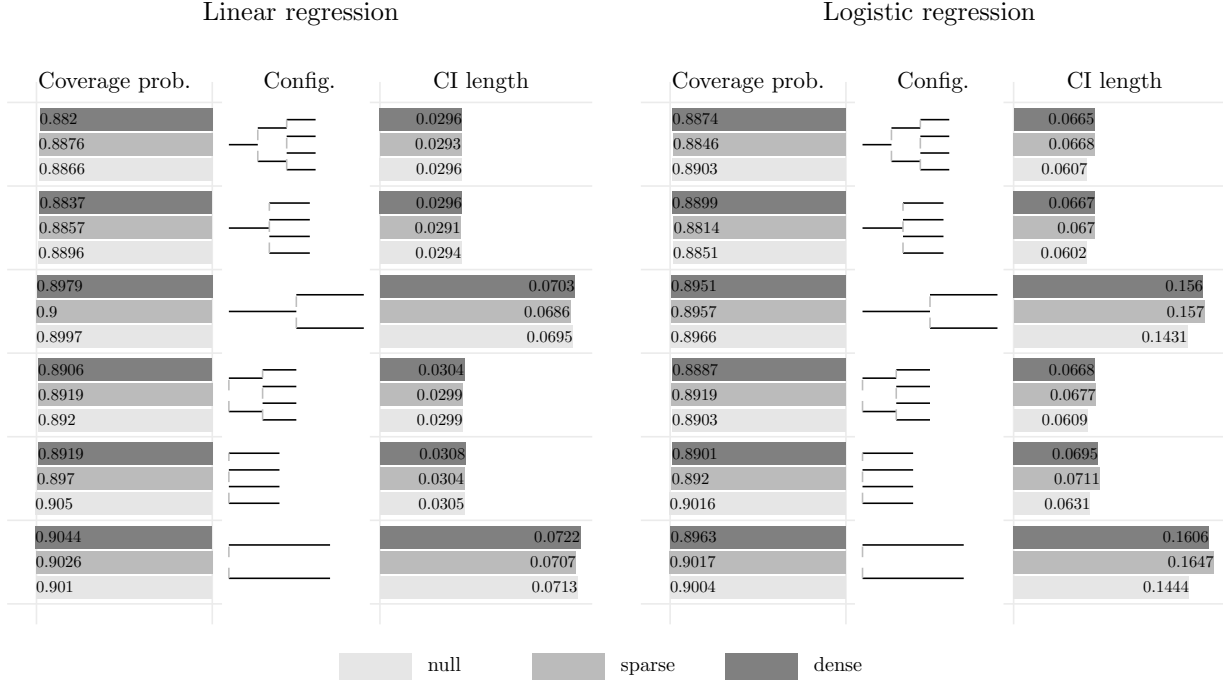

\paragraph{Coverage and informativeness.} In addition to the four configurations listed in Table~\ref{table:configs}, this exploration includes two more configurations, both with $K=1$ and $B_1 = 2$. The first is set to $n_0 = 0, n_1 = N/2$, and the second is set $n_0 = n_1 = N/3$, where $N = 10^6$. Given a configuration, the HiGrad procedure is performed for $L = 100$ times, each yielding a 90\% confidence interval $\mathrm{CI}_{\ell}$ for $\mu_X(\theta^\ast)$ with $X$ being sampled from $\N(0, I_d)$. Figure \ref{fig:simulation} shows a concise summary of the results in the form of bar plots, which average the empirical coverage probabilities
\[
\frac{1}{L}\sum_{\ell = 1}^{L} \mathbbm1(\mu_X(\theta^\ast)\in \mathrm{CI}_{\ell})
\]
and the average confidence interval length both over 100 independent copies of $X$. 
Note that for logistic regression, the confidence interval length is on the scale of $\mu_x$, the logit value, instead of the probability.
For all configurations, models and true parameter types, the coverage probabilities are close to the nominal level 90\%. In particular, the HiGrad configuration with two directly split threads (at the bottom of the plots) attains the coverage probability that is closest to 90\%. However, its confidence interval is the longest among all the six configurations (comparable to the configuration with two intermediately split threads). The configurations with $T = 4$ threads give similar levels of informativeness, which is consistent with the decreasing monotonicity of \eqref{eq:decrease}.

\paragraph{Sensitivity analysis of tuning parameters.} In addition to the HiGrad tree configurations, the practitioners need to determine the step sizes for running the algorithm. Following our theoretical results and aforementioned experimental settings, our focus is on step sizes taking the form
\[
\gamma_j = \frac{c_1}{(j + c_2)^{\alpha}},
\]
for constants $0.5 < \alpha < 1, c_1 > 0$, and $c_2$. Using the same setting as Figure~\ref{fig:simulation}, Figure~\ref{fig:stepsizes} shows that the performance of HiGrad is satisfactory and consistent over a range of the triplet of tuning parameters $(\alpha,c_1,c_2)$.

\begin{figure}[t]
\begin{center}
\begin{tabular}{cc}
\small{Linear regression} & \small{Logistic regression} \\ 
\\
\input{figs/fig_simulation_lm_add.tex} & \input{figs/fig_simulation_logistic_add.tex}\\
\end{tabular}
\input{figs/legend_simulation_add.tex}
\end{center}
\caption{Coverage probability and length of the HiGrad confidence intervals with respect to different choices of step sizes. In both panels, the middle column presents the configuration graphically; the left column shows the coverage probabilities 
(with the nominal coverage 90\% indicated by a vertical gray line); the right column illustrates the average lengths of the confidence intervals. The color of the bar indicates the triplet $(\alpha,c_1,c_2)$ that specifies the step sizes $\gamma_j = \frac{c_1}{(j + c_2)^{\alpha}}$ in Proposition~\ref{prop:t}. Here we show in particular the results for the dense parameter type.}
\label{fig:stepsizes}
\end{figure}


\paragraph{Summary.} To summarize the phenomena observed from the simulated studies, Table \ref{table:ratings} assigns each HiGrad configuration three qualitative ratings. For comparison, the vanilla SGD is included as the simplest example of HiGrad. The last configuration (two splits and $T=4$ threads) achieves the best overall performance according to the three criteria. As a caveat, the summary table is informal and should be confined to the present simulation context.

\begin{table}[h]
\begin{center}
\small
\begin{tabular}{>{\centering\arraybackslash}m{1.3in} | >{\centering\arraybackslash}m{0.9in} | >{\centering\arraybackslash}m{0.9in} | >{\centering\arraybackslash}m{0.9in} }
\hline\hline
Config. & Accuracy & Coverage & Informativeness \\
\hline
\includegraphics[height = 1.5em]{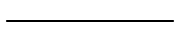} & \score{5}{5} & \score{0}{5} & \score{0}{5}\\
\includegraphics[height = 1.5em]{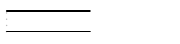} & \score{4.5}{5} & \score{5}{5} & \score{2}{5}\\
\includegraphics[height = 1.5em]{figs/legends/legend_0_4.pdf} & \score{3}{5} & \score{4.5}{5} & \score{4}{5}\\
\includegraphics[height = 1.5em]{figs/legends/legend_0_2_2.pdf} & \score{3.5}{5} & \score{4}{5} & \score{4}{5}\\
\includegraphics[height = 1.5em]{figs/legends/legend_1_4.pdf} & \score{3.5}{5} & \score{4}{5} & \score{4}{5}\\
\includegraphics[height = 1.5em]{figs/legends/legend_1_2_2.pdf} & \score{4}{5} & \score{4}{5} & \score{4}{5}\\
\hline\hline
\end{tabular}
\end{center}
\caption{Ratings of different HiGrad configurations.}
\label{table:ratings}
\end{table}




\subsection{A real data example}
\label{sec:real-data-example}

This section reports the results of applying HiGrad to the Adult dataset on the UCI repository \citep{lichman2013machine}, which is discussed in Introduction. The original dataset contains 14 features, of which 6 are continuous and 8 are categorical. 
We use the preprocessed version hosted on the LibSVM repository \citep{chang2011libsvm}, which has 123 binary features and contains 32,561 samples. We randomly pick 1,000 as a test set, and the rest as a training set. With the default configuration ($K = 2$, $B_1 = B_2 = 2$, $n_0 = n_1 = n_2 = N/7$), HiGrad is used to fit logistic regression on the training set with $N = 10^6$ iterations. The step size is taken to be $\gamma_j = 0.5j^{-0.505}$
and the initial points are chosen as earlier in Section \ref{sec:simulations-1}.

In this real-world example, the coverage of confidence intervals cannot be evaluated because the true probabilities of the test samples are unknown. Instead, we consider the HiGrad prediction interval as a way of measuring the randomness of the algorithm. Explicitly, HiGrad is repeated for $L = 500$ times in the setting specified above. In the $\ell$th run, HiGrad obtains the predicted probability $\widehat p_{i\ell}$ and the 90\% prediction interval $\mathrm{PI}_{i\ell}$ for the $i$th unit in the test set\footnote{HiGrad first constructs estimates and intervals for the logit $x^\top \theta$ and then transform them to probabilities using $\exp(x^\top \theta)/(\exp(x^\top \theta) + 1)$.}, where $i = 1, \ldots, 1000$. We consider the  empirical coverage probability for the $i$th unit given as
\begin{equation}\label{eq:hi_pi}
\frac{1}{L(L-1)} \sum_{\ell_1 \neq \ell_2}\mathbbm{1}\left(\widehat p_{i\ell_1} \in \mathrm{PI}_{i\ell_2}\right),
\end{equation}
where the summation is over all $L(L-1)$ pairs of $(\ell_1, \ell_2)$ such that $1 \le \ell_1 \ne \ell_2 \le L$.

In addition, we investigate the coverage property of the HiGrad prediction intervals for SGD estimates. To this end, SGD with $N = 10^6$ iterations is repeatedly performed for $L = 100$ times, each yielding an ``oracle sample'' predicted probability $\widehat p_{i \ell}'$ for the $i$th test unit. The empirical coverage probability for the $i$th unit is
\begin{equation}\label{eq:sgd_pi}
\frac{1}{L^2}\sum_{\ell_1 = 1}^{L}\sum_{\ell_2 = 1}^L \mathbbm1\left(\widehat p_{i\ell_1}' \in \mathrm{PI}_{i\ell_2}\right).
\end{equation}

Figure \ref{fig:adult_coverage} plots the histograms of the empirical coverage probabilities \eqref{eq:hi_pi} and \eqref{eq:sgd_pi} for the 1,000 test sample units, respectively. Note that both the left and right histograms use the same prediction intervals to cover different estimates. In both histograms, the coverage probability is highly concentrated around the nominal level 90\%, showing that the HiGrad prediction intervals achieve a reasonable coverage probability for most units in the test set. A noticeable left tail is observed in both histograms, however, indicating that more epochs in HiGrad and SGD are needed to invoke the asymptotic results for such a small fraction of units.

\begin{figure}[htp]
\begin{center}
\input{figs/adult_coverage.tex}
\caption{Histogram of average coverage probabilities for the 100 samples in the test set. The left plot corresponds to a fresh HiGrad prediction \eqref{eq:hi_pi}, and the right corresponds to an SGD prediction \eqref{eq:sgd_pi}. The average coverage probabilities are calculated based on an set of ``oracle samples'' and 100 runs of HiGrad.}
\label{fig:adult_coverage}
\end{center}
\end{figure}
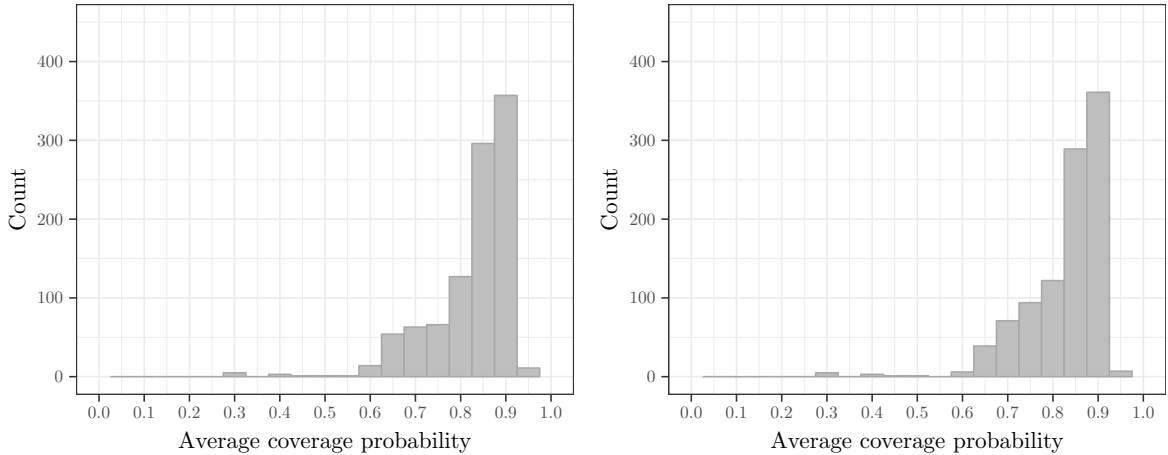


%% file: figs/accuracy_ggplot.tex
\begin{tikzpicture}[x=1pt,y=1pt]
\definecolor{fillColor}{RGB}{255,255,255}
\path[use as bounding box,fill=fillColor,fill opacity=0.00] (0,0) rectangle (794.97,433.62);
\begin{scope}
\path[clip] (  0.00,216.81) rectangle (264.99,433.62);
\definecolor{drawColor}{RGB}{255,255,255}
\definecolor{fillColor}{RGB}{255,255,255}

\path[draw=drawColor,line width= 0.6pt,line join=round,line cap=round,fill=fillColor] (  0.00,216.81) rectangle (264.99,433.62);
\end{scope}
\begin{scope}
\path[clip] ( 35.74,248.34) rectangle (259.49,410.48);
\definecolor{fillColor}{RGB}{255,255,255}

\path[fill=fillColor] ( 35.74,248.34) rectangle (259.49,410.48);
\definecolor{drawColor}{gray}{0.92}

\path[draw=drawColor,line width= 0.3pt,line join=round] ( 35.74,276.77) --
	(259.49,276.77);

\path[draw=drawColor,line width= 0.3pt,line join=round] ( 35.74,318.88) --
	(259.49,318.88);

\path[draw=drawColor,line width= 0.3pt,line join=round] ( 35.74,361.00) --
	(259.49,361.00);

\path[draw=drawColor,line width= 0.3pt,line join=round] ( 35.74,403.11) --
	(259.49,403.11);

\path[draw=drawColor,line width= 0.3pt,line join=round] ( 96.76,248.34) --
	( 96.76,410.48);

\path[draw=drawColor,line width= 0.3pt,line join=round] (198.47,248.34) --
	(198.47,410.48);

\path[draw=drawColor,line width= 0.6pt,line join=round] ( 35.74,255.71) --
	(259.49,255.71);

\path[draw=drawColor,line width= 0.6pt,line join=round] ( 35.74,297.83) --
	(259.49,297.83);

\path[draw=drawColor,line width= 0.6pt,line join=round] ( 35.74,339.94) --
	(259.49,339.94);

\path[draw=drawColor,line width= 0.6pt,line join=round] ( 35.74,382.06) --
	(259.49,382.06);

\path[draw=drawColor,line width= 0.6pt,line join=round] ( 45.91,248.34) --
	( 45.91,410.48);

\path[draw=drawColor,line width= 0.6pt,line join=round] (147.61,248.34) --
	(147.61,410.48);

\path[draw=drawColor,line width= 0.6pt,line join=round] (249.32,248.34) --
	(249.32,410.48);
\definecolor{drawColor}{RGB}{0,0,0}

\path[draw=drawColor,line width= 0.6pt,dash pattern=on 4pt off 4pt ,line join=round] ( 45.91,294.59) --
	( 66.25,286.82) --
	( 86.59,279.74) --
	(106.93,272.28) --
	(127.27,269.46) --
	(147.61,263.84) --
	(167.96,261.86) --
	(188.30,262.56) --
	(208.64,260.17) --
	(228.98,260.18) --
	(249.32,259.02);
\definecolor{drawColor}{RGB}{255,0,0}

\path[draw=drawColor,line width= 0.6pt,dash pattern=on 1pt off 3pt on 4pt off 3pt ,line join=round] ( 45.91,290.33) --
	( 66.25,282.07) --
	( 86.59,276.33) --
	(106.93,274.19) --
	(127.27,270.28) --
	(147.61,266.40) --
	(167.96,264.05) --
	(188.30,263.39) --
	(208.64,260.67) --
	(228.98,259.69) --
	(249.32,259.31);
\definecolor{drawColor}{RGB}{0,0,0}

\path[draw=drawColor,line width= 0.6pt,line join=round] ( 45.91,308.67) --
	( 66.25,289.20) --
	( 86.59,285.59) --
	(106.93,275.82) --
	(127.27,271.52) --
	(147.61,265.66) --
	(167.96,262.67) --
	(188.30,261.39) --
	(208.64,259.68) --
	(228.98,260.02) --
	(249.32,258.56);
\definecolor{drawColor}{RGB}{255,0,0}

\path[draw=drawColor,line width= 0.6pt,dash pattern=on 1pt off 3pt ,line join=round] ( 45.91,317.95) --
	( 66.25,303.27) --
	( 86.59,294.45) --
	(106.93,288.60) --
	(127.27,281.85) --
	(147.61,275.37) --
	(167.96,270.45) --
	(188.30,269.21) --
	(208.64,265.90) --
	(228.98,263.04) --
	(249.32,260.71);
\definecolor{drawColor}{gray}{0.20}

\path[draw=drawColor,line width= 0.6pt,line join=round,line cap=round] ( 35.74,248.34) rectangle (259.49,410.48);
\end{scope}
\begin{scope}
\path[clip] (  0.00,  0.00) rectangle (794.97,433.62);
\definecolor{drawColor}{gray}{0.30}

\node[text=drawColor,anchor=base east,inner sep=0pt, outer sep=0pt, scale=  0.88] at ( 30.79,252.68) {1.0};

\node[text=drawColor,anchor=base east,inner sep=0pt, outer sep=0pt, scale=  0.88] at ( 30.79,294.80) {1.1};

\node[text=drawColor,anchor=base east,inner sep=0pt, outer sep=0pt, scale=  0.88] at ( 30.79,336.91) {1.2};

\node[text=drawColor,anchor=base east,inner sep=0pt, outer sep=0pt, scale=  0.88] at ( 30.79,379.03) {1.3};
\end{scope}
\begin{scope}
\path[clip] (  0.00,  0.00) rectangle (794.97,433.62);
\definecolor{drawColor}{gray}{0.20}

\path[draw=drawColor,line width= 0.6pt,line join=round] ( 32.99,255.71) --
	( 35.74,255.71);

\path[draw=drawColor,line width= 0.6pt,line join=round] ( 32.99,297.83) --
	( 35.74,297.83);

\path[draw=drawColor,line width= 0.6pt,line join=round] ( 32.99,339.94) --
	( 35.74,339.94);

\path[draw=drawColor,line width= 0.6pt,line join=round] ( 32.99,382.06) --
	( 35.74,382.06);
\end{scope}
\begin{scope}
\path[clip] (  0.00,  0.00) rectangle (794.97,433.62);
\definecolor{drawColor}{gray}{0.20}

\path[draw=drawColor,line width= 0.6pt,line join=round] ( 45.91,245.59) --
	( 45.91,248.34);

\path[draw=drawColor,line width= 0.6pt,line join=round] (147.61,245.59) --
	(147.61,248.34);

\path[draw=drawColor,line width= 0.6pt,line join=round] (249.32,245.59) --
	(249.32,248.34);
\end{scope}
\begin{scope}
\path[clip] (  0.00,  0.00) rectangle (794.97,433.62);
\definecolor{drawColor}{gray}{0.30}

\node[text=drawColor,anchor=base,inner sep=0pt, outer sep=0pt, scale=  0.88] at ( 45.91,237.33) {1e+04};

\node[text=drawColor,anchor=base,inner sep=0pt, outer sep=0pt, scale=  0.88] at (147.61,237.33) {1e+05};

\node[text=drawColor,anchor=base,inner sep=0pt, outer sep=0pt, scale=  0.88] at (249.32,237.33) {1e+06};
\end{scope}
\begin{scope}
\path[clip] (  0.00,  0.00) rectangle (794.97,433.62);
\definecolor{drawColor}{RGB}{0,0,0}

\node[text=drawColor,anchor=base,inner sep=0pt, outer sep=0pt, scale=  1.10] at (147.61,224.25) {Total number of steps};
\end{scope}
\begin{scope}
\path[clip] (  0.00,  0.00) rectangle (794.97,433.62);
\definecolor{drawColor}{RGB}{0,0,0}

\node[text=drawColor,rotate= 90.00,anchor=base,inner sep=0pt, outer sep=0pt, scale=  1.10] at ( 13.08,329.41) {Normalized risk};
\end{scope}
\begin{scope}
\path[clip] (  0.00,  0.00) rectangle (794.97,433.62);
\definecolor{drawColor}{RGB}{0,0,0}

\node[text=drawColor,anchor=base west,inner sep=0pt, outer sep=0pt, scale=  1.32] at ( 35.74,419.03) {Linear regression, null};
\end{scope}
\begin{scope}
\path[clip] (  0.00,  0.00) rectangle (264.99,216.81);
\definecolor{drawColor}{RGB}{255,255,255}
\definecolor{fillColor}{RGB}{255,255,255}

\path[draw=drawColor,line width= 0.6pt,line join=round,line cap=round,fill=fillColor] (  0.00,  0.00) rectangle (264.99,216.81);
\end{scope}
\begin{scope}
\path[clip] ( 35.74, 31.53) rectangle (259.49,193.67);
\definecolor{fillColor}{RGB}{255,255,255}

\path[fill=fillColor] ( 35.74, 31.53) rectangle (259.49,193.67);
\definecolor{drawColor}{gray}{0.92}

\path[draw=drawColor,line width= 0.3pt,line join=round] ( 35.74, 59.96) --
	(259.49, 59.96);

\path[draw=drawColor,line width= 0.3pt,line join=round] ( 35.74,102.07) --
	(259.49,102.07);

\path[draw=drawColor,line width= 0.3pt,line join=round] ( 35.74,144.19) --
	(259.49,144.19);

\path[draw=drawColor,line width= 0.3pt,line join=round] ( 35.74,186.30) --
	(259.49,186.30);

\path[draw=drawColor,line width= 0.3pt,line join=round] ( 96.76, 31.53) --
	( 96.76,193.67);

\path[draw=drawColor,line width= 0.3pt,line join=round] (198.47, 31.53) --
	(198.47,193.67);

\path[draw=drawColor,line width= 0.6pt,line join=round] ( 35.74, 38.90) --
	(259.49, 38.90);

\path[draw=drawColor,line width= 0.6pt,line join=round] ( 35.74, 81.02) --
	(259.49, 81.02);

\path[draw=drawColor,line width= 0.6pt,line join=round] ( 35.74,123.13) --
	(259.49,123.13);

\path[draw=drawColor,line width= 0.6pt,line join=round] ( 35.74,165.25) --
	(259.49,165.25);

\path[draw=drawColor,line width= 0.6pt,line join=round] ( 45.91, 31.53) --
	( 45.91,193.67);

\path[draw=drawColor,line width= 0.6pt,line join=round] (147.61, 31.53) --
	(147.61,193.67);

\path[draw=drawColor,line width= 0.6pt,line join=round] (249.32, 31.53) --
	(249.32,193.67);
\definecolor{drawColor}{RGB}{0,0,0}

\path[draw=drawColor,line width= 0.6pt,dash pattern=on 4pt off 4pt ,line join=round] ( 45.91,103.78) --
	( 66.25, 91.68) --
	( 86.59, 79.83) --
	(106.93, 69.18) --
	(127.27, 61.25) --
	(147.61, 55.80) --
	(167.96, 51.61) --
	(188.30, 50.96) --
	(208.64, 45.47) --
	(228.98, 45.19) --
	(249.32, 44.13);
\definecolor{drawColor}{RGB}{255,0,0}

\path[draw=drawColor,line width= 0.6pt,dash pattern=on 1pt off 3pt on 4pt off 3pt ,line join=round] ( 45.91, 80.55) --
	( 66.25, 69.22) --
	( 86.59, 65.88) --
	(106.93, 57.49) --
	(127.27, 55.58) --
	(147.61, 51.49) --
	(167.96, 48.14) --
	(188.30, 47.25) --
	(208.64, 44.58) --
	(228.98, 43.78) --
	(249.32, 43.17);
\definecolor{drawColor}{RGB}{0,0,0}

\path[draw=drawColor,line width= 0.6pt,line join=round] ( 45.91,144.55) --
	( 66.25,123.34) --
	( 86.59,104.25) --
	(106.93, 88.72) --
	(127.27, 76.05) --
	(147.61, 64.20) --
	(167.96, 56.73) --
	(188.30, 54.92) --
	(208.64, 48.59) --
	(228.98, 47.54) --
	(249.32, 44.98);
\definecolor{drawColor}{RGB}{255,0,0}

\path[draw=drawColor,line width= 0.6pt,dash pattern=on 1pt off 3pt ,line join=round] ( 45.91,104.45) --
	( 66.25, 90.54) --
	( 86.59, 83.04) --
	(106.93, 72.22) --
	(127.27, 67.45) --
	(147.61, 57.58) --
	(167.96, 55.52) --
	(188.30, 54.63) --
	(208.64, 48.76) --
	(228.98, 45.69) --
	(249.32, 45.99);
\definecolor{drawColor}{gray}{0.20}

\path[draw=drawColor,line width= 0.6pt,line join=round,line cap=round] ( 35.74, 31.53) rectangle (259.49,193.67);
\end{scope}
\begin{scope}
\path[clip] (  0.00,  0.00) rectangle (794.97,433.62);
\definecolor{drawColor}{gray}{0.30}

\node[text=drawColor,anchor=base east,inner sep=0pt, outer sep=0pt, scale=  0.88] at ( 30.79, 35.87) {1.0};

\node[text=drawColor,anchor=base east,inner sep=0pt, outer sep=0pt, scale=  0.88] at ( 30.79, 77.99) {1.1};

\node[text=drawColor,anchor=base east,inner sep=0pt, outer sep=0pt, scale=  0.88] at ( 30.79,120.10) {1.2};

\node[text=drawColor,anchor=base east,inner sep=0pt, outer sep=0pt, scale=  0.88] at ( 30.79,162.22) {1.3};
\end{scope}
\begin{scope}
\path[clip] (  0.00,  0.00) rectangle (794.97,433.62);
\definecolor{drawColor}{gray}{0.20}

\path[draw=drawColor,line width= 0.6pt,line join=round] ( 32.99, 38.90) --
	( 35.74, 38.90);

\path[draw=drawColor,line width= 0.6pt,line join=round] ( 32.99, 81.02) --
	( 35.74, 81.02);

\path[draw=drawColor,line width= 0.6pt,line join=round] ( 32.99,123.13) --
	( 35.74,123.13);

\path[draw=drawColor,line width= 0.6pt,line join=round] ( 32.99,165.25) --
	( 35.74,165.25);
\end{scope}
\begin{scope}
\path[clip] (  0.00,  0.00) rectangle (794.97,433.62);
\definecolor{drawColor}{gray}{0.20}

\path[draw=drawColor,line width= 0.6pt,line join=round] ( 45.91, 28.78) --
	( 45.91, 31.53);

\path[draw=drawColor,line width= 0.6pt,line join=round] (147.61, 28.78) --
	(147.61, 31.53);

\path[draw=drawColor,line width= 0.6pt,line join=round] (249.32, 28.78) --
	(249.32, 31.53);
\end{scope}
\begin{scope}
\path[clip] (  0.00,  0.00) rectangle (794.97,433.62);
\definecolor{drawColor}{gray}{0.30}

\node[text=drawColor,anchor=base,inner sep=0pt, outer sep=0pt, scale=  0.88] at ( 45.91, 20.52) {1e+04};

\node[text=drawColor,anchor=base,inner sep=0pt, outer sep=0pt, scale=  0.88] at (147.61, 20.52) {1e+05};

\node[text=drawColor,anchor=base,inner sep=0pt, outer sep=0pt, scale=  0.88] at (249.32, 20.52) {1e+06};
\end{scope}
\begin{scope}
\path[clip] (  0.00,  0.00) rectangle (794.97,433.62);
\definecolor{drawColor}{RGB}{0,0,0}

\node[text=drawColor,anchor=base,inner sep=0pt, outer sep=0pt, scale=  1.10] at (147.61,  7.44) {Total number of steps};
\end{scope}
\begin{scope}
\path[clip] (  0.00,  0.00) rectangle (794.97,433.62);
\definecolor{drawColor}{RGB}{0,0,0}

\node[text=drawColor,rotate= 90.00,anchor=base,inner sep=0pt, outer sep=0pt, scale=  1.10] at ( 13.08,112.60) {Normalized risk};
\end{scope}
\begin{scope}
\path[clip] (  0.00,  0.00) rectangle (794.97,433.62);
\definecolor{drawColor}{RGB}{0,0,0}

\node[text=drawColor,anchor=base west,inner sep=0pt, outer sep=0pt, scale=  1.32] at ( 35.74,202.22) {Logistic regression, null};
\end{scope}
\begin{scope}
\path[clip] (264.99,216.81) rectangle (529.98,433.62);
\definecolor{drawColor}{RGB}{255,255,255}
\definecolor{fillColor}{RGB}{255,255,255}

\path[draw=drawColor,line width= 0.6pt,line join=round,line cap=round,fill=fillColor] (264.99,216.81) rectangle (529.98,433.62);
\end{scope}
\begin{scope}
\path[clip] (300.73,248.34) rectangle (524.48,410.48);
\definecolor{fillColor}{RGB}{255,255,255}

\path[fill=fillColor] (300.73,248.34) rectangle (524.48,410.48);
\definecolor{drawColor}{gray}{0.92}

\path[draw=drawColor,line width= 0.3pt,line join=round] (300.73,276.77) --
	(524.48,276.77);

\path[draw=drawColor,line width= 0.3pt,line join=round] (300.73,318.88) --
	(524.48,318.88);

\path[draw=drawColor,line width= 0.3pt,line join=round] (300.73,361.00) --
	(524.48,361.00);

\path[draw=drawColor,line width= 0.3pt,line join=round] (300.73,403.11) --
	(524.48,403.11);

\path[draw=drawColor,line width= 0.3pt,line join=round] (361.75,248.34) --
	(361.75,410.48);

\path[draw=drawColor,line width= 0.3pt,line join=round] (463.46,248.34) --
	(463.46,410.48);

\path[draw=drawColor,line width= 0.6pt,line join=round] (300.73,255.71) --
	(524.48,255.71);

\path[draw=drawColor,line width= 0.6pt,line join=round] (300.73,297.83) --
	(524.48,297.83);

\path[draw=drawColor,line width= 0.6pt,line join=round] (300.73,339.94) --
	(524.48,339.94);

\path[draw=drawColor,line width= 0.6pt,line join=round] (300.73,382.06) --
	(524.48,382.06);

\path[draw=drawColor,line width= 0.6pt,line join=round] (310.90,248.34) --
	(310.90,410.48);

\path[draw=drawColor,line width= 0.6pt,line join=round] (412.60,248.34) --
	(412.60,410.48);

\path[draw=drawColor,line width= 0.6pt,line join=round] (514.31,248.34) --
	(514.31,410.48);
\definecolor{drawColor}{RGB}{0,0,0}

\path[draw=drawColor,line width= 0.6pt,dash pattern=on 4pt off 4pt ,line join=round] (310.90,306.41) --
	(331.24,292.75) --
	(351.58,284.64) --
	(371.92,274.44) --
	(392.26,272.15) --
	(412.60,266.05) --
	(432.95,262.56) --
	(453.29,263.28) --
	(473.63,260.44) --
	(493.97,260.72) --
	(514.31,259.27);
\definecolor{drawColor}{RGB}{255,0,0}

\path[draw=drawColor,line width= 0.6pt,dash pattern=on 1pt off 3pt on 4pt off 3pt ,line join=round] (310.90,288.54) --
	(331.24,281.41) --
	(351.58,275.93) --
	(371.92,274.26) --
	(392.26,270.35) --
	(412.60,266.45) --
	(432.95,264.03) --
	(453.29,263.38) --
	(473.63,260.68) --
	(493.97,259.65) --
	(514.31,259.33);
\definecolor{drawColor}{RGB}{0,0,0}

\path[draw=drawColor,line width= 0.6pt,line join=round] (310.90,343.55) --
	(331.24,313.23) --
	(351.58,299.94) --
	(371.92,286.32) --
	(392.26,279.17) --
	(412.60,270.70) --
	(432.95,265.14) --
	(453.29,263.63) --
	(473.63,260.68) --
	(493.97,260.94) --
	(514.31,259.75);
\definecolor{drawColor}{RGB}{255,0,0}

\path[draw=drawColor,line width= 0.6pt,dash pattern=on 1pt off 3pt ,line join=round] (310.90,314.47) --
	(331.24,302.34) --
	(351.58,293.59) --
	(371.92,288.56) --
	(392.26,281.77) --
	(412.60,275.45) --
	(432.95,270.34) --
	(453.29,269.22) --
	(473.63,265.82) --
	(493.97,263.00) --
	(514.31,260.72);
\definecolor{drawColor}{gray}{0.20}

\path[draw=drawColor,line width= 0.6pt,line join=round,line cap=round] (300.73,248.34) rectangle (524.48,410.48);
\end{scope}
\begin{scope}
\path[clip] (  0.00,  0.00) rectangle (794.97,433.62);
\definecolor{drawColor}{gray}{0.30}

\node[text=drawColor,anchor=base east,inner sep=0pt, outer sep=0pt, scale=  0.88] at (295.78,252.68) {1.0};

\node[text=drawColor,anchor=base east,inner sep=0pt, outer sep=0pt, scale=  0.88] at (295.78,294.80) {1.1};

\node[text=drawColor,anchor=base east,inner sep=0pt, outer sep=0pt, scale=  0.88] at (295.78,336.91) {1.2};

\node[text=drawColor,anchor=base east,inner sep=0pt, outer sep=0pt, scale=  0.88] at (295.78,379.03) {1.3};
\end{scope}
\begin{scope}
\path[clip] (  0.00,  0.00) rectangle (794.97,433.62);
\definecolor{drawColor}{gray}{0.20}

\path[draw=drawColor,line width= 0.6pt,line join=round] (297.98,255.71) --
	(300.73,255.71);

\path[draw=drawColor,line width= 0.6pt,line join=round] (297.98,297.83) --
	(300.73,297.83);

\path[draw=drawColor,line width= 0.6pt,line join=round] (297.98,339.94) --
	(300.73,339.94);

\path[draw=drawColor,line width= 0.6pt,line join=round] (297.98,382.06) --
	(300.73,382.06);
\end{scope}
\begin{scope}
\path[clip] (  0.00,  0.00) rectangle (794.97,433.62);
\definecolor{drawColor}{gray}{0.20}

\path[draw=drawColor,line width= 0.6pt,line join=round] (310.90,245.59) --
	(310.90,248.34);

\path[draw=drawColor,line width= 0.6pt,line join=round] (412.60,245.59) --
	(412.60,248.34);

\path[draw=drawColor,line width= 0.6pt,line join=round] (514.31,245.59) --
	(514.31,248.34);
\end{scope}
\begin{scope}
\path[clip] (  0.00,  0.00) rectangle (794.97,433.62);
\definecolor{drawColor}{gray}{0.30}

\node[text=drawColor,anchor=base,inner sep=0pt, outer sep=0pt, scale=  0.88] at (310.90,237.33) {1e+04};

\node[text=drawColor,anchor=base,inner sep=0pt, outer sep=0pt, scale=  0.88] at (412.60,237.33) {1e+05};

\node[text=drawColor,anchor=base,inner sep=0pt, outer sep=0pt, scale=  0.88] at (514.31,237.33) {1e+06};
\end{scope}
\begin{scope}
\path[clip] (  0.00,  0.00) rectangle (794.97,433.62);
\definecolor{drawColor}{RGB}{0,0,0}

\node[text=drawColor,anchor=base,inner sep=0pt, outer sep=0pt, scale=  1.10] at (412.60,224.25) {Total number of steps};
\end{scope}
\begin{scope}
\path[clip] (  0.00,  0.00) rectangle (794.97,433.62);
\definecolor{drawColor}{RGB}{0,0,0}

\node[text=drawColor,rotate= 90.00,anchor=base,inner sep=0pt, outer sep=0pt, scale=  1.10] at (278.07,329.41) {Normalized risk};
\end{scope}
\begin{scope}
\path[clip] (  0.00,  0.00) rectangle (794.97,433.62);
\definecolor{drawColor}{RGB}{0,0,0}

\node[text=drawColor,anchor=base west,inner sep=0pt, outer sep=0pt, scale=  1.32] at (300.73,419.03) {Linear regression, sparse};
\end{scope}
\begin{scope}
\path[clip] (264.99,  0.00) rectangle (529.98,216.81);
\definecolor{drawColor}{RGB}{255,255,255}
\definecolor{fillColor}{RGB}{255,255,255}

\path[draw=drawColor,line width= 0.6pt,line join=round,line cap=round,fill=fillColor] (264.99,  0.00) rectangle (529.98,216.81);
\end{scope}
\begin{scope}
\path[clip] (300.73, 31.53) rectangle (524.48,193.67);
\definecolor{fillColor}{RGB}{255,255,255}

\path[fill=fillColor] (300.73, 31.53) rectangle (524.48,193.67);
\definecolor{drawColor}{gray}{0.92}

\path[draw=drawColor,line width= 0.3pt,line join=round] (300.73, 59.96) --
	(524.48, 59.96);

\path[draw=drawColor,line width= 0.3pt,line join=round] (300.73,102.07) --
	(524.48,102.07);

\path[draw=drawColor,line width= 0.3pt,line join=round] (300.73,144.19) --
	(524.48,144.19);

\path[draw=drawColor,line width= 0.3pt,line join=round] (300.73,186.30) --
	(524.48,186.30);

\path[draw=drawColor,line width= 0.3pt,line join=round] (361.75, 31.53) --
	(361.75,193.67);

\path[draw=drawColor,line width= 0.3pt,line join=round] (463.46, 31.53) --
	(463.46,193.67);

\path[draw=drawColor,line width= 0.6pt,line join=round] (300.73, 38.90) --
	(524.48, 38.90);

\path[draw=drawColor,line width= 0.6pt,line join=round] (300.73, 81.02) --
	(524.48, 81.02);

\path[draw=drawColor,line width= 0.6pt,line join=round] (300.73,123.13) --
	(524.48,123.13);

\path[draw=drawColor,line width= 0.6pt,line join=round] (300.73,165.25) --
	(524.48,165.25);

\path[draw=drawColor,line width= 0.6pt,line join=round] (310.90, 31.53) --
	(310.90,193.67);

\path[draw=drawColor,line width= 0.6pt,line join=round] (412.60, 31.53) --
	(412.60,193.67);

\path[draw=drawColor,line width= 0.6pt,line join=round] (514.31, 31.53) --
	(514.31,193.67);
\definecolor{drawColor}{RGB}{0,0,0}

\path[draw=drawColor,line width= 0.6pt,dash pattern=on 4pt off 4pt ,line join=round] (310.90,131.03) --
	(331.24,114.79) --
	(351.58,106.49) --
	(371.92, 92.38) --
	(392.26, 84.88) --
	(412.60, 81.26) --
	(432.95, 74.56) --
	(453.29, 73.73) --
	(473.63, 67.11) --
	(493.97, 68.41) --
	(514.31, 65.05);
\definecolor{drawColor}{RGB}{255,0,0}

\path[draw=drawColor,line width= 0.6pt,dash pattern=on 1pt off 3pt on 4pt off 3pt ,line join=round] (310.90, 97.63) --
	(331.24, 83.76) --
	(351.58, 78.56) --
	(371.92, 65.46) --
	(392.26, 62.12) --
	(412.60, 57.78) --
	(432.95, 56.27) --
	(453.29, 54.41) --
	(473.63, 51.57) --
	(493.97, 51.60) --
	(514.31, 49.74);
\definecolor{drawColor}{RGB}{0,0,0}

\path[draw=drawColor,line width= 0.6pt,line join=round] (310.90,170.34) --
	(331.24,162.41) --
	(351.58,158.96) --
	(371.92,140.35) --
	(392.26,132.21) --
	(412.60,125.33) --
	(432.95,115.33) --
	(453.29,112.35) --
	(473.63,103.46) --
	(493.97,103.22) --
	(514.31, 98.61);
\definecolor{drawColor}{RGB}{255,0,0}

\path[draw=drawColor,line width= 0.6pt,dash pattern=on 1pt off 3pt ,line join=round] (310.90,130.59) --
	(331.24,111.69) --
	(351.58, 94.87) --
	(371.92, 81.21) --
	(392.26, 75.20) --
	(412.60, 68.61) --
	(432.95, 64.58) --
	(453.29, 58.80) --
	(473.63, 56.39) --
	(493.97, 57.08) --
	(514.31, 55.02);
\definecolor{drawColor}{gray}{0.20}

\path[draw=drawColor,line width= 0.6pt,line join=round,line cap=round] (300.73, 31.53) rectangle (524.48,193.67);
\end{scope}
\begin{scope}
\path[clip] (  0.00,  0.00) rectangle (794.97,433.62);
\definecolor{drawColor}{gray}{0.30}

\node[text=drawColor,anchor=base east,inner sep=0pt, outer sep=0pt, scale=  0.88] at (295.78, 35.87) {1.0};

\node[text=drawColor,anchor=base east,inner sep=0pt, outer sep=0pt, scale=  0.88] at (295.78, 77.99) {1.1};

\node[text=drawColor,anchor=base east,inner sep=0pt, outer sep=0pt, scale=  0.88] at (295.78,120.10) {1.2};

\node[text=drawColor,anchor=base east,inner sep=0pt, outer sep=0pt, scale=  0.88] at (295.78,162.22) {1.3};
\end{scope}
\begin{scope}
\path[clip] (  0.00,  0.00) rectangle (794.97,433.62);
\definecolor{drawColor}{gray}{0.20}

\path[draw=drawColor,line width= 0.6pt,line join=round] (297.98, 38.90) --
	(300.73, 38.90);

\path[draw=drawColor,line width= 0.6pt,line join=round] (297.98, 81.02) --
	(300.73, 81.02);

\path[draw=drawColor,line width= 0.6pt,line join=round] (297.98,123.13) --
	(300.73,123.13);

\path[draw=drawColor,line width= 0.6pt,line join=round] (297.98,165.25) --
	(300.73,165.25);
\end{scope}
\begin{scope}
\path[clip] (  0.00,  0.00) rectangle (794.97,433.62);
\definecolor{drawColor}{gray}{0.20}

\path[draw=drawColor,line width= 0.6pt,line join=round] (310.90, 28.78) --
	(310.90, 31.53);

\path[draw=drawColor,line width= 0.6pt,line join=round] (412.60, 28.78) --
	(412.60, 31.53);

\path[draw=drawColor,line width= 0.6pt,line join=round] (514.31, 28.78) --
	(514.31, 31.53);
\end{scope}
\begin{scope}
\path[clip] (  0.00,  0.00) rectangle (794.97,433.62);
\definecolor{drawColor}{gray}{0.30}

\node[text=drawColor,anchor=base,inner sep=0pt, outer sep=0pt, scale=  0.88] at (310.90, 20.52) {1e+04};

\node[text=drawColor,anchor=base,inner sep=0pt, outer sep=0pt, scale=  0.88] at (412.60, 20.52) {1e+05};

\node[text=drawColor,anchor=base,inner sep=0pt, outer sep=0pt, scale=  0.88] at (514.31, 20.52) {1e+06};
\end{scope}
\begin{scope}
\path[clip] (  0.00,  0.00) rectangle (794.97,433.62);
\definecolor{drawColor}{RGB}{0,0,0}

\node[text=drawColor,anchor=base,inner sep=0pt, outer sep=0pt, scale=  1.10] at (412.60,  7.44) {Total number of steps};
\end{scope}
\begin{scope}
\path[clip] (  0.00,  0.00) rectangle (794.97,433.62);
\definecolor{drawColor}{RGB}{0,0,0}

\node[text=drawColor,rotate= 90.00,anchor=base,inner sep=0pt, outer sep=0pt, scale=  1.10] at (278.07,112.60) {Normalized risk};
\end{scope}
\begin{scope}
\path[clip] (  0.00,  0.00) rectangle (794.97,433.62);
\definecolor{drawColor}{RGB}{0,0,0}

\node[text=drawColor,anchor=base west,inner sep=0pt, outer sep=0pt, scale=  1.32] at (300.73,202.22) {Logistic regression, sparse};
\end{scope}
\begin{scope}
\path[clip] (529.98,216.81) rectangle (794.97,433.62);
\definecolor{drawColor}{RGB}{255,255,255}
\definecolor{fillColor}{RGB}{255,255,255}

\path[draw=drawColor,line width= 0.6pt,line join=round,line cap=round,fill=fillColor] (529.98,216.81) rectangle (794.97,433.62);
\end{scope}
\begin{scope}
\path[clip] (565.72,248.34) rectangle (789.47,410.48);
\definecolor{fillColor}{RGB}{255,255,255}

\path[fill=fillColor] (565.72,248.34) rectangle (789.47,410.48);
\definecolor{drawColor}{gray}{0.92}

\path[draw=drawColor,line width= 0.3pt,line join=round] (565.72,276.77) --
	(789.47,276.77);

\path[draw=drawColor,line width= 0.3pt,line join=round] (565.72,318.88) --
	(789.47,318.88);

\path[draw=drawColor,line width= 0.3pt,line join=round] (565.72,361.00) --
	(789.47,361.00);

\path[draw=drawColor,line width= 0.3pt,line join=round] (565.72,403.11) --
	(789.47,403.11);

\path[draw=drawColor,line width= 0.3pt,line join=round] (626.74,248.34) --
	(626.74,410.48);

\path[draw=drawColor,line width= 0.3pt,line join=round] (728.45,248.34) --
	(728.45,410.48);

\path[draw=drawColor,line width= 0.6pt,line join=round] (565.72,255.71) --
	(789.47,255.71);

\path[draw=drawColor,line width= 0.6pt,line join=round] (565.72,297.83) --
	(789.47,297.83);

\path[draw=drawColor,line width= 0.6pt,line join=round] (565.72,339.94) --
	(789.47,339.94);

\path[draw=drawColor,line width= 0.6pt,line join=round] (565.72,382.06) --
	(789.47,382.06);

\path[draw=drawColor,line width= 0.6pt,line join=round] (575.89,248.34) --
	(575.89,410.48);

\path[draw=drawColor,line width= 0.6pt,line join=round] (677.59,248.34) --
	(677.59,410.48);

\path[draw=drawColor,line width= 0.6pt,line join=round] (779.30,248.34) --
	(779.30,410.48);
\definecolor{drawColor}{RGB}{0,0,0}

\path[draw=drawColor,line width= 0.6pt,dash pattern=on 4pt off 4pt ,line join=round] (575.89,304.19) --
	(596.23,293.50) --
	(616.57,284.62) --
	(636.91,275.47) --
	(657.25,272.33) --
	(677.59,265.78) --
	(697.94,262.82) --
	(718.28,263.36) --
	(738.62,260.48) --
	(758.96,260.60) --
	(779.30,259.16);
\definecolor{drawColor}{RGB}{255,0,0}

\path[draw=drawColor,line width= 0.6pt,dash pattern=on 1pt off 3pt on 4pt off 3pt ,line join=round] (575.89,288.60) --
	(596.23,281.42) --
	(616.57,276.16) --
	(636.91,274.40) --
	(657.25,270.32) --
	(677.59,266.33) --
	(697.94,264.07) --
	(718.28,263.34) --
	(738.62,260.70) --
	(758.96,259.68) --
	(779.30,259.30);
\definecolor{drawColor}{RGB}{0,0,0}

\path[draw=drawColor,line width= 0.6pt,line join=round] (575.89,342.16) --
	(596.23,311.42) --
	(616.57,303.72) --
	(636.91,287.18) --
	(657.25,278.75) --
	(677.59,271.62) --
	(697.94,265.38) --
	(718.28,263.85) --
	(738.62,260.90) --
	(758.96,261.32) --
	(779.30,259.21);
\definecolor{drawColor}{RGB}{255,0,0}

\path[draw=drawColor,line width= 0.6pt,dash pattern=on 1pt off 3pt ,line join=round] (575.89,314.07) --
	(596.23,301.37) --
	(616.57,293.89) --
	(636.91,288.65) --
	(657.25,281.73) --
	(677.59,275.33) --
	(697.94,270.52) --
	(718.28,269.23) --
	(738.62,265.83) --
	(758.96,262.96) --
	(779.30,260.69);
\definecolor{drawColor}{gray}{0.20}

\path[draw=drawColor,line width= 0.6pt,line join=round,line cap=round] (565.72,248.34) rectangle (789.47,410.48);
\end{scope}
\begin{scope}
\path[clip] (  0.00,  0.00) rectangle (794.97,433.62);
\definecolor{drawColor}{gray}{0.30}

\node[text=drawColor,anchor=base east,inner sep=0pt, outer sep=0pt, scale=  0.88] at (560.77,252.68) {1.0};

\node[text=drawColor,anchor=base east,inner sep=0pt, outer sep=0pt, scale=  0.88] at (560.77,294.80) {1.1};

\node[text=drawColor,anchor=base east,inner sep=0pt, outer sep=0pt, scale=  0.88] at (560.77,336.91) {1.2};

\node[text=drawColor,anchor=base east,inner sep=0pt, outer sep=0pt, scale=  0.88] at (560.77,379.03) {1.3};
\end{scope}
\begin{scope}
\path[clip] (  0.00,  0.00) rectangle (794.97,433.62);
\definecolor{drawColor}{gray}{0.20}

\path[draw=drawColor,line width= 0.6pt,line join=round] (562.97,255.71) --
	(565.72,255.71);

\path[draw=drawColor,line width= 0.6pt,line join=round] (562.97,297.83) --
	(565.72,297.83);

\path[draw=drawColor,line width= 0.6pt,line join=round] (562.97,339.94) --
	(565.72,339.94);

\path[draw=drawColor,line width= 0.6pt,line join=round] (562.97,382.06) --
	(565.72,382.06);
\end{scope}
\begin{scope}
\path[clip] (  0.00,  0.00) rectangle (794.97,433.62);
\definecolor{drawColor}{gray}{0.20}

\path[draw=drawColor,line width= 0.6pt,line join=round] (575.89,245.59) --
	(575.89,248.34);

\path[draw=drawColor,line width= 0.6pt,line join=round] (677.59,245.59) --
	(677.59,248.34);

\path[draw=drawColor,line width= 0.6pt,line join=round] (779.30,245.59) --
	(779.30,248.34);
\end{scope}
\begin{scope}
\path[clip] (  0.00,  0.00) rectangle (794.97,433.62);
\definecolor{drawColor}{gray}{0.30}

\node[text=drawColor,anchor=base,inner sep=0pt, outer sep=0pt, scale=  0.88] at (575.89,237.33) {1e+04};

\node[text=drawColor,anchor=base,inner sep=0pt, outer sep=0pt, scale=  0.88] at (677.59,237.33) {1e+05};

\node[text=drawColor,anchor=base,inner sep=0pt, outer sep=0pt, scale=  0.88] at (779.30,237.33) {1e+06};
\end{scope}
\begin{scope}
\path[clip] (  0.00,  0.00) rectangle (794.97,433.62);
\definecolor{drawColor}{RGB}{0,0,0}

\node[text=drawColor,anchor=base,inner sep=0pt, outer sep=0pt, scale=  1.10] at (677.59,224.25) {Total number of steps};
\end{scope}
\begin{scope}
\path[clip] (  0.00,  0.00) rectangle (794.97,433.62);
\definecolor{drawColor}{RGB}{0,0,0}

\node[text=drawColor,rotate= 90.00,anchor=base,inner sep=0pt, outer sep=0pt, scale=  1.10] at (543.06,329.41) {Normalized risk};
\end{scope}
\begin{scope}
\path[clip] (  0.00,  0.00) rectangle (794.97,433.62);
\definecolor{drawColor}{RGB}{0,0,0}

\node[text=drawColor,anchor=base west,inner sep=0pt, outer sep=0pt, scale=  1.32] at (565.72,419.03) {Linear regression, dense};
\end{scope}
\begin{scope}
\path[clip] (529.98,  0.00) rectangle (794.97,216.81);
\definecolor{drawColor}{RGB}{255,255,255}
\definecolor{fillColor}{RGB}{255,255,255}

\path[draw=drawColor,line width= 0.6pt,line join=round,line cap=round,fill=fillColor] (529.98,  0.00) rectangle (794.97,216.81);
\end{scope}
\begin{scope}
\path[clip] (565.72, 31.53) rectangle (789.47,193.67);
\definecolor{fillColor}{RGB}{255,255,255}

\path[fill=fillColor] (565.72, 31.53) rectangle (789.47,193.67);
\definecolor{drawColor}{gray}{0.92}

\path[draw=drawColor,line width= 0.3pt,line join=round] (565.72, 59.96) --
	(789.47, 59.96);

\path[draw=drawColor,line width= 0.3pt,line join=round] (565.72,102.07) --
	(789.47,102.07);

\path[draw=drawColor,line width= 0.3pt,line join=round] (565.72,144.19) --
	(789.47,144.19);

\path[draw=drawColor,line width= 0.3pt,line join=round] (565.72,186.30) --
	(789.47,186.30);

\path[draw=drawColor,line width= 0.3pt,line join=round] (626.74, 31.53) --
	(626.74,193.67);

\path[draw=drawColor,line width= 0.3pt,line join=round] (728.45, 31.53) --
	(728.45,193.67);

\path[draw=drawColor,line width= 0.6pt,line join=round] (565.72, 38.90) --
	(789.47, 38.90);

\path[draw=drawColor,line width= 0.6pt,line join=round] (565.72, 81.02) --
	(789.47, 81.02);

\path[draw=drawColor,line width= 0.6pt,line join=round] (565.72,123.13) --
	(789.47,123.13);

\path[draw=drawColor,line width= 0.6pt,line join=round] (565.72,165.25) --
	(789.47,165.25);

\path[draw=drawColor,line width= 0.6pt,line join=round] (575.89, 31.53) --
	(575.89,193.67);

\path[draw=drawColor,line width= 0.6pt,line join=round] (677.59, 31.53) --
	(677.59,193.67);

\path[draw=drawColor,line width= 0.6pt,line join=round] (779.30, 31.53) --
	(779.30,193.67);
\definecolor{drawColor}{RGB}{0,0,0}

\path[draw=drawColor,line width= 0.6pt,dash pattern=on 4pt off 4pt ,line join=round] (575.89,128.59) --
	(596.23,115.51) --
	(616.57,103.10) --
	(636.91, 92.75) --
	(657.25, 84.37) --
	(677.59, 79.68) --
	(697.94, 75.52) --
	(718.28, 75.20) --
	(738.62, 68.33) --
	(758.96, 67.60) --
	(779.30, 64.78);
\definecolor{drawColor}{RGB}{255,0,0}

\path[draw=drawColor,line width= 0.6pt,dash pattern=on 1pt off 3pt on 4pt off 3pt ,line join=round] (575.89, 95.47) --
	(596.23, 81.31) --
	(616.57, 73.67) --
	(636.91, 65.94) --
	(657.25, 63.62) --
	(677.59, 58.52) --
	(697.94, 55.59) --
	(718.28, 53.89) --
	(738.62, 52.47) --
	(758.96, 50.22) --
	(779.30, 49.36);
\definecolor{drawColor}{RGB}{0,0,0}

\path[draw=drawColor,line width= 0.6pt,line join=round] (575.89,167.07) --
	(596.23,166.60) --
	(616.57,152.37) --
	(636.91,143.85) --
	(657.25,128.37) --
	(677.59,121.50) --
	(697.94,114.93) --
	(718.28,112.71) --
	(738.62,106.69) --
	(758.96,103.32) --
	(779.30, 96.39);
\definecolor{drawColor}{RGB}{255,0,0}

\path[draw=drawColor,line width= 0.6pt,dash pattern=on 1pt off 3pt ,line join=round] (575.89,130.19) --
	(596.23,107.14) --
	(616.57, 94.55) --
	(636.91, 83.65) --
	(657.25, 75.73) --
	(677.59, 68.13) --
	(697.94, 63.08) --
	(718.28, 61.78) --
	(738.62, 57.98) --
	(758.96, 56.90) --
	(779.30, 52.41);
\definecolor{drawColor}{gray}{0.20}

\path[draw=drawColor,line width= 0.6pt,line join=round,line cap=round] (565.72, 31.53) rectangle (789.47,193.67);
\end{scope}
\begin{scope}
\path[clip] (  0.00,  0.00) rectangle (794.97,433.62);
\definecolor{drawColor}{gray}{0.30}

\node[text=drawColor,anchor=base east,inner sep=0pt, outer sep=0pt, scale=  0.88] at (560.77, 35.87) {1.0};

\node[text=drawColor,anchor=base east,inner sep=0pt, outer sep=0pt, scale=  0.88] at (560.77, 77.99) {1.1};

\node[text=drawColor,anchor=base east,inner sep=0pt, outer sep=0pt, scale=  0.88] at (560.77,120.10) {1.2};

\node[text=drawColor,anchor=base east,inner sep=0pt, outer sep=0pt, scale=  0.88] at (560.77,162.22) {1.3};
\end{scope}
\begin{scope}
\path[clip] (  0.00,  0.00) rectangle (794.97,433.62);
\definecolor{drawColor}{gray}{0.20}

\path[draw=drawColor,line width= 0.6pt,line join=round] (562.97, 38.90) --
	(565.72, 38.90);

\path[draw=drawColor,line width= 0.6pt,line join=round] (562.97, 81.02) --
	(565.72, 81.02);

\path[draw=drawColor,line width= 0.6pt,line join=round] (562.97,123.13) --
	(565.72,123.13);

\path[draw=drawColor,line width= 0.6pt,line join=round] (562.97,165.25) --
	(565.72,165.25);
\end{scope}
\begin{scope}
\path[clip] (  0.00,  0.00) rectangle (794.97,433.62);
\definecolor{drawColor}{gray}{0.20}

\path[draw=drawColor,line width= 0.6pt,line join=round] (575.89, 28.78) --
	(575.89, 31.53);

\path[draw=drawColor,line width= 0.6pt,line join=round] (677.59, 28.78) --
	(677.59, 31.53);

\path[draw=drawColor,line width= 0.6pt,line join=round] (779.30, 28.78) --
	(779.30, 31.53);
\end{scope}
\begin{scope}
\path[clip] (  0.00,  0.00) rectangle (794.97,433.62);
\definecolor{drawColor}{gray}{0.30}

\node[text=drawColor,anchor=base,inner sep=0pt, outer sep=0pt, scale=  0.88] at (575.89, 20.52) {1e+04};

\node[text=drawColor,anchor=base,inner sep=0pt, outer sep=0pt, scale=  0.88] at (677.59, 20.52) {1e+05};

\node[text=drawColor,anchor=base,inner sep=0pt, outer sep=0pt, scale=  0.88] at (779.30, 20.52) {1e+06};
\end{scope}
\begin{scope}
\path[clip] (  0.00,  0.00) rectangle (794.97,433.62);
\definecolor{drawColor}{RGB}{0,0,0}

\node[text=drawColor,anchor=base,inner sep=0pt, outer sep=0pt, scale=  1.10] at (677.59,  7.44) {Total number of steps};
\end{scope}
\begin{scope}
\path[clip] (  0.00,  0.00) rectangle (794.97,433.62);
\definecolor{drawColor}{RGB}{0,0,0}

\node[text=drawColor,rotate= 90.00,anchor=base,inner sep=0pt, outer sep=0pt, scale=  1.10] at (543.06,112.60) {Normalized risk};
\end{scope}
\begin{scope}
\path[clip] (  0.00,  0.00) rectangle (794.97,433.62);
\definecolor{drawColor}{RGB}{0,0,0}

\node[text=drawColor,anchor=base west,inner sep=0pt, outer sep=0pt, scale=  1.32] at (565.72,202.22) {Logistic regression, dense};
\end{scope}
\end{tikzpicture}

%% file: figs/accuracy_legend.tex
\begin{tikzpicture}[x=1pt,y=1pt]
\definecolor{fillColor}{RGB}{255,255,255}
\path[use as bounding box,fill=fillColor,fill opacity=0.00] (0,0) rectangle (361.35, 18.07);
\begin{scope}
\path[clip] (  0.00,  0.00) rectangle ( 90.34, 18.07);
\definecolor{drawColor}{RGB}{0,0,0}

\path[draw=drawColor,line width= 0.4pt,line join=round,line cap=round] ( 45.17,  2.76) -- ( 66.08,  2.76);

\path[draw=drawColor,line width= 0.4pt,line join=round,line cap=round] ( 45.17,  6.94) -- ( 66.08,  6.94);

\path[draw=drawColor,line width= 0.4pt,line join=round,line cap=round] ( 45.17, 11.12) -- ( 66.08, 11.12);

\path[draw=drawColor,line width= 0.4pt,line join=round,line cap=round] ( 45.17, 15.31) -- ( 66.08, 15.31);
\definecolor{drawColor}{RGB}{190,190,190}

\path[draw=drawColor,line width= 0.4pt,dash pattern=on 1pt off 3pt ,line join=round,line cap=round] ( 45.17,  2.76) -- ( 45.17, 15.31);
\definecolor{drawColor}{RGB}{0,0,0}

\path[draw=drawColor,line width= 0.4pt,line join=round,line cap=round] (  5.94,  9.03) -- ( 29.70,  9.03);

\node[text=drawColor,anchor=base west,inner sep=0pt, outer sep=0pt, scale=  0.66] at ( 35.64,  6.76) {:};
\end{scope}
\begin{scope}
\path[clip] ( 90.34,  0.00) rectangle (180.67, 18.07);
\definecolor{drawColor}{RGB}{0,0,0}

\path[draw=drawColor,line width= 0.4pt,line join=round,line cap=round] (135.51,  9.03) -- (152.24,  9.03);

\path[draw=drawColor,line width= 0.4pt,line join=round,line cap=round] (152.24,  2.76) -- (168.96,  2.76);

\path[draw=drawColor,line width= 0.4pt,line join=round,line cap=round] (152.24,  6.94) -- (168.96,  6.94);

\path[draw=drawColor,line width= 0.4pt,line join=round,line cap=round] (152.24, 11.12) -- (168.96, 11.12);

\path[draw=drawColor,line width= 0.4pt,line join=round,line cap=round] (152.24, 15.31) -- (168.96, 15.31);
\definecolor{drawColor}{RGB}{190,190,190}

\path[draw=drawColor,line width= 0.4pt,dash pattern=on 1pt off 3pt ,line join=round,line cap=round] (152.24,  2.76) -- (152.24, 15.31);
\definecolor{drawColor}{RGB}{255,0,0}

\path[draw=drawColor,line width= 0.4pt,dash pattern=on 1pt off 3pt ,line join=round,line cap=round] ( 96.28,  9.03) -- (120.04,  9.03);
\definecolor{drawColor}{RGB}{0,0,0}

\node[text=drawColor,anchor=base west,inner sep=0pt, outer sep=0pt, scale=  0.66] at (125.98,  6.76) {:};
\end{scope}
\begin{scope}
\path[clip] (180.67,  0.00) rectangle (271.01, 18.07);
\definecolor{drawColor}{RGB}{0,0,0}

\path[draw=drawColor,line width= 0.4pt,line join=round,line cap=round] (225.84,  4.85) -- (239.78,  4.85);

\path[draw=drawColor,line width= 0.4pt,line join=round,line cap=round] (225.84, 13.22) -- (239.78, 13.22);

\path[draw=drawColor,line width= 0.4pt,line join=round,line cap=round] (239.78,  2.76) -- (253.73,  2.76);

\path[draw=drawColor,line width= 0.4pt,line join=round,line cap=round] (239.78,  6.94) -- (253.73,  6.94);

\path[draw=drawColor,line width= 0.4pt,line join=round,line cap=round] (239.78, 11.12) -- (253.73, 11.12);

\path[draw=drawColor,line width= 0.4pt,line join=round,line cap=round] (239.78, 15.31) -- (253.73, 15.31);
\definecolor{drawColor}{RGB}{190,190,190}

\path[draw=drawColor,line width= 0.4pt,dash pattern=on 1pt off 3pt ,line join=round,line cap=round] (225.84,  4.85) -- (225.84, 13.22);

\path[draw=drawColor,line width= 0.4pt,dash pattern=on 1pt off 3pt ,line join=round,line cap=round] (239.78,  2.76) -- (239.78,  6.94);

\path[draw=drawColor,line width= 0.4pt,dash pattern=on 1pt off 3pt ,line join=round,line cap=round] (239.78, 11.12) -- (239.78, 15.31);
\definecolor{drawColor}{RGB}{0,0,0}

\path[draw=drawColor,line width= 0.4pt,dash pattern=on 4pt off 4pt ,line join=round,line cap=round] (186.61,  9.03) -- (210.37,  9.03);

\node[text=drawColor,anchor=base west,inner sep=0pt, outer sep=0pt, scale=  0.66] at (216.31,  6.76) {:};
\end{scope}
\begin{scope}
\path[clip] (271.01,  0.00) rectangle (361.35, 18.07);
\definecolor{drawColor}{RGB}{0,0,0}

\path[draw=drawColor,line width= 0.4pt,line join=round,line cap=round] (316.18,  9.03) -- (328.13,  9.03);

\path[draw=drawColor,line width= 0.4pt,line join=round,line cap=round] (328.13,  4.85) -- (340.08,  4.85);

\path[draw=drawColor,line width= 0.4pt,line join=round,line cap=round] (328.13, 13.22) -- (340.08, 13.22);

\path[draw=drawColor,line width= 0.4pt,line join=round,line cap=round] (340.08,  2.76) -- (352.03,  2.76);

\path[draw=drawColor,line width= 0.4pt,line join=round,line cap=round] (340.08,  6.94) -- (352.03,  6.94);

\path[draw=drawColor,line width= 0.4pt,line join=round,line cap=round] (340.08, 11.12) -- (352.03, 11.12);

\path[draw=drawColor,line width= 0.4pt,line join=round,line cap=round] (340.08, 15.31) -- (352.03, 15.31);
\definecolor{drawColor}{RGB}{190,190,190}

\path[draw=drawColor,line width= 0.4pt,dash pattern=on 1pt off 3pt ,line join=round,line cap=round] (328.13,  4.85) -- (328.13, 13.22);

\path[draw=drawColor,line width= 0.4pt,dash pattern=on 1pt off 3pt ,line join=round,line cap=round] (340.08,  2.76) -- (340.08,  6.94);

\path[draw=drawColor,line width= 0.4pt,dash pattern=on 1pt off 3pt ,line join=round,line cap=round] (340.08, 11.12) -- (340.08, 15.31);
\definecolor{drawColor}{RGB}{255,0,0}

\path[draw=drawColor,line width= 0.4pt,dash pattern=on 1pt off 3pt on 4pt off 3pt ,line join=round,line cap=round] (276.95,  9.03) -- (300.71,  9.03);
\definecolor{drawColor}{RGB}{0,0,0}

\node[text=drawColor,anchor=base west,inner sep=0pt, outer sep=0pt, scale=  0.66] at (306.65,  6.76) {:};
\end{scope}
\end{tikzpicture}

%% file: figs/fig_simulation_lm.tex
\begin{tikzpicture}[x=1pt,y=1pt]
\definecolor{fillColor}{RGB}{255,255,255}
\path[use as bounding box,fill=fillColor,fill opacity=0.00] (0,0) rectangle (224.04,216.81);
\begin{scope}
\path[clip] (  0.00,  0.00) rectangle ( 84.01,216.81);
\definecolor{drawColor}{RGB}{255,255,255}
\definecolor{fillColor}{RGB}{255,255,255}

\path[draw=drawColor,line width= 0.6pt,line join=round,line cap=round,fill=fillColor] (  0.00,  0.00) rectangle ( 84.01,216.81);
\end{scope}
\begin{scope}
\path[clip] (  2.75,  5.60) rectangle ( 84.01,210.92);
\definecolor{fillColor}{RGB}{255,255,255}

\path[fill=fillColor] (  2.75,  5.60) rectangle ( 84.01,210.92);
\definecolor{drawColor}{gray}{0.92}

\path[draw=drawColor,line width= 0.6pt,line join=round] (  2.75, 13.35) --
	( 84.01, 13.35);

\path[draw=drawColor,line width= 0.6pt,line join=round] (  2.75, 44.98) --
	( 84.01, 44.98);

\path[draw=drawColor,line width= 0.6pt,line join=round] (  2.75, 76.62) --
	( 84.01, 76.62);

\path[draw=drawColor,line width= 0.6pt,line join=round] (  2.75,108.26) --
	( 84.01,108.26);

\path[draw=drawColor,line width= 0.6pt,line join=round] (  2.75,139.89) --
	( 84.01,139.89);

\path[draw=drawColor,line width= 0.6pt,line join=round] (  2.75,171.53) --
	( 84.01,171.53);

\path[draw=drawColor,line width= 0.6pt,line join=round] (  2.75,203.17) --
	( 84.01,203.17);

\path[draw=drawColor,line width= 0.6pt,line join=round] ( 80.32,  5.60) --
	( 80.32,210.92);

\path[draw=drawColor,line width= 0.6pt,line join=round] ( 13.83,  5.60) --
	( 13.83,210.92);
\definecolor{fillColor}{RGB}{128,128,128}

\path[fill=fillColor] ( 13.51, 34.44) rectangle ( 80.32, 42.87);
\definecolor{fillColor}{RGB}{187,187,187}

\path[fill=fillColor] ( 13.64, 24.95) rectangle ( 80.32, 33.38);
\definecolor{fillColor}{RGB}{230,230,230}

\path[fill=fillColor] ( 13.76, 15.46) rectangle ( 80.32, 23.89);
\definecolor{fillColor}{RGB}{128,128,128}

\path[fill=fillColor] ( 14.43, 66.07) rectangle ( 80.32, 74.51);
\definecolor{fillColor}{RGB}{187,187,187}

\path[fill=fillColor] ( 14.05, 56.58) rectangle ( 80.32, 65.02);
\definecolor{fillColor}{RGB}{230,230,230}

\path[fill=fillColor] ( 13.46, 47.09) rectangle ( 80.32, 55.53);
\definecolor{fillColor}{RGB}{128,128,128}

\path[fill=fillColor] ( 14.53, 97.71) rectangle ( 80.32,106.15);
\definecolor{fillColor}{RGB}{187,187,187}

\path[fill=fillColor] ( 14.43, 88.22) rectangle ( 80.32, 96.66);
\definecolor{fillColor}{RGB}{230,230,230}

\path[fill=fillColor] ( 14.42, 78.73) rectangle ( 80.32, 87.17);
\definecolor{fillColor}{RGB}{128,128,128}

\path[fill=fillColor] ( 13.99,129.35) rectangle ( 80.32,137.79);
\definecolor{fillColor}{RGB}{187,187,187}

\path[fill=fillColor] ( 13.83,119.86) rectangle ( 80.32,128.29);
\definecolor{fillColor}{RGB}{230,230,230}

\path[fill=fillColor] ( 13.85,110.37) rectangle ( 80.32,118.80);
\definecolor{fillColor}{RGB}{128,128,128}

\path[fill=fillColor] ( 15.04,160.99) rectangle ( 80.32,169.42);
\definecolor{fillColor}{RGB}{187,187,187}

\path[fill=fillColor] ( 14.89,151.50) rectangle ( 80.32,159.93);
\definecolor{fillColor}{RGB}{230,230,230}

\path[fill=fillColor] ( 14.60,142.00) rectangle ( 80.32,150.44);
\definecolor{fillColor}{RGB}{128,128,128}

\path[fill=fillColor] ( 15.16,192.62) rectangle ( 80.32,201.06);
\definecolor{fillColor}{RGB}{187,187,187}

\path[fill=fillColor] ( 14.75,183.13) rectangle ( 80.32,191.57);
\definecolor{fillColor}{RGB}{230,230,230}

\path[fill=fillColor] ( 14.82,173.64) rectangle ( 80.32,182.08);
\definecolor{drawColor}{RGB}{0,0,0}

\node[text=drawColor,anchor=base west,inner sep=0pt, outer sep=0pt, scale=  0.57] at ( 13.51, 36.70) {0.9044};

\node[text=drawColor,anchor=base west,inner sep=0pt, outer sep=0pt, scale=  0.57] at ( 13.64, 27.21) {0.9026};

\node[text=drawColor,anchor=base west,inner sep=0pt, outer sep=0pt, scale=  0.57] at ( 13.76, 17.71) {0.901};

\node[text=drawColor,anchor=base west,inner sep=0pt, outer sep=0pt, scale=  0.57] at ( 14.43, 68.33) {0.8919};

\node[text=drawColor,anchor=base west,inner sep=0pt, outer sep=0pt, scale=  0.57] at ( 14.05, 58.84) {0.897};

\node[text=drawColor,anchor=base west,inner sep=0pt, outer sep=0pt, scale=  0.57] at ( 13.46, 49.35) {0.905};

\node[text=drawColor,anchor=base west,inner sep=0pt, outer sep=0pt, scale=  0.57] at ( 14.53, 99.97) {0.8906};

\node[text=drawColor,anchor=base west,inner sep=0pt, outer sep=0pt, scale=  0.57] at ( 14.43, 90.48) {0.8919};

\node[text=drawColor,anchor=base west,inner sep=0pt, outer sep=0pt, scale=  0.57] at ( 14.42, 80.99) {0.892};

\node[text=drawColor,anchor=base west,inner sep=0pt, outer sep=0pt, scale=  0.57] at ( 13.99,131.61) {0.8979};

\node[text=drawColor,anchor=base west,inner sep=0pt, outer sep=0pt, scale=  0.57] at ( 13.83,122.12) {0.9};

\node[text=drawColor,anchor=base west,inner sep=0pt, outer sep=0pt, scale=  0.57] at ( 13.85,112.63) {0.8997};

\node[text=drawColor,anchor=base west,inner sep=0pt, outer sep=0pt, scale=  0.57] at ( 15.04,163.24) {0.8837};

\node[text=drawColor,anchor=base west,inner sep=0pt, outer sep=0pt, scale=  0.57] at ( 14.89,153.75) {0.8857};

\node[text=drawColor,anchor=base west,inner sep=0pt, outer sep=0pt, scale=  0.57] at ( 14.60,144.26) {0.8896};

\node[text=drawColor,anchor=base west,inner sep=0pt, outer sep=0pt, scale=  0.57] at ( 15.16,194.88) {0.882};

\node[text=drawColor,anchor=base west,inner sep=0pt, outer sep=0pt, scale=  0.57] at ( 14.75,185.39) {0.8876};

\node[text=drawColor,anchor=base west,inner sep=0pt, outer sep=0pt, scale=  0.57] at ( 14.82,175.90) {0.8866};
\end{scope}
\begin{scope}
\path[clip] (  0.00,  0.00) rectangle (224.04,216.81);
\definecolor{drawColor}{RGB}{0,0,0}

\node[text=drawColor,anchor=base,inner sep=0pt, outer sep=0pt, scale=  0.80] at ( 43.38,208.45) {Coverage prob.};
\end{scope}
\begin{scope}
\path[clip] ( 84.01,  0.00) rectangle (140.02,216.81);
\definecolor{drawColor}{RGB}{255,255,255}
\definecolor{fillColor}{RGB}{255,255,255}

\path[draw=drawColor,line width= 0.6pt,line join=round,line cap=round,fill=fillColor] ( 84.01,  0.00) rectangle (140.02,216.81);
\end{scope}
\begin{scope}
\path[clip] ( 83.92,  5.60) rectangle (142.87,210.92);
\definecolor{drawColor}{RGB}{0,0,0}

\path[draw=drawColor,line width= 0.6pt,line join=round] ( 86.60, 35.65) -- (124.88, 35.65);

\path[draw=drawColor,line width= 0.6pt,line join=round] ( 86.60, 23.02) -- (124.88, 23.02);
\definecolor{drawColor}{RGB}{190,190,190}

\path[draw=drawColor,line width= 0.6pt,dash pattern=on 4pt off 4pt ,line join=round] ( 86.60, 23.02) -- ( 86.60, 35.65);
\definecolor{drawColor}{RGB}{0,0,0}

\path[draw=drawColor,line width= 0.6pt,line join=round] ( 86.60, 64.06) -- (105.74, 64.06);

\path[draw=drawColor,line width= 0.6pt,line join=round] ( 86.60, 57.75) -- (105.74, 57.75);

\path[draw=drawColor,line width= 0.6pt,line join=round] ( 86.60, 70.37) -- (105.74, 70.37);

\path[draw=drawColor,line width= 0.6pt,line join=round] ( 86.60, 51.43) -- (105.74, 51.43);
\definecolor{drawColor}{RGB}{190,190,190}

\path[draw=drawColor,line width= 0.6pt,dash pattern=on 4pt off 4pt ,line join=round] ( 86.60, 51.43) -- ( 86.60, 70.37);
\definecolor{drawColor}{RGB}{0,0,0}

\path[draw=drawColor,line width= 0.6pt,line join=round] ( 86.60, 98.79) -- ( 99.36, 98.79);

\path[draw=drawColor,line width= 0.6pt,line join=round] ( 86.60, 86.16) -- ( 99.36, 86.16);

\path[draw=drawColor,line width= 0.6pt,line join=round] ( 99.36, 95.63) -- (112.12, 95.63);

\path[draw=drawColor,line width= 0.6pt,line join=round] ( 99.36, 89.32) -- (112.12, 89.32);

\path[draw=drawColor,line width= 0.6pt,line join=round] ( 99.36,101.94) -- (112.12,101.94);

\path[draw=drawColor,line width= 0.6pt,line join=round] ( 99.36, 83.00) -- (112.12, 83.00);
\definecolor{drawColor}{RGB}{190,190,190}

\path[draw=drawColor,line width= 0.6pt,dash pattern=on 4pt off 4pt ,line join=round] ( 86.60, 86.16) -- ( 86.60, 98.79);

\path[draw=drawColor,line width= 0.6pt,dash pattern=on 4pt off 4pt ,line join=round] ( 99.36, 83.00) -- ( 99.36,101.94);
\definecolor{drawColor}{RGB}{0,0,0}

\path[draw=drawColor,line width= 0.6pt,line join=round] ( 86.60,124.04) -- (112.12,124.04);

\path[draw=drawColor,line width= 0.6pt,line join=round] (112.12,130.36) -- (137.64,130.36);

\path[draw=drawColor,line width= 0.6pt,line join=round] (112.12,117.73) -- (137.64,117.73);
\definecolor{drawColor}{RGB}{190,190,190}

\path[draw=drawColor,line width= 0.6pt,dash pattern=on 4pt off 4pt ,line join=round] (112.12,117.73) -- (112.12,130.36);
\definecolor{drawColor}{RGB}{0,0,0}

\path[draw=drawColor,line width= 0.6pt,line join=round] ( 86.60,155.61) -- (101.91,155.61);

\path[draw=drawColor,line width= 0.6pt,line join=round] (101.91,158.77) -- (117.22,158.77);

\path[draw=drawColor,line width= 0.6pt,line join=round] (101.91,152.45) -- (117.22,152.45);

\path[draw=drawColor,line width= 0.6pt,line join=round] (101.91,165.08) -- (117.22,165.08);

\path[draw=drawColor,line width= 0.6pt,line join=round] (101.91,146.14) -- (117.22,146.14);
\definecolor{drawColor}{RGB}{190,190,190}

\path[draw=drawColor,line width= 0.6pt,dash pattern=on 4pt off 4pt ,line join=round] (101.91,146.14) -- (101.91,165.08);
\definecolor{drawColor}{RGB}{0,0,0}

\path[draw=drawColor,line width= 0.6pt,line join=round] ( 86.60,187.18) -- ( 97.54,187.18);

\path[draw=drawColor,line width= 0.6pt,line join=round] ( 97.54,193.49) -- (108.47,193.49);

\path[draw=drawColor,line width= 0.6pt,line join=round] ( 97.54,180.87) -- (108.47,180.87);

\path[draw=drawColor,line width= 0.6pt,line join=round] (108.47,190.34) -- (119.41,190.34);

\path[draw=drawColor,line width= 0.6pt,line join=round] (108.47,184.02) -- (119.41,184.02);

\path[draw=drawColor,line width= 0.6pt,line join=round] (108.47,196.65) -- (119.41,196.65);

\path[draw=drawColor,line width= 0.6pt,line join=round] (108.47,177.71) -- (119.41,177.71);
\definecolor{drawColor}{RGB}{190,190,190}

\path[draw=drawColor,line width= 0.6pt,dash pattern=on 4pt off 4pt ,line join=round] ( 97.54,180.87) -- ( 97.54,193.49);

\path[draw=drawColor,line width= 0.6pt,dash pattern=on 4pt off 4pt ,line join=round] (108.47,177.71) -- (108.47,196.65);
\end{scope}
\begin{scope}
\path[clip] (  0.00,  0.00) rectangle (224.04,216.81);
\definecolor{drawColor}{RGB}{0,0,0}

\node[text=drawColor,anchor=base,inner sep=0pt, outer sep=0pt, scale=  0.80] at (113.39,208.45) {Config.};
\end{scope}
\begin{scope}
\path[clip] (140.02,  0.00) rectangle (224.04,216.81);
\definecolor{drawColor}{RGB}{255,255,255}
\definecolor{fillColor}{RGB}{255,255,255}

\path[draw=drawColor,line width= 0.6pt,line join=round,line cap=round,fill=fillColor] (140.02,  0.00) rectangle (224.04,216.81);
\end{scope}
\begin{scope}
\path[clip] (139.93,  5.60) rectangle (224.04,210.92);
\definecolor{fillColor}{RGB}{255,255,255}

\path[fill=fillColor] (139.93,  5.60) rectangle (224.04,210.92);
\definecolor{drawColor}{gray}{0.92}

\path[draw=drawColor,line width= 0.6pt,line join=round] (139.93, 13.35) --
	(224.04, 13.35);

\path[draw=drawColor,line width= 0.6pt,line join=round] (139.93, 44.98) --
	(224.04, 44.98);

\path[draw=drawColor,line width= 0.6pt,line join=round] (139.93, 76.62) --
	(224.04, 76.62);

\path[draw=drawColor,line width= 0.6pt,line join=round] (139.93,108.26) --
	(224.04,108.26);

\path[draw=drawColor,line width= 0.6pt,line join=round] (139.93,139.89) --
	(224.04,139.89);

\path[draw=drawColor,line width= 0.6pt,line join=round] (139.93,171.53) --
	(224.04,171.53);

\path[draw=drawColor,line width= 0.6pt,line join=round] (139.93,203.17) --
	(224.04,203.17);

\path[draw=drawColor,line width= 0.6pt,line join=round] (143.75,  5.60) --
	(143.75,210.92);
\definecolor{fillColor}{RGB}{128,128,128}

\path[fill=fillColor] (143.75, 34.44) rectangle (219.46, 42.87);
\definecolor{fillColor}{RGB}{187,187,187}

\path[fill=fillColor] (143.75, 24.95) rectangle (217.88, 33.38);
\definecolor{fillColor}{RGB}{230,230,230}

\path[fill=fillColor] (143.75, 15.46) rectangle (218.51, 23.89);
\definecolor{fillColor}{RGB}{128,128,128}

\path[fill=fillColor] (143.75, 66.07) rectangle (176.05, 74.51);
\definecolor{fillColor}{RGB}{187,187,187}

\path[fill=fillColor] (143.75, 56.58) rectangle (175.63, 65.02);
\definecolor{fillColor}{RGB}{230,230,230}

\path[fill=fillColor] (143.75, 47.09) rectangle (175.73, 55.53);
\definecolor{fillColor}{RGB}{128,128,128}

\path[fill=fillColor] (143.75, 97.71) rectangle (175.63,106.15);
\definecolor{fillColor}{RGB}{187,187,187}

\path[fill=fillColor] (143.75, 88.22) rectangle (175.10, 96.66);
\definecolor{fillColor}{RGB}{230,230,230}

\path[fill=fillColor] (143.75, 78.73) rectangle (175.10, 87.17);
\definecolor{fillColor}{RGB}{128,128,128}

\path[fill=fillColor] (143.75,129.35) rectangle (217.46,137.79);
\definecolor{fillColor}{RGB}{187,187,187}

\path[fill=fillColor] (143.75,119.86) rectangle (215.68,128.29);
\definecolor{fillColor}{RGB}{230,230,230}

\path[fill=fillColor] (143.75,110.37) rectangle (216.63,118.80);
\definecolor{fillColor}{RGB}{128,128,128}

\path[fill=fillColor] (143.75,160.99) rectangle (174.79,169.42);
\definecolor{fillColor}{RGB}{187,187,187}

\path[fill=fillColor] (143.75,151.50) rectangle (174.26,159.93);
\definecolor{fillColor}{RGB}{230,230,230}

\path[fill=fillColor] (143.75,142.00) rectangle (174.58,150.44);
\definecolor{fillColor}{RGB}{128,128,128}

\path[fill=fillColor] (143.75,192.62) rectangle (174.79,201.06);
\definecolor{fillColor}{RGB}{187,187,187}

\path[fill=fillColor] (143.75,183.13) rectangle (174.47,191.57);
\definecolor{fillColor}{RGB}{230,230,230}

\path[fill=fillColor] (143.75,173.64) rectangle (174.79,182.08);
\definecolor{drawColor}{RGB}{0,0,0}

\node[text=drawColor,anchor=base east,inner sep=0pt, outer sep=0pt, scale=  0.57] at (219.46, 36.70) {0.0722};

\node[text=drawColor,anchor=base east,inner sep=0pt, outer sep=0pt, scale=  0.57] at (217.88, 27.21) {0.0707};

\node[text=drawColor,anchor=base east,inner sep=0pt, outer sep=0pt, scale=  0.57] at (218.51, 17.71) {0.0713};

\node[text=drawColor,anchor=base east,inner sep=0pt, outer sep=0pt, scale=  0.57] at (176.05, 68.33) {0.0308};

\node[text=drawColor,anchor=base east,inner sep=0pt, outer sep=0pt, scale=  0.57] at (175.63, 58.84) {0.0304};

\node[text=drawColor,anchor=base east,inner sep=0pt, outer sep=0pt, scale=  0.57] at (175.73, 49.35) {0.0305};

\node[text=drawColor,anchor=base east,inner sep=0pt, outer sep=0pt, scale=  0.57] at (175.63, 99.97) {0.0304};

\node[text=drawColor,anchor=base east,inner sep=0pt, outer sep=0pt, scale=  0.57] at (175.10, 90.48) {0.0299};

\node[text=drawColor,anchor=base east,inner sep=0pt, outer sep=0pt, scale=  0.57] at (175.10, 80.99) {0.0299};

\node[text=drawColor,anchor=base east,inner sep=0pt, outer sep=0pt, scale=  0.57] at (217.46,131.61) {0.0703};

\node[text=drawColor,anchor=base east,inner sep=0pt, outer sep=0pt, scale=  0.57] at (215.68,122.12) {0.0686};

\node[text=drawColor,anchor=base east,inner sep=0pt, outer sep=0pt, scale=  0.57] at (216.63,112.63) {0.0695};

\node[text=drawColor,anchor=base east,inner sep=0pt, outer sep=0pt, scale=  0.57] at (174.79,163.24) {0.0296};

\node[text=drawColor,anchor=base east,inner sep=0pt, outer sep=0pt, scale=  0.57] at (174.26,153.75) {0.0291};

\node[text=drawColor,anchor=base east,inner sep=0pt, outer sep=0pt, scale=  0.57] at (174.58,144.26) {0.0294};

\node[text=drawColor,anchor=base east,inner sep=0pt, outer sep=0pt, scale=  0.57] at (174.79,194.88) {0.0296};

\node[text=drawColor,anchor=base east,inner sep=0pt, outer sep=0pt, scale=  0.57] at (174.47,185.39) {0.0293};

\node[text=drawColor,anchor=base east,inner sep=0pt, outer sep=0pt, scale=  0.57] at (174.79,175.90) {0.0296};
\end{scope}
\begin{scope}
\path[clip] (  0.00,  0.00) rectangle (224.04,216.81);
\definecolor{drawColor}{RGB}{0,0,0}

\node[text=drawColor,anchor=base,inner sep=0pt, outer sep=0pt, scale=  0.80] at (181.98,208.45) {CI length};
\end{scope}
\end{tikzpicture}

%% file: figs/fig_simulation_logistic.tex
\begin{tikzpicture}[x=1pt,y=1pt]
\definecolor{fillColor}{RGB}{255,255,255}
\path[use as bounding box,fill=fillColor,fill opacity=0.00] (0,0) rectangle (224.04,216.81);
\begin{scope}
\path[clip] (  0.00,  0.00) rectangle ( 84.01,216.81);
\definecolor{drawColor}{RGB}{255,255,255}
\definecolor{fillColor}{RGB}{255,255,255}

\path[draw=drawColor,line width= 0.6pt,line join=round,line cap=round,fill=fillColor] (  0.00,  0.00) rectangle ( 84.01,216.81);
\end{scope}
\begin{scope}
\path[clip] (  2.75,  5.60) rectangle ( 84.01,210.92);
\definecolor{fillColor}{RGB}{255,255,255}

\path[fill=fillColor] (  2.75,  5.60) rectangle ( 84.01,210.92);
\definecolor{drawColor}{gray}{0.92}

\path[draw=drawColor,line width= 0.6pt,line join=round] (  2.75, 13.35) --
	( 84.01, 13.35);

\path[draw=drawColor,line width= 0.6pt,line join=round] (  2.75, 44.98) --
	( 84.01, 44.98);

\path[draw=drawColor,line width= 0.6pt,line join=round] (  2.75, 76.62) --
	( 84.01, 76.62);

\path[draw=drawColor,line width= 0.6pt,line join=round] (  2.75,108.26) --
	( 84.01,108.26);

\path[draw=drawColor,line width= 0.6pt,line join=round] (  2.75,139.89) --
	( 84.01,139.89);

\path[draw=drawColor,line width= 0.6pt,line join=round] (  2.75,171.53) --
	( 84.01,171.53);

\path[draw=drawColor,line width= 0.6pt,line join=round] (  2.75,203.17) --
	( 84.01,203.17);

\path[draw=drawColor,line width= 0.6pt,line join=round] ( 80.32,  5.60) --
	( 80.32,210.92);

\path[draw=drawColor,line width= 0.6pt,line join=round] ( 13.83,  5.60) --
	( 13.83,210.92);
\definecolor{fillColor}{RGB}{128,128,128}

\path[fill=fillColor] ( 14.10, 34.44) rectangle ( 80.32, 42.87);
\definecolor{fillColor}{RGB}{187,187,187}

\path[fill=fillColor] ( 13.71, 24.95) rectangle ( 80.32, 33.38);
\definecolor{fillColor}{RGB}{230,230,230}

\path[fill=fillColor] ( 13.80, 15.46) rectangle ( 80.32, 23.89);
\definecolor{fillColor}{RGB}{128,128,128}

\path[fill=fillColor] ( 14.56, 66.07) rectangle ( 80.32, 74.51);
\definecolor{fillColor}{RGB}{187,187,187}

\path[fill=fillColor] ( 14.42, 56.58) rectangle ( 80.32, 65.02);
\definecolor{fillColor}{RGB}{230,230,230}

\path[fill=fillColor] ( 13.71, 47.09) rectangle ( 80.32, 55.53);
\definecolor{fillColor}{RGB}{128,128,128}

\path[fill=fillColor] ( 14.67, 97.71) rectangle ( 80.32,106.15);
\definecolor{fillColor}{RGB}{187,187,187}

\path[fill=fillColor] ( 14.43, 88.22) rectangle ( 80.32, 96.66);
\definecolor{fillColor}{RGB}{230,230,230}

\path[fill=fillColor] ( 14.55, 78.73) rectangle ( 80.32, 87.17);
\definecolor{fillColor}{RGB}{128,128,128}

\path[fill=fillColor] ( 14.19,129.35) rectangle ( 80.32,137.79);
\definecolor{fillColor}{RGB}{187,187,187}

\path[fill=fillColor] ( 14.15,119.86) rectangle ( 80.32,128.29);
\definecolor{fillColor}{RGB}{230,230,230}

\path[fill=fillColor] ( 14.08,110.37) rectangle ( 80.32,118.80);
\definecolor{fillColor}{RGB}{128,128,128}

\path[fill=fillColor] ( 14.58,160.99) rectangle ( 80.32,169.42);
\definecolor{fillColor}{RGB}{187,187,187}

\path[fill=fillColor] ( 15.21,151.50) rectangle ( 80.32,159.93);
\definecolor{fillColor}{RGB}{230,230,230}

\path[fill=fillColor] ( 14.93,142.00) rectangle ( 80.32,150.44);
\definecolor{fillColor}{RGB}{128,128,128}

\path[fill=fillColor] ( 14.76,192.62) rectangle ( 80.32,201.06);
\definecolor{fillColor}{RGB}{187,187,187}

\path[fill=fillColor] ( 14.97,183.13) rectangle ( 80.32,191.57);
\definecolor{fillColor}{RGB}{230,230,230}

\path[fill=fillColor] ( 14.55,173.64) rectangle ( 80.32,182.08);
\definecolor{drawColor}{RGB}{0,0,0}

\node[text=drawColor,anchor=base west,inner sep=0pt, outer sep=0pt, scale=  0.57] at ( 14.10, 36.70) {0.8963};

\node[text=drawColor,anchor=base west,inner sep=0pt, outer sep=0pt, scale=  0.57] at ( 13.71, 27.21) {0.9017};

\node[text=drawColor,anchor=base west,inner sep=0pt, outer sep=0pt, scale=  0.57] at ( 13.80, 17.71) {0.9004};

\node[text=drawColor,anchor=base west,inner sep=0pt, outer sep=0pt, scale=  0.57] at ( 14.56, 68.33) {0.8901};

\node[text=drawColor,anchor=base west,inner sep=0pt, outer sep=0pt, scale=  0.57] at ( 14.42, 58.84) {0.892};

\node[text=drawColor,anchor=base west,inner sep=0pt, outer sep=0pt, scale=  0.57] at ( 13.71, 49.35) {0.9016};

\node[text=drawColor,anchor=base west,inner sep=0pt, outer sep=0pt, scale=  0.57] at ( 14.67, 99.97) {0.8887};

\node[text=drawColor,anchor=base west,inner sep=0pt, outer sep=0pt, scale=  0.57] at ( 14.43, 90.48) {0.8919};

\node[text=drawColor,anchor=base west,inner sep=0pt, outer sep=0pt, scale=  0.57] at ( 14.55, 80.99) {0.8903};

\node[text=drawColor,anchor=base west,inner sep=0pt, outer sep=0pt, scale=  0.57] at ( 14.19,131.61) {0.8951};

\node[text=drawColor,anchor=base west,inner sep=0pt, outer sep=0pt, scale=  0.57] at ( 14.15,122.12) {0.8957};

\node[text=drawColor,anchor=base west,inner sep=0pt, outer sep=0pt, scale=  0.57] at ( 14.08,112.63) {0.8966};

\node[text=drawColor,anchor=base west,inner sep=0pt, outer sep=0pt, scale=  0.57] at ( 14.58,163.24) {0.8899};

\node[text=drawColor,anchor=base west,inner sep=0pt, outer sep=0pt, scale=  0.57] at ( 15.21,153.75) {0.8814};

\node[text=drawColor,anchor=base west,inner sep=0pt, outer sep=0pt, scale=  0.57] at ( 14.93,144.26) {0.8851};

\node[text=drawColor,anchor=base west,inner sep=0pt, outer sep=0pt, scale=  0.57] at ( 14.76,194.88) {0.8874};

\node[text=drawColor,anchor=base west,inner sep=0pt, outer sep=0pt, scale=  0.57] at ( 14.97,185.39) {0.8846};

\node[text=drawColor,anchor=base west,inner sep=0pt, outer sep=0pt, scale=  0.57] at ( 14.55,175.90) {0.8903};
\end{scope}
\begin{scope}
\path[clip] (  0.00,  0.00) rectangle (224.04,216.81);
\definecolor{drawColor}{RGB}{0,0,0}

\node[text=drawColor,anchor=base,inner sep=0pt, outer sep=0pt, scale=  0.80] at ( 43.38,208.45) {Coverage prob.};
\end{scope}
\begin{scope}
\path[clip] ( 84.01,  0.00) rectangle (140.02,216.81);
\definecolor{drawColor}{RGB}{255,255,255}
\definecolor{fillColor}{RGB}{255,255,255}

\path[draw=drawColor,line width= 0.6pt,line join=round,line cap=round,fill=fillColor] ( 84.01,  0.00) rectangle (140.02,216.81);
\end{scope}
\begin{scope}
\path[clip] ( 83.92,  5.60) rectangle (142.87,210.92);
\definecolor{drawColor}{RGB}{0,0,0}

\path[draw=drawColor,line width= 0.6pt,line join=round] ( 86.60, 35.65) -- (124.88, 35.65);

\path[draw=drawColor,line width= 0.6pt,line join=round] ( 86.60, 23.02) -- (124.88, 23.02);
\definecolor{drawColor}{RGB}{190,190,190}

\path[draw=drawColor,line width= 0.6pt,dash pattern=on 4pt off 4pt ,line join=round] ( 86.60, 23.02) -- ( 86.60, 35.65);
\definecolor{drawColor}{RGB}{0,0,0}

\path[draw=drawColor,line width= 0.6pt,line join=round] ( 86.60, 64.06) -- (105.74, 64.06);

\path[draw=drawColor,line width= 0.6pt,line join=round] ( 86.60, 57.75) -- (105.74, 57.75);

\path[draw=drawColor,line width= 0.6pt,line join=round] ( 86.60, 70.37) -- (105.74, 70.37);

\path[draw=drawColor,line width= 0.6pt,line join=round] ( 86.60, 51.43) -- (105.74, 51.43);
\definecolor{drawColor}{RGB}{190,190,190}

\path[draw=drawColor,line width= 0.6pt,dash pattern=on 4pt off 4pt ,line join=round] ( 86.60, 51.43) -- ( 86.60, 70.37);
\definecolor{drawColor}{RGB}{0,0,0}

\path[draw=drawColor,line width= 0.6pt,line join=round] ( 86.60, 98.79) -- ( 99.36, 98.79);

\path[draw=drawColor,line width= 0.6pt,line join=round] ( 86.60, 86.16) -- ( 99.36, 86.16);

\path[draw=drawColor,line width= 0.6pt,line join=round] ( 99.36, 95.63) -- (112.12, 95.63);

\path[draw=drawColor,line width= 0.6pt,line join=round] ( 99.36, 89.32) -- (112.12, 89.32);

\path[draw=drawColor,line width= 0.6pt,line join=round] ( 99.36,101.94) -- (112.12,101.94);

\path[draw=drawColor,line width= 0.6pt,line join=round] ( 99.36, 83.00) -- (112.12, 83.00);
\definecolor{drawColor}{RGB}{190,190,190}

\path[draw=drawColor,line width= 0.6pt,dash pattern=on 4pt off 4pt ,line join=round] ( 86.60, 86.16) -- ( 86.60, 98.79);

\path[draw=drawColor,line width= 0.6pt,dash pattern=on 4pt off 4pt ,line join=round] ( 99.36, 83.00) -- ( 99.36,101.94);
\definecolor{drawColor}{RGB}{0,0,0}

\path[draw=drawColor,line width= 0.6pt,line join=round] ( 86.60,124.04) -- (112.12,124.04);

\path[draw=drawColor,line width= 0.6pt,line join=round] (112.12,130.36) -- (137.64,130.36);

\path[draw=drawColor,line width= 0.6pt,line join=round] (112.12,117.73) -- (137.64,117.73);
\definecolor{drawColor}{RGB}{190,190,190}

\path[draw=drawColor,line width= 0.6pt,dash pattern=on 4pt off 4pt ,line join=round] (112.12,117.73) -- (112.12,130.36);
\definecolor{drawColor}{RGB}{0,0,0}

\path[draw=drawColor,line width= 0.6pt,line join=round] ( 86.60,155.61) -- (101.91,155.61);

\path[draw=drawColor,line width= 0.6pt,line join=round] (101.91,158.77) -- (117.22,158.77);

\path[draw=drawColor,line width= 0.6pt,line join=round] (101.91,152.45) -- (117.22,152.45);

\path[draw=drawColor,line width= 0.6pt,line join=round] (101.91,165.08) -- (117.22,165.08);

\path[draw=drawColor,line width= 0.6pt,line join=round] (101.91,146.14) -- (117.22,146.14);
\definecolor{drawColor}{RGB}{190,190,190}

\path[draw=drawColor,line width= 0.6pt,dash pattern=on 4pt off 4pt ,line join=round] (101.91,146.14) -- (101.91,165.08);
\definecolor{drawColor}{RGB}{0,0,0}

\path[draw=drawColor,line width= 0.6pt,line join=round] ( 86.60,187.18) -- ( 97.54,187.18);

\path[draw=drawColor,line width= 0.6pt,line join=round] ( 97.54,193.49) -- (108.47,193.49);

\path[draw=drawColor,line width= 0.6pt,line join=round] ( 97.54,180.87) -- (108.47,180.87);

\path[draw=drawColor,line width= 0.6pt,line join=round] (108.47,190.34) -- (119.41,190.34);

\path[draw=drawColor,line width= 0.6pt,line join=round] (108.47,184.02) -- (119.41,184.02);

\path[draw=drawColor,line width= 0.6pt,line join=round] (108.47,196.65) -- (119.41,196.65);

\path[draw=drawColor,line width= 0.6pt,line join=round] (108.47,177.71) -- (119.41,177.71);
\definecolor{drawColor}{RGB}{190,190,190}

\path[draw=drawColor,line width= 0.6pt,dash pattern=on 4pt off 4pt ,line join=round] ( 97.54,180.87) -- ( 97.54,193.49);

\path[draw=drawColor,line width= 0.6pt,dash pattern=on 4pt off 4pt ,line join=round] (108.47,177.71) -- (108.47,196.65);
\end{scope}
\begin{scope}
\path[clip] (  0.00,  0.00) rectangle (224.04,216.81);
\definecolor{drawColor}{RGB}{0,0,0}

\node[text=drawColor,anchor=base,inner sep=0pt, outer sep=0pt, scale=  0.80] at (113.39,208.45) {Config.};
\end{scope}
\begin{scope}
\path[clip] (140.02,  0.00) rectangle (224.04,216.81);
\definecolor{drawColor}{RGB}{255,255,255}
\definecolor{fillColor}{RGB}{255,255,255}

\path[draw=drawColor,line width= 0.6pt,line join=round,line cap=round,fill=fillColor] (140.02,  0.00) rectangle (224.04,216.81);
\end{scope}
\begin{scope}
\path[clip] (139.93,  5.60) rectangle (224.04,210.92);
\definecolor{fillColor}{RGB}{255,255,255}

\path[fill=fillColor] (139.93,  5.60) rectangle (224.04,210.92);
\definecolor{drawColor}{gray}{0.92}

\path[draw=drawColor,line width= 0.6pt,line join=round] (139.93, 13.35) --
	(224.04, 13.35);

\path[draw=drawColor,line width= 0.6pt,line join=round] (139.93, 44.98) --
	(224.04, 44.98);

\path[draw=drawColor,line width= 0.6pt,line join=round] (139.93, 76.62) --
	(224.04, 76.62);

\path[draw=drawColor,line width= 0.6pt,line join=round] (139.93,108.26) --
	(224.04,108.26);

\path[draw=drawColor,line width= 0.6pt,line join=round] (139.93,139.89) --
	(224.04,139.89);

\path[draw=drawColor,line width= 0.6pt,line join=round] (139.93,171.53) --
	(224.04,171.53);

\path[draw=drawColor,line width= 0.6pt,line join=round] (139.93,203.17) --
	(224.04,203.17);

\path[draw=drawColor,line width= 0.6pt,line join=round] (143.75,  5.60) --
	(143.75,210.92);
\definecolor{fillColor}{RGB}{128,128,128}

\path[fill=fillColor] (143.75, 34.44) rectangle (217.57, 42.87);
\definecolor{fillColor}{RGB}{187,187,187}

\path[fill=fillColor] (143.75, 24.95) rectangle (219.46, 33.38);
\definecolor{fillColor}{RGB}{230,230,230}

\path[fill=fillColor] (143.75, 15.46) rectangle (210.13, 23.89);
\definecolor{fillColor}{RGB}{128,128,128}

\path[fill=fillColor] (143.75, 66.07) rectangle (175.70, 74.51);
\definecolor{fillColor}{RGB}{187,187,187}

\path[fill=fillColor] (143.75, 56.58) rectangle (176.43, 65.02);
\definecolor{fillColor}{RGB}{230,230,230}

\path[fill=fillColor] (143.75, 47.09) rectangle (172.76, 55.53);
\definecolor{fillColor}{RGB}{128,128,128}

\path[fill=fillColor] (143.75, 97.71) rectangle (174.46,106.15);
\definecolor{fillColor}{RGB}{187,187,187}

\path[fill=fillColor] (143.75, 88.22) rectangle (174.87, 96.66);
\definecolor{fillColor}{RGB}{230,230,230}

\path[fill=fillColor] (143.75, 78.73) rectangle (171.74, 87.17);
\definecolor{fillColor}{RGB}{128,128,128}

\path[fill=fillColor] (143.75,129.35) rectangle (215.46,137.79);
\definecolor{fillColor}{RGB}{187,187,187}

\path[fill=fillColor] (143.75,119.86) rectangle (215.92,128.29);
\definecolor{fillColor}{RGB}{230,230,230}

\path[fill=fillColor] (143.75,110.37) rectangle (209.53,118.80);
\definecolor{fillColor}{RGB}{128,128,128}

\path[fill=fillColor] (143.75,160.99) rectangle (174.41,169.42);
\definecolor{fillColor}{RGB}{187,187,187}

\path[fill=fillColor] (143.75,151.50) rectangle (174.55,159.93);
\definecolor{fillColor}{RGB}{230,230,230}

\path[fill=fillColor] (143.75,142.00) rectangle (171.42,150.44);
\definecolor{fillColor}{RGB}{128,128,128}

\path[fill=fillColor] (143.75,192.62) rectangle (174.32,201.06);
\definecolor{fillColor}{RGB}{187,187,187}

\path[fill=fillColor] (143.75,183.13) rectangle (174.46,191.57);
\definecolor{fillColor}{RGB}{230,230,230}

\path[fill=fillColor] (143.75,173.64) rectangle (171.65,182.08);
\definecolor{drawColor}{RGB}{0,0,0}

\node[text=drawColor,anchor=base east,inner sep=0pt, outer sep=0pt, scale=  0.57] at (217.57, 36.70) {0.1606};

\node[text=drawColor,anchor=base east,inner sep=0pt, outer sep=0pt, scale=  0.57] at (219.46, 27.21) {0.1647};

\node[text=drawColor,anchor=base east,inner sep=0pt, outer sep=0pt, scale=  0.57] at (210.13, 17.71) {0.1444};

\node[text=drawColor,anchor=base east,inner sep=0pt, outer sep=0pt, scale=  0.57] at (175.70, 68.33) {0.0695};

\node[text=drawColor,anchor=base east,inner sep=0pt, outer sep=0pt, scale=  0.57] at (176.43, 58.84) {0.0711};

\node[text=drawColor,anchor=base east,inner sep=0pt, outer sep=0pt, scale=  0.57] at (172.76, 49.35) {0.0631};

\node[text=drawColor,anchor=base east,inner sep=0pt, outer sep=0pt, scale=  0.57] at (174.46, 99.97) {0.0668};

\node[text=drawColor,anchor=base east,inner sep=0pt, outer sep=0pt, scale=  0.57] at (174.87, 90.48) {0.0677};

\node[text=drawColor,anchor=base east,inner sep=0pt, outer sep=0pt, scale=  0.57] at (171.74, 80.99) {0.0609};

\node[text=drawColor,anchor=base east,inner sep=0pt, outer sep=0pt, scale=  0.57] at (215.46,131.61) {0.156};

\node[text=drawColor,anchor=base east,inner sep=0pt, outer sep=0pt, scale=  0.57] at (215.92,122.12) {0.157};

\node[text=drawColor,anchor=base east,inner sep=0pt, outer sep=0pt, scale=  0.57] at (209.53,112.63) {0.1431};

\node[text=drawColor,anchor=base east,inner sep=0pt, outer sep=0pt, scale=  0.57] at (174.41,163.24) {0.0667};

\node[text=drawColor,anchor=base east,inner sep=0pt, outer sep=0pt, scale=  0.57] at (174.55,153.75) {0.067};

\node[text=drawColor,anchor=base east,inner sep=0pt, outer sep=0pt, scale=  0.57] at (171.42,144.26) {0.0602};

\node[text=drawColor,anchor=base east,inner sep=0pt, outer sep=0pt, scale=  0.57] at (174.32,194.88) {0.0665};

\node[text=drawColor,anchor=base east,inner sep=0pt, outer sep=0pt, scale=  0.57] at (174.46,185.39) {0.0668};

\node[text=drawColor,anchor=base east,inner sep=0pt, outer sep=0pt, scale=  0.57] at (171.65,175.90) {0.0607};
\end{scope}
\begin{scope}
\path[clip] (  0.00,  0.00) rectangle (224.04,216.81);
\definecolor{drawColor}{RGB}{0,0,0}

\node[text=drawColor,anchor=base,inner sep=0pt, outer sep=0pt, scale=  0.80] at (181.98,208.45) {CI length};
\end{scope}
\end{tikzpicture}

%% file: figs/legend_simulation.tex
\begin{tikzpicture}[x=1pt,y=1pt]
\definecolor{fillColor}{RGB}{255,255,255}
\path[use as bounding box,fill=fillColor,fill opacity=0.00] (0,0) rectangle (216.81, 10.84);
\begin{scope}
\path[clip] (  0.00,  0.00) rectangle (216.81, 10.84);
\definecolor{fillColor}{RGB}{230,230,230}

\path[fill=fillColor] (  8.03,  0.40) --
	( 35.40,  0.40) --
	( 35.40, 10.44) --
	(  8.03, 10.44) --
	cycle;
\definecolor{drawColor}{RGB}{0,0,0}

\node[text=drawColor,anchor=base west,inner sep=0pt, outer sep=0pt, scale=  0.70] at ( 41.40,  3.81) {null};
\definecolor{fillColor}{RGB}{187,187,187}

\path[fill=fillColor] ( 81.03,  0.40) --
	(108.40,  0.40) --
	(108.40, 10.44) --
	( 81.03, 10.44) --
	cycle;

\node[text=drawColor,anchor=base west,inner sep=0pt, outer sep=0pt, scale=  0.70] at (114.40,  3.81) {sparse};
\definecolor{fillColor}{RGB}{128,128,128}

\path[fill=fillColor] (154.03,  0.40) --
	(181.40,  0.40) --
	(181.40, 10.44) --
	(154.03, 10.44) --
	cycle;

\node[text=drawColor,anchor=base west,inner sep=0pt, outer sep=0pt, scale=  0.70] at (187.40,  3.81) {dense};
\end{scope}
\end{tikzpicture}

%% file: figs/fig_simulation_lm_add.tex
\begin{tikzpicture}[x=1pt,y=1pt]
\definecolor{fillColor}{RGB}{255,255,255}
\path[use as bounding box,fill=fillColor,fill opacity=0.00] (0,0) rectangle (224.04,216.81);
\begin{scope}
\path[clip] (  0.00,  0.00) rectangle ( 84.01,216.81);
\definecolor{drawColor}{RGB}{255,255,255}
\definecolor{fillColor}{RGB}{255,255,255}

\path[draw=drawColor,line width= 0.6pt,line join=round,line cap=round,fill=fillColor] (  0.00,  0.00) rectangle ( 84.01,216.81);
\end{scope}
\begin{scope}
\path[clip] (  2.75,  5.60) rectangle ( 84.01,211.77);
\definecolor{fillColor}{RGB}{255,255,255}

\path[fill=fillColor] (  2.75,  5.60) rectangle ( 84.01,211.77);
\definecolor{drawColor}{gray}{0.92}

\path[draw=drawColor,line width= 0.6pt,line join=round] (  2.75, 13.38) --
	( 84.01, 13.38);

\path[draw=drawColor,line width= 0.6pt,line join=round] (  2.75, 45.15) --
	( 84.01, 45.15);

\path[draw=drawColor,line width= 0.6pt,line join=round] (  2.75, 76.91) --
	( 84.01, 76.91);

\path[draw=drawColor,line width= 0.6pt,line join=round] (  2.75,108.68) --
	( 84.01,108.68);

\path[draw=drawColor,line width= 0.6pt,line join=round] (  2.75,140.45) --
	( 84.01,140.45);

\path[draw=drawColor,line width= 0.6pt,line join=round] (  2.75,172.21) --
	( 84.01,172.21);

\path[draw=drawColor,line width= 0.6pt,line join=round] (  2.75,203.98) --
	( 84.01,203.98);

\path[draw=drawColor,line width= 0.6pt,line join=round] ( 80.32,  5.60) --
	( 80.32,211.77);

\path[draw=drawColor,line width= 0.6pt,line join=round] ( 13.83,  5.60) --
	( 13.83,211.77);
\definecolor{fillColor}{RGB}{128,128,128}

\path[fill=fillColor] ( 13.77, 39.06) rectangle ( 80.32, 43.29);
\definecolor{fillColor}{RGB}{155,155,155}

\path[fill=fillColor] ( 13.78, 34.29) rectangle ( 80.32, 38.53);
\definecolor{fillColor}{RGB}{177,177,177}

\path[fill=fillColor] ( 13.73, 29.53) rectangle ( 80.32, 33.76);
\definecolor{fillColor}{gray}{0.77}

\path[fill=fillColor] ( 13.74, 24.76) rectangle ( 80.32, 29.00);
\definecolor{fillColor}{gray}{0.84}

\path[fill=fillColor] ( 13.60, 20.00) rectangle ( 80.32, 24.23);
\definecolor{fillColor}{RGB}{230,230,230}

\path[fill=fillColor] ( 13.87, 15.23) rectangle ( 80.32, 19.47);
\definecolor{fillColor}{RGB}{128,128,128}

\path[fill=fillColor] ( 13.97, 70.82) rectangle ( 80.32, 75.06);
\definecolor{fillColor}{RGB}{155,155,155}

\path[fill=fillColor] ( 13.74, 66.06) rectangle ( 80.32, 70.29);
\definecolor{fillColor}{RGB}{177,177,177}

\path[fill=fillColor] ( 13.85, 61.29) rectangle ( 80.32, 65.53);
\definecolor{fillColor}{gray}{0.77}

\path[fill=fillColor] ( 13.99, 56.53) rectangle ( 80.32, 60.76);
\definecolor{fillColor}{gray}{0.84}

\path[fill=fillColor] ( 13.74, 51.76) rectangle ( 80.32, 56.00);
\definecolor{fillColor}{RGB}{230,230,230}

\path[fill=fillColor] ( 13.95, 47.00) rectangle ( 80.32, 51.23);
\definecolor{fillColor}{RGB}{128,128,128}

\path[fill=fillColor] ( 13.99,102.59) rectangle ( 80.32,106.83);
\definecolor{fillColor}{RGB}{155,155,155}

\path[fill=fillColor] ( 13.88, 97.83) rectangle ( 80.32,102.06);
\definecolor{fillColor}{RGB}{177,177,177}

\path[fill=fillColor] ( 14.37, 93.06) rectangle ( 80.32, 97.30);
\definecolor{fillColor}{gray}{0.77}

\path[fill=fillColor] ( 14.26, 88.30) rectangle ( 80.32, 92.53);
\definecolor{fillColor}{gray}{0.84}

\path[fill=fillColor] ( 14.32, 83.53) rectangle ( 80.32, 87.77);
\definecolor{fillColor}{RGB}{230,230,230}

\path[fill=fillColor] ( 14.19, 78.77) rectangle ( 80.32, 83.00);
\definecolor{fillColor}{RGB}{128,128,128}

\path[fill=fillColor] ( 13.85,134.36) rectangle ( 80.32,138.59);
\definecolor{fillColor}{RGB}{155,155,155}

\path[fill=fillColor] ( 13.81,129.59) rectangle ( 80.32,133.83);
\definecolor{fillColor}{RGB}{177,177,177}

\path[fill=fillColor] ( 14.14,124.83) rectangle ( 80.32,129.06);
\definecolor{fillColor}{gray}{0.77}

\path[fill=fillColor] ( 13.94,120.06) rectangle ( 80.32,124.30);
\definecolor{fillColor}{gray}{0.84}

\path[fill=fillColor] ( 13.96,115.30) rectangle ( 80.32,119.53);
\definecolor{fillColor}{RGB}{230,230,230}

\path[fill=fillColor] ( 14.04,110.53) rectangle ( 80.32,114.77);
\definecolor{fillColor}{RGB}{128,128,128}

\path[fill=fillColor] ( 14.14,166.13) rectangle ( 80.32,170.36);
\definecolor{fillColor}{RGB}{155,155,155}

\path[fill=fillColor] ( 13.96,161.36) rectangle ( 80.32,165.60);
\definecolor{fillColor}{RGB}{177,177,177}

\path[fill=fillColor] ( 14.87,156.60) rectangle ( 80.32,160.83);
\definecolor{fillColor}{gray}{0.77}

\path[fill=fillColor] ( 14.37,151.83) rectangle ( 80.32,156.07);
\definecolor{fillColor}{gray}{0.84}

\path[fill=fillColor] ( 14.91,147.07) rectangle ( 80.32,151.30);
\definecolor{fillColor}{RGB}{230,230,230}

\path[fill=fillColor] ( 14.53,142.30) rectangle ( 80.32,146.54);
\definecolor{fillColor}{RGB}{128,128,128}

\path[fill=fillColor] ( 13.98,197.89) rectangle ( 80.32,202.13);
\definecolor{fillColor}{RGB}{155,155,155}

\path[fill=fillColor] ( 14.08,193.13) rectangle ( 80.32,197.36);
\definecolor{fillColor}{RGB}{177,177,177}

\path[fill=fillColor] ( 14.79,188.36) rectangle ( 80.32,192.60);
\definecolor{fillColor}{gray}{0.77}

\path[fill=fillColor] ( 14.50,183.60) rectangle ( 80.32,187.83);
\definecolor{fillColor}{gray}{0.84}

\path[fill=fillColor] ( 14.73,178.83) rectangle ( 80.32,183.07);
\definecolor{fillColor}{RGB}{230,230,230}

\path[fill=fillColor] ( 14.51,174.07) rectangle ( 80.32,178.30);
\definecolor{drawColor}{RGB}{0,0,0}

\node[text=drawColor,anchor=base west,inner sep=0pt, outer sep=0pt, scale=  0.57] at ( 13.77, 39.21) {0.9008};

\node[text=drawColor,anchor=base west,inner sep=0pt, outer sep=0pt, scale=  0.57] at ( 13.78, 34.45) {0.9007};

\node[text=drawColor,anchor=base west,inner sep=0pt, outer sep=0pt, scale=  0.57] at ( 13.73, 29.68) {0.9014};

\node[text=drawColor,anchor=base west,inner sep=0pt, outer sep=0pt, scale=  0.57] at ( 13.74, 24.92) {0.9013};

\node[text=drawColor,anchor=base west,inner sep=0pt, outer sep=0pt, scale=  0.57] at ( 13.60, 20.15) {0.9031};

\node[text=drawColor,anchor=base west,inner sep=0pt, outer sep=0pt, scale=  0.57] at ( 13.87, 15.39) {0.8995};

\node[text=drawColor,anchor=base west,inner sep=0pt, outer sep=0pt, scale=  0.57] at ( 13.97, 70.98) {0.8981};

\node[text=drawColor,anchor=base west,inner sep=0pt, outer sep=0pt, scale=  0.57] at ( 13.74, 66.22) {0.9013};

\node[text=drawColor,anchor=base west,inner sep=0pt, outer sep=0pt, scale=  0.57] at ( 13.85, 61.45) {0.8998};

\node[text=drawColor,anchor=base west,inner sep=0pt, outer sep=0pt, scale=  0.57] at ( 13.99, 56.69) {0.8979};

\node[text=drawColor,anchor=base west,inner sep=0pt, outer sep=0pt, scale=  0.57] at ( 13.74, 51.92) {0.9012};

\node[text=drawColor,anchor=base west,inner sep=0pt, outer sep=0pt, scale=  0.57] at ( 13.95, 47.16) {0.8984};

\node[text=drawColor,anchor=base west,inner sep=0pt, outer sep=0pt, scale=  0.57] at ( 13.99,102.75) {0.8978};

\node[text=drawColor,anchor=base west,inner sep=0pt, outer sep=0pt, scale=  0.57] at ( 13.88, 97.98) {0.8993};

\node[text=drawColor,anchor=base west,inner sep=0pt, outer sep=0pt, scale=  0.57] at ( 14.37, 93.22) {0.8927};

\node[text=drawColor,anchor=base west,inner sep=0pt, outer sep=0pt, scale=  0.57] at ( 14.26, 88.45) {0.8942};

\node[text=drawColor,anchor=base west,inner sep=0pt, outer sep=0pt, scale=  0.57] at ( 14.32, 83.69) {0.8934};

\node[text=drawColor,anchor=base west,inner sep=0pt, outer sep=0pt, scale=  0.57] at ( 14.19, 78.92) {0.8952};

\node[text=drawColor,anchor=base west,inner sep=0pt, outer sep=0pt, scale=  0.57] at ( 13.85,134.52) {0.8998};

\node[text=drawColor,anchor=base west,inner sep=0pt, outer sep=0pt, scale=  0.57] at ( 13.81,129.75) {0.9003};

\node[text=drawColor,anchor=base west,inner sep=0pt, outer sep=0pt, scale=  0.57] at ( 14.14,124.99) {0.8958};

\node[text=drawColor,anchor=base west,inner sep=0pt, outer sep=0pt, scale=  0.57] at ( 13.94,120.22) {0.8985};

\node[text=drawColor,anchor=base west,inner sep=0pt, outer sep=0pt, scale=  0.57] at ( 13.96,115.46) {0.8982};

\node[text=drawColor,anchor=base west,inner sep=0pt, outer sep=0pt, scale=  0.57] at ( 14.04,110.69) {0.8972};

\node[text=drawColor,anchor=base west,inner sep=0pt, outer sep=0pt, scale=  0.57] at ( 14.14,166.28) {0.8958};

\node[text=drawColor,anchor=base west,inner sep=0pt, outer sep=0pt, scale=  0.57] at ( 13.96,161.52) {0.8983};

\node[text=drawColor,anchor=base west,inner sep=0pt, outer sep=0pt, scale=  0.57] at ( 14.87,156.75) {0.886};

\node[text=drawColor,anchor=base west,inner sep=0pt, outer sep=0pt, scale=  0.57] at ( 14.37,151.99) {0.8927};

\node[text=drawColor,anchor=base west,inner sep=0pt, outer sep=0pt, scale=  0.57] at ( 14.91,147.22) {0.8854};

\node[text=drawColor,anchor=base west,inner sep=0pt, outer sep=0pt, scale=  0.57] at ( 14.53,142.46) {0.8905};

\node[text=drawColor,anchor=base west,inner sep=0pt, outer sep=0pt, scale=  0.57] at ( 13.98,198.05) {0.898};

\node[text=drawColor,anchor=base west,inner sep=0pt, outer sep=0pt, scale=  0.57] at ( 14.08,193.29) {0.8967};

\node[text=drawColor,anchor=base west,inner sep=0pt, outer sep=0pt, scale=  0.57] at ( 14.79,188.52) {0.887};

\node[text=drawColor,anchor=base west,inner sep=0pt, outer sep=0pt, scale=  0.57] at ( 14.50,183.76) {0.8909};

\node[text=drawColor,anchor=base west,inner sep=0pt, outer sep=0pt, scale=  0.57] at ( 14.73,178.99) {0.8878};

\node[text=drawColor,anchor=base west,inner sep=0pt, outer sep=0pt, scale=  0.57] at ( 14.51,174.23) {0.8908};
\end{scope}
\begin{scope}
\path[clip] (  0.00,  0.00) rectangle (224.04,216.81);
\definecolor{drawColor}{RGB}{0,0,0}

\node[text=drawColor,anchor=base,inner sep=0pt, outer sep=0pt, scale=  0.80] at ( 43.38,208.45) {Coverage prob.};
\end{scope}
\begin{scope}
\path[clip] ( 84.01,  0.00) rectangle (140.02,216.81);
\definecolor{drawColor}{RGB}{255,255,255}
\definecolor{fillColor}{RGB}{255,255,255}

\path[draw=drawColor,line width= 0.6pt,line join=round,line cap=round,fill=fillColor] ( 84.01,  0.00) rectangle (140.02,216.81);
\end{scope}
\begin{scope}
\path[clip] ( 83.92,  5.60) rectangle (142.87,211.77);
\definecolor{drawColor}{RGB}{0,0,0}

\path[draw=drawColor,line width= 0.6pt,line join=round] ( 86.60, 35.77) -- (124.88, 35.77);

\path[draw=drawColor,line width= 0.6pt,line join=round] ( 86.60, 23.09) -- (124.88, 23.09);
\definecolor{drawColor}{RGB}{190,190,190}

\path[draw=drawColor,line width= 0.6pt,dash pattern=on 4pt off 4pt ,line join=round] ( 86.60, 23.09) -- ( 86.60, 35.77);
\definecolor{drawColor}{RGB}{0,0,0}

\path[draw=drawColor,line width= 0.6pt,line join=round] ( 86.60, 64.30) -- (105.74, 64.30);

\path[draw=drawColor,line width= 0.6pt,line join=round] ( 86.60, 57.96) -- (105.74, 57.96);

\path[draw=drawColor,line width= 0.6pt,line join=round] ( 86.60, 70.64) -- (105.74, 70.64);

\path[draw=drawColor,line width= 0.6pt,line join=round] ( 86.60, 51.62) -- (105.74, 51.62);
\definecolor{drawColor}{RGB}{190,190,190}

\path[draw=drawColor,line width= 0.6pt,dash pattern=on 4pt off 4pt ,line join=round] ( 86.60, 51.62) -- ( 86.60, 70.64);
\definecolor{drawColor}{RGB}{0,0,0}

\path[draw=drawColor,line width= 0.6pt,line join=round] ( 86.60, 99.17) -- ( 99.36, 99.17);

\path[draw=drawColor,line width= 0.6pt,line join=round] ( 86.60, 86.49) -- ( 99.36, 86.49);

\path[draw=drawColor,line width= 0.6pt,line join=round] ( 99.36, 96.00) -- (112.12, 96.00);

\path[draw=drawColor,line width= 0.6pt,line join=round] ( 99.36, 89.66) -- (112.12, 89.66);

\path[draw=drawColor,line width= 0.6pt,line join=round] ( 99.36,102.34) -- (112.12,102.34);

\path[draw=drawColor,line width= 0.6pt,line join=round] ( 99.36, 83.32) -- (112.12, 83.32);
\definecolor{drawColor}{RGB}{190,190,190}

\path[draw=drawColor,line width= 0.6pt,dash pattern=on 4pt off 4pt ,line join=round] ( 86.60, 86.49) -- ( 86.60, 99.17);

\path[draw=drawColor,line width= 0.6pt,dash pattern=on 4pt off 4pt ,line join=round] ( 99.36, 83.32) -- ( 99.36,102.34);
\definecolor{drawColor}{RGB}{0,0,0}

\path[draw=drawColor,line width= 0.6pt,line join=round] ( 86.60,124.53) -- (112.12,124.53);

\path[draw=drawColor,line width= 0.6pt,line join=round] (112.12,130.87) -- (137.64,130.87);

\path[draw=drawColor,line width= 0.6pt,line join=round] (112.12,118.19) -- (137.64,118.19);
\definecolor{drawColor}{RGB}{190,190,190}

\path[draw=drawColor,line width= 0.6pt,dash pattern=on 4pt off 4pt ,line join=round] (112.12,118.19) -- (112.12,130.87);
\definecolor{drawColor}{RGB}{0,0,0}

\path[draw=drawColor,line width= 0.6pt,line join=round] ( 86.60,156.23) -- (101.91,156.23);

\path[draw=drawColor,line width= 0.6pt,line join=round] (101.91,159.40) -- (117.22,159.40);

\path[draw=drawColor,line width= 0.6pt,line join=round] (101.91,153.06) -- (117.22,153.06);

\path[draw=drawColor,line width= 0.6pt,line join=round] (101.91,165.74) -- (117.22,165.74);

\path[draw=drawColor,line width= 0.6pt,line join=round] (101.91,146.72) -- (117.22,146.72);
\definecolor{drawColor}{RGB}{190,190,190}

\path[draw=drawColor,line width= 0.6pt,dash pattern=on 4pt off 4pt ,line join=round] (101.91,146.72) -- (101.91,165.74);
\definecolor{drawColor}{RGB}{0,0,0}

\path[draw=drawColor,line width= 0.6pt,line join=round] ( 86.60,187.93) -- ( 97.54,187.93);

\path[draw=drawColor,line width= 0.6pt,line join=round] ( 97.54,194.27) -- (108.47,194.27);

\path[draw=drawColor,line width= 0.6pt,line join=round] ( 97.54,181.59) -- (108.47,181.59);

\path[draw=drawColor,line width= 0.6pt,line join=round] (108.47,191.10) -- (119.41,191.10);

\path[draw=drawColor,line width= 0.6pt,line join=round] (108.47,184.76) -- (119.41,184.76);

\path[draw=drawColor,line width= 0.6pt,line join=round] (108.47,197.44) -- (119.41,197.44);

\path[draw=drawColor,line width= 0.6pt,line join=round] (108.47,178.42) -- (119.41,178.42);
\definecolor{drawColor}{RGB}{190,190,190}

\path[draw=drawColor,line width= 0.6pt,dash pattern=on 4pt off 4pt ,line join=round] ( 97.54,181.59) -- ( 97.54,194.27);

\path[draw=drawColor,line width= 0.6pt,dash pattern=on 4pt off 4pt ,line join=round] (108.47,178.42) -- (108.47,197.44);
\end{scope}
\begin{scope}
\path[clip] (  0.00,  0.00) rectangle (224.04,216.81);
\definecolor{drawColor}{RGB}{0,0,0}

\node[text=drawColor,anchor=base,inner sep=0pt, outer sep=0pt, scale=  0.80] at (113.39,208.45) {Config};
\end{scope}
\begin{scope}
\path[clip] (140.02,  0.00) rectangle (224.04,216.81);
\definecolor{drawColor}{RGB}{255,255,255}
\definecolor{fillColor}{RGB}{255,255,255}

\path[draw=drawColor,line width= 0.6pt,line join=round,line cap=round,fill=fillColor] (140.02,  0.00) rectangle (224.04,216.81);
\end{scope}
\begin{scope}
\path[clip] (139.93,  5.60) rectangle (224.04,211.77);
\definecolor{fillColor}{RGB}{255,255,255}

\path[fill=fillColor] (139.93,  5.60) rectangle (224.04,211.77);
\definecolor{drawColor}{gray}{0.92}

\path[draw=drawColor,line width= 0.6pt,line join=round] (139.93, 13.38) --
	(224.04, 13.38);

\path[draw=drawColor,line width= 0.6pt,line join=round] (139.93, 45.15) --
	(224.04, 45.15);

\path[draw=drawColor,line width= 0.6pt,line join=round] (139.93, 76.91) --
	(224.04, 76.91);

\path[draw=drawColor,line width= 0.6pt,line join=round] (139.93,108.68) --
	(224.04,108.68);

\path[draw=drawColor,line width= 0.6pt,line join=round] (139.93,140.45) --
	(224.04,140.45);

\path[draw=drawColor,line width= 0.6pt,line join=round] (139.93,172.21) --
	(224.04,172.21);

\path[draw=drawColor,line width= 0.6pt,line join=round] (139.93,203.98) --
	(224.04,203.98);

\path[draw=drawColor,line width= 0.6pt,line join=round] (143.75,  5.60) --
	(143.75,211.77);
\definecolor{fillColor}{RGB}{128,128,128}

\path[fill=fillColor] (143.75, 39.06) rectangle (217.91, 43.29);
\definecolor{fillColor}{RGB}{155,155,155}

\path[fill=fillColor] (143.75, 34.29) rectangle (218.22, 38.53);
\definecolor{fillColor}{RGB}{177,177,177}

\path[fill=fillColor] (143.75, 29.53) rectangle (217.81, 33.76);
\definecolor{fillColor}{gray}{0.77}

\path[fill=fillColor] (143.75, 24.76) rectangle (218.84, 29.00);
\definecolor{fillColor}{gray}{0.84}

\path[fill=fillColor] (143.75, 20.00) rectangle (219.46, 24.23);
\definecolor{fillColor}{RGB}{230,230,230}

\path[fill=fillColor] (143.75, 15.23) rectangle (216.68, 19.47);
\definecolor{fillColor}{RGB}{128,128,128}

\path[fill=fillColor] (143.75, 70.82) rectangle (175.74, 75.06);
\definecolor{fillColor}{RGB}{155,155,155}

\path[fill=fillColor] (143.75, 66.06) rectangle (176.05, 70.29);
\definecolor{fillColor}{RGB}{177,177,177}

\path[fill=fillColor] (143.75, 61.29) rectangle (175.74, 65.53);
\definecolor{fillColor}{gray}{0.77}

\path[fill=fillColor] (143.75, 56.53) rectangle (175.84, 60.76);
\definecolor{fillColor}{gray}{0.84}

\path[fill=fillColor] (143.75, 51.76) rectangle (176.26, 56.00);
\definecolor{fillColor}{RGB}{230,230,230}

\path[fill=fillColor] (143.75, 47.00) rectangle (175.23, 51.23);
\definecolor{fillColor}{RGB}{128,128,128}

\path[fill=fillColor] (143.75,102.59) rectangle (175.23,106.83);
\definecolor{fillColor}{RGB}{155,155,155}

\path[fill=fillColor] (143.75, 97.83) rectangle (175.54,102.06);
\definecolor{fillColor}{RGB}{177,177,177}

\path[fill=fillColor] (143.75, 93.06) rectangle (174.82, 97.30);
\definecolor{fillColor}{gray}{0.77}

\path[fill=fillColor] (143.75, 88.30) rectangle (175.33, 92.53);
\definecolor{fillColor}{gray}{0.84}

\path[fill=fillColor] (143.75, 83.53) rectangle (175.23, 87.77);
\definecolor{fillColor}{RGB}{230,230,230}

\path[fill=fillColor] (143.75, 78.77) rectangle (174.51, 83.00);
\definecolor{fillColor}{RGB}{128,128,128}

\path[fill=fillColor] (143.75,134.36) rectangle (216.99,138.59);
\definecolor{fillColor}{RGB}{155,155,155}

\path[fill=fillColor] (143.75,129.59) rectangle (216.78,133.83);
\definecolor{fillColor}{RGB}{177,177,177}

\path[fill=fillColor] (143.75,124.83) rectangle (215.34,129.06);
\definecolor{fillColor}{gray}{0.77}

\path[fill=fillColor] (143.75,120.06) rectangle (217.09,124.30);
\definecolor{fillColor}{gray}{0.84}

\path[fill=fillColor] (143.75,115.30) rectangle (216.89,119.53);
\definecolor{fillColor}{RGB}{230,230,230}

\path[fill=fillColor] (143.75,110.53) rectangle (215.34,114.77);
\definecolor{fillColor}{RGB}{128,128,128}

\path[fill=fillColor] (143.75,166.13) rectangle (175.02,170.36);
\definecolor{fillColor}{RGB}{155,155,155}

\path[fill=fillColor] (143.75,161.36) rectangle (175.43,165.60);
\definecolor{fillColor}{RGB}{177,177,177}

\path[fill=fillColor] (143.75,156.60) rectangle (174.40,160.83);
\definecolor{fillColor}{gray}{0.77}

\path[fill=fillColor] (143.75,151.83) rectangle (175.02,156.07);
\definecolor{fillColor}{gray}{0.84}

\path[fill=fillColor] (143.75,147.07) rectangle (174.82,151.30);
\definecolor{fillColor}{RGB}{230,230,230}

\path[fill=fillColor] (143.75,142.30) rectangle (174.30,146.54);
\definecolor{fillColor}{RGB}{128,128,128}

\path[fill=fillColor] (143.75,197.89) rectangle (175.02,202.13);
\definecolor{fillColor}{RGB}{155,155,155}

\path[fill=fillColor] (143.75,193.13) rectangle (175.33,197.36);
\definecolor{fillColor}{RGB}{177,177,177}

\path[fill=fillColor] (143.75,188.36) rectangle (174.20,192.60);
\definecolor{fillColor}{gray}{0.77}

\path[fill=fillColor] (143.75,183.60) rectangle (174.92,187.83);
\definecolor{fillColor}{gray}{0.84}

\path[fill=fillColor] (143.75,178.83) rectangle (174.51,183.07);
\definecolor{fillColor}{RGB}{230,230,230}

\path[fill=fillColor] (143.75,174.07) rectangle (174.10,178.30);
\definecolor{drawColor}{RGB}{0,0,0}

\node[text=drawColor,anchor=base east,inner sep=0pt, outer sep=0pt, scale=  0.57] at (217.91, 39.21) {0.0721};

\node[text=drawColor,anchor=base east,inner sep=0pt, outer sep=0pt, scale=  0.57] at (218.22, 34.45) {0.0724};

\node[text=drawColor,anchor=base east,inner sep=0pt, outer sep=0pt, scale=  0.57] at (217.81, 29.68) {0.072};

\node[text=drawColor,anchor=base east,inner sep=0pt, outer sep=0pt, scale=  0.57] at (218.84, 24.92) {0.073};

\node[text=drawColor,anchor=base east,inner sep=0pt, outer sep=0pt, scale=  0.57] at (219.46, 20.15) {0.0736};

\node[text=drawColor,anchor=base east,inner sep=0pt, outer sep=0pt, scale=  0.57] at (216.68, 15.39) {0.0709};

\node[text=drawColor,anchor=base east,inner sep=0pt, outer sep=0pt, scale=  0.57] at (175.74, 70.98) {0.0311};

\node[text=drawColor,anchor=base east,inner sep=0pt, outer sep=0pt, scale=  0.57] at (176.05, 66.22) {0.0314};

\node[text=drawColor,anchor=base east,inner sep=0pt, outer sep=0pt, scale=  0.57] at (175.74, 61.45) {0.0311};

\node[text=drawColor,anchor=base east,inner sep=0pt, outer sep=0pt, scale=  0.57] at (175.84, 56.69) {0.0312};

\node[text=drawColor,anchor=base east,inner sep=0pt, outer sep=0pt, scale=  0.57] at (176.26, 51.92) {0.0316};

\node[text=drawColor,anchor=base east,inner sep=0pt, outer sep=0pt, scale=  0.57] at (175.23, 47.16) {0.0306};

\node[text=drawColor,anchor=base east,inner sep=0pt, outer sep=0pt, scale=  0.57] at (175.23,102.75) {0.0306};

\node[text=drawColor,anchor=base east,inner sep=0pt, outer sep=0pt, scale=  0.57] at (175.54, 97.98) {0.0309};

\node[text=drawColor,anchor=base east,inner sep=0pt, outer sep=0pt, scale=  0.57] at (174.82, 93.22) {0.0302};

\node[text=drawColor,anchor=base east,inner sep=0pt, outer sep=0pt, scale=  0.57] at (175.33, 88.45) {0.0307};

\node[text=drawColor,anchor=base east,inner sep=0pt, outer sep=0pt, scale=  0.57] at (175.23, 83.69) {0.0306};

\node[text=drawColor,anchor=base east,inner sep=0pt, outer sep=0pt, scale=  0.57] at (174.51, 78.92) {0.0299};

\node[text=drawColor,anchor=base east,inner sep=0pt, outer sep=0pt, scale=  0.57] at (216.99,134.52) {0.0712};

\node[text=drawColor,anchor=base east,inner sep=0pt, outer sep=0pt, scale=  0.57] at (216.78,129.75) {0.071};

\node[text=drawColor,anchor=base east,inner sep=0pt, outer sep=0pt, scale=  0.57] at (215.34,124.99) {0.0696};

\node[text=drawColor,anchor=base east,inner sep=0pt, outer sep=0pt, scale=  0.57] at (217.09,120.22) {0.0713};

\node[text=drawColor,anchor=base east,inner sep=0pt, outer sep=0pt, scale=  0.57] at (216.89,115.46) {0.0711};

\node[text=drawColor,anchor=base east,inner sep=0pt, outer sep=0pt, scale=  0.57] at (215.34,110.69) {0.0696};

\node[text=drawColor,anchor=base east,inner sep=0pt, outer sep=0pt, scale=  0.57] at (175.02,166.28) {0.0304};

\node[text=drawColor,anchor=base east,inner sep=0pt, outer sep=0pt, scale=  0.57] at (175.43,161.52) {0.0308};

\node[text=drawColor,anchor=base east,inner sep=0pt, outer sep=0pt, scale=  0.57] at (174.40,156.75) {0.0298};

\node[text=drawColor,anchor=base east,inner sep=0pt, outer sep=0pt, scale=  0.57] at (175.02,151.99) {0.0304};

\node[text=drawColor,anchor=base east,inner sep=0pt, outer sep=0pt, scale=  0.57] at (174.82,147.22) {0.0302};

\node[text=drawColor,anchor=base east,inner sep=0pt, outer sep=0pt, scale=  0.57] at (174.30,142.46) {0.0297};

\node[text=drawColor,anchor=base east,inner sep=0pt, outer sep=0pt, scale=  0.57] at (175.02,198.05) {0.0304};

\node[text=drawColor,anchor=base east,inner sep=0pt, outer sep=0pt, scale=  0.57] at (175.33,193.29) {0.0307};

\node[text=drawColor,anchor=base east,inner sep=0pt, outer sep=0pt, scale=  0.57] at (174.20,188.52) {0.0296};

\node[text=drawColor,anchor=base east,inner sep=0pt, outer sep=0pt, scale=  0.57] at (174.92,183.76) {0.0303};

\node[text=drawColor,anchor=base east,inner sep=0pt, outer sep=0pt, scale=  0.57] at (174.51,178.99) {0.0299};

\node[text=drawColor,anchor=base east,inner sep=0pt, outer sep=0pt, scale=  0.57] at (174.10,174.23) {0.0295};
\end{scope}
\begin{scope}
\path[clip] (  0.00,  0.00) rectangle (224.04,216.81);
\definecolor{drawColor}{RGB}{0,0,0}

\node[text=drawColor,anchor=base,inner sep=0pt, outer sep=0pt, scale=  0.80] at (181.98,208.45) {CI length};
\end{scope}
\end{tikzpicture}

%% file: figs/fig_simulation_logistic_add.tex
\begin{tikzpicture}[x=1pt,y=1pt]
\definecolor{fillColor}{RGB}{255,255,255}
\path[use as bounding box,fill=fillColor,fill opacity=0.00] (0,0) rectangle (224.04,216.81);
\begin{scope}
\path[clip] (  0.00,  0.00) rectangle ( 84.01,216.81);
\definecolor{drawColor}{RGB}{255,255,255}
\definecolor{fillColor}{RGB}{255,255,255}

\path[draw=drawColor,line width= 0.6pt,line join=round,line cap=round,fill=fillColor] (  0.00,  0.00) rectangle ( 84.01,216.81);
\end{scope}
\begin{scope}
\path[clip] (  2.75,  5.60) rectangle ( 84.01,211.77);
\definecolor{fillColor}{RGB}{255,255,255}

\path[fill=fillColor] (  2.75,  5.60) rectangle ( 84.01,211.77);
\definecolor{drawColor}{gray}{0.92}

\path[draw=drawColor,line width= 0.6pt,line join=round] (  2.75, 13.38) --
	( 84.01, 13.38);

\path[draw=drawColor,line width= 0.6pt,line join=round] (  2.75, 45.15) --
	( 84.01, 45.15);

\path[draw=drawColor,line width= 0.6pt,line join=round] (  2.75, 76.91) --
	( 84.01, 76.91);

\path[draw=drawColor,line width= 0.6pt,line join=round] (  2.75,108.68) --
	( 84.01,108.68);

\path[draw=drawColor,line width= 0.6pt,line join=round] (  2.75,140.45) --
	( 84.01,140.45);

\path[draw=drawColor,line width= 0.6pt,line join=round] (  2.75,172.21) --
	( 84.01,172.21);

\path[draw=drawColor,line width= 0.6pt,line join=round] (  2.75,203.98) --
	( 84.01,203.98);

\path[draw=drawColor,line width= 0.6pt,line join=round] ( 80.32,  5.60) --
	( 80.32,211.77);

\path[draw=drawColor,line width= 0.6pt,line join=round] ( 13.83,  5.60) --
	( 13.83,211.77);
\definecolor{fillColor}{RGB}{128,128,128}

\path[fill=fillColor] ( 13.81, 39.06) rectangle ( 80.32, 43.29);
\definecolor{fillColor}{RGB}{155,155,155}

\path[fill=fillColor] ( 14.43, 34.29) rectangle ( 80.32, 38.53);
\definecolor{fillColor}{RGB}{177,177,177}

\path[fill=fillColor] ( 14.80, 29.53) rectangle ( 80.32, 33.76);
\definecolor{fillColor}{gray}{0.77}

\path[fill=fillColor] ( 14.12, 24.76) rectangle ( 80.32, 29.00);
\definecolor{fillColor}{gray}{0.84}

\path[fill=fillColor] ( 14.48, 20.00) rectangle ( 80.32, 24.23);
\definecolor{fillColor}{RGB}{230,230,230}

\path[fill=fillColor] ( 14.07, 15.23) rectangle ( 80.32, 19.47);
\definecolor{fillColor}{RGB}{128,128,128}

\path[fill=fillColor] ( 14.54, 70.82) rectangle ( 80.32, 75.06);
\definecolor{fillColor}{RGB}{155,155,155}

\path[fill=fillColor] ( 16.51, 66.06) rectangle ( 80.32, 70.29);
\definecolor{fillColor}{RGB}{177,177,177}

\path[fill=fillColor] ( 22.40, 61.29) rectangle ( 80.32, 65.53);
\definecolor{fillColor}{gray}{0.77}

\path[fill=fillColor] ( 15.66, 56.53) rectangle ( 80.32, 60.76);
\definecolor{fillColor}{gray}{0.84}

\path[fill=fillColor] ( 17.38, 51.76) rectangle ( 80.32, 56.00);
\definecolor{fillColor}{RGB}{230,230,230}

\path[fill=fillColor] ( 15.22, 47.00) rectangle ( 80.32, 51.23);
\definecolor{fillColor}{RGB}{128,128,128}

\path[fill=fillColor] ( 14.80,102.59) rectangle ( 80.32,106.83);
\definecolor{fillColor}{RGB}{155,155,155}

\path[fill=fillColor] ( 15.88, 97.83) rectangle ( 80.32,102.06);
\definecolor{fillColor}{RGB}{177,177,177}

\path[fill=fillColor] ( 19.84, 93.06) rectangle ( 80.32, 97.30);
\definecolor{fillColor}{gray}{0.77}

\path[fill=fillColor] ( 16.20, 88.30) rectangle ( 80.32, 92.53);
\definecolor{fillColor}{gray}{0.84}

\path[fill=fillColor] ( 18.46, 83.53) rectangle ( 80.32, 87.77);
\definecolor{fillColor}{RGB}{230,230,230}

\path[fill=fillColor] ( 16.06, 78.77) rectangle ( 80.32, 83.00);
\definecolor{fillColor}{RGB}{128,128,128}

\path[fill=fillColor] ( 14.24,134.36) rectangle ( 80.32,138.59);
\definecolor{fillColor}{RGB}{155,155,155}

\path[fill=fillColor] ( 14.36,129.59) rectangle ( 80.32,133.83);
\definecolor{fillColor}{RGB}{177,177,177}

\path[fill=fillColor] ( 15.49,124.83) rectangle ( 80.32,129.06);
\definecolor{fillColor}{gray}{0.77}

\path[fill=fillColor] ( 14.47,120.06) rectangle ( 80.32,124.30);
\definecolor{fillColor}{gray}{0.84}

\path[fill=fillColor] ( 15.70,115.30) rectangle ( 80.32,119.53);
\definecolor{fillColor}{RGB}{230,230,230}

\path[fill=fillColor] ( 14.63,110.53) rectangle ( 80.32,114.77);
\definecolor{fillColor}{RGB}{128,128,128}

\path[fill=fillColor] ( 15.07,166.13) rectangle ( 80.32,170.36);
\definecolor{fillColor}{RGB}{155,155,155}

\path[fill=fillColor] ( 15.55,161.36) rectangle ( 80.32,165.60);
\definecolor{fillColor}{RGB}{177,177,177}

\path[fill=fillColor] ( 21.03,156.60) rectangle ( 80.32,160.83);
\definecolor{fillColor}{gray}{0.77}

\path[fill=fillColor] ( 17.54,151.83) rectangle ( 80.32,156.07);
\definecolor{fillColor}{gray}{0.84}

\path[fill=fillColor] ( 20.90,147.07) rectangle ( 80.32,151.30);
\definecolor{fillColor}{RGB}{230,230,230}

\path[fill=fillColor] ( 17.59,142.30) rectangle ( 80.32,146.54);
\definecolor{fillColor}{RGB}{128,128,128}

\path[fill=fillColor] ( 15.05,197.89) rectangle ( 80.32,202.13);
\definecolor{fillColor}{RGB}{155,155,155}

\path[fill=fillColor] ( 15.43,193.13) rectangle ( 80.32,197.36);
\definecolor{fillColor}{RGB}{177,177,177}

\path[fill=fillColor] ( 20.40,188.36) rectangle ( 80.32,192.60);
\definecolor{fillColor}{gray}{0.77}

\path[fill=fillColor] ( 17.23,183.60) rectangle ( 80.32,187.83);
\definecolor{fillColor}{gray}{0.84}

\path[fill=fillColor] ( 20.28,178.83) rectangle ( 80.32,183.07);
\definecolor{fillColor}{RGB}{230,230,230}

\path[fill=fillColor] ( 17.26,174.07) rectangle ( 80.32,178.30);
\definecolor{drawColor}{RGB}{0,0,0}

\node[text=drawColor,anchor=base west,inner sep=0pt, outer sep=0pt, scale=  0.57] at ( 13.81, 39.21) {0.9003};

\node[text=drawColor,anchor=base west,inner sep=0pt, outer sep=0pt, scale=  0.57] at ( 14.43, 34.45) {0.8919};

\node[text=drawColor,anchor=base west,inner sep=0pt, outer sep=0pt, scale=  0.57] at ( 14.80, 29.68) {0.8869};

\node[text=drawColor,anchor=base west,inner sep=0pt, outer sep=0pt, scale=  0.57] at ( 14.12, 24.92) {0.8961};

\node[text=drawColor,anchor=base west,inner sep=0pt, outer sep=0pt, scale=  0.57] at ( 14.48, 20.15) {0.8912};

\node[text=drawColor,anchor=base west,inner sep=0pt, outer sep=0pt, scale=  0.57] at ( 14.07, 15.39) {0.8968};

\node[text=drawColor,anchor=base west,inner sep=0pt, outer sep=0pt, scale=  0.57] at ( 14.54, 70.98) {0.8904};

\node[text=drawColor,anchor=base west,inner sep=0pt, outer sep=0pt, scale=  0.57] at ( 16.51, 66.22) {0.8637};

\node[text=drawColor,anchor=base west,inner sep=0pt, outer sep=0pt, scale=  0.57] at ( 22.40, 61.45) {0.784};

\node[text=drawColor,anchor=base west,inner sep=0pt, outer sep=0pt, scale=  0.57] at ( 15.66, 56.69) {0.8753};

\node[text=drawColor,anchor=base west,inner sep=0pt, outer sep=0pt, scale=  0.57] at ( 17.38, 51.92) {0.852};

\node[text=drawColor,anchor=base west,inner sep=0pt, outer sep=0pt, scale=  0.57] at ( 15.22, 47.16) {0.8812};

\node[text=drawColor,anchor=base west,inner sep=0pt, outer sep=0pt, scale=  0.57] at ( 14.80,102.75) {0.8869};

\node[text=drawColor,anchor=base west,inner sep=0pt, outer sep=0pt, scale=  0.57] at ( 15.88, 97.98) {0.8723};

\node[text=drawColor,anchor=base west,inner sep=0pt, outer sep=0pt, scale=  0.57] at ( 19.84, 93.22) {0.8186};

\node[text=drawColor,anchor=base west,inner sep=0pt, outer sep=0pt, scale=  0.57] at ( 16.20, 88.45) {0.8679};

\node[text=drawColor,anchor=base west,inner sep=0pt, outer sep=0pt, scale=  0.57] at ( 18.46, 83.69) {0.8374};

\node[text=drawColor,anchor=base west,inner sep=0pt, outer sep=0pt, scale=  0.57] at ( 16.06, 78.92) {0.8699};

\node[text=drawColor,anchor=base west,inner sep=0pt, outer sep=0pt, scale=  0.57] at ( 14.24,134.52) {0.8945};

\node[text=drawColor,anchor=base west,inner sep=0pt, outer sep=0pt, scale=  0.57] at ( 14.36,129.75) {0.8929};

\node[text=drawColor,anchor=base west,inner sep=0pt, outer sep=0pt, scale=  0.57] at ( 15.49,124.99) {0.8776};

\node[text=drawColor,anchor=base west,inner sep=0pt, outer sep=0pt, scale=  0.57] at ( 14.47,120.22) {0.8913};

\node[text=drawColor,anchor=base west,inner sep=0pt, outer sep=0pt, scale=  0.57] at ( 15.70,115.46) {0.8747};

\node[text=drawColor,anchor=base west,inner sep=0pt, outer sep=0pt, scale=  0.57] at ( 14.63,110.69) {0.8892};

\node[text=drawColor,anchor=base west,inner sep=0pt, outer sep=0pt, scale=  0.57] at ( 15.07,166.28) {0.8832};

\node[text=drawColor,anchor=base west,inner sep=0pt, outer sep=0pt, scale=  0.57] at ( 15.55,161.52) {0.8768};

\node[text=drawColor,anchor=base west,inner sep=0pt, outer sep=0pt, scale=  0.57] at ( 21.03,156.75) {0.8025};

\node[text=drawColor,anchor=base west,inner sep=0pt, outer sep=0pt, scale=  0.57] at ( 17.54,151.99) {0.8498};

\node[text=drawColor,anchor=base west,inner sep=0pt, outer sep=0pt, scale=  0.57] at ( 20.90,147.22) {0.8043};

\node[text=drawColor,anchor=base west,inner sep=0pt, outer sep=0pt, scale=  0.57] at ( 17.59,142.46) {0.8491};

\node[text=drawColor,anchor=base west,inner sep=0pt, outer sep=0pt, scale=  0.57] at ( 15.05,198.05) {0.8835};

\node[text=drawColor,anchor=base west,inner sep=0pt, outer sep=0pt, scale=  0.57] at ( 15.43,193.29) {0.8783};

\node[text=drawColor,anchor=base west,inner sep=0pt, outer sep=0pt, scale=  0.57] at ( 20.40,188.52) {0.8111};

\node[text=drawColor,anchor=base west,inner sep=0pt, outer sep=0pt, scale=  0.57] at ( 17.23,183.76) {0.854};

\node[text=drawColor,anchor=base west,inner sep=0pt, outer sep=0pt, scale=  0.57] at ( 20.28,178.99) {0.8127};

\node[text=drawColor,anchor=base west,inner sep=0pt, outer sep=0pt, scale=  0.57] at ( 17.26,174.23) {0.8536};
\end{scope}
\begin{scope}
\path[clip] (  0.00,  0.00) rectangle (224.04,216.81);
\definecolor{drawColor}{RGB}{0,0,0}

\node[text=drawColor,anchor=base,inner sep=0pt, outer sep=0pt, scale=  0.80] at ( 43.38,208.45) {Coverage prob.};
\end{scope}
\begin{scope}
\path[clip] ( 84.01,  0.00) rectangle (140.02,216.81);
\definecolor{drawColor}{RGB}{255,255,255}
\definecolor{fillColor}{RGB}{255,255,255}

\path[draw=drawColor,line width= 0.6pt,line join=round,line cap=round,fill=fillColor] ( 84.01,  0.00) rectangle (140.02,216.81);
\end{scope}
\begin{scope}
\path[clip] ( 83.92,  5.60) rectangle (142.87,211.77);
\definecolor{drawColor}{RGB}{0,0,0}

\path[draw=drawColor,line width= 0.6pt,line join=round] ( 86.60, 35.77) -- (124.88, 35.77);

\path[draw=drawColor,line width= 0.6pt,line join=round] ( 86.60, 23.09) -- (124.88, 23.09);
\definecolor{drawColor}{RGB}{190,190,190}

\path[draw=drawColor,line width= 0.6pt,dash pattern=on 4pt off 4pt ,line join=round] ( 86.60, 23.09) -- ( 86.60, 35.77);
\definecolor{drawColor}{RGB}{0,0,0}

\path[draw=drawColor,line width= 0.6pt,line join=round] ( 86.60, 64.30) -- (105.74, 64.30);

\path[draw=drawColor,line width= 0.6pt,line join=round] ( 86.60, 57.96) -- (105.74, 57.96);

\path[draw=drawColor,line width= 0.6pt,line join=round] ( 86.60, 70.64) -- (105.74, 70.64);

\path[draw=drawColor,line width= 0.6pt,line join=round] ( 86.60, 51.62) -- (105.74, 51.62);
\definecolor{drawColor}{RGB}{190,190,190}

\path[draw=drawColor,line width= 0.6pt,dash pattern=on 4pt off 4pt ,line join=round] ( 86.60, 51.62) -- ( 86.60, 70.64);
\definecolor{drawColor}{RGB}{0,0,0}

\path[draw=drawColor,line width= 0.6pt,line join=round] ( 86.60, 99.17) -- ( 99.36, 99.17);

\path[draw=drawColor,line width= 0.6pt,line join=round] ( 86.60, 86.49) -- ( 99.36, 86.49);

\path[draw=drawColor,line width= 0.6pt,line join=round] ( 99.36, 96.00) -- (112.12, 96.00);

\path[draw=drawColor,line width= 0.6pt,line join=round] ( 99.36, 89.66) -- (112.12, 89.66);

\path[draw=drawColor,line width= 0.6pt,line join=round] ( 99.36,102.34) -- (112.12,102.34);

\path[draw=drawColor,line width= 0.6pt,line join=round] ( 99.36, 83.32) -- (112.12, 83.32);
\definecolor{drawColor}{RGB}{190,190,190}

\path[draw=drawColor,line width= 0.6pt,dash pattern=on 4pt off 4pt ,line join=round] ( 86.60, 86.49) -- ( 86.60, 99.17);

\path[draw=drawColor,line width= 0.6pt,dash pattern=on 4pt off 4pt ,line join=round] ( 99.36, 83.32) -- ( 99.36,102.34);
\definecolor{drawColor}{RGB}{0,0,0}

\path[draw=drawColor,line width= 0.6pt,line join=round] ( 86.60,124.53) -- (112.12,124.53);

\path[draw=drawColor,line width= 0.6pt,line join=round] (112.12,130.87) -- (137.64,130.87);

\path[draw=drawColor,line width= 0.6pt,line join=round] (112.12,118.19) -- (137.64,118.19);
\definecolor{drawColor}{RGB}{190,190,190}

\path[draw=drawColor,line width= 0.6pt,dash pattern=on 4pt off 4pt ,line join=round] (112.12,118.19) -- (112.12,130.87);
\definecolor{drawColor}{RGB}{0,0,0}

\path[draw=drawColor,line width= 0.6pt,line join=round] ( 86.60,156.23) -- (101.91,156.23);

\path[draw=drawColor,line width= 0.6pt,line join=round] (101.91,159.40) -- (117.22,159.40);

\path[draw=drawColor,line width= 0.6pt,line join=round] (101.91,153.06) -- (117.22,153.06);

\path[draw=drawColor,line width= 0.6pt,line join=round] (101.91,165.74) -- (117.22,165.74);

\path[draw=drawColor,line width= 0.6pt,line join=round] (101.91,146.72) -- (117.22,146.72);
\definecolor{drawColor}{RGB}{190,190,190}

\path[draw=drawColor,line width= 0.6pt,dash pattern=on 4pt off 4pt ,line join=round] (101.91,146.72) -- (101.91,165.74);
\definecolor{drawColor}{RGB}{0,0,0}

\path[draw=drawColor,line width= 0.6pt,line join=round] ( 86.60,187.93) -- ( 97.54,187.93);

\path[draw=drawColor,line width= 0.6pt,line join=round] ( 97.54,194.27) -- (108.47,194.27);

\path[draw=drawColor,line width= 0.6pt,line join=round] ( 97.54,181.59) -- (108.47,181.59);

\path[draw=drawColor,line width= 0.6pt,line join=round] (108.47,191.10) -- (119.41,191.10);

\path[draw=drawColor,line width= 0.6pt,line join=round] (108.47,184.76) -- (119.41,184.76);

\path[draw=drawColor,line width= 0.6pt,line join=round] (108.47,197.44) -- (119.41,197.44);

\path[draw=drawColor,line width= 0.6pt,line join=round] (108.47,178.42) -- (119.41,178.42);
\definecolor{drawColor}{RGB}{190,190,190}

\path[draw=drawColor,line width= 0.6pt,dash pattern=on 4pt off 4pt ,line join=round] ( 97.54,181.59) -- ( 97.54,194.27);

\path[draw=drawColor,line width= 0.6pt,dash pattern=on 4pt off 4pt ,line join=round] (108.47,178.42) -- (108.47,197.44);
\end{scope}
\begin{scope}
\path[clip] (  0.00,  0.00) rectangle (224.04,216.81);
\definecolor{drawColor}{RGB}{0,0,0}

\node[text=drawColor,anchor=base,inner sep=0pt, outer sep=0pt, scale=  0.80] at (113.39,208.45) {Config};
\end{scope}
\begin{scope}
\path[clip] (140.02,  0.00) rectangle (224.04,216.81);
\definecolor{drawColor}{RGB}{255,255,255}
\definecolor{fillColor}{RGB}{255,255,255}

\path[draw=drawColor,line width= 0.6pt,line join=round,line cap=round,fill=fillColor] (140.02,  0.00) rectangle (224.04,216.81);
\end{scope}
\begin{scope}
\path[clip] (139.93,  5.60) rectangle (224.04,211.77);
\definecolor{fillColor}{RGB}{255,255,255}

\path[fill=fillColor] (139.93,  5.60) rectangle (224.04,211.77);
\definecolor{drawColor}{gray}{0.92}

\path[draw=drawColor,line width= 0.6pt,line join=round] (139.93, 13.38) --
	(224.04, 13.38);

\path[draw=drawColor,line width= 0.6pt,line join=round] (139.93, 45.15) --
	(224.04, 45.15);

\path[draw=drawColor,line width= 0.6pt,line join=round] (139.93, 76.91) --
	(224.04, 76.91);

\path[draw=drawColor,line width= 0.6pt,line join=round] (139.93,108.68) --
	(224.04,108.68);

\path[draw=drawColor,line width= 0.6pt,line join=round] (139.93,140.45) --
	(224.04,140.45);

\path[draw=drawColor,line width= 0.6pt,line join=round] (139.93,172.21) --
	(224.04,172.21);

\path[draw=drawColor,line width= 0.6pt,line join=round] (139.93,203.98) --
	(224.04,203.98);

\path[draw=drawColor,line width= 0.6pt,line join=round] (143.75,  5.60) --
	(143.75,211.77);
\definecolor{fillColor}{RGB}{128,128,128}

\path[fill=fillColor] (143.75, 39.06) rectangle (217.47, 43.29);
\definecolor{fillColor}{RGB}{155,155,155}

\path[fill=fillColor] (143.75, 34.29) rectangle (217.84, 38.53);
\definecolor{fillColor}{RGB}{177,177,177}

\path[fill=fillColor] (143.75, 29.53) rectangle (219.46, 33.76);
\definecolor{fillColor}{gray}{0.77}

\path[fill=fillColor] (143.75, 24.76) rectangle (217.52, 29.00);
\definecolor{fillColor}{gray}{0.84}

\path[fill=fillColor] (143.75, 20.00) rectangle (218.40, 24.23);
\definecolor{fillColor}{RGB}{230,230,230}

\path[fill=fillColor] (143.75, 15.23) rectangle (218.44, 19.47);
\definecolor{fillColor}{RGB}{128,128,128}

\path[fill=fillColor] (143.75, 70.82) rectangle (175.38, 75.06);
\definecolor{fillColor}{RGB}{155,155,155}

\path[fill=fillColor] (143.75, 66.06) rectangle (175.66, 70.29);
\definecolor{fillColor}{RGB}{177,177,177}

\path[fill=fillColor] (143.75, 61.29) rectangle (176.53, 65.53);
\definecolor{fillColor}{gray}{0.77}

\path[fill=fillColor] (143.75, 56.53) rectangle (175.52, 60.76);
\definecolor{fillColor}{gray}{0.84}

\path[fill=fillColor] (143.75, 51.76) rectangle (176.53, 56.00);
\definecolor{fillColor}{RGB}{230,230,230}

\path[fill=fillColor] (143.75, 47.00) rectangle (176.16, 51.23);
\definecolor{fillColor}{RGB}{128,128,128}

\path[fill=fillColor] (143.75,102.59) rectangle (174.32,106.83);
\definecolor{fillColor}{RGB}{155,155,155}

\path[fill=fillColor] (143.75, 97.83) rectangle (174.83,102.06);
\definecolor{fillColor}{RGB}{177,177,177}

\path[fill=fillColor] (143.75, 93.06) rectangle (172.15, 97.30);
\definecolor{fillColor}{gray}{0.77}

\path[fill=fillColor] (143.75, 88.30) rectangle (173.03, 92.53);
\definecolor{fillColor}{gray}{0.84}

\path[fill=fillColor] (143.75, 83.53) rectangle (171.78, 87.77);
\definecolor{fillColor}{RGB}{230,230,230}

\path[fill=fillColor] (143.75, 78.77) rectangle (173.44, 83.00);
\definecolor{fillColor}{RGB}{128,128,128}

\path[fill=fillColor] (143.75,134.36) rectangle (214.66,138.59);
\definecolor{fillColor}{RGB}{155,155,155}

\path[fill=fillColor] (143.75,129.59) rectangle (216.32,133.83);
\definecolor{fillColor}{RGB}{177,177,177}

\path[fill=fillColor] (143.75,124.83) rectangle (208.07,129.06);
\definecolor{fillColor}{gray}{0.77}

\path[fill=fillColor] (143.75,120.06) rectangle (211.30,124.30);
\definecolor{fillColor}{gray}{0.84}

\path[fill=fillColor] (143.75,115.30) rectangle (206.36,119.53);
\definecolor{fillColor}{RGB}{230,230,230}

\path[fill=fillColor] (143.75,110.53) rectangle (211.94,114.77);
\definecolor{fillColor}{RGB}{128,128,128}

\path[fill=fillColor] (143.75,166.13) rectangle (174.00,170.36);
\definecolor{fillColor}{RGB}{155,155,155}

\path[fill=fillColor] (143.75,161.36) rectangle (174.69,165.60);
\definecolor{fillColor}{RGB}{177,177,177}

\path[fill=fillColor] (143.75,156.60) rectangle (169.94,160.83);
\definecolor{fillColor}{gray}{0.77}

\path[fill=fillColor] (143.75,151.83) rectangle (171.83,156.07);
\definecolor{fillColor}{gray}{0.84}

\path[fill=fillColor] (143.75,147.07) rectangle (169.39,151.30);
\definecolor{fillColor}{RGB}{230,230,230}

\path[fill=fillColor] (143.75,142.30) rectangle (172.20,146.54);
\definecolor{fillColor}{RGB}{128,128,128}

\path[fill=fillColor] (143.75,197.89) rectangle (173.81,202.13);
\definecolor{fillColor}{RGB}{155,155,155}

\path[fill=fillColor] (143.75,193.13) rectangle (174.64,197.36);
\definecolor{fillColor}{RGB}{177,177,177}

\path[fill=fillColor] (143.75,188.36) rectangle (169.48,192.60);
\definecolor{fillColor}{gray}{0.77}

\path[fill=fillColor] (143.75,183.60) rectangle (171.55,187.83);
\definecolor{fillColor}{gray}{0.84}

\path[fill=fillColor] (143.75,178.83) rectangle (169.02,183.07);
\definecolor{fillColor}{RGB}{230,230,230}

\path[fill=fillColor] (143.75,174.07) rectangle (171.88,178.30);
\definecolor{drawColor}{RGB}{0,0,0}

\node[text=drawColor,anchor=base east,inner sep=0pt, outer sep=0pt, scale=  0.57] at (217.47, 39.21) {0.1599};

\node[text=drawColor,anchor=base east,inner sep=0pt, outer sep=0pt, scale=  0.57] at (217.84, 34.45) {0.1607};

\node[text=drawColor,anchor=base east,inner sep=0pt, outer sep=0pt, scale=  0.57] at (219.46, 29.68) {0.1642};

\node[text=drawColor,anchor=base east,inner sep=0pt, outer sep=0pt, scale=  0.57] at (217.52, 24.92) {0.16};

\node[text=drawColor,anchor=base east,inner sep=0pt, outer sep=0pt, scale=  0.57] at (218.40, 20.15) {0.1619};

\node[text=drawColor,anchor=base east,inner sep=0pt, outer sep=0pt, scale=  0.57] at (218.44, 15.39) {0.162};

\node[text=drawColor,anchor=base east,inner sep=0pt, outer sep=0pt, scale=  0.57] at (175.38, 70.98) {0.0686};

\node[text=drawColor,anchor=base east,inner sep=0pt, outer sep=0pt, scale=  0.57] at (175.66, 66.22) {0.0692};

\node[text=drawColor,anchor=base east,inner sep=0pt, outer sep=0pt, scale=  0.57] at (176.53, 61.45) {0.0711};

\node[text=drawColor,anchor=base east,inner sep=0pt, outer sep=0pt, scale=  0.57] at (175.52, 56.69) {0.0689};

\node[text=drawColor,anchor=base east,inner sep=0pt, outer sep=0pt, scale=  0.57] at (176.53, 51.92) {0.0711};

\node[text=drawColor,anchor=base east,inner sep=0pt, outer sep=0pt, scale=  0.57] at (176.16, 47.16) {0.0703};

\node[text=drawColor,anchor=base east,inner sep=0pt, outer sep=0pt, scale=  0.57] at (174.32,102.75) {0.0663};

\node[text=drawColor,anchor=base east,inner sep=0pt, outer sep=0pt, scale=  0.57] at (174.83, 97.98) {0.0674};

\node[text=drawColor,anchor=base east,inner sep=0pt, outer sep=0pt, scale=  0.57] at (172.15, 93.22) {0.0616};

\node[text=drawColor,anchor=base east,inner sep=0pt, outer sep=0pt, scale=  0.57] at (173.03, 88.45) {0.0635};

\node[text=drawColor,anchor=base east,inner sep=0pt, outer sep=0pt, scale=  0.57] at (171.78, 83.69) {0.0608};

\node[text=drawColor,anchor=base east,inner sep=0pt, outer sep=0pt, scale=  0.57] at (173.44, 78.92) {0.0644};

\node[text=drawColor,anchor=base east,inner sep=0pt, outer sep=0pt, scale=  0.57] at (214.66,134.52) {0.1538};

\node[text=drawColor,anchor=base east,inner sep=0pt, outer sep=0pt, scale=  0.57] at (216.32,129.75) {0.1574};

\node[text=drawColor,anchor=base east,inner sep=0pt, outer sep=0pt, scale=  0.57] at (208.07,124.99) {0.1395};

\node[text=drawColor,anchor=base east,inner sep=0pt, outer sep=0pt, scale=  0.57] at (211.30,120.22) {0.1465};

\node[text=drawColor,anchor=base east,inner sep=0pt, outer sep=0pt, scale=  0.57] at (206.36,115.46) {0.1358};

\node[text=drawColor,anchor=base east,inner sep=0pt, outer sep=0pt, scale=  0.57] at (211.94,110.69) {0.1479};

\node[text=drawColor,anchor=base east,inner sep=0pt, outer sep=0pt, scale=  0.57] at (174.00,166.28) {0.0656};

\node[text=drawColor,anchor=base east,inner sep=0pt, outer sep=0pt, scale=  0.57] at (174.69,161.52) {0.0671};

\node[text=drawColor,anchor=base east,inner sep=0pt, outer sep=0pt, scale=  0.57] at (169.94,156.75) {0.0568};

\node[text=drawColor,anchor=base east,inner sep=0pt, outer sep=0pt, scale=  0.57] at (171.83,151.99) {0.0609};

\node[text=drawColor,anchor=base east,inner sep=0pt, outer sep=0pt, scale=  0.57] at (169.39,147.22) {0.0556};

\node[text=drawColor,anchor=base east,inner sep=0pt, outer sep=0pt, scale=  0.57] at (172.20,142.46) {0.0617};

\node[text=drawColor,anchor=base east,inner sep=0pt, outer sep=0pt, scale=  0.57] at (173.81,198.05) {0.0652};

\node[text=drawColor,anchor=base east,inner sep=0pt, outer sep=0pt, scale=  0.57] at (174.64,193.29) {0.067};

\node[text=drawColor,anchor=base east,inner sep=0pt, outer sep=0pt, scale=  0.57] at (169.48,188.52) {0.0558};

\node[text=drawColor,anchor=base east,inner sep=0pt, outer sep=0pt, scale=  0.57] at (171.55,183.76) {0.0603};

\node[text=drawColor,anchor=base east,inner sep=0pt, outer sep=0pt, scale=  0.57] at (169.02,178.99) {0.0548};

\node[text=drawColor,anchor=base east,inner sep=0pt, outer sep=0pt, scale=  0.57] at (171.88,174.23) {0.061};
\end{scope}
\begin{scope}
\path[clip] (  0.00,  0.00) rectangle (224.04,216.81);
\definecolor{drawColor}{RGB}{0,0,0}

\node[text=drawColor,anchor=base,inner sep=0pt, outer sep=0pt, scale=  0.80] at (181.98,208.45) {CI length};
\end{scope}
\end{tikzpicture}

%% file: figs/legend_simulation_add.tex
\begin{tikzpicture}[x=1pt,y=1pt]
\definecolor{fillColor}{RGB}{255,255,255}
\path[use as bounding box,fill=fillColor,fill opacity=0.00] (0,0) rectangle (361.35, 18.07);
\begin{scope}
\path[clip] (  0.00,  0.00) rectangle (361.35, 18.07);
\definecolor{fillColor}{RGB}{230,230,230}

\path[fill=fillColor] ( 13.38, 10.71) --
	( 59.01, 10.71) --
	( 59.01, 17.40) --
	( 13.38, 17.40) --
	cycle;
\definecolor{drawColor}{RGB}{0,0,0}

\node[text=drawColor,anchor=base west,inner sep=0pt, outer sep=0pt, scale=  0.70] at ( 65.01, 12.45) {(0.51,0.1,0)};
\definecolor{fillColor}{gray}{0.84}

\path[fill=fillColor] (135.05, 10.71) --
	(180.67, 10.71) --
	(180.67, 17.40) --
	(135.05, 17.40) --
	cycle;

\node[text=drawColor,anchor=base west,inner sep=0pt, outer sep=0pt, scale=  0.70] at (186.67, 12.45) {(0.55,0.1,0)};
\definecolor{fillColor}{gray}{0.77}

\path[fill=fillColor] (256.72, 10.71) --
	(302.34, 10.71) --
	(302.34, 17.40) --
	(256.72, 17.40) --
	cycle;

\node[text=drawColor,anchor=base west,inner sep=0pt, outer sep=0pt, scale=  0.70] at (308.34, 12.45) {(0.51,0.1,100)};
\definecolor{fillColor}{RGB}{177,177,177}

\path[fill=fillColor] ( 13.38,  0.67) --
	( 59.01,  0.67) --
	( 59.01,  7.36) --
	( 13.38,  7.36) --
	cycle;

\node[text=drawColor,anchor=base west,inner sep=0pt, outer sep=0pt, scale=  0.70] at ( 65.01,  2.41) {(0.55,0.1,100)};
\definecolor{fillColor}{RGB}{155,155,155}

\path[fill=fillColor] (135.05,  0.67) --
	(180.67,  0.67) --
	(180.67,  7.36) --
	(135.05,  7.36) --
	cycle;

\node[text=drawColor,anchor=base west,inner sep=0pt, outer sep=0pt, scale=  0.70] at (186.67,  2.41) {(0.51,0.5,100)};
\definecolor{fillColor}{RGB}{128,128,128}

\path[fill=fillColor] (256.72,  0.67) --
	(302.34,  0.67) --
	(302.34,  7.36) --
	(256.72,  7.36) --
	cycle;

\node[text=drawColor,anchor=base west,inner sep=0pt, outer sep=0pt, scale=  0.70] at (308.34,  2.41) {(0.55,0.5,100)};
\end{scope}
\end{tikzpicture}

%% file: figs/adult_coverage.tex
\begin{tikzpicture}[x=1pt,y=1pt]
\definecolor{fillColor}{RGB}{255,255,255}
\path[use as bounding box,fill=fillColor,fill opacity=0.00] (0,0) rectangle (448.07,180.67);
\begin{scope}
\path[clip] (  0.00,  0.00) rectangle (224.04,180.67);
\definecolor{drawColor}{RGB}{255,255,255}
\definecolor{fillColor}{RGB}{255,255,255}

\path[draw=drawColor,line width= 0.6pt,line join=round,line cap=round,fill=fillColor] ( -0.00,  0.00) rectangle (224.04,180.68);
\end{scope}
\begin{scope}
\path[clip] ( 30.46, 27.54) rectangle (218.54,175.17);
\definecolor{fillColor}{RGB}{255,255,255}

\path[fill=fillColor] ( 30.46, 27.54) rectangle (218.54,175.17);
\definecolor{drawColor}{gray}{0.92}

\path[draw=drawColor,line width= 0.3pt,line join=round] ( 30.46, 49.16) --
	(218.54, 49.16);

\path[draw=drawColor,line width= 0.3pt,line join=round] ( 30.46, 78.99) --
	(218.54, 78.99);

\path[draw=drawColor,line width= 0.3pt,line join=round] ( 30.46,108.81) --
	(218.54,108.81);

\path[draw=drawColor,line width= 0.3pt,line join=round] ( 30.46,138.64) --
	(218.54,138.64);

\path[draw=drawColor,line width= 0.3pt,line join=round] ( 30.46,168.46) --
	(218.54,168.46);

\path[draw=drawColor,line width= 0.3pt,line join=round] ( 30.46, 27.54) --
	( 30.46,175.17);

\path[draw=drawColor,line width= 0.3pt,line join=round] ( 47.56, 27.54) --
	( 47.56,175.17);

\path[draw=drawColor,line width= 0.3pt,line join=round] ( 64.65, 27.54) --
	( 64.65,175.17);

\path[draw=drawColor,line width= 0.3pt,line join=round] ( 81.75, 27.54) --
	( 81.75,175.17);

\path[draw=drawColor,line width= 0.3pt,line join=round] ( 98.85, 27.54) --
	( 98.85,175.17);

\path[draw=drawColor,line width= 0.3pt,line join=round] (115.95, 27.54) --
	(115.95,175.17);

\path[draw=drawColor,line width= 0.3pt,line join=round] (133.05, 27.54) --
	(133.05,175.17);

\path[draw=drawColor,line width= 0.3pt,line join=round] (150.14, 27.54) --
	(150.14,175.17);

\path[draw=drawColor,line width= 0.3pt,line join=round] (167.24, 27.54) --
	(167.24,175.17);

\path[draw=drawColor,line width= 0.3pt,line join=round] (184.34, 27.54) --
	(184.34,175.17);

\path[draw=drawColor,line width= 0.3pt,line join=round] (201.44, 27.54) --
	(201.44,175.17);

\path[draw=drawColor,line width= 0.3pt,line join=round] (218.54, 27.54) --
	(218.54,175.17);

\path[draw=drawColor,line width= 0.6pt,line join=round] ( 30.46, 34.25) --
	(218.54, 34.25);

\path[draw=drawColor,line width= 0.6pt,line join=round] ( 30.46, 64.07) --
	(218.54, 64.07);

\path[draw=drawColor,line width= 0.6pt,line join=round] ( 30.46, 93.90) --
	(218.54, 93.90);

\path[draw=drawColor,line width= 0.6pt,line join=round] ( 30.46,123.73) --
	(218.54,123.73);

\path[draw=drawColor,line width= 0.6pt,line join=round] ( 30.46,153.55) --
	(218.54,153.55);

\path[draw=drawColor,line width= 0.6pt,line join=round] ( 39.01, 27.54) --
	( 39.01,175.17);

\path[draw=drawColor,line width= 0.6pt,line join=round] ( 56.10, 27.54) --
	( 56.10,175.17);

\path[draw=drawColor,line width= 0.6pt,line join=round] ( 73.20, 27.54) --
	( 73.20,175.17);

\path[draw=drawColor,line width= 0.6pt,line join=round] ( 90.30, 27.54) --
	( 90.30,175.17);

\path[draw=drawColor,line width= 0.6pt,line join=round] (107.40, 27.54) --
	(107.40,175.17);

\path[draw=drawColor,line width= 0.6pt,line join=round] (124.50, 27.54) --
	(124.50,175.17);

\path[draw=drawColor,line width= 0.6pt,line join=round] (141.60, 27.54) --
	(141.60,175.17);

\path[draw=drawColor,line width= 0.6pt,line join=round] (158.69, 27.54) --
	(158.69,175.17);

\path[draw=drawColor,line width= 0.6pt,line join=round] (175.79, 27.54) --
	(175.79,175.17);

\path[draw=drawColor,line width= 0.6pt,line join=round] (192.89, 27.54) --
	(192.89,175.17);

\path[draw=drawColor,line width= 0.6pt,line join=round] (209.99, 27.54) --
	(209.99,175.17);
\definecolor{drawColor}{RGB}{169,169,169}
\definecolor{fillColor}{RGB}{190,190,190}

\path[draw=drawColor,line width= 0.6pt,line join=round,fill=fillColor] ( 43.28, 34.25) rectangle ( 51.83, 34.25);

\path[draw=drawColor,line width= 0.6pt,line join=round,fill=fillColor] ( 51.83, 34.25) rectangle ( 60.38, 34.25);

\path[draw=drawColor,line width= 0.6pt,line join=round,fill=fillColor] ( 60.38, 34.25) rectangle ( 68.93, 34.25);

\path[draw=drawColor,line width= 0.6pt,line join=round,fill=fillColor] ( 68.93, 34.25) rectangle ( 77.48, 34.25);

\path[draw=drawColor,line width= 0.6pt,line join=round,fill=fillColor] ( 77.48, 34.25) rectangle ( 86.03, 34.25);

\path[draw=drawColor,line width= 0.6pt,line join=round,fill=fillColor] ( 86.03, 34.25) rectangle ( 94.58, 35.74);

\path[draw=drawColor,line width= 0.6pt,line join=round,fill=fillColor] ( 94.58, 34.25) rectangle (103.12, 34.25);

\path[draw=drawColor,line width= 0.6pt,line join=round,fill=fillColor] (103.12, 34.25) rectangle (111.67, 35.14);

\path[draw=drawColor,line width= 0.6pt,line join=round,fill=fillColor] (111.67, 34.25) rectangle (120.22, 34.55);

\path[draw=drawColor,line width= 0.6pt,line join=round,fill=fillColor] (120.22, 34.25) rectangle (128.77, 34.55);

\path[draw=drawColor,line width= 0.6pt,line join=round,fill=fillColor] (128.77, 34.25) rectangle (137.32, 34.55);

\path[draw=drawColor,line width= 0.6pt,line join=round,fill=fillColor] (137.32, 34.25) rectangle (145.87, 38.42);

\path[draw=drawColor,line width= 0.6pt,line join=round,fill=fillColor] (145.87, 34.25) rectangle (154.42, 50.35);

\path[draw=drawColor,line width= 0.6pt,line join=round,fill=fillColor] (154.42, 34.25) rectangle (162.97, 53.04);

\path[draw=drawColor,line width= 0.6pt,line join=round,fill=fillColor] (162.97, 34.25) rectangle (171.52, 53.93);

\path[draw=drawColor,line width= 0.6pt,line join=round,fill=fillColor] (171.52, 34.25) rectangle (180.07, 72.13);

\path[draw=drawColor,line width= 0.6pt,line join=round,fill=fillColor] (180.07, 34.25) rectangle (188.62,122.53);

\path[draw=drawColor,line width= 0.6pt,line join=round,fill=fillColor] (188.62, 34.25) rectangle (197.16,140.73);

\path[draw=drawColor,line width= 0.6pt,line join=round,fill=fillColor] (197.16, 34.25) rectangle (205.71, 37.53);
\definecolor{drawColor}{gray}{0.20}

\path[draw=drawColor,line width= 0.6pt,line join=round,line cap=round] ( 30.46, 27.54) rectangle (218.54,175.17);
\end{scope}
\begin{scope}
\path[clip] (  0.00,  0.00) rectangle (448.07,180.67);
\definecolor{drawColor}{gray}{0.30}

\node[text=drawColor,anchor=base east,inner sep=0pt, outer sep=0pt, scale=  0.60] at ( 25.51, 32.18) {0};

\node[text=drawColor,anchor=base east,inner sep=0pt, outer sep=0pt, scale=  0.60] at ( 25.51, 62.01) {100};

\node[text=drawColor,anchor=base east,inner sep=0pt, outer sep=0pt, scale=  0.60] at ( 25.51, 91.83) {200};

\node[text=drawColor,anchor=base east,inner sep=0pt, outer sep=0pt, scale=  0.60] at ( 25.51,121.66) {300};

\node[text=drawColor,anchor=base east,inner sep=0pt, outer sep=0pt, scale=  0.60] at ( 25.51,151.48) {400};
\end{scope}
\begin{scope}
\path[clip] (  0.00,  0.00) rectangle (448.07,180.67);
\definecolor{drawColor}{gray}{0.20}

\path[draw=drawColor,line width= 0.6pt,line join=round] ( 27.71, 34.25) --
	( 30.46, 34.25);

\path[draw=drawColor,line width= 0.6pt,line join=round] ( 27.71, 64.07) --
	( 30.46, 64.07);

\path[draw=drawColor,line width= 0.6pt,line join=round] ( 27.71, 93.90) --
	( 30.46, 93.90);

\path[draw=drawColor,line width= 0.6pt,line join=round] ( 27.71,123.73) --
	( 30.46,123.73);

\path[draw=drawColor,line width= 0.6pt,line join=round] ( 27.71,153.55) --
	( 30.46,153.55);
\end{scope}
\begin{scope}
\path[clip] (  0.00,  0.00) rectangle (448.07,180.67);
\definecolor{drawColor}{gray}{0.20}

\path[draw=drawColor,line width= 0.6pt,line join=round] ( 39.01, 24.79) --
	( 39.01, 27.54);

\path[draw=drawColor,line width= 0.6pt,line join=round] ( 56.10, 24.79) --
	( 56.10, 27.54);

\path[draw=drawColor,line width= 0.6pt,line join=round] ( 73.20, 24.79) --
	( 73.20, 27.54);

\path[draw=drawColor,line width= 0.6pt,line join=round] ( 90.30, 24.79) --
	( 90.30, 27.54);

\path[draw=drawColor,line width= 0.6pt,line join=round] (107.40, 24.79) --
	(107.40, 27.54);

\path[draw=drawColor,line width= 0.6pt,line join=round] (124.50, 24.79) --
	(124.50, 27.54);

\path[draw=drawColor,line width= 0.6pt,line join=round] (141.60, 24.79) --
	(141.60, 27.54);

\path[draw=drawColor,line width= 0.6pt,line join=round] (158.69, 24.79) --
	(158.69, 27.54);

\path[draw=drawColor,line width= 0.6pt,line join=round] (175.79, 24.79) --
	(175.79, 27.54);

\path[draw=drawColor,line width= 0.6pt,line join=round] (192.89, 24.79) --
	(192.89, 27.54);

\path[draw=drawColor,line width= 0.6pt,line join=round] (209.99, 24.79) --
	(209.99, 27.54);
\end{scope}
\begin{scope}
\path[clip] (  0.00,  0.00) rectangle (448.07,180.67);
\definecolor{drawColor}{gray}{0.30}

\node[text=drawColor,anchor=base,inner sep=0pt, outer sep=0pt, scale=  0.60] at ( 39.01, 18.45) {0.0};

\node[text=drawColor,anchor=base,inner sep=0pt, outer sep=0pt, scale=  0.60] at ( 56.10, 18.45) {0.1};

\node[text=drawColor,anchor=base,inner sep=0pt, outer sep=0pt, scale=  0.60] at ( 73.20, 18.45) {0.2};

\node[text=drawColor,anchor=base,inner sep=0pt, outer sep=0pt, scale=  0.60] at ( 90.30, 18.45) {0.3};

\node[text=drawColor,anchor=base,inner sep=0pt, outer sep=0pt, scale=  0.60] at (107.40, 18.45) {0.4};

\node[text=drawColor,anchor=base,inner sep=0pt, outer sep=0pt, scale=  0.60] at (124.50, 18.45) {0.5};

\node[text=drawColor,anchor=base,inner sep=0pt, outer sep=0pt, scale=  0.60] at (141.60, 18.45) {0.6};

\node[text=drawColor,anchor=base,inner sep=0pt, outer sep=0pt, scale=  0.60] at (158.69, 18.45) {0.7};

\node[text=drawColor,anchor=base,inner sep=0pt, outer sep=0pt, scale=  0.60] at (175.79, 18.45) {0.8};

\node[text=drawColor,anchor=base,inner sep=0pt, outer sep=0pt, scale=  0.60] at (192.89, 18.45) {0.9};

\node[text=drawColor,anchor=base,inner sep=0pt, outer sep=0pt, scale=  0.60] at (209.99, 18.45) {1.0};
\end{scope}
\begin{scope}
\path[clip] (  0.00,  0.00) rectangle (448.07,180.67);
\definecolor{drawColor}{RGB}{0,0,0}

\node[text=drawColor,anchor=base,inner sep=0pt, outer sep=0pt, scale=  0.80] at (124.50,  7.44) {Average coverage probability};
\end{scope}
\begin{scope}
\path[clip] (  0.00,  0.00) rectangle (448.07,180.67);
\definecolor{drawColor}{RGB}{0,0,0}

\node[text=drawColor,rotate= 90.00,anchor=base,inner sep=0pt, outer sep=0pt, scale=  0.80] at ( 11.01,101.36) {Count};
\end{scope}
\begin{scope}
\path[clip] (224.04,  0.00) rectangle (448.07,180.67);
\definecolor{drawColor}{RGB}{255,255,255}
\definecolor{fillColor}{RGB}{255,255,255}

\path[draw=drawColor,line width= 0.6pt,line join=round,line cap=round,fill=fillColor] (224.04,  0.00) rectangle (448.07,180.68);
\end{scope}
\begin{scope}
\path[clip] (254.49, 27.54) rectangle (442.57,175.17);
\definecolor{fillColor}{RGB}{255,255,255}

\path[fill=fillColor] (254.49, 27.54) rectangle (442.57,175.17);
\definecolor{drawColor}{gray}{0.92}

\path[draw=drawColor,line width= 0.3pt,line join=round] (254.49, 49.16) --
	(442.57, 49.16);

\path[draw=drawColor,line width= 0.3pt,line join=round] (254.49, 78.99) --
	(442.57, 78.99);

\path[draw=drawColor,line width= 0.3pt,line join=round] (254.49,108.81) --
	(442.57,108.81);

\path[draw=drawColor,line width= 0.3pt,line join=round] (254.49,138.64) --
	(442.57,138.64);

\path[draw=drawColor,line width= 0.3pt,line join=round] (254.49,168.46) --
	(442.57,168.46);

\path[draw=drawColor,line width= 0.3pt,line join=round] (254.49, 27.54) --
	(254.49,175.17);

\path[draw=drawColor,line width= 0.3pt,line join=round] (271.59, 27.54) --
	(271.59,175.17);

\path[draw=drawColor,line width= 0.3pt,line join=round] (288.69, 27.54) --
	(288.69,175.17);

\path[draw=drawColor,line width= 0.3pt,line join=round] (305.79, 27.54) --
	(305.79,175.17);

\path[draw=drawColor,line width= 0.3pt,line join=round] (322.89, 27.54) --
	(322.89,175.17);

\path[draw=drawColor,line width= 0.3pt,line join=round] (339.99, 27.54) --
	(339.99,175.17);

\path[draw=drawColor,line width= 0.3pt,line join=round] (357.08, 27.54) --
	(357.08,175.17);

\path[draw=drawColor,line width= 0.3pt,line join=round] (374.18, 27.54) --
	(374.18,175.17);

\path[draw=drawColor,line width= 0.3pt,line join=round] (391.28, 27.54) --
	(391.28,175.17);

\path[draw=drawColor,line width= 0.3pt,line join=round] (408.38, 27.54) --
	(408.38,175.17);

\path[draw=drawColor,line width= 0.3pt,line join=round] (425.48, 27.54) --
	(425.48,175.17);

\path[draw=drawColor,line width= 0.3pt,line join=round] (442.57, 27.54) --
	(442.57,175.17);

\path[draw=drawColor,line width= 0.6pt,line join=round] (254.49, 34.25) --
	(442.57, 34.25);

\path[draw=drawColor,line width= 0.6pt,line join=round] (254.49, 64.07) --
	(442.57, 64.07);

\path[draw=drawColor,line width= 0.6pt,line join=round] (254.49, 93.90) --
	(442.57, 93.90);

\path[draw=drawColor,line width= 0.6pt,line join=round] (254.49,123.73) --
	(442.57,123.73);

\path[draw=drawColor,line width= 0.6pt,line join=round] (254.49,153.55) --
	(442.57,153.55);

\path[draw=drawColor,line width= 0.6pt,line join=round] (263.04, 27.54) --
	(263.04,175.17);

\path[draw=drawColor,line width= 0.6pt,line join=round] (280.14, 27.54) --
	(280.14,175.17);

\path[draw=drawColor,line width= 0.6pt,line join=round] (297.24, 27.54) --
	(297.24,175.17);

\path[draw=drawColor,line width= 0.6pt,line join=round] (314.34, 27.54) --
	(314.34,175.17);

\path[draw=drawColor,line width= 0.6pt,line join=round] (331.44, 27.54) --
	(331.44,175.17);

\path[draw=drawColor,line width= 0.6pt,line join=round] (348.53, 27.54) --
	(348.53,175.17);

\path[draw=drawColor,line width= 0.6pt,line join=round] (365.63, 27.54) --
	(365.63,175.17);

\path[draw=drawColor,line width= 0.6pt,line join=round] (382.73, 27.54) --
	(382.73,175.17);

\path[draw=drawColor,line width= 0.6pt,line join=round] (399.83, 27.54) --
	(399.83,175.17);

\path[draw=drawColor,line width= 0.6pt,line join=round] (416.93, 27.54) --
	(416.93,175.17);

\path[draw=drawColor,line width= 0.6pt,line join=round] (434.02, 27.54) --
	(434.02,175.17);
\definecolor{drawColor}{RGB}{169,169,169}
\definecolor{fillColor}{RGB}{190,190,190}

\path[draw=drawColor,line width= 0.6pt,line join=round,fill=fillColor] (267.32, 34.25) rectangle (275.87, 34.25);

\path[draw=drawColor,line width= 0.6pt,line join=round,fill=fillColor] (275.87, 34.25) rectangle (284.42, 34.25);

\path[draw=drawColor,line width= 0.6pt,line join=round,fill=fillColor] (284.42, 34.25) rectangle (292.97, 34.25);

\path[draw=drawColor,line width= 0.6pt,line join=round,fill=fillColor] (292.97, 34.25) rectangle (301.51, 34.25);

\path[draw=drawColor,line width= 0.6pt,line join=round,fill=fillColor] (301.51, 34.25) rectangle (310.06, 34.25);

\path[draw=drawColor,line width= 0.6pt,line join=round,fill=fillColor] (310.06, 34.25) rectangle (318.61, 35.74);

\path[draw=drawColor,line width= 0.6pt,line join=round,fill=fillColor] (318.61, 34.25) rectangle (327.16, 34.25);

\path[draw=drawColor,line width= 0.6pt,line join=round,fill=fillColor] (327.16, 34.25) rectangle (335.71, 35.14);

\path[draw=drawColor,line width= 0.6pt,line join=round,fill=fillColor] (335.71, 34.25) rectangle (344.26, 34.55);

\path[draw=drawColor,line width= 0.6pt,line join=round,fill=fillColor] (344.26, 34.25) rectangle (352.81, 34.55);

\path[draw=drawColor,line width= 0.6pt,line join=round,fill=fillColor] (352.81, 34.25) rectangle (361.36, 34.25);

\path[draw=drawColor,line width= 0.6pt,line join=round,fill=fillColor] (361.36, 34.25) rectangle (369.91, 36.04);

\path[draw=drawColor,line width= 0.6pt,line join=round,fill=fillColor] (369.91, 34.25) rectangle (378.46, 45.88);

\path[draw=drawColor,line width= 0.6pt,line join=round,fill=fillColor] (378.46, 34.25) rectangle (387.01, 55.42);

\path[draw=drawColor,line width= 0.6pt,line join=round,fill=fillColor] (387.01, 34.25) rectangle (395.55, 62.28);

\path[draw=drawColor,line width= 0.6pt,line join=round,fill=fillColor] (395.55, 34.25) rectangle (404.10, 70.63);

\path[draw=drawColor,line width= 0.6pt,line join=round,fill=fillColor] (404.10, 34.25) rectangle (412.65,120.44);

\path[draw=drawColor,line width= 0.6pt,line join=round,fill=fillColor] (412.65, 34.25) rectangle (421.20,141.92);

\path[draw=drawColor,line width= 0.6pt,line join=round,fill=fillColor] (421.20, 34.25) rectangle (429.75, 36.33);
\definecolor{drawColor}{gray}{0.20}

\path[draw=drawColor,line width= 0.6pt,line join=round,line cap=round] (254.49, 27.54) rectangle (442.57,175.17);
\end{scope}
\begin{scope}
\path[clip] (  0.00,  0.00) rectangle (448.07,180.67);
\definecolor{drawColor}{gray}{0.30}

\node[text=drawColor,anchor=base east,inner sep=0pt, outer sep=0pt, scale=  0.60] at (249.54, 32.18) {0};

\node[text=drawColor,anchor=base east,inner sep=0pt, outer sep=0pt, scale=  0.60] at (249.54, 62.01) {100};

\node[text=drawColor,anchor=base east,inner sep=0pt, outer sep=0pt, scale=  0.60] at (249.54, 91.83) {200};

\node[text=drawColor,anchor=base east,inner sep=0pt, outer sep=0pt, scale=  0.60] at (249.54,121.66) {300};

\node[text=drawColor,anchor=base east,inner sep=0pt, outer sep=0pt, scale=  0.60] at (249.54,151.48) {400};
\end{scope}
\begin{scope}
\path[clip] (  0.00,  0.00) rectangle (448.07,180.67);
\definecolor{drawColor}{gray}{0.20}

\path[draw=drawColor,line width= 0.6pt,line join=round] (251.74, 34.25) --
	(254.49, 34.25);

\path[draw=drawColor,line width= 0.6pt,line join=round] (251.74, 64.07) --
	(254.49, 64.07);

\path[draw=drawColor,line width= 0.6pt,line join=round] (251.74, 93.90) --
	(254.49, 93.90);

\path[draw=drawColor,line width= 0.6pt,line join=round] (251.74,123.73) --
	(254.49,123.73);

\path[draw=drawColor,line width= 0.6pt,line join=round] (251.74,153.55) --
	(254.49,153.55);
\end{scope}
\begin{scope}
\path[clip] (  0.00,  0.00) rectangle (448.07,180.67);
\definecolor{drawColor}{gray}{0.20}

\path[draw=drawColor,line width= 0.6pt,line join=round] (263.04, 24.79) --
	(263.04, 27.54);

\path[draw=drawColor,line width= 0.6pt,line join=round] (280.14, 24.79) --
	(280.14, 27.54);

\path[draw=drawColor,line width= 0.6pt,line join=round] (297.24, 24.79) --
	(297.24, 27.54);

\path[draw=drawColor,line width= 0.6pt,line join=round] (314.34, 24.79) --
	(314.34, 27.54);

\path[draw=drawColor,line width= 0.6pt,line join=round] (331.44, 24.79) --
	(331.44, 27.54);

\path[draw=drawColor,line width= 0.6pt,line join=round] (348.53, 24.79) --
	(348.53, 27.54);

\path[draw=drawColor,line width= 0.6pt,line join=round] (365.63, 24.79) --
	(365.63, 27.54);

\path[draw=drawColor,line width= 0.6pt,line join=round] (382.73, 24.79) --
	(382.73, 27.54);

\path[draw=drawColor,line width= 0.6pt,line join=round] (399.83, 24.79) --
	(399.83, 27.54);

\path[draw=drawColor,line width= 0.6pt,line join=round] (416.93, 24.79) --
	(416.93, 27.54);

\path[draw=drawColor,line width= 0.6pt,line join=round] (434.02, 24.79) --
	(434.02, 27.54);
\end{scope}
\begin{scope}
\path[clip] (  0.00,  0.00) rectangle (448.07,180.67);
\definecolor{drawColor}{gray}{0.30}

\node[text=drawColor,anchor=base,inner sep=0pt, outer sep=0pt, scale=  0.60] at (263.04, 18.45) {0.0};

\node[text=drawColor,anchor=base,inner sep=0pt, outer sep=0pt, scale=  0.60] at (280.14, 18.45) {0.1};

\node[text=drawColor,anchor=base,inner sep=0pt, outer sep=0pt, scale=  0.60] at (297.24, 18.45) {0.2};

\node[text=drawColor,anchor=base,inner sep=0pt, outer sep=0pt, scale=  0.60] at (314.34, 18.45) {0.3};

\node[text=drawColor,anchor=base,inner sep=0pt, outer sep=0pt, scale=  0.60] at (331.44, 18.45) {0.4};

\node[text=drawColor,anchor=base,inner sep=0pt, outer sep=0pt, scale=  0.60] at (348.53, 18.45) {0.5};

\node[text=drawColor,anchor=base,inner sep=0pt, outer sep=0pt, scale=  0.60] at (365.63, 18.45) {0.6};

\node[text=drawColor,anchor=base,inner sep=0pt, outer sep=0pt, scale=  0.60] at (382.73, 18.45) {0.7};

\node[text=drawColor,anchor=base,inner sep=0pt, outer sep=0pt, scale=  0.60] at (399.83, 18.45) {0.8};

\node[text=drawColor,anchor=base,inner sep=0pt, outer sep=0pt, scale=  0.60] at (416.93, 18.45) {0.9};

\node[text=drawColor,anchor=base,inner sep=0pt, outer sep=0pt, scale=  0.60] at (434.02, 18.45) {1.0};
\end{scope}
\begin{scope}
\path[clip] (  0.00,  0.00) rectangle (448.07,180.67);
\definecolor{drawColor}{RGB}{0,0,0}

\node[text=drawColor,anchor=base,inner sep=0pt, outer sep=0pt, scale=  0.80] at (348.53,  7.44) {Average coverage probability};
\end{scope}
\begin{scope}
\path[clip] (  0.00,  0.00) rectangle (448.07,180.67);
\definecolor{drawColor}{RGB}{0,0,0}

\node[text=drawColor,rotate= 90.00,anchor=base,inner sep=0pt, outer sep=0pt, scale=  0.80] at (235.05,101.36) {Count};
\end{scope}
\end{tikzpicture}

%% file: discuss.tex
\section{Discussion and Future Work}
\label{sec:discussion}

This paper has proposed a method called HiGrad for statistical inference in online learning. This novel procedure, compared with SGD, attaches a confidence interval to its predictions without incurring additional computational or memory cost. The HiGrad confidence interval has been rigorously shown to asymptotically achieve the correct coverage probability for smooth (locally) strongly convex objectives. Moreover, the associated estimator has the same asymptotic variance as the vanilla SGD and can even attain the Cram\'er--Rao lower bound in the case of model specification. In both simulations and a real data example, HiGrad is empirically observed to yield good finite-sample performance using a default set of structural parameters derived by balancing the two competing criteria, namely sharing and contrasting. In the spirit of reproducibility, code to generate the figures in the paper is available at \url{http://stat.wharton.upenn.edu/~suw/higrad}.

HiGrad admits several potential theoretical refinements and practical extensions. The theoretical results presented in Section \ref{sec:method-1} are asymptotic in nature. As such, it is tempting to investigate the finite-sample properties of HiGrad, and doing so might provide insights to improve the coverage properties of the confidence intervals. This extension requires significant technical improvement as the Ruppert--Polyak averaging scheme is asymptotic in nature. That being said, a promising approach is to carry over the finite-sample results in \cite{chen2016} in the HiGrad setting. A direction related, or roughly equivalent, to the prior one is to generalize the HiGrad procedure to the high-dimensional setting where the dimension of the unknown vector $\theta^\ast$ can increase. From a technical perspective, it requires to extend our main technical ingredient, the Ruppert--Polyak normality result, to the case of a growing dimension. An interesting step towards this direction has been explored in \cite{gadat2017optimal}. In the non-convex setting, the multiple threads of HiGrad might help increase the odds of escaping saddle points. A preliminary study of this important question has been provided by \cite{sordello2019robust}, which however does not attempt to construct confidence intervals for non-convex stochastic optimization.

In essence, the HiGrad algorithm provides a broad class of templates for online learning. It would be interesting to investigate how to best parallelize this new algorithm and to explore the use of the HiGrad in addition to uncertainty quantification, for example, treating the confidence interval length as a stopping criterion. More broadly, any variants of SGD would presumably carry over to HiGrad, and the question to ask is how to obtain some form of uncertainty quantification of the results. In conjunction with applying HiGrad, variants worth considering include adaptive strategies for
choosing step sizes \citep{duchi2011adaptive,kingma2014adam}, variance reduction techniques \citep{johnson2013accelerating,defazio2014saga}, normality of the last SGD iterate \citep{toulis2014statistical}, and the implicit SGD \citep{toulis2017asymptotic}. Moreover, for
non-smooth or non-convex problems (for example, SVM, online EM \citep{lange1995gradient,cappe2009line}, and multilayer neural networks \citep{lecun2015deep}), although exact normality results are unlikely to hold and the averaged iterates are often replaced by the last iterate, particularly in the non-convex setting, the hope is that the splitting strategy might help HiGrad get a panoramic view of the landscape of the objective function, allowing it to better understand its algorithmic variability.


%% file: appendix_weak.tex
\section{Proofs Under Assumption \ref{ass:cvx_global}}
\label{sec:appendix}
This appendix provides a self-contained proof of Lemma \ref{lm:pj}, which is a Donsker-style generalization of the Ruppert--Polyak theorem. The original Ruppert--Polyak normality follows as a byproduct. Throughout Appendix \ref{sec:appendix}, we work on Assumption \ref{ass:cvx_global}, which is stronger than Assumption \ref{ass:cvx} in the main text. This stronger assumption helps better highlight the main ideas of this celebrated normality result. Later in Appendix \ref{sec:extend-proof-lemma}, we move back to Assumption \ref{ass:cvx} and the proof only needs minor modifications to the present one. The proofs presented in this appendix make use of a range of ideas in \cite{polyak1992,moulines2011non,fort2012central}. 


Below, we present the assumption adopted throughout Appendix \ref{sec:appendix}.

\begin{customassm}{1'}[Global strong convexity of $f$]\label{ass:cvx_global}
The objective function $f(\theta)$ is continuously differentiable and strongly convex with parameter $\rho > 0$, that is, for any $\theta_1, \theta_2$,
\[
f(\theta_2) \ge f(\theta_1) + (\theta_2 - \theta_1)^\top \nabla f(\theta_1) + \frac{\rho}{2} \|\theta_2 - \theta_1\|^2.
\]
In addition, assume that $\nabla f$ is Lipschitz continuous:
\[
\|\nabla f(\theta_1) - \nabla f(\theta_2)\| \le L \|\theta_1 -\theta_2\|
\]
for some $L > 0$. Last, the Hessian of $f$ exists and is Lipschitz continuous in a neighborhood of $\theta^\ast$, that is, there is some $\delta_1 > 0$ such that
\[
\left\| \nabla^2 f(\theta) - \nabla^2 f(\theta^\ast) \right\| \le L' \|\theta - \theta^\ast\|_2
\]
for some $L' > 0$ if $\|\theta - \theta^\ast\| \le \delta_1$.
\end{customassm}

Before moving to the proofs, we introduce below a much less restrictive assumption on the step sizes compared with the one given in the main text. In fact, both the proofs in the present appendix and Appendix \ref{sec:extend-proof-lemma} holds with a broader class of step sizes as characterized below. The choice $\gamma_j = c_1(j + c_2)^{-\alpha}$ for $\alpha \in (0.5, 1)$ in the main body of the paper is included as a special example.

\begin{assumption}[Slowly decaying step sizes]\label{ass:size}
Assume that the sequence of positive step sizes $\{\gamma_j\}_{j=1}^{\infty}$ that obey
\begin{align}
&\sum_{j=1}^{\infty} \gamma_j^2 < \infty \label{eq:sq_finite}\\
&\sum_{j=1}^\infty \frac{\gamma_j}{\sqrt j} < \infty \label{eq:r_need}\\
& \lim_{j \goto \infty} j \gamma_j = \infty \label{eq:step_large}\\
&\lim_{j \goto \infty}\frac1{\gamma_j} \log\frac{\gamma_j}{\gamma_{j+1}} = 0 \label{eq:slow_decay}\\
& \lim_{j\goto\infty} \frac1{\sqrt j} \sum_{l=1}^{j} \frac{1}{\sqrt{\gamma_l}} \left| \frac{\gamma_l}{\gamma_{l+1}} - 1 \right| = 0
\label{eq:theta_diff_need}.
\end{align}

\end{assumption}

Some direct consequences of this assumption are given as a lemma below. Its proof is deferred to Section \ref{sec:proof-lemma-refxxxx}.

\begin{lemma}\label{lm:step_all}
Let $\{\gamma_j\}_{j=1}^{\infty}$ be an arbitrary sequence of positive numbers. Then each of the following statements is true. 
\begin{enumerate}
\item 
Equation \eqref{eq:slow_decay} implies that
\[
\sum_{j=1}^{\infty} \gamma_j = \infty.
\]

\item
If $\gamma_j$ is a non-increasing sequence, then \eqref{eq:r_need} implies \eqref{eq:sq_finite}.

\item
If $\gamma_j$ is a non-increasing sequence and $\gamma_{j+1}/\gamma_j$ is bounded below away from 0, then \eqref{eq:step_large} implies \eqref{eq:theta_diff_need}.

\end{enumerate}
\end{lemma}

\subsection{Proof of Lemma \ref{lm:pj}}
\label{sec:proof-lemma-reflm:pj}

It suffices to prove Lemma \ref{lm:pj} in the case of $K = 1$. We will prove: Under Assumptions \ref{ass:cvx}, \ref{ass:reg}, and \ref{ass:size}, we have that the SGD iterates obey
\[
\sqrt{n_1} \left( \overline \theta_{n_1} - \theta^\ast \right)
\]
and
\[
\sqrt{n_2} \left( \overline \theta_{n_1+1 : n_1 + n_2} - \theta^\ast \right)
\]
are asymptotically distributed as two \iid normal random variables with mean 0. In fact, below we prove a stronger version of the  normality of the Ruppert--Polyak averaging scheme.

Let $Z_1, Z_2, \ldots, Z_{n_0}$ and $Z'_1, Z'_2, \ldots, Z'_{n_1}$ be $n_0 + n_1$ i.i.d.~random variables with the same distribution as $Z$, and consider the following iterations:
\[
\theta_j = \theta_{j-1} - \gamma_j g(\theta_{j-1}, Z_j)
\]
for $j = 1, \ldots, n_0$ and
\[
\theta_i' = \theta_{i-1}' - \gamma_{i}' g(\theta_{i-1}', Z_i')
\]
for $i = 1, \ldots, n_1$, with $\theta'_0 = \theta_{n_0}$. Above, $\gamma_i' = \gamma_{n_0+i}$. Assume both $n_0$ and $n_1$ tend to infinity.

We write
\[
\epsilon_j = g(\theta_{j-1}, Z_j) - \nabla f(\theta_{j-1})
\]
for $j = 1, \ldots, n_0$ and
\[
\epsilon'_i = g(\theta'_{i-1}, Z'_i) - \nabla f(\theta'_{i-1})
\]
for $i = 1, \ldots, n_1$. Note that $\theta_j$ is adapted to the filtration $\mathcal{F}_j := \sigma(Z_1, \ldots, Z_j)$ and $\theta'_i$ is adapted to the filtration $\mathcal{F}_i' := \sigma(\mathcal F^{\infty}, Z'_1, \ldots, Z'_i)$, where $\mathcal F^{\infty} := \cup_{j} \F_j$. 

Now, we write the SGD update as
\[
\theta_j = \theta_{j-1}  - \gamma_j \nabla f(\theta_{j-1}) - \gamma_j\epsilon_j,
\]
which can be alternatively written as
\begin{equation}\label{eq:1grad_alt}
\nabla f(\theta_{j-1})  = \frac{\theta_{j-1} - \theta_j}{\gamma_j} - \epsilon_j.
\end{equation}
Intuitively, assuming that $\theta_{j-1}$ is close to $\theta^\ast$, then $\nabla f(\theta_{j-1})$ is approximately equal to $H(\theta - \theta^\ast)$ (recall the notation $H = \nabla^2 f(\theta^\ast)$). This suggests us to write
\begin{equation}\label{eq:non_quad_e}
\nabla f(\theta) =  H (\theta - \theta^\ast) + r_{\theta},
\end{equation}
where $r_{\theta}$ shall be shown to be sufficiently small later. Making use of \eqref{eq:1grad_alt} and summing \eqref{eq:non_quad_e} over $\theta_{j-1}$ for $j = 1, \ldots, n_0$ give
\[
\sum_{j=1}^{n_0} H(\theta_{j-1} - \theta^\ast) = - \underbrace{\sum_{j=1}^{n_0} \epsilon_j}_{I_1} - \underbrace{\sum_{j=1}^{n_0} r_{\theta_{j-1}}}_{I_2} + \underbrace{\sum_{j=1}^{n_0} \frac{\theta_{j-1} - \theta_j}{\gamma_j}}_{I_3}
\]
and, similarly,
\[
\sum_{i=1}^{n_1} H(\theta'_{i-1} - \theta^\ast) = - \underbrace{\sum_{i=1}^{n_1} \epsilon'_i}_{I_1'} - \underbrace{\sum_{i=1}^{n_1} r_{\theta'_{i-1}}}_{I_2'} + \underbrace{\sum_{i=1}^{n_1} \frac{\theta'_{i-1} - \theta'_i}{\gamma'_{i}}}_{I_3'}.
\]

Below, we state three lemmas characterizing $I_1, I_1', I_2, I_2', I_3$, and $I_3'$. While the first lemma shows that both re-scaled $I_1$ and $I_1'$ jointly converge to two \iid normals with zero mean, the other two lemmas say that $I_2, I_2', I_3$, and $I_3'$ are negligible with appropriate scaling. Taking the three lemmas as given for the moment, one has
\[
\begin{aligned}
\sqrt{n_0} \left(\frac1{n_0} \sum_{j=1}^{n_0}\theta_j - \theta^\ast \right) &= \frac1{\sqrt{n_0}} \sum_{j=1}^{n_0}(\theta_{j-1} - \theta^\ast) - \frac{\theta_0 - \theta_{n_0}}{\sqrt{n_0}}\\
&= \frac1{\sqrt{n_0}} \sum_{j=1}^{n_0}(\theta_{j-1} - \theta^\ast) - o_{\P}(1)\\
&= -H^{-1} \frac{I_1}{\sqrt n_0} - H^{-1} \frac{I_2}{\sqrt n_0} + H^{-1} \frac{I_3}{\sqrt n_0} - o_{\P}(1)\\
&= -H^{-1} \frac{I_1}{\sqrt n_0} - o_{\P}(1) + o_{\P}(1) - o_{\P}(1)\\
&= -H^{-1} \frac{I_1}{\sqrt n_0} + o_{\P}(1),\\
\end{aligned}
\]
where we make use of the fact that $\frac{\theta_0 - \theta_{n_0}}{\sqrt{n_0}} = \frac{\theta_0 - \theta^\ast + o_{\P}(1)}{\sqrt{n_0}} = o_{\P}(1)$ by Lemma \ref{lm:consis}. Recognizing the fact $\frac{I_1}{\sqrt n_0} \Rightarrow \N(0, V)$ (recall that $V = \E g(\theta^\ast, Z) g(\theta^\ast, Z)^\top$) given by Lemma \ref{lm:normal}, we readily get
\[
\sqrt{n_0} \left(\frac1{n_0} \sum_{j=1}^{n_0}\theta_j - \theta^\ast \right) \Rightarrow \N(0, H^{-1} V H^{-1}).
\]
Likewise, 
\[
\sqrt{n_1} \left(\frac1{n_1} \sum_{i=1}^{n_1}\theta'_i - \theta^\ast \right) \Rightarrow \N(0, H^{-1} V H^{-1}),
\]
and the asymptotic independence is implied by the second half of Lemma \ref{lm:normal}.

The discussion above indicates that the proof of Lemma \ref{lm:pj} would be completed once we establish these three lemmas. This is the subject of Sections \ref{sec:norm-sum_j=1n}, \ref{sec:negl-sum_j=1n-r_th-1}, and \ref{sec:negl-other-term}. We write $\xi_n = o_{\P}(1)$ if $\xi_n \Rightarrow 0$ weakly.

Throughout this section, assume both $n_0, n_1 \goto \infty$ and $n_0/n_1$ converges to a number in $(0, \infty)$.
\begin{lemma}[Normality of $I_1$ and $I_1'$]\label{lm:normal}
Under Assumptions \ref{ass:cvx_global}, \ref{ass:reg}, and \ref{ass:size}, then
\[
\frac1{\sqrt{n_0}}\sum_{j=1}^{n_0}\epsilon_j \Rightarrow \N(0, V), \quad \text{and }\frac1{\sqrt{n_1}}\sum_{i=1}^{n_1}\epsilon_i' \Rightarrow \N(0, V).
\]
Moreover, they are asymptotically independent.
\end{lemma}

\begin{lemma}[Negligibility of $I_2$ and $I_2'$]\label{lm:i2_small}
Under Assumptions \ref{ass:cvx_global}, \ref{ass:reg}, and \ref{ass:size}, then
\[
\frac1{\sqrt{n_0}}\sum_{j=1}^{n_0} r_{\theta_{j-1}} = o_{\P}(1)
\]
and
\[
\frac1{\sqrt{n_1}}\sum_{i=1}^{n_1} r_{\theta'_{i-1}} = o_{\P}(1).
\]

\end{lemma}

\begin{lemma}[Negligibility of $I_3$ and $I_3'$]\label{lm:theta_diff_small}
Under Assumptions \ref{ass:cvx_global}, \ref{ass:reg}, and \ref{ass:size}, then
\[
\frac1{\sqrt{n_0}}\sum_{j=1}^{n_0} \frac{\theta_{j-1} - \theta_j}{\gamma_j} = o_{\P}(1)
\]
and
\[
\frac1{\sqrt{n_1}}\sum_{i=1}^{n_1} \frac{\theta_{i-1}' - \theta_i'}{\gamma'_i} = o_{\P}(1).
\]
\end{lemma}

\subsection{Proof of Lemma \ref{lm:normal}}
\label{sec:norm-sum_j=1n}

By definition, $\{\epsilon_j\}_{j=1}^{n_0}$ is a martingale difference with respect to $\{\F_j\}_{j=1}^{n_0}$:
\[
\E(\epsilon_j | \mathcal F_{j-1}) = 0
\]
for $1 \le j \le n_0$ (as a convention, set $\F_0 = \{\emptyset, \Omega\}$ if $\theta_0$ is deterministic, otherwise $\F_0 = \sigma(\theta_0)$ ) and, similarly, 
\[
\E(\epsilon_i' | \mathcal F'_{i-1}) = 0
\]
for $1 \le i \le n_1$.

The lemma below is a martingale equivalent of the central limit theorem (CLT). As a convention, set $\mathcal G_0 = \{\emptyset, \Omega\}$. The proof of this lemma is standard (for example, using characteristic functions) and is thus omitted. Interested readers can find its proof, for example, in \cite{hall2014martingale}. 
\begin{lemma}[Martingale difference CLT in the Lyapunov form]\label{lm:mtgl}
Let $\{M_l\}_{l=1}^{\infty}$ be a martingale difference adapted to a filtration $\{\mathcal G_l\}_{l=1}^{\infty}$ satisfying
\begin{equation}\label{eq:mtgl1}
\frac1{n} \sum_{l=1}^n \E(M_l^2 | \mathcal G_{l-1}) \Rightarrow \sigma^2
\end{equation}
for some constant $\sigma^2 \ge 0$ as $n \goto \infty$ and
\begin{equation}\label{eq:mtgl2}
\frac1{n^{1+\kappa/2}} \sum_{l=1}^n \E (M_l^{2+\kappa}) \goto 0
\end{equation}
for some constant $\kappa > 0$ as $n \goto \infty$. Then, this martingale difference satisfies
\[
\frac{\sum_{l=1}^n M_l}{\sqrt{n}} \Rightarrow \N(0,\sigma^2)
\]
as $n \goto \infty$.
\end{lemma}

Note that this lemma includes $\sigma^2 = 0$ as an example. In that case, we interpret $\N(0,\sigma^2)$ as a point mass at $0$. It can be proved using the theory of characteristics functions. Before turning to the proof of Lemma \ref{lm:normal}, we state the following two lemmas. Lemma \ref{lm:consis} claims that the SGD iterates are consistent for $\theta^\ast$ and its proof is provided at the end of the present section. The proof relies heavily on Lemma \ref{lm:r_s}, a well-known auxiliary result in stochastic approximation. Interested readers can find the proof of Lemma \ref{lm:r_s} in \cite{robbins1985convergence}.

\begin{lemma}[Robbins--Siegmund theorem]\label{lm:r_s}
Let $\{D_l, \beta_l, \eta_l, \zeta_l \}_{l=1}^{\infty}$ be non-negative and adapted to a filtration $\{ \mathcal G_l \}_{l=1}^{\infty}$. Assume
\[
\E[D_{l+1} | \mathcal G_l] \le (1 + \beta_l) D_l + \eta_l - \zeta_l
\]
for all $l \ge 1$ and, in addition, both $\sum \beta_l < \infty$ and $\sum \eta_l < \infty$ almost surely. Then, with probability one, $D_l$ converges to a random variable $0 \le D_{\infty} < \infty$ and $\sum \zeta_l < \infty$.
\end{lemma}

\begin{lemma}\label{lm:consis}
Under Assumptions \ref{ass:cvx_global}, \ref{ass:reg}, and \ref{ass:size}, we have $\theta_l \goto \theta^\ast$ almost surely.

\end{lemma}
As a consequence of the lemma above, for any $\upsilon > 0$, the cardinality $\#\{l: \|\theta_l - \theta^\ast \| > \upsilon\}$ is finite almost surely. This lemma is mostly used through this fact. Next, we present the proof of Lemma \ref{lm:normal}.

\begin{proof}[Proof of Lemma \ref{lm:normal}]
Fix some $0 < \nu \le \delta$ define
\[
\widetilde\epsilon_j = 
\begin{cases}
\epsilon_j, &\text{if } \|\theta_{j-1} - \theta^\ast\| \le \nu\\
g(\theta^\ast, Z_j), &\text{if } \|\theta_{j-1} - \theta^\ast\| > \nu\\
\end{cases}
\]
for $j = 1 , \ldots, n_0$ and
\[
\widetilde\epsilon_i' = 
\begin{cases}
\epsilon'_i, &\text{if } \|\theta'_{i-1} - \theta^\ast\| \le \nu\\
g(\theta^\ast, Z'_i), &\text{if } \|\theta'_{i-1} - \theta^\ast\| > \nu\\
\end{cases}
\]
for $i = 1 , \ldots, n_1$. Our approach is to prove that
\begin{equation}\label{eq:tilde_norm}
\frac1{\sqrt{n_0}}\sum_{j=1}^{n_0} \widetilde\epsilon_j \Rightarrow \N(0, V)
\end{equation}
and
\begin{equation}\label{eq:diff_zero}
\frac1{\sqrt{n_0}}\sum_{j=1}^{n_0} \epsilon_j - \frac1{\sqrt{n_0}}\sum_{j=1}^{n_0} \widetilde\epsilon_j \Rightarrow 0,
\end{equation}
and
\begin{equation}\label{eq:tilde_norm2}
\frac1{\sqrt{n_1}}\sum_{i=1}^{n_1} \widetilde\epsilon_i' \Rightarrow \N(0, V)
\end{equation}
and
\begin{equation}\label{eq:diff_zero2}
\frac1{\sqrt{n_1}}\sum_{i=1}^{n_1} \epsilon'_i - \frac1{\sqrt{n_1}}\sum_{i=1}^{n_1} \widetilde\epsilon'_i \Rightarrow 0.
\end{equation}

Below, the first step is to prove \eqref{eq:tilde_norm} and \eqref{eq:diff_zero}. Then, going forward, we will prove \eqref{eq:tilde_norm2} and \eqref{eq:diff_zero2}, meanwhile showing that \eqref{eq:tilde_norm} and \eqref{eq:tilde_norm2} are asymptotically independent. This shall complete the proof of Lemma \ref{lm:normal}.

To show \eqref{eq:diff_zero}, note that Lemma \ref{lm:consis} ensures that, almost surely, the number of $1 \le j < \infty$ such that $\epsilon_j$ and $\widetilde\epsilon_j$ differ is finite. 

Now, we first turn to prove \eqref{eq:tilde_norm}. For any fixed vector $a$ of the same dimension as $\epsilon_j$, to prove \eqref{eq:tilde_norm} it suffices to show that
\[
\frac1{\sqrt{n_0}}\sum_{j=1}^{n_0} a^\top \widetilde\epsilon_j \Rightarrow \N(0, a^\top V a^\top).
\]
Write $M_j = a^\top \widetilde\epsilon_j$. We aim to verify \eqref{eq:mtgl1} and \eqref{eq:mtgl2} in Lemma \ref{lm:mtgl}. Observe that
\[
\begin{aligned}
\frac1{n_0} \sum_{j=1}^{n_0} \E(M_j^2 | \mathcal F_{j-1}) &= \frac1{n_0} a^\top \left[ \sum_{j=1}^{n_0} \E(\widetilde\epsilon_j \widetilde\epsilon_j^\top |\mathcal F_{j-1}) \right] a\\
\end{aligned}
\]
By construction, it obeys
\[
\E(\widetilde\epsilon_j \widetilde\epsilon_j^\top |\mathcal F_{j-1}) =
\begin{cases}
\E_{\theta_{j-1}} \epsilon \epsilon^\top, & \text{if } \|\theta_{j-1} - \theta^\ast\| \le \nu\\
V, & \text{if } \|\theta_{j-1} - \theta^\ast\| >  \nu,
\end{cases}
\]
which, together with Assumption \ref{ass:reg}, implies that it always satisfies
\[
\left\| \E(\widetilde\epsilon_j \widetilde\epsilon_j^\top |\mathcal F_{j-1}) - V \right\| \le C \min\{ \|\theta_{j-1} - \theta^\ast\|, \nu \}
\]
for all $j \ge 1$. As a result,
\[
\left\| \frac1{n_0} \sum_{j=1}^{n_0} \E(\widetilde\epsilon_j \widetilde\epsilon_j^\top |\mathcal F_{j-1}) - V \right\| \le \frac{C}{n_0} \sum_{j=1}^{n_0} \min\{ \|\theta_{j-1} - \theta^\ast\|, \nu \}.
\]
By Lemma \ref{lm:consis}, the right-hand side term above diminishes in probability since $\theta_j \goto \theta^\ast$ almost surely, that is,
\[
\frac{C}{n_0} \sum_{j=1}^{n_0} \min\{ \|\theta_{j-1} - \theta^\ast\|, \nu \} \goto 0
\]
almost surely. Hence,
\[
\frac1{n_0} a^\top \left[ \sum_{j=1}^{n_0} \E(\widetilde\epsilon_j \widetilde\epsilon_j^\top |\mathcal F_{j-1}) \right] a \goto a^\top V a
\]
almost surely. Thus, \eqref{eq:mtgl1} is satisfied for the martingale difference $\widetilde\epsilon_j$. Next, we turn to verify \eqref{eq:mtgl2}. By the last point in Assumption \ref{ass:reg}, taking $\kappa = \delta$ gives
\[
\begin{aligned}
\E(|M_j|^{2+\delta} | \mathcal{F}_{j-1}) &= \E(|a^\top \widetilde\epsilon_j|^{2+\delta} \big| \mathcal{F}_{j-1})\\
&= \E(|a^\top \widetilde\epsilon_j|^{2+\delta} \big| \mathcal{F}_{j-1})\\
&\le \|a\|^{2+\delta} \E(\|\widetilde\epsilon_j\|^{2+\delta} \big| \mathcal{F}_{j-1})\\
& \lesssim \|a\|^{2+\delta}.
\end{aligned}
\]
Hence, we get
\[
\begin{aligned}
\frac1{n_0^{1+\delta/2}} \sum_{j=1}^{n_0} \E |M_j|^{2+\delta}  \lesssim \frac1{n_0^{1+\delta/2}} \sum_{j=1}^{n_0} \|a\|^{2+\delta} = \frac{\|a\|^{2+\delta}}{n_0^{\delta/2}},
\end{aligned}
\]
which clearly diminishes to zero as $n_0 \goto 0$. Summarizing these results validates \eqref{eq:tilde_norm}.

Now, we move to \eqref{eq:tilde_norm2} and \eqref{eq:diff_zero2}. Conditional on $\theta'_0 \equiv \theta_{n_0}$, the proof of \eqref{eq:tilde_norm} and \eqref{eq:diff_zero} can seamlessly carry over to this case. And, if we can further show that convergence in \eqref{eq:tilde_norm2} and \eqref{eq:diff_zero2} does not depend on the initial point $\theta_{n_0}$ as $n_0 \goto \infty$, then we would both verify the desired independence between \eqref{eq:tilde_norm} and \eqref{eq:tilde_norm2} and establish the unconditional convergence in \eqref{eq:tilde_norm2} and \eqref{eq:diff_zero2}. To see this, note that Lemma \ref{lm:consis} ensures that $\theta'_0 =  \theta^\ast + o_{\P}(1)$ as $n_0 \goto \infty$. (Actually, a stronger result holds
\[
\sup_{1 \le i < \infty} \|\theta'_i - \theta^\ast\| = o_{\P}(1)
\]
as $n_0 \goto \infty$ due to the fact that $\{Z_j\}_{j=1}^{\infty}$ has the same distribution as $\{Z_j\}_{j=1}^{n_0} \cup \{Z_i'\}_{i=1}^{\infty}$.) Put it differently, with probability tending to one, the second segment of the SGD starts at a point uniformly close to $\theta^\ast$, implying the convergence in \eqref{eq:tilde_norm2} and \eqref{eq:diff_zero2} is asymptotically independent of $\theta_{n_0}$. This completes the proof of this lemma.

\end{proof}

\begin{remark}
The fact that $n_0 \asymp n_1$ is not used in the proof above. As a matter of fact, the proof can be significantly simplified provided that $n_0 \asymp n_1$ by invoking Donsker's theorem for martingales.
\end{remark}

We conclude this section by presenting the proof of Lemma \ref{lm:consis}.
\begin{proof}[Proof of Lemma \ref{lm:consis}]
Write $\Delta_l$ for the suboptimality $f(\theta_l) - f^\ast$. By the $L$-smoothness of $f$, we get
\[
f(\theta_l) = f(\theta_{l-1} - \gamma_l g(\theta_{l-1}, Z_l)) \le f(\theta_{l-1}) - \gamma_l g(\theta_{l-1}, Z_l)^\top \nabla f(\theta_{l-1}) + \frac{L}{2} \|\gamma_{l} g(\theta_{l-1}, Z_l)\|^2,
\]
from which we get
\[
\begin{aligned}
\E(\Delta_l | \F_{l-1}) &= \E(f(\theta_l) - f^\ast | \F_{l-1}) \\
&\le \Delta_{l-1} - \gamma_l \|\nabla f(\theta_{l-1})\|^2 + \E\left[\frac{L}{2} \|\gamma_{l} g(\theta_{l-1}, Z_l)\|^2 \Big| \F_{l-1}\right]\\
&= \Delta_{l-1} - \gamma_l \|\nabla f(\theta_{l-1})\|^2 + \frac{L\gamma_l^2 \|\nabla f (\theta_{l-1})\|^2}{2} + \frac{L\gamma_l^2}{2}\E_{\theta_{l-1}} \|\epsilon\|^2\\
&\le \Delta_{l-1} - \gamma_l \times 2\rho \Delta_{l-1} + \frac{L\gamma_l^2 \times 2L \Delta_{l-1}}{2} + \frac{L\gamma_l^2}{2} \times c'(1 + \|\theta_{l-1} - \theta^\ast\|^2)\\
&\le \Delta_{l-1} - \gamma_l \times 2\rho \Delta_{l-1} + \frac{L\gamma_l^2 \times 2L \Delta_{l-1}}{2} + \frac{c'L\gamma_l^2}{2} + \frac{c'L\gamma_l^2}{2} \frac{2\Delta_{l-1}}{\rho}\\
\end{aligned}
\]
where we have made use of the inequalities $2\rho(f(\theta) - f^\ast) \le \|\nabla f(\theta)\|^2 \le 2L (f(\theta) - f^\ast)$. Rearranging the inequality above, we get
\begin{equation}\label{eq:r_s_later}
\E(\Delta_l | \F_{l-1}) \le (1 + c_1 \gamma_l^2)\Delta_{l-1} + c_2 \gamma_l^2 - c_3 \gamma_l \Delta_{l-1},
\end{equation}
where
\[
c_1 = L^2 + \frac{c'L}{\rho}, c_2 = \frac{c'L}{2}, c_3 = 2\rho.
\]

To conclude the proof, we need to apply the Robbins--Siegmund theorem (Lemma \ref{lm:r_s}). Observe that, by Assumption \ref{ass:size},
\[
\sum_{l=1}^{\infty} c_1 \gamma_l^2  < \infty, \text{and } \sum_{l=1}^{\infty} c_2 \gamma_l^2 < \infty
\]
Hence, it follows from the Robbins--Siegmund theorem that, almost surely, $\Delta_l \equiv f(\theta_l) - f^\ast$ converges to a finite random variable, say, $\Delta_{\infty} \ge 0$. Moreover, this theorem ensures
\begin{equation}\label{eq:siegmund1}
\sum_{l=1}^{\infty} c_3 \gamma_l \Delta_{l-1} < \infty.
\end{equation}
If $\P(\Delta_{\infty} > 0) > 0$, then the left-hand side of \eqref{eq:siegmund1} would be infinite with positive probability due to the fact $\sum_{l=1}^{\infty} \gamma_l = \infty$, a contradiction to \eqref{eq:siegmund1}. This reveals that $f(\theta_l) \goto f^\ast$ almost surely and, as a consequence, $\theta_l \goto \theta^\ast$ with probability one.

\end{proof}

\subsection{Proof of Lemma \ref{lm:i2_small}}
\label{sec:negl-sum_j=1n-r_th-1}
Recognizing that $n_0 \asymp n_1$ and the relationship that
\[
\frac1{\sqrt{n_1}}\sum_{i=1}^{n_1} r_{\theta'_{i-1}} =\sqrt{\frac{n_0+n_1}{n_1}} \cdot  \frac1{\sqrt{n_0 + n_1}}\sum_{l=1}^{n_0+n_1} r_{\theta_{l-1}} - \sqrt{\frac{n_0}{n_1}} \cdot  \frac1{\sqrt{n_0}}\sum_{j=1}^{n_0} r_{\theta_{j-1}},
\]
where the convention $\theta_l = \theta'_{l-n_0}$ for $l \ge n_0$ is made. Thus, it suffices to prove that
\[
\frac1{\sqrt{n}}\sum_{l=1}^{n} \|r_{\theta_{l-1}}\| = o_{\P}(1)
\]
as $n \goto \infty$.

Recall that $r_{\theta} = \nabla f(\theta) - H(\theta - \theta^\ast)$, where $H$ is the Hessian of $f$ at $\theta^\ast$. Using the vector-valued mean value theorem, we get for any $\theta$
\[
\begin{aligned}
\|r_{\theta}\| &= \left\| \nabla f(\theta) - H (\theta - \theta^\ast)  \right\| \\
&= \left\| \left[\nabla f(\theta) - H (\theta - \theta^\ast) \right] - \left[\nabla f(\theta^\ast) - H (\theta^\ast - \theta^\ast) \right] \right\|\\
& \le \left\| \left( \nabla^2 f(\theta') - H \right) (\theta - \theta^\ast)\right\|,
\end{aligned}
\]
where $\theta' = c \theta^\ast + (1 - c)\theta$ for some $0 < c < 1$. To proceed, note that from Assumption \ref{ass:cvx}, if $\|\theta - \theta^\ast\| \le \delta$, then
\[
\begin{aligned}
\left\| \left( \nabla^2 f(\theta') - H \right) (\theta - \theta^\ast)\right\| &\le \left\| \nabla^2 f(\theta') - H \right\| \left\|\theta - \theta^\ast \right\|\\
& \le L' \left\|\theta' - \theta^\ast \right\| \left\|\theta - \theta^\ast \right\|\\
& \le L' \left\|\theta - \theta^\ast \right\| \left\|\theta - \theta^\ast \right\|\\
& = L' \left\|\theta - \theta^\ast \right\|^2.
\end{aligned}
\]
That is, $\|r_{\theta}\| \le L' \|\theta - \theta^\ast\|^2$ if $\|\theta - \theta^\ast\| \le \delta$. In general, we have
\[
\begin{aligned}
\| r_{\theta} \| &\le \|\nabla f(\theta)\| + \|H(\theta - \theta^\ast)\| \\
&\le L \|\theta - \theta^\ast\| + L \|\theta - \theta^\ast\|\\
&\le 2L \|\theta - \theta^\ast\|,
\end{aligned}
\]
no matter $\|\theta - \theta^\ast\| \le \delta$ or $\|\theta - \theta^\ast\| > \delta$. Applying the results above yields
\[
\begin{aligned}
\sum_{l=1}^{n} \| r_{\theta_{l-1}} \| &\le 2L\sum_{l=1}^{n} \|\theta_l - \theta^\ast\| \ind_{\|\theta_{l-1} - \theta^\ast\| > \delta} + L'\sum_{l=1}^{n} \|\theta_{l-1} - \theta^\ast\|^2 \ind_{\|\theta_{l-1} - \theta^\ast\| \le \delta}\\
&\le 2L\sum_{l=1}^{n} \|\theta_l - \theta^\ast\| \ind_{\|\theta_{l-1} - \theta^\ast\| > \delta} + L'\sum_{l=1}^{n} \|\theta_{l-1} - \theta^\ast\|^2.
\end{aligned}
\]
By Lemma \ref{lm:consis}, we have
\[
\frac{2L}{\sqrt{n}}\sum_{l=1}^{n} \|\theta_l - \theta^\ast\| \ind_{\|\theta_{l-1} - \theta^\ast\| > \delta} \goto 0
\]
as $n \goto \infty$ with probability one. Hence, it suffices to show that $\sum_{l=1}^{n} \|\theta_{l-1} - \theta^\ast\|^2/\sqrt{n} = o_{\P}(1)$, which is implied by
\begin{equation}\label{eq:r_small1}
\frac1{\sqrt n}\sum_{l=1}^{n} \E \|\theta_l - \theta^\ast\|^2 = o(1).
\end{equation}

Now, the proof of Lemma \ref{lm:i2_small} is reduced to showing \eqref{eq:r_small1}. Before presenting the proof, we list two useful lemmas. The first lemma is due to Leopold Kronecker and a proof can be found in \cite{shiryaev1996probability}, and the proof of Lemma \ref{lm:series} is given at the end of this section for self-containedness.

\begin{lemma}[Kronecker's lemma]\label{lm:kronecker}
Let $\{a_l\}_{l=1}^\infty$ be an infinite sequence that has a convergent sum $\sum_{l=1}^\infty a_l$. Then, for an arbitrary $\{b_l\}_{l=1}^\infty$ satisfying $0 < b_1 \leq b_2 \leq b_3 \leq \cdots$ and $b_l \to \infty$, it must hold
\[
\lim_{n \to \infty}\frac1{b_n}\sum_{l=1}^n b_l a_l = 0.
\]
\end{lemma}

\begin{lemma}\label{lm:series}
Let $c_1$ and $c_2$ be arbitrary positive constants. Under Assumption \ref{ass:size} and $\gamma_j \goto 0$, if $B_j > 0$ obeys
\[
B_l \le \frac{\gamma_{l-1}(1 - c_1 \gamma_l)}{\gamma_l} B_{l-1} + c_2 \gamma_l,
\]
then $\sup_l B_l < \infty$.
\end{lemma}

\begin{proof}[Proof of Lemma \ref{lm:i2_small}]

We begin by pointing out that \eqref{eq:r_small1} is implied by
\begin{equation}\label{eq:r_gamma}
\E\|\theta_l - \theta^\ast\|^2 \le C \gamma_l
\end{equation}
for some fixed $C$ and all $l \ge 1$. To see this, note that \eqref{eq:r_gamma} gives
\[
\begin{aligned}
\frac1{\sqrt n}\sum_{l=1}^{n} \E \|\theta_l - \theta^\ast\|^2 &\le \frac1{\sqrt n}\sum_{l=1}^{n} C\gamma_l\\
&\le C \cdot\frac1{\sqrt n}\sum_{l=1}^{n} \sqrt{l} \cdot \frac{\gamma_l}{\sqrt l}.
\end{aligned}
\]
Assumption \ref{ass:size} says that $\sum_{l=1}^{\infty} \frac{\gamma_l}{\sqrt l}$ converges. Hence, taking $b_l = \sqrt{l}$, Kronecker's lemma readily yields
\[
\frac1{\sqrt n}\sum_{l=1}^{n} \sqrt{l} \cdot \frac{\gamma_l}{\sqrt l} \goto 0.
\]

The rest of the proof is devoted to verifying \eqref{eq:r_gamma}. Rewriting \eqref{eq:r_s_later} and taking expectations on both sides give (recall that $\Delta_l = f(\theta_l) - f^\ast$)
\[
\E \Delta_l  \le (1 -c_3\gamma_l + c_1 \gamma_l^2) \E\Delta_{l-1} + c_2 \gamma_l^2,
\]
which is equivalent to
\[
\frac{\E \Delta_l}{\gamma_l}  \le \frac{\gamma_{l-1}(1 -c_3\gamma_l + c_1 \gamma_l^2)}{\gamma_l} \frac{\E\Delta_{l-1}}{\gamma_{l-1}} + c_2 \gamma_l.
\]
Since $\gamma_l \goto 0$, then for sufficiently large $l$, we have $c_1 \gamma_l^2 < c_3\gamma_l/2$. Plugging this inequality to the display above gives
\begin{equation}\label{eq:delta_gamma_d}
\frac{\E \Delta_l}{\gamma_l}  \le \frac{\gamma_{l-1}(1 - 0.5 c_3\gamma_l)}{\gamma_l} \frac{\E\Delta_{l-1}}{\gamma_{l-1}} + c_2 \gamma_l
\end{equation}
for sufficiently large $l$. With \eqref{eq:delta_gamma_d} in place, Lemma \ref{lm:series} immediately concludes that
\[
\sup_{1 \le l < \infty}\frac{\E \Delta_l}{\gamma_l} < \infty
\]
or, equivalent,
\[
\E (f(\theta_l) - f^\ast) \le C' \gamma_l
\]
for all $l$ and some constant $C'$. Since $f(\theta_l) - f^\ast \ge \frac{\rho}{2}\|\theta_l - \theta^\ast\|^2$ due to the strong convexity of $f$, it follows that
\[
\E \|\theta_l - \theta^\ast\|^2 \le \frac{2C'}{\rho} \gamma_l.
\]
This completes the proof of the lemma.

\end{proof}

\begin{proof}[Proof of Lemma \ref{lm:series}]
Suppose on the contrary that $\sup B_l = \infty$. Consider a sequence $\{A_l\}$ defined as
\[
A_l = \frac{\gamma_{l-1}(1 - c_1 \gamma_l)}{\gamma_l} A_{l-1} + c_2 \gamma_l
\]
for all $l \ge 1$ (set $\gamma_0$ to some appropriate constant). It is clear that $A_l \ge B_l$ for all $l$. Together with the assumption $\sup B_l = \infty$, this implies that $\sup A_l = \infty$. Now, observe that
\[
\begin{aligned}
\frac{\gamma_{l-1}(1 - c_1 \gamma_l)}{\gamma_l} A_{l-1} + c_2 \gamma_l &= \left(1 + \frac{\gamma_{l-1}-\gamma_l}{\gamma_l} \right) (1 - c_1\gamma_l)A_{l-1} + c_2 \gamma_l\\
&= (1 + o(\gamma_l))(1 - c_1 \gamma_l)A_{l-1} + c_2 \gamma_l\\
&= \left(1 - (c_1-o(1)) \gamma_l \right) A_{l-1} + c_2 \gamma_l\\
&= A_{l-1} - \left[\left( c_1-o(1) \right) A_{l-1} - c_2 \right] \gamma_l.
\end{aligned}
\]
That is,
\[
A_l = A_{l-1} - \left[\left( c_1-o(1) \right) A_{l-1} - c_2 \right] \gamma_l.
\]
Thus, once $A_l \ge (1+o(1))c_2/c_1$ for some $l$, this sequence starts to decreases until it falls below the cutoff $(1+o(1))c_2/c_1$. Therefore, this sequence can not diverge to $\infty$.

\end{proof}

\subsection{Proof of Lemma \ref{lm:theta_diff_small}}
\label{sec:negl-other-term}

As earlier in the proof of Lemma \ref{lm:i2_small}, we only need to prove that
\[
\frac1{\sqrt{n}}\sum_{l=1}^{n} \frac{\theta_{l-1} - \theta_l}{\gamma_l} = o_{\P}(1)
\]
as $n \goto \infty$. Applying the Abel summation, we get
\[
\sum_{l=1}^n \frac{\theta_{l-1} - \theta_l}{\gamma_l} = \sum_{l=1}^{n-1} (\theta_l - \theta^\ast)(\gamma_{l+1}^{-1} - \gamma_l^{-1}) - (\theta_n - \theta^\ast) \gamma_n^{-1} + (\theta_0 - \theta^\ast) \gamma_1^{-1}.
\]
Recognizing \eqref{eq:step_large} in Assumption \ref{ass:size} and \eqref{eq:r_gamma}, we get
\[
\frac1{\sqrt n} \frac{\E \|\theta_n - \theta^\ast\|}{\gamma_n} \le \frac{\sqrt{\E \|\theta_n - \theta^\ast\|^2}}{\sqrt{n} \gamma_n} \le \frac{\sqrt{C \gamma_n}}{ \sqrt{n} \gamma_n} = \sqrt{\frac{C}{n \gamma_n}} \goto 0,
\]
which together with
\[
\frac1{\sqrt{n}} \frac{\theta_0 - \theta^\ast}{\gamma_1} \goto 0
\]
demonstrates that it suffices to show that
\begin{equation}\nonumber
\frac1{\sqrt n}\sum_{l=1}^{n-1} \left| \gamma_{l+1}^{-1} - \gamma_l^{-1} \right| \|\theta_l - \theta^\ast\| \goto 0.
\end{equation}
The display above immediately follows from
\begin{equation}\label{eq:gamma_inv}
\lim_{n \goto \infty} \frac1{\sqrt n}\sum_{l=1}^{n-1} \left| \gamma_{l+1}^{-1} - \gamma_l^{-1} \right| \E \|\theta_l - \theta^\ast\| = 0.
\end{equation}

Next, we turn to prove \eqref{eq:gamma_inv}. In fact, \eqref{eq:r_gamma} gives
\begin{equation}\nonumber
\begin{aligned}
\frac1{\sqrt n}\sum_{l=1}^{n-1} |\gamma_{l+1}^{-1} - \gamma_l^{-1}| \E \|\theta_l - \theta^\ast\| &\le \frac1{\sqrt n}\sum_{l=1}^{n-1} |\gamma_{l+1}^{-1} - \gamma_l^{-1}| \sqrt{\E \|\theta_l - \theta^\ast\|^2}\\
&\le \frac1{\sqrt n}\sum_{l=1}^{n-1} |\gamma_{l+1}^{-1} - \gamma_l^{-1}| \sqrt{C\gamma_l}\\
&\le \sqrt{\frac{C}{n}}\sum_{l=1}^{n-1}\gamma_l^{\frac12} \left|\gamma_{l+1}^{-1} - \gamma_l^{-1}\right|\\
& \goto 0
\end{aligned}
\end{equation}
as $n \goto \infty$. This is given by \eqref{eq:theta_diff_need} in Assumption \ref{ass:size}, thereby establishing \eqref{eq:gamma_inv}.

\subsection{Proof of Lemma \ref{lm:step_all}}
\label{sec:proof-lemma-refxxxx}

\begin{proof}[Proof of Lemma \ref{lm:step_all}]
We prove the three statements one by one, as follows.
\begin{enumerate}
\item 
Given \eqref{eq:slow_decay}, suppose on the contrary that
\begin{equation}\label{eq:slow_decay_oppo}
\sum \gamma_j < \infty.
\end{equation}
Let
\[
a_j = \frac1{\gamma_j} \log\frac{\gamma_j}{\gamma_{j+1}}.
\]
Then, from this definition we get
\[
\gamma_{m+1} =  \gamma_1 \e^{-a_m \gamma_m - \cdots - a_1 \gamma_1}.
\]
The exponent, namely $-a_m \gamma_m - \cdots - a_1 \gamma_1$, satisfies
\[
\begin{aligned}
-a_m \gamma_{m} - \cdots - a_1 \gamma_1 &\ge -\sup_{l \ge 1} |a_l| \sum_{l=1}^{m} \gamma_l\\
&\ge -\sup_{l \ge 1} |a_l| \sum_{l=1}^{\infty} \gamma_l.\\
\end{aligned}
\]
Due to \eqref{eq:slow_decay} and \eqref{eq:slow_decay_oppo}, it must have a finite $\sup_{l \ge 1} |a_l| \sum_{l=1}^{\infty} \gamma_l$. As a consequence,
\[
\gamma_{m+1} =  \gamma_1 \e^{-a_m \gamma_m - \cdots - a_1 \gamma_1} \ge \gamma_1 \exp\left[ -{\sup_{l \ge 1} |a_l| \sum_{l=1}^{\infty} \gamma_l} \right].
\]
This contradicts \eqref{eq:slow_decay_oppo}. Therefore, \eqref{eq:slow_decay} implies
\[
\sum_{j=1}^\infty \gamma_j = \infty.
\]

\item
From \eqref{eq:r_need} we can assume that $\sum \gamma_j/\sqrt{j} = C \in (0, \infty)$. Now, we consider the problem of maximizing
\[
\sum_{j=1}^\infty \gamma_j^2
\]
over all $\gamma_j$ satisfying $\gamma_1 \ge \gamma_2 \ge \cdots \ge 0$ and $\sum \gamma_j/\sqrt{j} = C$. To this end, recognize that $\sum_{j=1}^\infty \gamma_j^2$ is a convex function and the feasible set is convex. This implies that the function must attain the maximum at a vertex of the feasible set
\[
\left\{ \{\gamma_j\}_{j=1}^\infty: \gamma_1 \ge \gamma_2 \ge \cdots \ge 0, \sum_{j=1}^{\infty} \frac{\gamma_j}{\sqrt{j}} = C \right\}.
\]
It can be shown that a vertex must take the form $\gamma_1 = \cdots = \gamma_l, \gamma_{l+1} = \cdots = 0$ for some $l \ge 1$. Hence,
\[
\gamma_1 = \cdots = \gamma_l = \frac{C}{\sum_{j=1}^l j^{-\frac12}}.
\] 
Further, the maximum of $\sum_{j=1}^\infty \gamma_j^2$ must be
\[
C^2 \sup_l \frac{ l}{(\sum_{j=1}^l j^{-\frac12})^2} = C^2 \frac{1}{(\sum_{j=1}^1 j^{-\frac12})^2} = C^2,
\]
which is finite. Thus, \eqref{eq:sq_finite} holds.

\item 
The proof of this part can be found on page 24 of \cite{fort2012central}. For self-containedness, we complete the proof here. Due to the non-increasing of the sequence $\gamma_j$, it follows that
\[
\begin{aligned}
\sum_{j=1}^{n} \frac{1}{\sqrt{\gamma_j}} \left| \frac{\gamma_j}{\gamma_{j+1}} - 1 \right| &= \sum_{j=1}^{n} \frac{1}{\sqrt{\gamma_j}} \left( \frac{\gamma_j}{\gamma_{j+1}} - 1 \right) \\
& = \sum_{j=1}^{n} \sqrt{\gamma_j} \left( \frac1{\gamma_{j+1}} - \frac1{\gamma_j} \right) \\
& = \sum_{j=2}^{n+1} \frac1{\gamma_j} \left(\sqrt{\gamma_{j-1}} - \sqrt{\gamma_j}\right) - \frac1{\sqrt{\gamma_1}} + \frac1{\sqrt{\gamma_{n+1}}}\\
& \lesssim \sum_{j=2}^{n+1} \frac1{\sqrt{\gamma_j\gamma_{j-1}}} \left(\sqrt{\gamma_{j-1}} - \sqrt{\gamma_j}\right) - \frac1{\sqrt{\gamma_1}} + \frac1{\sqrt{\gamma_{n+1}}}\\
&= \frac2{\sqrt{\gamma_{n+1}}} - \frac2{\sqrt{\gamma_1}},
\end{aligned}
\]
where $\lesssim$ follows from the boundedness of $\gamma_{j-1}/\gamma_j$. Hence, the display above together with \eqref{eq:step_large} gives
\[
\frac1{\sqrt n}\sum_{j=1}^{n} \frac{1}{\sqrt{\gamma_j}} \left| \frac{\gamma_j}{\gamma_{j+1}} - 1 \right| \goto 0
\]
as $n \goto \infty$.
\end{enumerate}
\end{proof}


%% file: appendix_strong.tex
\section{Proofs Under Assumption \ref{ass:cvx}}
\label{sec:extend-proof-lemma}

Appendix \ref{sec:extend-proof-lemma} presents a proof of Lemma \ref{lm:pj} under the less restrictive Assumption \ref{ass:cvx}, which only assumes a form of local strong convexity of the objective function. The proof is built on top of the one given in the previous appendix.

\subsection{Proof of Lemma \ref{lm:consis} with local strong convexity}
\label{sec:proof-lemma-reflm:c}
Denote by $\widetilde\Delta_l = \|\theta_l - \theta^\ast\|^2$. Recognizing the SGD update \eqref{eq:sgd}, we get
\begin{equation}\label{eq:wide_delta}
\begin{aligned}
\E(\widetilde\Delta_l | \F_{l-1}) &= \|\theta_{l-1} - \theta^\ast - \gamma_l \nabla f(\theta_{l-1})\|^2 + \gamma_l^2 \E_{\theta_{l-1}} \|\epsilon\|^2\\
&= \|\theta_{l-1} - \theta^\ast\|^2  - 2 \gamma_l  (\theta_{l-1} - \theta^\ast)^\top \nabla f(\theta_{l-1}) + \gamma_l^2 \| \nabla f(\theta_{l-1})\|^2+ \gamma_l^2 \E_{\theta_{l-1}} \|\epsilon\|^2.
\end{aligned}
\end{equation}
By Assumption \ref{ass:reg}, for any $\theta$,
\[
\E_{\theta} \|\epsilon\|^2 \le C'(1 + \|\theta - \theta^\ast\|^2)
\]
for some $C' > 0$. Then, we get from \eqref{eq:wide_delta} that
\begin{equation}\label{eq:wide_delta_fi}
\begin{aligned}
\E(\widetilde\Delta_l | \F_{l-1}) &\le \|\theta_{l-1} - \theta^\ast\|^2  - 2 \gamma_l  (\theta_{l-1} - \theta^\ast)^\top \nabla f(\theta_{l-1}) + \gamma_l^2 \| \nabla f(\theta_{l-1})\|^2+ C'\gamma_l^2 (1 + \|\theta_{l-1} - \theta^\ast\|^2)\\
&\le \|\theta_{l-1} - \theta^\ast\|^2  - 2 \gamma_l  (\theta_{l-1} - \theta^\ast)^\top \nabla f(\theta_{l-1}) + \gamma_l^2 L^2 \| \theta_{l-1} - \theta^\ast\|^2+ C'\gamma_l^2 (1 + \|\theta_{l-1} - \theta^\ast\|^2)\\
&\le (1 + L^2\gamma_l^2 + C'\gamma_l^2)\|\theta_{l-1} - \theta^\ast\|^2 + C'\gamma_l^2 - 2 \gamma_l  (\theta_{l-1} - \theta^\ast)^\top \nabla f(\theta_{l-1})
\end{aligned}
\end{equation}
To proceed, we need a lemma.
\begin{lemma}\label{lm:str_other}
Let $F$ be differentiable convex function defined on a Euclidean space and $F(x) - \frac{\rho}{2}\|x\|^2 $ is convex on the ball $\{x: \|x - x^\ast\| \le r\}$ centered at the minimizer $x^\ast$ of $F$ and $r > 0$. Then,
\[
(x - x^\ast)^\top \nabla F(x) \ge \rho \|x - x^\ast\| \min\{\|x - x^\ast\|, r\}.
\]
\end{lemma}

By Assumption \ref{ass:cvx}, we see that $f(\theta)$ is strongly convex at a neighborhood of $\theta^\ast$. Hence, there there exists $\delta' > 0$ such that $f(\theta) - \frac{\delta'}{2}\|\theta\|^2$ is convex on $\{\theta: \|\theta - \theta^\ast\| \le \delta'\}$. Thus, applying Lemma \ref{lm:str_other}, together with \eqref{eq:wide_delta_fi} gives
\begin{equation}\label{eq:expand_delta}
\E(\widetilde\Delta_l | \F_{l-1}) \le (1 + L^2\gamma_l^2 + C'\gamma_l^2)\|\theta_{l-1} - \theta^\ast\|^2  + C'\gamma_l^2 - 2 \gamma_l  \delta' \|\theta_{l-1} - \theta^\ast\| \min\{ \|\theta_{l-1} - \theta^\ast\|, \delta' \}.
\end{equation}
Since both $\sum (L^2\gamma_l^2 + C'\gamma_l^2) < 0$ and $\sum C'\gamma_l^2 < \infty$, Lemma \ref{lm:r_s} shows that $\widetilde\Delta_l \equiv \|\theta_l - \theta^\ast\|^2$ converges to a random variable, say, $\widetilde\Delta_{\infty}$, almost surely, and 
\[
\sum_{l=1}^\infty 2 \gamma_l  \delta' \|\theta_{l-1} - \theta^\ast\| \min\{ \|\theta_{l-1} - \theta^\ast\|, \delta' \} < \infty
\]
almost surely. This implies that, almost surely,
\[
\sum_{l=1}^\infty 2 \gamma_l  \delta' \sqrt{\widetilde\Delta_{\infty}} \min\{\widetilde\Delta_{\infty}, \delta' \} < \infty,
\]
which together with the fact that $\sum \gamma_l = \infty$ yields that
\[
\widetilde\Delta_{\infty} = 0
\]
almost surely. This reveals that $\theta_l \goto \theta^\ast$ with probability one.

\begin{proof}[Proof of Lemma \ref{lm:str_other}]
First, consider the case where $\|x - x^\ast\| \le r$. The gradient of $F(x) - \frac{\rho}{2} \|x\|^2$ is a monotone operator on $\{x: \|x - x^\ast\| \le r\}$ due to its convexity. Hence, we get
\[
\left\langle x - x^\ast, \nabla F(x) - \rho x - \nabla F(x^\ast) + \rho x^\ast \right\rangle \ge 0,
\]
which can be written as
\begin{equation}\label{eq:monotone_inside}
(x - x^\ast)^\top \nabla F(x) \ge \rho \|x - x^\ast\|^2.
\end{equation}

Now, consider $x$ such that $\|x - x^\ast\| > r$. Denote by $\widetilde x$ the projection of $x$ onto the ball $\{x: \|x - x^\ast\| \le r\}$. Then, using the property of monotone operator $\nabla F$ gives
\[
\left\langle x - \widetilde x, \nabla F(x) - \nabla F(\widetilde x) \right\rangle \ge 0,
\]
from which we get
\[
(x - \widetilde x)^\top \nabla F(x) \ge (x - \widetilde x)^\top \nabla F(\widetilde x).
\]
To proceed, note that \eqref{eq:monotone_inside} is also satisfied for $\widetilde x$. That is,
\[
(\widetilde x - x^\ast)^\top \nabla F(\widetilde x) \ge \rho \|\widetilde x - x^\ast\|^2 = \rho r^2.
\]
From the geometry of projection, it follows that
\[
(x - \widetilde x)^\top \nabla F(x) = \frac{\|x - x^\ast\| - r}{\|x - x^\ast\|}(x - x^\ast)^\top \nabla F(x)
\]
and
\[
(x - \widetilde x)^\top \nabla F(\widetilde x) = \frac{\|x - x^\ast\| - r}{r} (\widetilde x - x^\ast)^\top \nabla F(\widetilde x).
\]
Taking all the displays above gives
\[
\begin{aligned}
(x - x^\ast)^\top \nabla F(x) &= \frac{\|x - x^\ast\|}{\|x - x^\ast\| - r} (x - \widetilde x)^\top \nabla F(x)\\
& \ge  \frac{\|x - x^\ast\|}{\|x - x^\ast\| - r} (x - \widetilde x)^\top \nabla F(\widetilde x)\\
& =  \frac{\|x - x^\ast\|}{\|x - x^\ast\| - r} \cdot \frac{\|x - x^\ast\| - r}{r} (\widetilde x - x^\ast)^\top \nabla F(\widetilde x)\\
& =  \frac{\|x - x^\ast\|}{r} \rho r^2\\
& =  \rho r \|x - x^\ast\|,
\end{aligned}
\]
as desired.

\end{proof}

\subsection{Proof of Lemma \ref{lm:i2_small} without (\ref{eq:r_gamma})}
\label{sec:proofx}
As in the proof presented in the preceding section, $f$ is $\delta'$-strongly convex on $\{\theta: \|\theta - \theta^\ast\| \le \delta'\}$. Denote by
\[
\tau_{m} := \inf_{l \ge m} \left\{ l: \|\theta_l - \theta^\ast\| > \delta' \right\}.
\]
Note that $\tau_{m}$ is a stopping time adapted to $\mathcal F = \{\mathcal F_l\}_{l = 1}^{\infty}$. Denote by $\widehat\Delta_l = \|\theta_l - \theta^\ast\|^2 \ind_{\tau_{m} > l}$. Hence, for $l \ge m+1$, using \eqref{eq:expand_delta} we have
\[
\begin{aligned}
\E \left[ \widehat\Delta_l | \mathcal F_{l-1}\right] &\le \E \left[\|\theta_l - \theta^\ast\|^2 \ind_{\tau_{m} > l-1} | \mathcal F_{l-1} \right]\\
&= \ind_{\tau_{m} > l-1}  \E \left[\|\theta_l - \theta^\ast\|^2 | \mathcal F_{l-1} \right]\\
& \le \ind_{\tau_{m} > l-1}  \left[(1 + C''\gamma_l^2)\|\theta_{l-1} - \theta^\ast\|^2  + C'\gamma_l^2 - 2 \gamma_l  \delta' \|\theta_{l-1} - \theta^\ast\| \min\{ \|\theta_{l-1} - \theta^\ast\|, \delta' \} \right]\\
& = \ind_{\tau_{m} > l-1}  \left[(1 + C''\gamma_l^2)\|\theta_{l-1} - \theta^\ast\|^2  + C'\gamma_l^2 - 2 \gamma_l  \delta' \|\theta_{l-1} - \theta^\ast\|^2 \right]\\
& = \ind_{\tau_{m} > l-1}  \left[(1 + C''\gamma_l^2 - 2\delta'\gamma_l )\|\theta_{l-1} - \theta^\ast\|^2  + C'\gamma_l^2 \right]\\
& \le (1 + C''\gamma_l^2 - 2\delta'\gamma_l )\widehat\Delta_{l-1}  + C'\gamma_l^2.
\end{aligned}
\]
Since $\gamma_l \goto 0$, for sufficiently large $l$, say, $l \ge l_0$, we get
\[
\E \widehat\Delta_l \le (1 - \delta'\gamma_l ) \E \widehat\Delta_{l-1}  + C'\gamma_l^2,
\]
which is equivalent to
\[
\frac{\E \widehat\Delta_l}{\gamma_l} \le \frac{\gamma_{l-1}(1 - \delta'\gamma_l )}{\gamma_l} \frac{\E \widehat\Delta_{l-1}}{\gamma_{l-1}}  + C'\gamma_l.
\]
Making use of Lemma \ref{lm:series} gives
\begin{equation}\label{eq:delta_bound}
\sup_{1 \le l < \infty} \frac{\E \widehat\Delta_l}{\gamma_l} < \infty.
\end{equation}

Recall that our goal is to prove
\begin{equation}\label{eq:r_small1_local}
\frac1{\sqrt n}\sum_{l=1}^{n} \|\theta_l - \theta^\ast\|^2 = o_{\P}(1).
\end{equation}
Write the display above as
\[
\frac1{\sqrt n}\sum_{l=1}^{n} \|\theta_l - \theta^\ast\|^2 =  \frac1{\sqrt n}\sum_{l=1}^m \|\theta_l - \theta^\ast\|^2  + \frac1{\sqrt n}\sum_{l=m+1}^{n} \|\theta_l - \theta^\ast\|^2.
\]
Let $\A_{m}$ be the event that $\tau_{m} = \infty$. We see
\[
\begin{aligned}
\E \left[ \frac1{\sqrt n}\sum_{l=m+1}^n \|\theta_{l} - \theta^\ast\|^2; \A_{m} \right] &\le \frac1{\sqrt n}\sum_{l=m+1}^{n} \E \left[  \|\theta_l - \theta^\ast\|^2; \tau_{m} > l \right]\\
&= \frac1{\sqrt n} \sum_{l=m+1}^n \E \widehat\Delta_l\\
&\lesssim \frac1{\sqrt n}\sum_{l=m}^{n-1} \gamma_l,
\end{aligned}
\]
which, by Kronecker's lemma, goes to zero as $n \goto \infty$. By Lemma \ref{lm:consis}, $\P(A_{m,\delta}) \goto 1$ as $m \goto \infty$. Note that
\[
\P\left( \frac1{\sqrt n}\sum_{l=m+1}^{n} \|\theta_l - \theta^\ast\|^2 \ne \frac1{\sqrt n}\sum_{l=m+1}^{n} \|\theta_l - \theta^\ast\|^2 \ind_{\A_m} \right) = \P(\overline\A_m)
\]
Recognize $\P(\overline\A_m) \goto 0$ as $m \goto \infty$ by Lemma \ref{lm:consis} (its extension in the previous section). Hence, by first taking $n \goto \infty$ and then $m \goto \infty$ we complete the proof of \eqref{eq:r_small1_local}. 

\subsection{Proof of Lemma \ref{lm:theta_diff_small} without (\ref{eq:r_gamma})}
As earlier, it suffices to show that
\[
\frac1{\sqrt{n}}\sum_{l=1}^n \frac{\theta_{l-1} - \theta_l}{\gamma_l} = \frac1{\sqrt{n}}\sum_{l=1}^{n-1} (\theta_l - \theta^\ast)(\gamma_{l+1}^{-1} - \gamma_l^{-1}) - \frac1{\sqrt{n}}(\theta_n - \theta^\ast) \gamma_n^{-1} + \frac1{\sqrt{n}}(\theta_0 - \theta^\ast) \gamma_1^{-1} = o_{\P}(1)
\]
as $n \goto \infty$. This is a consequence of
\begin{equation}\label{eq:without_here}
\frac1{\sqrt{n}}\sum_{l=m}^{n-1} \|\theta_l - \theta^\ast \| |\gamma_{l+1}^{-1} - \gamma_l^{-1}|  = o_{\P}(1), \quad \frac1{\sqrt{n}} \|\theta_n - \theta^\ast\| \gamma_n^{-1} = o_{\P}(1).
\end{equation}
Recall that $\A_{m}$ denotes the event that $\|\theta_l - \theta^\ast\| \le \delta'$ for all $l \ge m$. Then, note that
\[
\begin{aligned}
\frac1{\sqrt n}\sum_{l=m}^{n-1} |\gamma_{l+1}^{-1} - \gamma_l^{-1}| \E \left[ \|\theta_l - \theta^\ast\|; \A_m \right] &\le \frac1{\sqrt n}\sum_{l=m}^{n-1} |\gamma_{l+1}^{-1} - \gamma_l^{-1}| \sqrt{\P(\A_m)\E \left[ \|\theta_l - \theta^\ast\|^2; \A_m \right]}\\
&\le \frac1{\sqrt n}\sum_{l=m}^{n-1} |\gamma_{l+1}^{-1} - \gamma_l^{-1}| \sqrt{\E \left[ \|\theta_l - \theta^\ast\|^2; \A_m \right]}\\
&\le \frac1{\sqrt n}\sum_{l=m}^{n-1} |\gamma_{l+1}^{-1} - \gamma_l^{-1}| \sqrt{\E \left[ \|\theta_l - \theta^\ast\|^2; \tau_{m} > l \right]}\\
&\lesssim \frac1{\sqrt n}\sum_{l=m}^{n-1} |\gamma_{l+1}^{-1} - \gamma_l^{-1}| \gamma_l^{\frac12},
\end{aligned}
\]
where the last inequality makes use of \eqref{eq:delta_bound}. From \eqref{eq:theta_diff_need} in Assumption~\ref{ass:size} it follows that 
\[
\frac1{\sqrt n}\sum_{l=m}^{n-1} |\gamma_{l+1}^{-1} - \gamma_l^{-1}| \gamma_l^{\frac12} \goto 0
\]
as $n \goto \infty$. Hence, \eqref{eq:without_here} holds on the complement of $\A_m$. Recognizing that $\P(\A_m) \goto 1$ as $m \goto \infty$, the proof is completed.



%% file: appendix_mis.tex
\section{Other Proofs}
\label{sec:other-proofs}

\subsection{Verifying assumptions for logistic regression, ridge regression, and Huber regression}
\label{sec:verify-assumpt-logis}

Before verifying the assumptions, we state the following lemma, which will be helpful in verifying the second assumption. 

\begin{lemma}\label{lem:sufficient_assumption2}
Suppose that $\mathbb E\|\nabla f(\theta^*, Z)\|^2$ and $\mathbb E\|\nabla^2 f(\theta, Z)\|^2$ are bounded for all $\theta$.
Then
\[
\|\mathbb E\, \epsilon\epsilon^\top - V\| \leq C  \|\theta - \theta^*\| + C'  \|\theta - \theta^*\|^2
\]
for some constant $C$ and $C'$.
\end{lemma}
\begin{proof}[Proof of Lemma \ref{lem:sufficient_assumption2}]
To validate Assumption 2, we write
\begin{align*}
&\mathbb E\, \epsilon\epsilon^\top \\
&= \mathbb E\left( (\nabla f(\theta, Z) - \nabla f(\theta))(\nabla f(\theta, Z) - \nabla f(\theta))^\top\right)\\
& = \mathbb E\left( (\nabla f(\theta, Z) - \nabla f(\theta) - \nabla f(\theta^*, Z) + \nabla f(\theta^*, Z))(\nabla f(\theta, Z) - \nabla f(\theta) - \nabla f(\theta^*, Z) + \nabla f(\theta^*, Z))^\top\right)\\
& = \mathbb E\left( (\nabla f(\theta, Z) - \nabla f(\theta) - \nabla f(\theta^*, Z))(\nabla f(\theta, Z) - \nabla f(\theta) - \nabla f(\theta^*, Z))^\top\right)\\
& \quad + 2\mathbb E\left(\nabla f(\theta^*, Z)(\nabla f(\theta, Z) - \nabla f(\theta) - \nabla f(\theta^*, Z))^\top\right)
+ \mathbb E\left(\nabla f(\theta^*, Z)\nabla f(\theta^*, Z)^\top\right).
\end{align*}
Therefore, noting that $V = \mathbb E\left(\nabla f(\theta^*, Z)\nabla f(\theta^*, Z)^\top\right)$
\begin{align*}
&\|\mathbb E\, \epsilon\epsilon^\top - V\| \\
&\leq \underbrace{\left\|\mathbb E\left( (\nabla f(\theta, Z) - \nabla f(\theta) - \nabla f(\theta^*, Z))(\nabla f(\theta, Z) - \nabla f(\theta) - \nabla f(\theta^*, Z))^\top\right)\right\|}_{A} \\
&\quad + 2 \underbrace{\left\|\mathbb E\left(\nabla f(\theta^*, Z)(\nabla f(\theta, Z) - \nabla f(\theta) - \nabla f(\theta^*, Z))^\top\right)\right\|}_{B}.
\end{align*}
Next, we bound the two terms $A$ and $B$ separately. 
Firstly, 
\begin{align*}
A & = \left\|\mathbb E\left( (\nabla f(\theta, Z) - \nabla f(\theta) - \nabla f(\theta^*, Z))(\nabla f(\theta, Z) - \nabla f(\theta) - \nabla f(\theta^*, Z))^\top\right)\right\| \\
& \leq \mathbb E\left\|(\nabla f(\theta, Z) - \nabla f(\theta) - \nabla f(\theta^*, Z))(\nabla f(\theta, Z) - \nabla f(\theta) - \nabla f(\theta^*, Z))^\top\right\| \\
& = \mathbb E\left\|\nabla f(\theta, Z) - \nabla f(\theta) - \nabla f(\theta^*, Z)\right\|^2\\
& = \mathbb E\left\|(\nabla f(\theta, Z) - \nabla f(\theta^*, Z)) - (\nabla f(\theta) - \nabla f(\theta^*))\right\|^2 \\
& \leq 2 \left(\mathbb E\left\|\nabla f(\theta, Z) - \nabla f(\theta^*, Z)\right\|^2 + \left\|\nabla f(\theta) - \nabla f(\theta^*)\right\|^2\right) \\
& = 2 \left(\mathbb E\left\|\nabla f(\theta, Z) - \nabla f(\theta^*, Z)\right\|^2 + \left\|\mathbb E\left(\nabla f(\theta, Z) - \nabla f(\theta^*, Z)\right)\right\|^2\right) \\
& \leq 2 \left(\mathbb E\left\|\nabla f(\theta, Z) - \nabla f(\theta^*, Z)\right\|^2 + \left(\mathbb E\left\|\nabla f(\theta, Z) - \nabla f(\theta^*, Z)\right\|\right)^2\right) \\
& \leq 2 \left(\mathbb E\left\|\nabla f(\theta, Z) - \nabla f(\theta^*, Z)\right\|^2 + \mathbb E\left\|\nabla f(\theta, Z) - \nabla f(\theta^*, Z)\right\|^2 \right) \\
& = 4\mathbb E\left\|\nabla f(\theta, Z) - \nabla f(\theta^*, Z)\right\|^2.
\end{align*}
By the mean value theorem that for some $\theta'$ we have
\[
A = 4\mathbb E\left\|\nabla^2 f(\theta', Z)(\theta - \theta^*)\right\|^2 
\leq 4\|\theta - \theta^*\|^2\mathbb E\left\|\nabla^2 f(\theta', Z)\right\|^2.
\]
Now to bound the term $B$, by Cauchy-Schwarz inequality
\begin{align*}
B & = \left\|\mathbb E\left(\nabla f(\theta^*, Z)(\nabla f(\theta, Z) - \nabla f(\theta) - \nabla f(\theta^*, Z))^\top\right)\right\| \\
& \leq \mathbb E\left\|\nabla f(\theta^*, Z)(\nabla f(\theta, Z) - \nabla f(\theta) - \nabla f(\theta^*, Z))^\top\right\| \\
& \leq \mathbb E\left(\|\nabla f(\theta^*, Z)\|\left\|\nabla f(\theta, Z) - \nabla f(\theta) - \nabla f(\theta^*, Z)\right\|\right) \\
& \leq \left(\mathbb E\left\|\nabla f(\theta^*, Z)\right\|^2 \mathbb E\left\|\nabla f(\theta, Z) - \nabla f(\theta) - \nabla f(\theta^*, Z)\right\|^2\right)^{1/2} \\
& = \left(\mathbb E\left\|\nabla f(\theta^*, Z)\right\|^2\right)^{1/2} A^{1/2} \\
& \leq 2\left(\mathbb E\left\|\nabla f(\theta^*, Z)\right\|^2\mathbb E\left\|\nabla^2 f(\theta', Z)\right\|^2\right)^{1/2} \|\theta - \theta^*\|.
\end{align*}
Combining the above two upper bounds we have
\begin{align*}
\|\mathbb E\, \epsilon\epsilon^\top - V\| 
& \leq 4\left(\mathbb E\left\|\nabla f(\theta^*, Z)\right\|^2\mathbb E\left\|\nabla^2 f(\theta', Z)\right\|^2\right)^{1/2} \|\theta - \theta^*\|+ 4\mathbb E\left\|\nabla^2 f(\theta', Z)\right\|^2\cdot\|\theta - \theta^*\|^2.
\end{align*}
Using the assumption that $\mathbb E\|\nabla f(\theta^*, Z)\|^2$ and $\mathbb E\|\nabla^2 f(\theta, Z)\|^2$ are bounded for all $\theta$, 
we come to the desired conclusion. 
\end{proof}

\paragraph{Linear regression.}
To begin with, note that
\[
f(\theta) = \E \frac12 (Y - X^\top\theta)^2 = \frac12 \theta^\top \left[\E X X^\top \right] \theta - \left[\E Y X\right]^\top \theta + \frac12 \E Y^2,
\]
which is a simple quadratic function. Hence, Assumption \ref{ass:cvx} readily follows as long as $\E X X^\top$ exists, that is, $\|X\|$ has a second moment, and is positive-definite. The positive-definiteness holds if the distribution of $X \in \R^d$ is in a generic position, for example, its distribution has probability density well-defined for an arbitrarily small region. 
%
%

Next, to verify Assumption \ref{ass:reg}, we calculate
\[
\mathbb E\left\|\nabla f(\theta^*, Z)\right\|^2 = \mathbb E\left\|(Y - X^\top\theta^*)X\right\|^2
\leq 2\mathbb E\left(|Y|^2\|X\|^2\right) + 2\|\theta^*\|^2\mathbb E\|X\|^4,
\]
\[
\mathbb E\left\|\nabla^2 f(\theta, Z)\right\|^2 = \mathbb E\left\|XX^\top\right\|^2 = \mathbb E\|X\|^4.
\]
Then applying Lemma \ref{lem:sufficient_assumption2}, we conclude that 
the first part of Assumption 2 is satisfied as long as both $\mathbb E\left(|Y|^2\|X\|^2\right)$
and $\mathbb E\|X\|^4$ exist and are finite. 
Furthermore, we have
\begin{align*}
\mathbb E\|\epsilon\|^{2+\delta} &= \mathbb E\left\|\nabla f(\theta, Z) - \nabla f(\theta)\right\|^{2+\delta} \\
& \leq 2^{1+\delta}\mathbb E\left(\|\nabla f(\theta, Z)\|^{2 + \delta} + \|\nabla f(\theta)\|^{2+\delta}\right) \\
& \leq 2^{2+\delta}\mathbb E\left\|\nabla f(\theta, Z)\right\|^{2+\delta} \\
& \leq 2^{3+2\delta}\mathbb E\left(|Y|^{2+\delta}\|X\|^{2+\delta}\right) + 2^{3+\delta}\|\theta\|^{2+\delta}\mathbb E\|X\|^{4+2\delta}.
\end{align*}
Hence, we conclude that given that $\mathbb E\left(|Y|^{2+\delta}\|X\|^{2+\delta}\right)<\infty$ and 
$\mathbb E\|X\|^{4+2\delta}<\infty$ for some $\delta > 0$, Assumption \ref{ass:reg} is satisfied.

\paragraph{Logistic regression.}
For logistic regression, to ease the calculation, we use a formulation where $Y\in\{-1, +1\}$
and the log likelihood function takes the form
\[
f(\theta, z) = \log\left(1 + \exp(-y\cdot x^\top \theta)\right).
\]
Thus, we have
\[
f(\theta) = \mathbb E f(\theta, Z) = \mathbb E \log\left(1 + \exp(-Y\cdot X^\top \theta)\right),
\]
\[
\nabla f(\theta) = -\mathbb E \frac{YX}{1 + \exp(Y\cdot X^\top\theta)},
\]
\[
\nabla^2 f(\theta) = \mathbb E\frac{XX^\top}{\left(1+\exp(Y\cdot X^\top\theta)\right)\left(1+\exp(-Y\cdot X^\top\theta)\right)}.
\]
Noting that 
\[
\|\nabla^2 f(\theta)\| \leq \mathbb E\frac{\|XX^\top\|}{\left(1+\exp(Y\cdot X^\top\theta)\right)\left(1+\exp(-Y\cdot X^\top\theta)\right)}
\leq \frac{1}{4}\mathbb E\|X\|^2,
\]
we conclude that as long as $\|X\|$ has a second moment, 
$\nabla f(\theta)$ is Lipschitz continuous. 
On the other hand, 
suppose that the distribution of $X$ has positive probability density defined for an open region in $\mathbb R^d$.
Let $u$ be the unit vector associated with the smallest eigenvector of $\nabla^2 f(\theta)$. 
We then have
\[
u^\top\nabla^2 f(\theta)u = \mathbb E\frac{(u^\top X)^2}{\left(1+\exp(Y\cdot X^\top\theta)\right)\left(1+\exp(-Y\cdot X^\top\theta)\right)} > 0
\]
for a fixed $\theta$.
Hence $\nabla^2f(\theta)$ exists and is positive definite for a neighborhood of $\theta^*$.

To verify Assumption \ref{ass:reg}, we apply Lemma \ref{lem:sufficient_assumption2}. Note that 
\[
\mathbb E\left\|\nabla f(\theta^*, Z)\right\|^2 = \mathbb E\left\|\frac{YX}{1+\exp(-Y\cdot X^\top\theta^*)}\right\|^2 \leq \mathbb E\left\|X\right\|^2
\]
\[
\mathbb E\left\|\nabla^2 f(\theta, Z)\right\|^2 = \mathbb E\left\|\frac{XX^\top}{\left(1+\exp(Y\cdot X^\top\theta)\right)\left(1+\exp(-Y\cdot X^\top\theta)\right)}\right\|^2
\leq \frac{1}{16}\mathbb E\left\|X\right\|^4.
\]
Therefore, as long as $X$ has a finite fourth moment, we have
\[
\|\mathbb E\, \epsilon\epsilon^\top - V\| \leq C \left(\|\theta - \theta^*\| + \|\theta - \theta^*\|^2\right).
\]
In addition, following a similar argument in the discussion on linear regression, we have
\[
\mathbb E\|\epsilon\|^{2+\delta} \leq 2^{2+\delta}\mathbb E\|\nabla f(\theta, Z)\|^{2+\delta} 
\leq 2^{2+\delta}\mathbb E\|X\|^{2+\delta}.
\]
Thus, Assumption \ref{ass:reg} is shown to be satisfied given that $\mathbb E\|X\|^4<\infty$.

\paragraph{Penalized generalized linear regression.}
For the $\ell_2$-penalized generalized linear regression, we have $f(\theta, z) = -yx^\top\theta + b(x^\top\theta) + \lambda\|\theta\|^2$.
Furthermore, we assume that the function $b(w)$ is twice differentiable and that $|b''(w)|\leq L$ for all $w$.
To verify Assumption \ref{ass:cvx}, we calculate
\begin{align*}
\nabla f(\theta) = \mathbb E\nabla f(\theta, Z) = \mathbb E \left(-YX + b'(X^\top\theta)X + 2\lambda\theta\right)
\end{align*}
\begin{align*}
\nabla^2 f(\theta) = \mathbb E\nabla^2 f(\theta, Z) = \mathbb E \left(b''(X^\top\theta)XX^\top + 2\lambda I\right) 
= \mathbb E \left(b''(X^\top\theta)XX^\top\right) + 2\lambda I 
\end{align*}
By the definition of the density function of the exponential family, we have
\[
b'(x^\top\theta) = \mathbb E\left(Y|X=x\right)\text{ and }b''(x^\top\theta) = \mathrm{Var}(Y|X = x).
\]
Therefore, $b''(X^\top\theta)\geq 0$ and the minimum eigenvalue of $\nabla^2 f(\theta)$ is at least $\lambda$.
It then follows that the function $f(\theta)$ is strongly convex. 
On the other hand, we have
\begin{align*}
\left\|\nabla^2 f(\theta)\right\| \leq \left\|\mathbb E \left(b''(X^\top\theta)XX^\top\right)\right\| + 2\lambda
\leq L\mathbb E\left\|X\right\|^2 + 2\lambda.
\end{align*}

Now let's investigate Assumption \ref{ass:reg} using Lemma \ref{lem:sufficient_assumption2}. 
In fact, we have
\begin{align*}
\mathbb E\left\|\nabla f(\theta^*, Z)\right\|^2 & \leq 3\mathbb E\left(|Y|^2\|X\|^2\right) + 3\mathbb E\left(\left|b'(X^\top\theta^*)\right|^2\|X\|^2\right) + 12\lambda^2\left\|\theta^*\right\|^2\\
& = 3\mathbb E\left(|Y|^2\|X\|^2\right) + 3\mathbb E\left(|Y^*|^2\|X\|^2\right) + 12\lambda^2\|\theta^*\|^2
\end{align*}
where $Y^*$ given $X = x$ follows the conditional distribution specified by $p(y|x) = f(\theta^*, z)$.
\begin{align*}
\mathbb E\left\|\nabla^2 f(\theta, Z)\right\|^2 & = \mathbb E\left\|b''(X^\top\theta)XX^\top + 2\lambda I\right\|^2 \\
&\leq 2L^2\mathbb E\left\|X\right\|^4 + 8 \lambda^2.
\end{align*}
Therefore, by Lemma \ref{lem:sufficient_assumption2}, as long as $\mathbb E\left(|Y|^2\|X\|^2\right)$, $\mathbb E\left(|Y^*|^2\|X\|^2\right)$ and $\mathbb E\|X\|^4$ are finite, we obtain
\[
\|\mathbb E\, \epsilon\epsilon^\top - V\| \leq C  \|\theta - \theta^*\| + C'  \|\theta - \theta^*\|^2.
\]
Next, we compute as in the previous discussion
\begin{align*}
\mathbb E\|\epsilon\|^{2+\delta} &\leq 2^{2+\delta}\mathbb E\|\nabla f(\theta, Z)\|^{2+\delta} \\
&\leq 2^{2+\delta}\left(3^{1+\delta}\mathbb E\left(|Y|^{2+\delta}\|X\|^{2+\delta}\right) + 3^{1+\delta}\mathbb E\left(|Y_\theta|^{2+\delta}\|X\|^{2+\delta}\right) + 3^{1+\delta}2^{2+\delta}\lambda^{2+\delta}\|\theta\|^{2+\delta}\right),
\end{align*}
which is bounded for $\theta$ such that $\|\theta-\theta^*\|\leq \delta$
provided that $\mathbb E\left(|Y_\theta|^{2+\delta}\|X\|^{2+\delta}\right)<\infty$
and $\mathbb E\left(|Y_\theta|^{2+\delta}\|X\|^{2+\delta}\right)<\infty$.

\paragraph{Huber regression}
For the Huber regression, we have $f(\theta, z) = \rho_\lambda(y - x^\top\theta)$ where $\rho_\lambda(a) = a^2/2$ for $|a|\leq \lambda$
and $\rho_\lambda(a) = \lambda|a| - \lambda^2/2$ otherwise.
We then have
\[
\nabla f(\theta) = \mathbb E\nabla f(\theta, Z) = \mathbb E\rho_\lambda'(Y - X^\top\theta)X,
\]
\[
\nabla^2 f(\theta) = \mathbb E\nabla^2f(\theta, Z) = \mathbb E\left( \mathds 1(|Y - X^\top\theta|\leq \lambda)XX^\top \right).
\]
The only thing we need to check is the local strong convexity; 
the other parts follow directly from the discussion of linear regression. 
All we need to show is that for $\theta^*$, the minimizer of $f(\theta)$,
$\mathbb P(|Y - X^\top\theta^*| < \lambda) > 0$.
Towards that end, we assume that $X$ and $Y$ both have a open and connected support.
In addition, we augment $X$ with an extra element of 1 (the intercept).
Now suppose that $\mathbb P(|Y - X^\top\theta^*| < \lambda) = 0$.
Then since the domains are both connected, we know that
either $\mathbb P(Y - X^\top\theta^* > \lambda) = 1$ or $\mathbb P(Y - X^\top\theta^* < -\lambda) = 1$.
Without loss of generality, we assume that $\mathbb P(Y - X^\top\theta^* > \lambda) = 1$.
Then it follows that
\begin{align*}
0 = \nabla f(\theta^*) = \mathbb E\rho_\lambda'(Y - X^\top \theta^\ast) X = \lambda \mathbb E X,
\end{align*}
which leads to contradiction with the augmentation of the intercept. 
Therefore, with the other requirements for the linear regression case such as $X$ being in a generic position, 
and $\mathbb E\left(|Y|^{2+\delta}\|X\|^{2+\delta}\right)<\infty$ and $\mathbb E\|X\|^{4+2\delta}<\infty$,
Huber regression is shown to satisfy the two assumptions.

%% file: paper_jmlr.bbl
\begin{thebibliography}{53}
\providecommand{\natexlab}[1]{#1}
\providecommand{\url}[1]{\texttt{#1}}
\expandafter\ifx\csname urlstyle\endcsname\relax
  \providecommand{\doi}[1]{doi: #1}\else
  \providecommand{\doi}{doi: \begingroup \urlstyle{rm}\Url}\fi

\bibitem[Agarwal et~al.(2012)Agarwal, Bartlett, Ravikumar, and
  Wainwright]{agarwal2010information}
Alekh Agarwal, Peter~L Bartlett, Pradeep Ravikumar, and Martin~J Wainwright.
\newblock Information-theoretic lower bounds on the oracle complexity of
  stochastic convex optimization.
\newblock \emph{IEEE Transactions on Information Theory}, 58:\penalty0
  3235--3249, 2012.

\bibitem[Bach and Moulines(2011)]{moulines2011non}
Francis Bach and Eric Moulines.
\newblock Non-asymptotic analysis of stochastic approximation algorithms for
  machine learning.
\newblock In \emph{Advances in Neural Information Processing Systems}, pages
  451--459, 2011.

\bibitem[Bach and Moulines(2013)]{bach2013non}
Francis Bach and Eric Moulines.
\newblock Non-strongly-convex smooth stochastic approximation with convergence
  rate {$O(1/n)$}.
\newblock In \emph{Advances in Neural Information Processing Systems}, pages
  773--781, 2013.

\bibitem[Bach(2014)]{bach2014adaptivity}
Francis~R Bach.
\newblock Adaptivity of averaged stochastic gradient descent to local strong
  convexity for logistic regression.
\newblock \emph{Journal of Machine Learning Research}, 15\penalty0
  (1):\penalty0 595--627, 2014.

\bibitem[Capp{\'e} and Moulines(2009)]{cappe2009line}
Olivier Capp{\'e} and Eric Moulines.
\newblock On-line expectation--maximization algorithm for latent data models.
\newblock \emph{Journal of the Royal Statistical Society: Series B (Statistical
  Methodology)}, 71\penalty0 (3):\penalty0 593--613, 2009.

\bibitem[Cardot et~al.(2013)Cardot, C{\'e}nac, and Zitt]{cardot2013efficient}
Herv{\'e} Cardot, Peggy C{\'e}nac, and Pierre-Andr{\'e} Zitt.
\newblock Efficient and fast estimation of the geometric median in {H}ilbert
  spaces with an averaged stochastic gradient algorithm.
\newblock \emph{Bernoulli}, 19\penalty0 (1):\penalty0 18--43, 2013.

\bibitem[Chang and Lin(2011)]{chang2011libsvm}
Chih-Chung Chang and Chih-Jen Lin.
\newblock Libsvm: a library for support vector machines.
\newblock \emph{ACM transactions on intelligent systems and technology (TIST)},
  2\penalty0 (3):\penalty0 27, 2011.

\bibitem[Chee and Toulis(2017)]{chee2017convergence}
Jerry Chee and Panos Toulis.
\newblock Convergence diagnostics for stochastic gradient descent with constant
  step size.
\newblock \emph{arXiv preprint arXiv:1710.06382}, 2017.

\bibitem[Chen et~al.(2020)Chen, Lee, Tong, and Zhang]{chen2016}
Xi~Chen, Jason~D Lee, Xin~T Tong, and Yichen Zhang.
\newblock Statistical inference for model parameters in stochastic gradient
  descent.
\newblock \emph{The Annals of Statistics}, 48\penalty0 (1):\penalty0 251--273,
  2020.

\bibitem[Defazio et~al.(2014)Defazio, Bach, and
  Lacoste-Julien]{defazio2014saga}
Aaron Defazio, Francis Bach, and Simon Lacoste-Julien.
\newblock {SAGA}: {A} fast incremental gradient method with support for
  non-strongly convex composite objectives.
\newblock In \emph{Advances in Neural Information Processing Systems}, pages
  1646--1654, 2014.

\bibitem[Dieuleveut and Bach(2016)]{dieuleveut2016nonparametric}
Aymeric Dieuleveut and Francis Bach.
\newblock Nonparametric stochastic approximation with large step-sizes.
\newblock \emph{The Annals of Statistics}, 44\penalty0 (4):\penalty0
  1363--1399, 2016.

\bibitem[Duchi and Ruan(2016)]{duchi2016local}
John Duchi and Feng Ruan.
\newblock Local asymptotics for some stochastic optimization problems:
  Optimality, constraint identification, and dual averaging.
\newblock \emph{arXiv preprint arXiv:1612.05612}, 2016.

\bibitem[Duchi et~al.(2011)Duchi, Hazan, and Singer]{duchi2011adaptive}
John Duchi, Elad Hazan, and Yoram Singer.
\newblock Adaptive subgradient methods for online learning and stochastic
  optimization.
\newblock \emph{Journal of Machine Learning Research}, 12\penalty0
  (Jul):\penalty0 2121--2159, 2011.

\bibitem[Fan et~al.(2018)Fan, Gong, Li, and Sun]{fan2018statistical}
Jianqing Fan, Wenyan Gong, Chris~Junchi Li, and Qiang Sun.
\newblock Statistical sparse online regression: {A} diffusion approximation
  perspective.
\newblock In \emph{Proceedings of the 21th International Conference on
  Artificial Intelligence and Statistics}, 2018.

\bibitem[Fang et~al.(2017)Fang, Xu, and Yang]{fang2017scalable}
Yixin Fang, Jinfeng Xu, and Lei Yang.
\newblock On scalable inference with stochastic gradient descent.
\newblock \emph{arXiv preprint arXiv:1707.00192}, 2017.

\bibitem[Fort(2012)]{fort2012central}
Gersende Fort.
\newblock Central limit theorems for stochastic approximation algorithms.
\newblock Technical report, Technical report, 2012.

\bibitem[Gadat and Panloup(2017)]{gadat2017optimal}
S{\'e}bastien Gadat and Fabien Panloup.
\newblock Optimal non-asymptotic bound of the {R}uppert--{P}olyak averaging
  without strong convexity.
\newblock \emph{arXiv preprint arXiv:1709.03342}, 2017.

\bibitem[Giles(2015)]{giles2015multilevel}
Michael~B Giles.
\newblock Multilevel {M}onte {C}arlo methods.
\newblock \emph{Acta Numerica}, 24:\penalty0 259--328, 2015.

\bibitem[Hall and Heyde(2014)]{hall2014martingale}
Peter Hall and Christopher~C Heyde.
\newblock \emph{Martingale limit theory and its application}.
\newblock Academic Press, 2014.

\bibitem[Huber(1964)]{huber1964robust}
Peter~J Huber.
\newblock Robust estimation of a location parameter.
\newblock \emph{The Annals of Mathematical Statistics}, pages 73--101, 1964.

\bibitem[Jain et~al.(2016)Jain, Kakade, Kidambi, Netrapalli, and
  Sidford]{jain2016parallelizing}
Prateek Jain, Sham~M Kakade, Rahul Kidambi, Praneeth Netrapalli, and Aaron
  Sidford.
\newblock Parallelizing stochastic approximation through mini-batching and
  tail-averaging.
\newblock \emph{arXiv preprint arXiv:1610.03774}, 2016.

\bibitem[Jain et~al.(2017)Jain, Kakade, Kidambi, Netrapalli, Pillutla, and
  Sidford]{jain2017markov}
Prateek Jain, Sham~M Kakade, Rahul Kidambi, Praneeth Netrapalli,
  Venkata~Krishna Pillutla, and Aaron Sidford.
\newblock A markov chain theory approach to characterizing the minimax
  optimality of stochastic gradient descent (for least squares).
\newblock \emph{arXiv preprint arXiv:1710.09430}, 2017.

\bibitem[Johnson and Zhang(2013)]{johnson2013accelerating}
Rie Johnson and Tong Zhang.
\newblock Accelerating stochastic gradient descent using predictive variance
  reduction.
\newblock In \emph{Advances in Neural Information Processing Systems}, pages
  315--323, 2013.

\bibitem[Kiefer and Wolfowitz(1952)]{kiefer1952stochastic}
Jack Kiefer and Jacob Wolfowitz.
\newblock Stochastic estimation of the maximum of a regression function.
\newblock \emph{The Annals of Mathematical Statistics}, 23\penalty0
  (3):\penalty0 462--466, 1952.

\bibitem[Kingma and Ba(2015)]{kingma2014adam}
Diederik Kingma and Jimmy Ba.
\newblock Adam: {A} method for stochastic optimization.
\newblock In \emph{International Conference for Learning Representations},
  2015.

\bibitem[Kushner and Yin(2003)]{citeulike:2621242}
Harold~J Kushner and G~George Yin.
\newblock \emph{Stochastic approximation and recursive algorithms and
  applications}.
\newblock Springer, New York, 2003.

\bibitem[Lai(2003)]{lai2003stochastic}
Tze~Leung Lai.
\newblock Stochastic approximation.
\newblock \emph{The Annals of Statistics}, 31\penalty0 (2):\penalty0 391--406,
  2003.

\bibitem[Lan et~al.(2012)Lan, Nemirovski, and Shapiro]{lan2012validation}
Guanghui Lan, Arkadi Nemirovski, and Alexander Shapiro.
\newblock Validation analysis of mirror descent stochastic approximation
  method.
\newblock \emph{Mathematical Programming}, 134\penalty0 (2):\penalty0 425--458,
  2012.

\bibitem[Lange(1995)]{lange1995gradient}
Kenneth Lange.
\newblock A gradient algorithm locally equivalent to the {EM} algorithm.
\newblock \emph{Journal of the Royal Statistical Society. Series B
  (Methodological)}, pages 425--437, 1995.

\bibitem[LeCun et~al.(2015)LeCun, Bengio, and Hinton]{lecun2015deep}
Yann LeCun, Yoshua Bengio, and Geoffrey Hinton.
\newblock Deep learning.
\newblock \emph{Nature}, 521\penalty0 (7553):\penalty0 436--444, 2015.

\bibitem[Li et~al.(2017)Li, Liu, Kyrillidis, and Caramanis]{li2017statistical}
Tianyang Li, Liu Liu, Anastasios Kyrillidis, and Constantine Caramanis.
\newblock Statistical inference using {SGD}.
\newblock \emph{arXiv preprint arXiv:1705.07477}, 2017.

\bibitem[Liang and Su(2019)]{liang2017statistical}
Tengyuan Liang and Weijie~J Su.
\newblock Statistical inference for the population landscape via
  moment-adjusted stochastic gradients.
\newblock \emph{Journal of the Royal Statistical Society: Series B (Statistical
  Methodology)}, 81\penalty0 (2):\penalty0 431--456, 2019.

\bibitem[Lichman(2013)]{lichman2013machine}
M.~Lichman.
\newblock {UCI} machine learning repository, 2013.
\newblock URL \url{http://archive.ics.uci.edu/ml}.

\bibitem[Mandt et~al.(2017)Mandt, Hoffman, and Blei]{mandt2017stochastic}
Stephan Mandt, Matthew~D Hoffman, and David~M Blei.
\newblock Stochastic gradient descent as approximate {B}ayesian inference.
\newblock \emph{arXiv preprint arXiv:1704.04289}, 2017.

\bibitem[Nemirovski et~al.(2009)Nemirovski, Juditsky, Lan, and
  Shapiro]{nemirovski2009robust}
Arkadi Nemirovski, Anatoli Juditsky, Guanghui Lan, and Alexander Shapiro.
\newblock Robust stochastic approximation approach to stochastic programming.
\newblock \emph{SIAM Journal on optimization}, 19\penalty0 (4):\penalty0
  1574--1609, 2009.

\bibitem[Nemirovskii and Yudin(1983)]{nemirovskii1983problem}
Arkadii Nemirovskii and David~B Yudin.
\newblock \emph{Problem complexity and method efficiency in optimization}.
\newblock Wiley \& Sons, 1983.

\bibitem[O'Donoghue and Cand\`es(2015)]{o2015adaptive}
Brendan O'Donoghue and Emmanuel Cand\`es.
\newblock Adaptive restart for accelerated gradient schemes.
\newblock \emph{Foundations of Computational Mathematics}, 15\penalty0
  (3):\penalty0 715--732, 2015.

\bibitem[Polyak(1990)]{polyak1990}
Boris~T Polyak.
\newblock New stochastic approximation type procedures.
\newblock \emph{Automat. i Telemekh}, 7\penalty0 (2):\penalty0 98--107, 1990.

\bibitem[Polyak and Juditsky(1992)]{polyak1992}
Boris~T Polyak and Anatoli~B Juditsky.
\newblock Acceleration of stochastic approximation by averaging.
\newblock \emph{SIAM Journal on Control and Optimization}, 30\penalty0
  (4):\penalty0 838--855, 1992.

\bibitem[Rakhlin et~al.(2012)Rakhlin, Shamir, and Sridharan]{rakhlin2012making}
Alexander Rakhlin, Ohad Shamir, and Karthik Sridharan.
\newblock Making gradient descent optimal for strongly convex stochastic
  optimization.
\newblock In \emph{Proceedings of the 29th International Conference on Machine
  Learning}, pages 449--456, 2012.

\bibitem[Recht et~al.(2011)Recht, Re, Wright, and Niu]{recht2011hogwild}
Benjamin Recht, Christopher Re, Stephen Wright, and Feng Niu.
\newblock Hogwild: A lock-free approach to parallelizing stochastic gradient
  descent.
\newblock In \emph{Advances in Neural Information Processing Systems}, pages
  693--701, 2011.

\bibitem[Robbins and Monro(1951)]{robbins1951stochastic}
Herbert Robbins and Sutton Monro.
\newblock A stochastic approximation method.
\newblock \emph{The Annals of Mathematical Statistics}, pages 400--407, 1951.

\bibitem[Robbins and Siegmund(1985)]{robbins1985convergence}
Herbert Robbins and David Siegmund.
\newblock A convergence theorem for non negative almost supermartingales and
  some applications.
\newblock In \emph{Herbert Robbins Selected Papers}, pages 111--135. Springer,
  1985.

\bibitem[Ruppert(1988)]{ruppert1988}
David Ruppert.
\newblock Efficient estimations from a slowly convergent {R}obbins--{M}onro
  process.
\newblock Technical report, Operations Research and Industrial Engineering,
  Cornell University, Ithaca, NY, 1988.

\bibitem[Shiryaev(1996)]{shiryaev1996probability}
Albert~N Shiryaev.
\newblock \emph{Probability}.
\newblock Springer-Verlag, New York, 1996.

\bibitem[Sordello et~al.(2019)Sordello, He, and Su]{sordello2019robust}
Matteo Sordello, Hangfeng He, and Weijie Su.
\newblock Robust learning rate selection for stochastic optimization via
  splitting diagnostic.
\newblock \emph{arXiv preprint arXiv:1910.08597}, 2019.

\bibitem[Su et~al.(2016)Su, Boyd, and Cand\`es]{su2016differential}
Weijie Su, Stephen Boyd, and Emmanuel Cand\`es.
\newblock A differential equation for modeling {N}esterov's accelerated
  gradient method: Theory and insights.
\newblock \emph{Journal of Machine Learning Research}, 17\penalty0
  (153):\penalty0 1--43, 2016.

\bibitem[Tibshirani and Taylor(2012)]{tibshirani2012degrees}
Ryan~J Tibshirani and Jonathan Taylor.
\newblock Degrees of freedom in lasso problems.
\newblock \emph{The Annals of Statistics}, 40\penalty0 (2):\penalty0
  1198--1232, 2012.

\bibitem[Toulis et~al.(2014)Toulis, Airoldi, and Rennie]{toulis2014statistical}
Panagiotis Toulis, Edoardo Airoldi, and Jason Rennie.
\newblock Statistical analysis of stochastic gradient methods for generalized
  linear models.
\newblock In \emph{International Conference on Machine Learning}, pages
  667--675, 2014.

\bibitem[Toulis and Airoldi(2017)]{toulis2017asymptotic}
Panos Toulis and Edoardo~M Airoldi.
\newblock Asymptotic and finite-sample properties of estimators based on
  stochastic gradients.
\newblock \emph{The Annals of Statistics}, 45\penalty0 (4):\penalty0
  1694--1727, 2017.

\bibitem[White(1980)]{white1980heteroskedasticity}
Halbert White.
\newblock A heteroskedasticity-consistent covariance matrix estimator and a
  direct test for heteroskedasticity.
\newblock \emph{Econometrica}, pages 817--838, 1980.

\bibitem[Yin et~al.(2018)Yin, Pananjady, Lam, Papailiopoulos, Ramchandran, and
  Bartlett]{yin2018}
Dong Yin, Ashwin Pananjady, Max Lam, Dimitris Papailiopoulos, Kannan
  Ramchandran, and Peter~L. Bartlett.
\newblock Gradient diversity: a key ingredient for scalable distributed
  learning.
\newblock In \emph{Proceedings of the 21th International Conference on
  Artificial Intelligence and Statistics}, 2018.

\bibitem[Zhang(2004)]{zhang2004solving}
Tong Zhang.
\newblock Solving large scale linear prediction problems using stochastic
  gradient descent algorithms.
\newblock In \emph{Proceedings of the 21st International Conference on Machine
  learning}, page 116, 2004.

\end{thebibliography}
